\documentclass[b5paper,10pt]{book}

\usepackage{phdthesis}
\usepackage[
  paperwidth=17cm
  ,paperheight=24cm
  ,inner=2.6cm
  ,outer=2.3cm
  ,top=2.3cm
  ,bottom=2.3cm
]{geometry}

\usepackage{graphicx,times}
\usepackage[hyphens]{url}
\usepackage{tabularx,amsmath}
\usepackage{multirow}
\usepackage{booktabs}
\usepackage{amssymb}
\usepackage{amsthm}
\usepackage{adjustbox}
\usepackage{mleftright}
\usepackage{dsfont}
\usepackage{enumerate}
\usepackage{mathtools} %
\usepackage[square,comma,numbers,sort&compress,sectionbib]{natbib}
\usepackage{wrapfig}
\usepackage{supertabular}
\usepackage{multicol}
\usepackage{setspace}
\usepackage{import}
\usepackage{booktabs} 
\usepackage{graphicx}
\usepackage{url}
\usepackage{times}
\usepackage{microtype}
\usepackage{tabularx,amsmath}
\usepackage{amssymb}
\usepackage{multirow} 
\usepackage{color}
\usepackage{float} %
\usepackage[usenames,dvipsnames,table]{xcolor}
\usepackage[normalem]{ulem}
\usepackage[english]{babel}
\usepackage{statex}
\usepackage[noend]{algorithmic}
\usepackage{algorithm}
\usepackage{paralist}
\usepackage{flushend}
\usepackage{enumitem}
\usepackage{setspace}
\usepackage{textcomp}
\usepackage{listings}
\usepackage{keyval}
\usepackage[figuresleft]{rotating} %
\usepackage{acronym}
\usepackage{appendix}
\usepackage{chngcntr}
\usepackage{etoolbox}
\usepackage{lipsum}
\usepackage{bibentry}
\usepackage{subcaption}
\usepackage{caption}
\usepackage{thmtools, thm-restate}  %

\nobibliography*

\usepackage[backref=page]{hyperref}
\hypersetup{%
pdfborder = {0 0 0},
pdftitle = {Safe, Efficient, and Robust Reinforcement Learning for Ranking and Diffusion Models},
pdfsubject = {PhD Thesis},
pdfkeywords = {learning to rank, reinforcement learning, diffusion models},
pdfauthor = {Shashank Gupta}
}
\usepackage{bookmark}

\newcommand{\headernodot}[1]{\vspace{1mm}\noindent\textbf{#1}}
\newcommand{\header}[1]{\headernodot{#1.}}

\makeatletter
\renewcommand\tableofcontents{%
    \if@twocolumn
      \@restonecoltrue\onecolumn
    \else
      \@restonecolfalse
    \fi
    \chapter*{\contentsname
        \@mkboth{%
           \contentsname}{\contentsname}}%
    \@starttoc{toc}%
    \if@restonecol\twocolumn\fi
    }
\makeatother

\renewcommand*{\backref}[1]{} 
\renewcommand*{\backrefalt}[4]{%
\ifcase #1
\or (Cited on page~#2.)  %
\else 
(Cited on pages~#2.)  
\fi
}

\let\svthefootnote\thefootnote
\newcommand\blankfootnote[1]{%
  \let\thefootnote\relax\footnotetext{#1}%
  \let\thefootnote\svthefootnote%
}
\let\svfootnote\footnote
\renewcommand\footnote[2][?]{%
  \if\relax#1\relax%
    \blankfootnote{#2}%
  \else%
    \if?#1\svfootnote{#2}\else\svfootnote[#1]{#2}\fi%
  \fi
}

\acrodef{IR}{information retrieval}
\acrodef{CF}{collaborative filtering}
\acrodef{LTR}{learning to rank}
\acrodef{NDCG}{normalized discounted cumulative gain}
\acrodef{DCG}{discounted cumulative gain}
\acrodef{VAE}{variational autoencoder}
\acrodef{VAE}{variational autoencoder}
\acrodef{ELBO}{evidence lower bound objective}
\acrodef{IPS}{inverse propensity scoring}
\acrodef{BPR}{bayesian personalized ranking}
\acrodef{MF}{matrix factorization}
\acrodef{MNAR}{missing-not-at-random}
\acrodef{ULTR}{unbiased learning-to-rank}
\acrodef{CLTR}{counterfactual learning to rank}
\acrodef{LOLN}{law of large numbers}
\acrodef{CRM}{counterfactual risk minimization}
\acrodef{IS}{importance sampling}
\acrodef{i.i.d}{independent and identically distributed}
\acrodef{CRM}{counterfactual risk minimization}
\acrodef{PL}{Plackett-Luce}
\acrodef{CTR}{click through rate}
\acrodef{SEA}{safe exploration algorithm}
\acrodef{GENSPEC}{generalization and specialization }
\acrodef{VCRM}{variational counterfactual risk minimization}
\acrodef{SGD}{stochastic gradient descent}
\acrodef{DR}{doubly robust}
\acrodef{DM}{direct method}
\acrodef{SNIPS}{self-normalized importance sampling}
\acrodef{ERC}{exposure ratio clipping}
\acrodef{PRPO}{proximal ranking policy optimization}
\acrodef{PPO}{proximal policy optimization}
\acrodef{RL}{reinforcement learning}
\acrodef{PPO}{proximal policy optimization}
\acrodef{MDP}{Markov decision process}
\acrodef{LOOP}{leave-one-out PPO}
\acrodef{IS}{importance sampling}
\acrodef{RLOO}{reinforce leave-one-out}
\acrodef{SF}{score function}
\acrodef{CB}{contextual bandit}
\acrodef{SERP}{search engine result page}
\acrodef{OPE}{off-policy evaluation}
\acrodef{OPL}{off-policy learning}
\acrodef{LLM}{large language models}
\acrodef{LLMs}{large language models}

\DeclareMathOperator*{\argmax}{arg\,max}
\DeclareMathOperator*{\argmin}{arg\,min}

\AtBeginEnvironment{subappendices}{%
\chapter*{Appendices}
\addcontentsline{toc}{chapter}{Appendices}
\counterwithin{figure}{section}
\counterwithin{table}{section}
}
\allowdisplaybreaks   %

\AtEndEnvironment{subappendices}{%
\counterwithout{figure}{section}
\counterwithout{table}{section}
\counterwithout{algorithm}{section}  %
\counterwithin{figure}{chapter}
\counterwithin{table}{chapter}
\counterwithin{algorithm}{chapter}  %
}

\begin{document}

\theoremstyle{plain}
\newtheorem{theorem}{Theorem}[section]
\newtheorem{proposition}[theorem]{Proposition}
\newtheorem{lemma}[theorem]{Lemma}
\newtheorem{corollary}[theorem]{Corollary}
\theoremstyle{definition}
\newtheorem{definition}[theorem]{Definition}
\newtheorem{assumption}[theorem]{Assumption}
\theoremstyle{remark}
\newtheorem{remark}[theorem]{Remark}

\frontmatter

{\pagestyle{empty}
\newcommand{\printtitle}{%
{\Huge\bf Safe, Efficient, and Robust Reinforcement Learning for Ranking and Diffusion Models \\[0.8cm]
}}

\begin{titlepage}
\par\vskip 2cm
\begin{center}
\printtitle
\vfill
{\LARGE\bf Shashank Gupta}
\vskip 2cm
\end{center}
\end{titlepage}

\mbox{}\newpage
\setcounter{page}{1}

\clearpage
\par\vskip 2cm
\begin{center}
\printtitle
\par\vspace {4cm}
{\large \sc Academisch Proefschrift}
\par\vspace {1cm}
{\large ter verkrijging van de graad van doctor aan de \\
Universiteit van Amsterdam\\
op gezag van de Rector Magnificus\\
prof.\ dr.\ ir.\ P.P.C.C.\ Verbeek\\
ten overstaan van een door het College voor Promoties ingestelde \\
commissie, in het openbaar te verdedigen in \\
de Agnietenkapel\\
op maandag 13 oktober 2025, te 16:00 uur \\ } %
\par\vspace {1cm} {\large door}
\par \vspace {1cm}
{\Large Shashank Gupta}
\par\vspace {1cm}
{\large geboren te Jaipur} 
\end{center}

\clearpage
\noindent%
\textbf{Promotiecommissie} \\\\
\begin{tabular}{@{}l l l}
Promotor: 
& prof.\ dr.\ M. de Rijke & Universiteit van Amsterdam %
\\
Co-promotor: 
& dr. H. Oosterhuis & Radboud Universiteit 
\\[1.2ex]
Overige leden: 
& dr. O. Jeunen & Aampe, Belgium \\
& prof. dr. T. Joachims & Cornell University \\
& prof. dr. E. Kanoulas & Universiteit van Amsterdam \\
& dr. S. Magliacane & Universiteit van Amsterdam \\
& prof. dr. C. Snoek & Universiteit van Amsterdam \\
\end{tabular}

\bigskip\noindent%
Faculteit der Natuurwetenschappen, Wiskunde en Informatica\\

\vfill

\noindent
The research was carried out at the Information Retrieval Lab at the University of Amsterdam and in part during an internship at Meta AI.
The research carried out at the University of Amsterdam was funded by DreamsLab.  \\
\bigskip

\noindent
Copyright \copyright~2025 Shashank Gupta, Amsterdam, The Netherlands\\
Cover by Shashank Gupta\\
Printed by Proefschriftspecialist, Zaandam\\
\\
ISBN: 978-94-93483-03-3\\

\clearpage
}

{\pagestyle{empty}

{
\begin{center}

\noindent
\textbf{Acknowledgements} \\ \vskip .5cm
\end{center}
}
\addcontentsline{toc}{chapter}{Acknowledgements}

\noindent
Looking back eight years, right after I finished my master's degree, I thought I was done with academia. I planned to stay in India and in industry for good. Then the pandemic happened, and suddenly I had more time on my hands than ever before. 
Things changed, and here I am, writing my PhD thesis.
There are many people I want to thank for their support and kindness throughout this journey. But before that, I'd like to leave a small reflection—for anyone reading this (especially junior PhD students), or maybe just as a note to my future self.

Every PhD has its highs and lows. The highs: papers, internships, conferences etc. are documented in my CV. 
The lows are less visible, hidden in my own memory and in the patience of family and friends who shared them with me.
Looking back, the biggest lessons I've learned are:
Lesson $\#1$: Don't put too much pressure on yourself. For a long time, I treated the PhD as if it had to define my entire worth -- learning everything, publishing endlessly, building networks, chasing milestones etc. 
That pressure came at the cost of my mental and physical health, and it left little room to celebrate small wins along the way. 
A PhD is just another degree, not your whole identity.
Lesson $\#2$: Don't ignore your health. 
For most of my PhD years, I neglected exercise and let work consume me. 
Now, I feel the toll of that neglect; I am at my heaviest weight of my life as I am writing this thesis. 
Work can move slowly, but regaining lost health is far harder.
Lesson $\#3$: Have a life outside of work. I never stuck with hobbies or creative pursuits, and that made me dependent on academic success for happiness, a shaky foundation, given how uncertain academia can be in general. 
A hobby gives balance and even helps professionally in the long run.

If these reflections are useful to someone else, that would make me glad. 
At the very least, I hope they'll remind the future me of what really matters.

Now I want to thank the people who helped me on this journey.
First and foremost, I want to thank Maarten for agreeing to take me as his PhD student.
Thank you, Maarten, for being so patient with me, thank you for supporting me initially when I was struggling with the move.
I am also thankful to him for being so kind and caring, qualities I admire more than academic prowess.
Thank you for always pushing me positively whenever needed and being patient when I was a bit down.

I would then like to thank my second advisor, Harrie. I am really glad I got to work with him eventually, given that I applied for a PhD position with him but eventually accepted one with Maarten.
I learned academic rigor, discipline, and attention to detail by working with Harrie.
I also learned how to do more theoretically grounded research and write theoretically oriented papers from Harrie.
Thank you for being supportive, understanding, and optimistic outside of academics as well.
Thank you for helping me become a better researcher.

I am honored to have Olivier, Thorsten, Sara, Cees, and Evangelos on my committee and for taking the time to read and critique my thesis.

Next, I would like to thank some people who helped me on my journey so far (in no particular order).
I want to thank Maria for being a positive anchor and a great friend during my PhD. I tend to lean towards pessimism in general, so thank you for always being so positive.
Thank you for always being there to listen to me and give very practical advice in life in general.
Seeing Ida grow up and playing with her always brought joy to me. 
I am also glad I met Mathijs through Maria, who is one of the most fun and funny people I know.
I want to thank Philipp for being such a great friend throughout these four years. 
I am so glad you decided to pursue a PhD and move to Amsterdam.
I don't know how I would have managed my PhD (and life in Amsterdam in general) if it were not for you.
Thank you for always being there to talk, for always listening to me rant/complain, and just always being there for me.
Thank you for always making the time to accompany me to restaurants, bike trips, and conference travels.
I also want to thank Vaishali (and Prateek by extension) for always being there to listen to me, advising me, and accompanying me to Indian restaurants.
Thank you also for bringing familiarity from India. Also thank you for always supporting me emotionally whenever I felt down. 
Thank you, Samarth, for bringing a calm energy, and thank you for visiting me in NYC.
Thank you, Mariya, for being around in London. My time in London was much more fun because of you. And thank you for always taking the time to call me and meet me, despite your ever-busy schedule.
Thank you, Norman, for inviting me whenever you visited Amsterdam.
Thank you, Thilina and Clem, for being solid DreamsLab colleagues and for the fun chats at Thilina's dinners.
Thank you, Clara, for the fun work-from-café days in Amsterdam.
Thank you, Sharvaree, for cooking all the delicious Indian food and for inviting me for dinners.

I had the pleasure of visiting London and New York during my internships and working with the Meta recommender systems org.
I want to thank Eric, Hanchao, and Aga for my internship in London. Outside my internship, I want to thank Palak for being a fun friend in London.
For my internship in NYC, I want to thank my teammates: Satya, Tsung-Yu, Chaitanya, and Sreya for your mentorship during the internship.
I also want to thank François, Paris, and Jun for your support during the internship.
Outside my internship, I want to thank Krishna, Kinnari, Shiwangi, Sujal, and Pulkit for making my stay in NYC more fun.
Last but not least, I want to especially thank Yiming for all your support during (and outside) my time at Meta.
Thank you for your mentorship and guidance in navigating the internship during my first time in London, and thank you for helping me with getting the second internship and for always going out of your way (even when you did not have to) to help me with Meta-related stuff and for talking to me.
You have been very kind, thank you.

I would also like to thank all my colleagues at IRLab with whom I have had the pleasure to learn so much during my PhD:
Evangelos, Petra (thank you for supporting me initially during my move to Amsterdam), Pablo (thank you for your help during my PhD), Ivana (thank you for your help during my PhD), Ali (thank you for helping out with the tutorial), Amin, Amy, Ana, Andrew, Antonis, Arezoo,
Barrie, Chang, Chen, Chuan, Dan, Daniel, David, Fen, Gabriel, Gabrielle, Georgios,
Hongyi, Ilias, Jia-Hong, Jie, Jin (thank you for helping out with the tutorial), Jingwei (thank you for being an awesome officemate), Julien, Kidist (thank you for being an awesome officemate), Lu, Maarten Marx, Maartje,
Mohammad, Maryam, Maurits, Maxime, Ming (thanks for being an awesome officemate), Mohanna, Mounia,
Mozhdeh, Olivier, Panagiotis, Pooya (thank you for inspiring me with your insane fitness levels, though I never followed up on that), Romain, Roxana, Sami, Shaojie,
Simon, Svitlana, Teng, Thilina, Thong, Siddharth Mehrotra, Siddharth Singh (thank you for the fun talks and your attempt to speak Hindi with me), Vera,
Yangjun, Yibin, Yixing, Yuanna, Yongkang, Yougang, Yubao, Yuyue, Zahra, Zhirui,
Zihan, and Ziming. My apologies if I missed anyone here. 

I also want to thank my friends from India for always being there and for always supporting me since our college days: Mitesh, Vivek, Subhadeep, Rohit, Siddharth, Pooja, Kirti, Nayan, and Sri Vidhya. 
I also want to thank my family for being my pillar of strength: my mom, dad, my sisters, my nephew, and everyone else.

Finally, I want thank you to the original scholar in my family: my grandfather -- Gyan Swarop Gupta. 
I guess I learned how to do research by watching you do the same while growing up -- on the history of Indus Valley civilization.
I miss you dearly, Nana, I know you are watching me always. 

\vskip .5cm
\noindent
\begin{flushright}
Shashank Gupta\\
Jaipur\\
August 2025
\end{flushright}

\clearpage
}

\tableofcontents

\mainmatter

\chapter{Introduction}
\label{chapter:introduction}

Reinforcement learning (RL) has established itself as a robust framework for enhancing decision-making systems across diverse application domains~\cite{luong2019applications,wang2020deep}. 
In conventional reinforcement learning configurations, an agent aims to maximize cumulative rewards through sequential action selection while interacting with an environment. 
Traditionally, the concept of a ``reinforcement learning agent" has been linked with robotics or strategic games like chess or Go. 
However, in recent years, reinforcement learning methodologies have increasingly been applied to user-facing applications, including web search engines, recommender systems, and fine-tuning foundation models for interactive use cases~\cite{kuhnle2021reinforcement,zhang2020deep,afsar2022reinforcement}.

Among the most widely implemented forms of reinforcement learning in user-facing applications is the contextual bandit framework~\cite{barraza2020introduction,vandenAkker2024}. 
Contextual bandits represent a simplified reinforcement learning scenario involving a single state (context) for each interaction with the environment. 
Within recommender systems, the context typically encompasses personalized user features, the action constitutes the recommended item, and the reward derives from user engagement metrics such as clicks or purchases~\cite{vandenAkker2024}. 
For web search applications, the context is the user query (potentially enhanced with personalized features), actions comprise ranked document lists, and rewards stem from search engine result page (SERP) interaction metrics, including normalized discounted cumulative gain (NDCG) or click counts~\cite{kuhnle2021reinforcement}. 
In foundation model fine-tuning, the context corresponds to an input prompt, the action is a generated sequence of words, and rewards typically originate from learned reward models approximating human feedback~\cite{bai2022training}.

When compared to comprehensive multi-step reinforcement learning scenarios, contextual bandits offer computational simplicity and facilitate easier deployment, frequently using existing offline logged data in an efficient manner~\cite{saito2021counterfactual,vandenAkker2024}. 
Each context-action pair directly yields a reward, thus eliminating complexities associated with sequential decision-making. 
Given these practical advantages, this thesis specifically concentrates on contextual bandit methodologies.

Despite these benefits, implementing contextual bandits in ranking problems introduces notable challenges, particularly regarding biases in user feedback. Ranking policies gather data under specific display conditions, resulting in biased user interaction signals -- such as position bias or trust bias -- which inadequately represent true item relevance~\cite{joachims2017unbiased}. Items positioned lower in rankings receive fewer interactions, potentially leading to their misclassification as irrelevant.

To address this bias, counterfactual learning-to-rank (LTR) approaches employ inverse propensity scoring techniques, adjusting observed interaction signals to approximate unbiased relevance estimates~\cite{joachims2017unbiased,agarwal2019addressing,vardasbi2020inverse,gupta2024unbiased}. 
However, inverse propensity scoring methods typically suffer from high variance, especially when working with limited data, resulting in unstable or suboptimal deployment~\cite{oosterhuis2022reaching,gupta2024unbiased}. 
The first component of this thesis addresses safety concerns in counterfactual LTR. 
In this context, ``safety” describes how the new ranking policy performs compared with the current production policy. When the new policy performs worse than the production policy, it is considered \textit{unsafe}; when it performs better, it is considered safe.
We propose a safe counterfactual LTR method that theoretically ensures a new ranking policy that performs at least as well as the currently deployed policy. 
While existing safe methods provide guarantees under assumed click behavior models, these guarantees fail if actual user behavior diverges~\cite{oosterhuis2021robust,jagerman2020safe}. 
To tackle this issue, we introduce a robust safe counterfactual LTR approach that provides reliable guarantees even when user behavior deviates from assumptions.

In the second part of this thesis, we focus on enhancing sample efficiency in contextual bandit learning and evaluation, specifically, achieving lower error rate with limited data. 
In this setting, off-policy evaluation estimates how a new policy would perform using data collected under a different policy, while off-policy learning uses that same logged data to optimize the new policy itself. 
Both off-policy evaluation and off-policy learning typically exhibit high variance, causing instability in performance estimates. 
To reduce variance and improve sample efficiency, we propose an optimal baseline-correction method that significantly decreases the error in off-policy estimates while requiring fewer data points.

Beyond ranking and recommendation systems, diffusion models have recently achieved state-of-the-art results in generative tasks like text-to-image synthesis~\cite{yang2023diffusion}. 
Denoising diffusion probabilistic models iteratively refine random noise into meaningful outputs guided by learned distributions~\cite{ho2020denoising}. 
However, these models do not inherently optimize custom objectives such as aesthetic quality or prompt alignment after training.
By interpreting the denoising process as an RL action, diffusion models can incorporate user-defined reward functions~\cite{black2023training}.
While proximal policy optimization (PPO) is commonly used for reinforcement learning fine-tuning, it involves substantial computational costs and high variance, requiring multiple networks loaded simultaneously. 
In contrast, REINFORCE offers computational efficiency but suffers from high variance and poor sample efficiency~\cite{williams1992simple}. We propose an efficient reinforcement learning fine-tuning method for text-to-image diffusion models, combining REINFORCE's computational efficiency with PPO's improved sample efficiency in a novel method -- leave-one-out PPO (LOOP).

Collectively, the contributions made in this thesis highlight the shared challenges of safety, efficiency, and robustness in contextual bandit methods across ranking and diffusion modeling contexts within the reinforcement learning paradigm.

\section{Research Outline and Questions}
\label{section:introduction:rqs}

Counterfactual LTR corrects user interaction bias primarily using \ac{IPS}, weighting clicks inversely proportional to their selection probability. 
As we pointed out above, while unbiased in expectation, IPS-based estimators are known to suffer from high variance, especially with limited logged interaction data, potentially yielding suboptimal ranking policies~\cite{oosterhuis2022reaching,gupta2024unbiased}. 
Deploying such suboptimal policies poses significant risks to user experience and business metrics. Therefore, it is crucial to incorporate mechanisms ensuring safe deployment. This leads to the first research question:

\acrodef{rq:safe1}[\ref{rq:safe1}]{Can safety guarantees be provided for counterfactual LTR policies to ensure that the new policy is at least as good as the production policy?}
\acrodef{rq:safe2}[\ref{rq:safe2}]{Can we provide robust safety guarantees for counterfactual \ac{LTR} policies even under adversarial user behavior settings?}
\acrodef{rq:recsys1}[\ref{rq:recsys1}]{Can we unify variance reduction techniques using baseline corrections and a doubly robust estimator under a common framework?}
\acrodef{rq:recsys}[\ref{rq:recsys}]{Given a unified framework for variance reduction techniques under baseline corrections, can we derive a variance-optimal unbiased estimator?}
\acrodef{rq:loop}[\ref{rq:loop}]{Can we improve the sample efficiency of proximal policy optimization for fine-tuning text-to-image diffusion?}

\begin{enumerate}[label=\textbf{RQ\arabic*},ref={RQ\arabic*},resume]
\item \acl{rq:safe1}\label{rq:safe1}
\end{enumerate}

\noindent To address \ref{rq:safe1}, we derive a lower confidence bound for the counterfactual \ac{LTR} estimator, establishing a lower bound on the true ranking utility, the ideal target metric for optimization. 
In Chapter 2, we demonstrate that optimizing this lower bound ensures a ranking policy that is no worse than the current production policy.
This property proves particularly valuable when click data is scarce, mitigating the substantial risk of deploying potentially harmful policies.

While \ref{rq:safe1} provides a probabilistic safety guarantee by optimizing the lower bound on the utility, these guarantees depend critically on assumptions regarding user behavior (click model). 
Deviations from these assumptions invalidate the guarantees, motivating the second research question:

\begin{enumerate}[label=\textbf{RQ\arabic*},ref={RQ\arabic*},resume]
\item \acl{rq:safe2}\label{rq:safe2}
\end{enumerate}

\noindent In Chapter 3, we introduce \ac{PRPO}, a method ensuring safety for counterfactual \ac{LTR} without reliance on user behavior assumptions, guaranteeing robust safety even under adversarial conditions.

Thus far, the discussion has focused on contextual bandits within ranking scenarios involving combinatorial action spaces. 
Next, we examine contextual bandits generating single actions, such as top-1 recommendations, with a focus on improving the sample efficiency in off-policy evaluation and learning. 
Standard methods like IPS are unbiased in expectation, but suffer from high variance. 
Alternative methods, including doubly robust (DR) estimators and self-normalized IPS (SNIPS), reduce variance using additive and multiplicative baseline corrections, respectively~\cite{Swaminathan2015,joachims2018deep}, yet lack a unifying framework. 
This motivates our third research question:

\begin{enumerate}[label=\textbf{RQ\arabic*},ref={RQ\arabic*},resume]
\item \acl{rq:recsys1}\label{rq:recsys1}
\end{enumerate}

\noindent Chapter 4 proposes the $\beta$-IPS estimator, integrating inverse propensity scoring (IPS), doubly robust methods, and self-normalized IPS under a unified baseline correction framework.

\begin{enumerate}[label=\textbf{RQ\arabic*},ref={RQ\arabic*},resume]
\item \acl{rq:recsys}\label{rq:recsys}
\end{enumerate}

\noindent Using the unified $\beta$-IPS estimator framework, we investigate whether a variance-optimal baseline correction ($\beta^{*}$) can be analytically derived. 
In Chapter 4, we confirm this possibility, presenting a closed-form solution for $\beta^{*}$ that minimizes variance for both off-policy learning and evaluation tasks.

Contextual bandit theory as previously discussed emphasizes user interactions within ranking or recommendation systems. 
However, the framework has also been effectively employed in fine-tuning foundation models, such as large language models (LLMs) and diffusion models, typically using proximal policy optimization (PPO). 
Recent research highlights computational advantages of REINFORCE (policy gradient methods) over PPO for LLMs~\cite{ahmadian2024back}. 
Given PPO's challenges with variance and sample inefficiency, we consider improvements through our final research question:

\begin{enumerate}[label=\textbf{RQ\arabic*},ref={RQ\arabic*},resume]
\item \acl{rq:loop}\label{rq:loop}
\end{enumerate}

\noindent In Chapter 5, we systematically compare PPO and REINFORCE for diffusion model fine-tuning. We first demonstrate that REINFORCE exhibits inferior sample efficiency compared to PPO. Subsequently, we propose \ac{LOOP}, an enhancement to PPO achieving superior performance with the same number of input prompts by generating multiple actions per prompt.

\section{Main Contributions}
\label{section:introduction:contributions}

The main contributions of this thesis are categorized into algorithmic and theoretical components, summarized as follows.

\subsection{Algorithmic contributions}

\begin{itemize}
\item A safe counterfactual \ac{LTR} optimization framework that uses the REINFORCE policy gradient, guided by a derived generalization bound for the counterfactual \ac{LTR} estimator (Chapter~\ref{chapter:01-online-evaluation1}).
\item A robust safe counterfactual \ac{LTR} algorithm extending the REINFORCE policy gradient method with a clipping mechanism inspired by PPO from reinforcement learning literature (Chapter~\ref{chapter:01-online-evaluation2}).
\item A closed-form solution for the variance-optimal baseline correction term in off-policy evaluation and off-policy learning for contextual bandit methods, substantially reducing variance in practical applications (Chapter~\ref{chapter:01-online-evaluation3}).
\item An efficient reinforcement learning approach for fine-tuning text-to-image diffusion models that enhances sample efficiency by generating multiple diffusion trajectories per input prompt, thereby effectively reducing variance (Chapter~\ref{chapter:01-online-evaluation4}).
\end{itemize}

\subsection{Theoretical contributions}

\begin{itemize}
\item A generalization bound for the counterfactual \ac{LTR} estimator using a position-based click model combined with the inverse propensity scoring estimator (Theorem~\ref{CLTR-bound1}, Chapter~\ref{chapter:01-online-evaluation1}).
\item A generalization bound for the counterfactual \ac{LTR} estimator employing a trust-bias click model with the doubly robust estimator (Theorem~\ref{CLTR-bound}, Chapter~\ref{chapter:01-online-evaluation2}).
\item A proof extending the position-based inverse propensity scoring counterfactual LTR generalization bound to the trust-bias based counterfactual LTR and doubly robust estimator (Appendix~\ref{appendix-cikm}, Chapter~\ref{chapter:01-online-evaluation2}).
\item A proof establishing a closed-form, variance-optimal baseline correction term applicable to off-policy evaluation estimates and learning gradients in contextual bandits (Section~\ref{sec:estimator_var}, Chapter~\ref{chapter:01-online-evaluation3}).
\item A theoretical demonstration of the sub-optimality of the REINFORCE estimator when samples are reused across iterations in reinforcement learning-based diffusion model fine-tuning (Theorem~\ref{theorem:1}, Chapter~\ref{chapter:01-online-evaluation4}).
\item A proof demonstrating that the proposed \ac{LOOP} algorithm achieves lower variance compared to traditional proximal policy optimization methods in diffusion model fine-tuning (Theorem~\ref{prop:1}, Chapter~\ref{chapter:01-online-evaluation4}).
\end{itemize}

\subsection{Empirical contributions}
\begin{itemize}
  \item \textbf{Safe counterfactual learning to rank.}  
        Extensive simulations on three public learning to rank benchmarks (Yahoo!\ Webscope, MSLR-WEB30k, and Istella) show that the proposed exposure-based CRM method eliminates the long “unsafe” warm-up period of IPS, matching the production policy after \(\sim\)400 interactions and converging to the IPS performance asymptotically.

  \item \textbf{Robust safety guarantees in advanced CLTR.}  
        Across the same public learning to rank datasets -- even under an adversarial click model, the safe DR and PRPO algorithms reach logging policy performance within the first few hundred queries, more than three orders of magnitude earlier than doubly robust method, while still converging to the optimal ranking performance asymptotically.

  \item \textbf{Variance-optimal off-policy bandit learning.}  
        On a synthetic benchmark that sweeps action-space sizes and logging-policy optimality, the proposed \(\beta\text{-IPS}\) estimator reduces gradient variance by up to two orders of magnitude and yields consistently higher policy value and lower MSE than IPS, SNIPS, and DR in both full-batch and mini-batch training settings.

  \item \textbf{Efficient RL fine-tuning of diffusion models.}  
        On the T2I-CompBench, aesthetic, and the image-text-alignment tasks, the LOOP algorithm surpasses PPO and REINFORCE baselines: with \(k{=}4\) trajectories per prompt it improves attribute-binding scores by 10–15 points and raises aesthetic quality by \(+\!1.0\) reward points, while simultaneously lowering reward-variance throughout training.

\end{itemize}

\subsection{Resource contributions}
\begin{itemize}
    \item \textbf{Safe Counterfactual LTR implementation.}  
    All code and experiment scripts for the safe counterfactual LTR methods described in Chapter~\ref{chapter:01-online-evaluation1} and \ref{chapter:01-online-evaluation2} are released at:  
    \href{https://github.com/shashankg7/cikm-safeultr}{\texttt{safe-cltr}}.

    \item \textbf{Optimal baseline corrections for off-policy bandits.}  
    The PyTorch implementation that accompanies Chapter~\ref{chapter:01-online-evaluation3} is released at:  
    \href{https://github.com/shashankg7/recsys2024_optimal_baseline}{\texttt{optimal-baseline-cb}}.

\end{itemize}

\section{Thesis Overview}
\label{section:introduction:overview}

The dissertation begins with the introduction, which is the current chapter the reader is engaged with. 
The research chapters in this thesis are structured into two distinct parts, with the first part predominantly addressing the safety aspect of counterfactual \ac{LTR} methods in Chapters 2 and 3. 
Chapter 3 should preferably be read following Chapter 2, as it extends the safety framework established in Chapter 2. 
Chapter 4 can be approached independently of other chapters in the thesis, as it primarily examines the top-1 action setting in contextual bandits. 
Similarly, Chapter 5 stands as a self-contained unit that can be read separately, focusing mainly on the post-training refinement of text-to-image foundation models.

\section{Origins}
\label{section:introduction:origins}

In this section, we list the origins and the list of contributions for each chapter in the thesis.

\begin{enumerate}[label=\textbf{Chapter~\arabic*},align=left]
\setcounter{enumi}{1}

\item is based on the following paper:
\begin{itemize}
\item \bibentry{gupta2023safe}.
\item[] SG: Conceptualization, Formal Analysis, Investigation, Methodology, Resources, Software, Validation, Writing – Original Draft Preparation. HO: Conceptualization, Formal Analysis, Investigation, Methodology, Supervision, Writing – Review \& Editing. MdR: Conceptualization, Methodology, Supervision, Validation, Funding Acquisition,  Writing – Review \& Editing.
\end{itemize}

\item is based on the following paper:
\begin{itemize}
\item \bibentry{gupta-2024-practical}.
\item[] SG: Conceptualization, Formal Analysis, Investigation, Methodology, Resources, Software, Validation, Writing – Original Draft Preparation. HO: Conceptualization, Formal Analysis, Investigation, Methodology, Supervision, Writing – Review \& Editing. MdR: Conceptualization, Methodology, Supervision, Validation, Funding Acquisition,  Writing – Review \& Editing.
\end{itemize}

\item is based on the following paper:
\begin{itemize}
\item \bibentry{gupta-2024-optimal}.
\item[] SG: Conceptualization, Formal Analysis, Investigation, Methodology, Resources, Software, Validation, Writing – Original Draft Preparation. OJ: Conceptualization, Formal Analysis, Investigation, Methodology, Resources, Software, Validation, Writing – Original Draft Preparation.  HO: Conceptualization, Formal Analysis, Investigation, Methodology, Supervision, Writing – Review \& Editing. MdR: Conceptualization, Methodology, Supervision, Validation, Funding Acquisition,  Writing – Review \& Editing.
\end{itemize}

\item is based on the following paper:
\begin{itemize}
\item \bibentry{gupta2025simple}.
\item[] SG: Conceptualization, Formal Analysis, Investigation, Methodology, Resources, Software, Validation, Writing – Original Draft Preparation. CA: Formal Analysis, Writing – Review \& Editing.  TYL: Formal Analysis, Writing – Review \& Editing. SDR: Formal Analysis, Writing – Review. HO: Formal Analysis, Writing – Review \& Editing. MdR: Formal Analysis, Writing – Review \& Editing. SNS: Formal Analysis, Funding Acquisition, Writing – Review \& Editing
\end{itemize}

\noindent \!\!\!\!\!\!\!\!\!\!\!\!\!\!\! The writing of the thesis also benefited from work on the following publications:
\begin{itemize}
\item \bibentry{gupta2023ictir}.
\item \bibentry{gupta-2023-first-abstract}
\item \bibentry{gupta2024unbiased}
\item \bibentry{gupta2023recent}
\item \bibentry{gupta2023recentfire}
\item \bibentry{bakker2024simpler}
\item \bibentry{gupta2024twostaged}
\end{itemize}
\end{enumerate}

\part{Safe Deployment in Learning-to-rank}

\chapter{Safe Deployment for Counterfactual Learning-to-Rank via Exposure-based Risk Minimization}
\chaptermark{Safe Deployment for Counterfactual Learning-to-Rank}
\label{chapter:01-online-evaluation1}

\footnote[]{This chapter was published as~\citep{gupta2023safe}.}
The goal of \ac{CLTR} is to correct for the selection bias in user interaction data (clicks).
It does so by weighting click interactions by the inverse of the estimated effect of the selection bias. 
Mathematically, \ac{CLTR} produces unbiased estimates of the ``true" relevance signal from biased click signals, in expectation. 
However, it is well-known that \ac{CLTR} suffers from high-variance problems, which is exacerbated in the limited logged interaction data size setting~\cite{gupta2024unbiased,oosterhuis2022reaching}.
This can lead to unsafe behavior during deployment, as the deployment of a sub-optimal ranking policy can have negative effects on the user experience, and subsequently on the business metrics.
To avoid such scenarios, we need a safety mechanism to avoid such detrimental behavior. This brings us to the first research question:

\begin{enumerate}[label=\textbf{RQ1},ref={RQ\arabic*}]
\item \acl{rq:safe1}
\end{enumerate}

\noindent In this chapter, we aim to answer this question by first deriving a generalization bound of the \ac{CLTR} estimator for the position bias. 
We show that optimizing for the generalization bound results in a guarantee that the new ranking policy will be at least as good as the production/behavior policy.

\section{Introduction}
\label{sec:intro}
\Ac{LTR} methods optimize ranking systems so that the resulting ranking behavior maximizes a given ranking metric~\cite{liu2009learning}. 
Traditionally, most \ac{LTR} methods applied a supervised learning procedure based on manually-created relevance judgements.
However, obtaining such judgements is time-consuming, expensive and does not scale~\citep{chapelle2011yahoo, qin2010letor}. 
As an alternative, \ac{LTR} methods have been developed that rely on clicks, as they are much cheaper to obtain in abundance in the form of user interaction logs~\citep{joachims2002optimizing}.

Despite its low costs, click data is generally strongly affected by different forms of interaction bias.
Interactions with rankings often suffer from \textit{position bias}~\cite{craswell2008experimental}: the position at which an item was shown often affects its \ac{CTR} more than its relevance.
As a result, the clicks observed in interaction logs are often more reflective of where items were displayed during logging than how relevant users find them.
Thus, naively using this data for \ac{LTR}, without corrections, can result in heavily \emph{biased} models with suboptimal ranking performance~\citep{joachims2017unbiased, wang2016learning}.

To mitigate the bias problem in interaction data, the field of \ac{CLTR} has proposed methods to mitigate bias with unbiased estimation~\citep{joachims2017unbiased}.
\Ac{CLTR} mainly relies on exposure-based \ac{IPS}~\cite{oosterhuis2020unbiased,wang2018position}, a \ac{LTR} specific adaptation of the \ac{IPS} counterfactual estimation method~\cite{swaminathan2015batch, joachims2016counterfactual, horvitz1952generalization}.
Standard exposure-\ac{IPS} weights clicks by the inverse effect of position-bias on the clicked item.
This procedure thus gives more weight to clicks on items that are underrepresented due to position-bias, and vice versa.
In expectation, this removes the effect of position-bias from the loss that is optimized.
 
\header{Unsafe \ac{CLTR}}
Despite enabling unbiased optimization, \ac{IPS} is also known to suffer from high variance~\cite{joachims2017unbiased, oosterhuis2022doubly}. 
Specifically, in cases with a lack of click data or with large amounts of noise, high variance can make IPS-based \ac{CLTR} unreliable and lead to very sub-optimal ranking models~\citep{jagerman2020safe, oosterhuis2021robust}.
This problem can be so severe that the learned ranking models can be worse than the model used to log the interaction data.
Deploying such a learned model could thus result in a substantially degraded user experience. 
In other words, despite the improvements that IPS-based \ac{CLTR} can bring, it is also an \emph{unsafe} approach since it can lead to considerable deteriorations, under certain circumstances.

This (un)safety issue is not unique to \ac{IPS}-based \ac{CLTR}. 
\citet{swaminathan2015batch} address this issue for contextual bandit problems by applying a generalization bound.
Such a bound can provide a high-confidence upper limit on the difference between the true and estimated performance of a bandit policy~\cite{shalev2014understanding, thomas2015high}.
This allows for safer \emph{conservative} optimization.
For instance, \citet{wu2018variance} introduce a bound based on the divergence between the new policy and the logging policy.
This bound avoids policies that stray away from the logging policy, unless there is strong evidence that they are actual improvements.
This method might appear to be a great fit for \ac{CLTR}, but, unfortunately, it is based on action propensities that do not generalize well to the very large action spaces in \ac{CLTR}.
Therefore, there is a need for a conservative generalization bound that is practical and effective in the \ac{CLTR} setting.

\header{Safe \ac{CLTR}}
To address this gap, in this chapter we propose an exposure-based \ac{CRM} method that is specifically designed for safe \ac{CLTR}.
Similar to how exposure-based \ac{IPS} deals with the large action spaces in ranking settings, our method is based on an exposure-based alternative to action-based generalization bounds.
We first introduce a divergence measure based on differences between the distributions of exposure of a new policy and a safe logging policy.
Then we provide a novel generalization bound and prove that it is a high-confidence lower-bound on the performance of a learned policy.
When uncertain, this bound defaults to preferring the logging policy and thus avoids decreases in performance due to variance.
In other words, with high-confidence, ranking models optimized with this bound are guaranteed to never deteriorate the user experience, even when little data is available.

\header{Main contributions}
We are the first to address \ac{CRM} for \ac{CLTR} and contribute a novel exposure-based \ac{CRM} method for safe \ac{CLTR}.
Our experimental results show that our proposed method is effective at avoiding initial periods of bad performance when little date is available, while also maintaining high performance at convergence.
Our novel exposure-based \ac{CRM} method thus enables safe \ac{CLTR} that can mitigate many of risks attached to previous methods.

Accordingly, we hope that our contribution makes the adoption of \ac{CLTR} methods more attractive to practitioners working on real-world search and recommendation systems.

\section{Related Work}
In this section, we review related work on \ac{CLTR} and \ac{CRM} in off-policy learning. 

\subsection{Counterfactual learning to rank}
\Ac{LTR} is a well-established area of research that deals with learning optimal rankings to maximize a pre-defined notion of utility~\citep{liu2009learning}.
Traditionally, \ac{LTR} systems were optimized using supervised learning on manually-created relevance judgements~\citep{chapelle2011yahoo}.
However, the manual curation of relevance judgements is a time-consuming and costly process~\citep{chapelle2011yahoo,qin2010letor}.
Moreover, manually-graded relevance signals do not always align well with actual user preferences~\citep{sanderson2010user}. 
Due to these shortcomings, \ac{LTR} from user interactions has become a popular alternative to supervised \ac{LTR}~\citep{speretta2005personalized,jiang2016learning,chapelle2009dynamic,joachims2017unbiased}.

Learning from user interactions/click logs was introduced in the pioneering work of~\citet*{joachims2002optimizing}.
Click data is relatively cheap to collect and indicative of actual user preferences~\cite{radlinski2008does}.
In spite of these advantages, click data is known to be a noisy and biased estimate of the true user preferences~\cite{craswell2008experimental,oosterhuis2020unbiased}.
Some of the common biases identified in the \ac{LTR} literature are position bias~\cite{craswell2008experimental}:
trust bias~\cite{agarwal2019addressing}, and item-selection bias~\cite{oosterhuis2020policy}.
 
To counter the effect of bias, \citet{joachims2017unbiased} introduced counterfactual learning in the context of \ac{LTR}. 
They proposed the application of \acf{IPS}, a causal inference technique that has prevalence in the offline bandit learning literature~\citep{joachims2016counterfactual}.
\ac{IPS} models the probability of the user examining a document at a given displayed rank.
The inverse of the examination probability, i.e., the inverse propensity, is used to correct for the position bias. 
    As a result of the inverse weighing scheme, \ac{IPS}-based \ac{LTR} optimization is unaffected by position bias, in expectation~\citep{joachims2017unbiased}.
Since its introduction, there has been an increasing interest in the area, with several application of \ac{IPS} in the context of ranking~\citep{agarwal2019addressing,vardasbi2020inverse,oosterhuis2020policy,wang2018position}.
Recent work has also explored \ac{CLTR} under a stochastic logging policy, where some exploration is introduced, as opposed to pure exploitation~\citep{oosterhuis2021unifying,oosterhuis2020policy,yadav2021policy}.

With regard to safety in learning from user interactions, \citet{jagerman2020safe} introduced the notation of safe exploration for offline contextual bandit algorithms. 
The authors introduced \ac{SEA}, which applies high-confidence performance bounds to \emph{safely} choose between the deployment of a logging policy and a learned policy.
\citet{oosterhuis2021robust} applied this context to \ac{LTR} and introduced a generalization and specialization framework to safely choose between a generalized feature-based \ac{LTR} model, and a specialized tabular \ac{LTR} model.
The important difference between prior work and our work is that existing methods safely \emph{choose} between policies, whereas our method safely \emph{optimizes} a policy.
To the best of our knowledge, we are the first to consider notion of safety for the \emph{optimization} of \ac{LTR} models.

\subsection{Counterfactual risk minimization for offline learning from logs}
A relevant area closely related to \ac{CLTR} is off-policy learning, or offline learning from bandit feedback data~\citep{joachims2016counterfactual,saito2021counterfactual,swaminathan2015batch,he2019off}.
Off-policy learning tries to bridge the mismatch between the action distributions of a new policy and the logging policy~\cite{joachims2016counterfactual}.
The most common techniques used to achieve that goal are \ac{IPS} and importance sampling~\cite{horvitz1952generalization}.
However, as noted by~\citet{cortes2010learning}, the \ac{IPS} estimator can have unbounded variance, which can lead to large errors in its estimation.
Consequently, optimization with IPS can result in convergence problems and severely suboptimal policies. 

To account for this high-variance problem, \citet{swaminathan2015batch} introduced \acf{CRM}, an off-policy method that explicitly controls for the variance during off-policy learning from bandit feedback data. 
Specifically, their learning objective consists of both the \ac{IPS} loss and a variance regularization term, which minimizes the dissimilarity between the two policies.
This variance regularization term represents the \emph{risk} that stems from the variance of the IPS estimation, however,
computing it requires a pass over the entire data which does not scale well. 
As a scalable alternative, \citet{wu2018variance} introduced \ac{VCRM}, where the authors estimate the \emph{risk} of the new policy by random sampling from the logged data. 
The objective function to be optimized in the \ac{VCRM} method is derived from a generic theoretical analysis of learning from importance sampling~\cite{cortes2010learning}. 
The risk term in the \ac{VCRM} method is defined in terms of a specific divergence between the logging policy and the new policy, known as the R\'enyi divergence~\citep{renyi1961measures}. 
To the best of our knowledge, there is no existing work on \ac{CRM} in a \ac{LTR} setting, making the work in this chapter the first to propose a \ac{CRM} approach for the \ac{LTR} task.

\section{Background}

\subsection{Learning to rank}
The objective of learning to rank methods is to optimize a ranking policy ($\pi$), so that for user-issued queries ($q$) it provides the optimal ranking of their pre-selected candidate document sets ($D_q$)~\citep{liu2009learning}.
Formally, this objective can be expressed as the maximization of the following utility function:
\begin{equation}
    U(\pi) =  \mathbb{E}_{q} \mleft[ \sum_{d \in D_q} \rho(d \mid q, \pi) P(R=1 \mid d, q) \mright]. \label{true-utility1}
\end{equation}
where $\rho(d \mid q, \pi)$ is the weight $\pi$ gives to document $d$ for query $q$.
The choice of $\rho$ determines what metric is optimized, for instance, the well-known \ac{NDCG} metric~\citep{jarvelin2002cumulated}:
\begin{equation}
    \rho_{\text{DCG}}(d \mid q, \pi) = \mathbb{E}_{y \sim \pi( \cdot \mid q)} \mleft[ (\log_2(\textrm{rank}(d \mid y) + 1))^{-1} \mright]. \label{rho}
\end{equation}
where $y$ is a ranking sampled from the policy $\pi$. For this chapter, the aim is to optimize the expected number of clicks, the next subsection will explain how we choose $\rho$ accordingly.

\subsection{Counterfactual learning to rank}
\label{sec:background:cltr}

\textbf{Position bias in clicks}.
Optimizing the \ac{LTR} objective in Eq.~\ref{true-utility1} requires access to the true relevance labels ($P(R=1 \mid d, q)$), which is often impossible in real-world ranking settings. 
As an alternative, \ac{CLTR} uses clicks, since they are present in abundance as logged user interactions. 
However, clicks are a biased indicator of relevance; for this chapter, we will assume the relation between clicks and relevance is determined by a position-based click model~\citep{joachims2017unbiased, chuklin-click-2015}.
For a document $d$ displayed in ranking $y$ for query $q$, this means the click probability can be decomposed into a rank-based examination probability and a document-based relevance probability:
\begin{equation}
    P(C=1 \mid d, q, y) = P(E=1 \mid \text{rank}(d \mid y))  P(R=1 \mid d, q).   
    \label{click-model}
\end{equation}
The key characteristic of the position-based click model is that the probability of examination only depends on the rank at which a document is displayed: $P(E=1 \mid d, q, y) = P(E=1 \mid \text{rank}(d \mid y))$.
Furthermore, this model assumes that clicks only take place when a document is both relevant to a user and examined by them.
Consequently, the click signal is an indication of both the relevance and examination of documents.
Thus, the position at which a document is displayed can have a stronger effect on its click probability than its actual relevance~\citep{craswell2008experimental}.

\header{Inverse-propensity-scoring for \ac{CLTR}}
We assume a setting where $N$ interactions have been logged using the logging policy $\pi_0$, for each interaction $i$ the query $q_i$, the displayed ranking $y_i$, and the clicks $c_i$ are logged:
\begin{equation}
    \mathcal{D} = \big\{q_i, y_i, c_i \big\}^N_{i=1}. \label{logs}
\end{equation}
We will use $c_i(d) \in \{0,1\}$ to denote whether document $d$ was clicked at interaction $i$.
Furthermore, we choose $\rho$ to match the examination probabilities under $\pi$:
\begin{equation}
    \rho(d \mid q, \pi) =  \mathbb{E}_{y \sim \pi( \cdot \mid q)} \big[ P(E=1 \mid \text{rank}(d \mid y))  \big] = \rho(d).
    \label{eq:exposurenew}
\end{equation}
Hence, our optimization objective $U(\pi)$ is equal to the expected number of clicks (cf.\ Eq.~\ref{true-utility1} and~\ref{click-model}).

In order to apply \ac{IPS}, we need the propensity of each document~\citep{joachims2017unbiased}, following \citet{oosterhuis2021unifying} we use:
\begin{equation}
\begin{split}
    \rho(d \mid q, \pi_0) &=  P(E =1  \mid \pi_{0}, d, q)\\
        &=  \mathbb{E}_{y \sim \pi_{0}( \cdot \mid q)} \big[ P(E=1 \mid \text{rank}(d \mid y))  \big]\\ \nonumber
        & = \rho_0(d).
\label{policy-aware-exposure1}
\end{split}
\end{equation}
Thus, the exposure of $d$ represents how likely it is examined when using $\pi_0$ for logging.
Thereby, it indicates how much the clicks on $d$ underrepresent its relevance.
For the sake of brevity, we drop $q$, $\pi$ and $\pi_0$ from our notation when their values are clear from the context: i.e., $\rho(d \mid q, \pi) = \rho(d)$ and $\rho(d \mid q, \pi_0) = \rho_0(d)$.

The exposure-based \ac{IPS} estimator takes each click in $\mathcal{D}$ and weights it inversely to $\rho_{0}(d)$ to correct for position-bias~\citep{joachims2017unbiased, oosterhuis2021unifying}:
\begin{equation}
    \hat{U}(\pi) = \frac{1}{N} \sum_{i=1}^{N} \sum_{d \in D_{q_i}}  \frac{\rho(d)}{\rho_{0}(d)} c_i(d).
    \label{cltr-obj}
\end{equation}
In other words, to compensate that position bias lowers the click probability a document by a factor of $\rho_{0}(d)$, clicks are weighted by $1/\rho_{0}(d)$ to correct for this effect in expectation.
As a result, clicks on documents that $\pi_0$ is likely to show at positions with low examination probabilities (i.e., the bottom of a ranking) receive a higher \ac{IPS} weight to compensate.

\header{Statistical properties of the IPS estimator}
The \ac{IPS} estimator $\hat{U}(\pi)$ (Eq.~\ref{cltr-obj}) is an unbiased and consistent estimate of our \ac{LTR} objective $U(\pi)$ (Eq.~\ref{true-utility1})~\cite{oosterhuis2022reaching}. 
It is \emph{unbiased} since its expected value is equal to our objective:
\begin{equation}
\mathbb{E}_{q,y,c} \mleft[ \hat{U}(\pi) \mright] = U(\pi),
\label{eq:unbiased_prop}
\end{equation}
and it is \emph{consistent} because this equivalence also holds in the limit of infinite data:
\begin{equation}
\lim\limits_{N \to \infty} \hat{U}(\pi) = U(\pi).
\end{equation}
For proofs of these properties, we refer to previous work~\cite{oosterhuis2020policy,joachims2017unbiased,oosterhuis2020learning}.

It is important to note that the unbiasedness and consistency properties do not indicate that the actual IPS estimates will be reliable.
The reason for this is that the estimates produced by IPS are also affected by its variance:
\begin{equation}
    \mathrm{Var}_{y,c}\mleft[\hat{U}(\pi) \mid q\mright] = \sum_{d \in D_q}  \frac{\rho(d)^2}{\rho_{0}(d)^2} \mathrm{Var}_{y,c}\mleft[ c(d) \mid \pi_0, q\mright].
\end{equation}
As we can see, its variance is very large when some propensities are small, due to the $\rho_{0}(d)^{-2}$ term.
As a result, the actual estimates that IPS produces can contain very large errors, especially when $N$ is relatively small or clicks are very noisy.
In other words, $\hat{U}(\pi)$ can be far removed from the true $U(\pi)$, and therefore, optimization with IPS can be very unsafe and lead to unpredictable results.

\begin{figure*}[t]
\includegraphics[width=\linewidth,trim= 0 -1cm 0 0,clip]{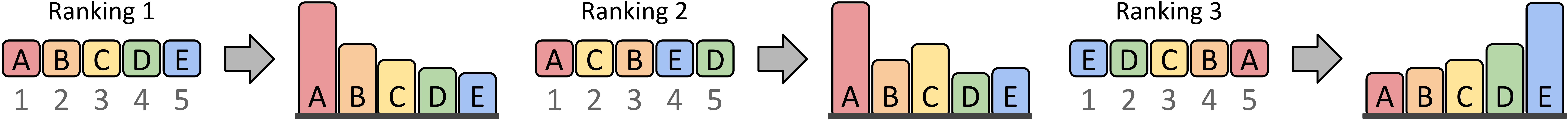}
\caption{
Three rankings and their normalized expected exposure distributions (Eq.~\ref{norm-exposure}) based on DCG weights (Eq.~\ref{rho}).
According to our exposure-based divergence, ranking 1 and ranking 2 are quite similar despite only agreeing on the placing of document A.
In contrast, ranking 1 and ranking 3 also agree on the placement of a single document (C) but have the highest possible dissimilarity, due to their highly mismatched exposure distributions.
}
\label{fig:exposure}
\end{figure*}

\subsection{Counterfactual risk minimization for offline bandit learning}\label{crm-bandit}

The foundational work by~\citet{swaminathan2015batch}  introduced the idea of \acf{CRM} for off-policy learning in a contextual bandit setup.
To avoid the negative effects of high-variance with \ac{IPS} estimation during bandit optimization, they use a generalization bound through the addition of a risk term~\citep{maurer2009empirical}.
With a probability of $1-\delta$, the \ac{IPS} estimate minus the risk term is a lower bound on the true utility of the policy: 
\begin{equation}
    P\big( U(\pi) \geq \hat{U}(\pi) - \text{Risk}(\delta) \big) > 1-\delta.
    \label{eq:generalizationbound}
\end{equation}
Therefore, optimization of the lower bound can be more reliable than solely optimizing the \ac{IPS} estimate ($\hat{U}(\pi)$), since it provides a high-confidence guarantee that a lower bound on the \emph{true} utility of the policy is maximized.

In particular, \citet{swaminathan2015batch} propose using the sample variance as the risk factor:
\begin{equation}
    \hat{U}_\text{action-CRM}(\pi) = \hat{U}_\text{action}(\pi) - \lambda \sqrt{\frac{1}{N} \mathrm{Var}\big[\hat{U}_\text{action}(\pi)\big]}, \label{crm-loss}
\end{equation}
where $\lambda \in \mathbb{R}^{>0}$ is an alternative to the $\delta$ parameter that also determines how probable it provides a bound on the true utility.
Importantly, this bound is based on an action-based \ac{IPS} estimator. 
For our \ac{LTR} setting this would translate to:
\begin{equation}
\hat{U}_\text{action}(\pi)
=
\frac{1}{N}\sum^{N}_{i=1} \frac{\pi(y_i \mid q_i)}{\pi_0(y_i \mid q_i)}
\sum_{d \in D_{q_i}} c_i(d).
\label{eq:actionips}
\end{equation}
However, action-based \ac{IPS} estimation does not work well in the \ac{LTR} setting because the very large number of possible rankings result in extremely small action propensities: $\pi_0(y_i \mid q_i)$, which creates a high-variance problem.
As discussed in Section~\ref{sec:background:cltr}, for this reason \ac{CLTR} uses exposure-based propensities instead (Eq.~\ref{policy-aware-exposure1} and~\ref{cltr-obj}), as they effectively avoid extremely small values.
As a result, the \ac{CRM} approach by \citet{swaminathan2015batch} is not effective for \ac{CLTR}, since the high-variance of its action-based \ac{IPS} make the method impractical in the ranking setting.

Another downside of the \ac{CRM} approach is that the computation of the sample-variance requires a full-pass over the training dataset, which is computationally costly for large-scale datasets. 
As a solution, \citet{wu2018variance} introduce variational \ac{CRM} (VCRM) which uses an upper bound on the variance term based on the R\'enyi divergence between the new policy and the logging policy~\citep{renyi1961measures}.
This R\'enyi divergence is approximated via random sampling, thus making the VCRM method suitable for stochastic gradient descent-based training methods~\cite{nowozin2016f}. 
Nevertheless, this \ac{CRM} approach still relies on action-based propensities, and therefore, does not provide an effective solution for the high-variance problem in \ac{CLTR}.

\section{A Novel Exposure-Based Generalization Bound for CLTR}

In order to develop a \ac{CRM} method for \ac{CLTR} with safety gaurantees, we aim to find a risk term that gives us a generalization bound as in Eq.~\ref{eq:generalizationbound}.
Importantly, this bound has to be effective in the \ac{LTR} setting, therefore, our approach should avoid action-based propensities.
 
We take inspiration from previous work by \citet{wu2018variance}, who use the fact that the R\'enyi divergence is an upper bound on the variance of an \ac{IPS} estimator:
\begin{equation}
    \mathrm{Var} \mleft[ \hat{U}_\text{action}(\pi) \mright] \leq d_2(\pi \,\Vert\, \pi_0),
    \label{divergence-bandit}
\end{equation}
 where $d_2$ is the exponentiated R\'enyi divergence between the new policy and the logging policy~\citep{renyi1961measures}:
 \begin{equation}
    d_2(\pi \,\Vert\, \pi_0) = \mathbb{E}_{q}\mleft[ \sum_{y} \mleft(\frac{\pi(y \mid q)}{\pi_0(y \mid q)}\mright)^2 \pi_0(y \mid q)  \mright].
    \label{eq:actionbaseddiv1}
\end{equation}
In other words, the dissimilarity between the logging policy and a new policy can be used to bound the variance of the IPS estimate of the new policy's performance.
 However, because this divergence is based on action propensities, it is not effective in the \ac{LTR} setting.
 
This section introduces a novel exposure-based measure of divergence that can produce a desired generalization bound for \ac{LTR} optimization.
Section~\ref{sec:normexposure} below introduces the concept of normalized exposure that treats rankings as exposure distributions.
Subsequently, Section~\ref{sec:varbound} proves that R\'enyi divergence based on normalized exposure can bound the variance of an exposure-based \ac{IPS} estimator.
Finally, Section~\ref{sec:perfbound1} uses this novel variance bound to construct a proven generalization bound for \ac{CLTR}.

\subsection{Normalized expected exposure}
\label{sec:normexposure}

R\'enyi divergence is only valid for probability distributions, e.g., $d_2(\pi \,\Vert\, \pi_0)$ with $\pi(y \mid q)$ and $\pi_0(y \mid q)$.
However, expected exposure is not a probability distribution, i.e., the values of $\rho(d)$ (Eq.~\ref{eq:exposurenew}) or $\rho_0(d)$ (Eq.~\ref{policy-aware-exposure1}) do not necessarily sum up to one, over all documents to be ranked.
This is because users generally examine more than a single item in a single displayed ranking~\citep{craswell2008experimental}, as a result, expected exposure can be seen as a distribution of multiple examinations.
Our insight is that a valid probability distribution can be obtained by normalizing the expected exposure:
\begin{equation}
        \rho'(d) = \frac{ \rho(d)}{\sum_{d' \in D} \rho(d')}
        = \frac{\rho(d)}{\mathrm{Z}},
    \label{norm-exposure}
\end{equation}
where the normalization factor is a constant that only depends on $K$, the (truncated) ranking length:
\begin{equation}
\begin{split}
    \mathrm{Z} = \sum_{d \in D}^{} \rho(d) & = \sum_{d \in D} \mathbb{E}_{y \sim \pi } \big[ P\big(E=1 \mid \text{rank}(d \mid y)\big)  \big]   \\
    & = \mathbb{E}_{y \sim \pi} \Big[  \sum_{d \in D}  P\big(E=1 \mid \text{rank}(d \mid y)\big)  \Big]    \\
    & = \mathbb{E}_{y \sim \pi} \Big[  \sum_{k=1}^K  P\big(E=1 \mid k \big)  \Big] \\
    & =  \sum_{k=1}^K  P\big(E=1 \mid k \big).  
    \end{split} 
    \label{total-exposure1}
\end{equation} 
In this way, $\mathrm{Z}$ can be seen as the expected amount of examination that any ranking will receive, and $\rho'$ as the probability distribution that indicates how it is expected to spread over documents.

An important property is that the ratio between two propensities is always equal to the ratio between their normalized counterparts:
\begin{equation}
    \frac{\rho(d)}{\rho_0(d)} = \frac{\rho'(d)}{\rho'_{0}(d)}.
    \label{eq:ratio}
\end{equation}
This is relevant to \ac{IPS} estimation since it only requires the ratios between propensities, the proofs in the remainder of this chapter make use of this property.

Finally, using the normalized expected exposure, we can introduce the exponentiated exposure-based R\'enyi divergence:
\begin{align}
    d_2(\rho \,\Vert\, \rho_0) &=  \mathbb{E}_{q} \bigg[ \sum_{d \in D_{q}}^{} \rho'_{0}(d) \mleft( \frac{\rho'(d)}{\rho'_{0}(d)} \mright)^{2}  \bigg].
    \label{eq:exposuredivergence}
\end{align}
The key difference between our exposure-based divergence and action-based divergence is that it allows policies to be very similar, even when they have no overlap in the rankings  they produce.
As an intuitive example, Figure~\ref{fig:exposure} displays three rankings and their associated normalized expected exposure distributions; these are the distributions for deterministic policies that give 100\% probability to one of the rankings.
Under action-based divergence, these policies would have the highest possible dissimilarity since they have no overlap in their possible actions, i.e., the rankings they give non-zero probability.
In contrast, exposure-based divergence gives high similarity between ranking 1 and ranking 2, since the differences in their exposure distribution are minor.
We note that these rankings still disagree on the placement of all documents except one.
Conversely, for ranking 1 and ranking 3, which also only agree on a single document placement, exposure-based divergence gives the lowest possible similarity score because their exposure distributions are highly mismatched.
Importantly, by solely considering differences in exposure distributions, exposure-based divergence naturally weighs differences at the bottom of rankings as less impactful than changes that affect the top.
As a result, exposure-based divergence more closely corresponds with common ranking metrics (Eq.~\ref{true-utility1}) than existing action-based divergences.

\subsection{Exposure-divergence bound on variance}
\label{sec:varbound}

We now provide proof that exposure-based divergence is an upper bound on the variance of \ac{IPS} estimators for \ac{CLTR}.\footnote{The following proof differs slightly from the original proof published in~\cite{gupta2023safe}, as we overlooked an additional constant term in the original paper.}
\begin{theorem}
\label{var-theorm}
    Given a ranking policy $\pi$ and logging policy $\pi_{0}$, with the expected exposures $\rho(d)$ and $\rho_{0}(d)$ respectively, the variance of the exposure-based \ac{IPS} estimate $\hat{U}(\pi)$ is upper-bounded by exposure-based divergence:
    \begin{equation}
        \begin{aligned}
            \mathrm{Var}_{q,y,c}\mleft[\hat{U}(\pi)\mright] \leq  \frac{\mathrm{Z} }{N} d_2(\rho \,\Vert\, \rho_0) + \frac{1}{N}.  \label{variance-full}
    \end{aligned}
    \end{equation} 
\end{theorem}
\begin{proof}
    From the definition of $\hat{U}(\pi)$ (Eq.~\ref{cltr-obj}) and the assumption that queries $q$ are \ac{i.i.d}, the variance of the counterfactual estimator can be expanded by applying the law of total variance as follows:
    \begin{equation}
        \mathrm{Var}_{q,y,c}\mleft[\hat{U}(\pi )\mright] =  \frac{1}{N} \Big( \mathbb{E}_{q} \mleft[ \mathrm{Var}_{y,c}\mleft[ \hat{U}(\pi) \mid q\mright] \mright] + \mathrm{Var}_{q} \mleft[ \mathbb{E}_{y,c}\mleft[ \hat{U}(\pi) \mid q\mright] \mright] \Big) . 
        \label{eq:vardecom1}
    \end{equation}
    The second term (variance over queries) can be expanded as follows:
    \begin{align}
        \mathrm{Var}_{q} \mleft[ \mathbb{E}_{y,c}\mleft[ \hat{U}(\pi) \mid q\mright] \mright]
        & = \mathbb{E}_{q} \mleft[ \mathbb{E}_{y,c}\mleft[ \hat{U}(\pi) \mid q\mright] ^2 \mright] -  \mathbb{E}_{q} \mleft[ \mathbb{E}_{y,c}\mleft[ \hat{U}(\pi) \mid q\mright] \mright] ^2   \\ \nonumber
        & \leq \mathbb{E}_{q} \mleft[ \mathbb{E}_{y,c}\mleft[ \hat{U}(\pi) \mid q\mright] ^2 \mright] \\ \nonumber
        & = \mathbb{E}_{q} \mleft[ \mleft[ U(\pi) \mid q\mright] ^2 \mright] \\ \nonumber
        & \leq  1, \nonumber 
    \end{align}
    where in the second step, we use the unbiasedness property (Eq.~\ref{eq:unbiased_prop}) of the counterfactual estimator, and use the fact that the true utility is non-zero, i.e., $U(\pi) \geq 0$.
    In the last step, we make use of the fact that the true utility is bounded, and is upper bounded by $1$. This is a safe assumption if the utility is normalized, as is the case for: normalized discounted cumulative gain or click-through rate.

    Substituting it back in the counterfactual utility variance (Eq.~\ref{eq:vardecom1}), we get:
    \begin{equation}
        \mathrm{Var}_{q,y,c}\mleft[\hat{U}(\pi )\mright] \leq  \frac{1}{N} \Big( \mathbb{E}_{q} \mleft[ \mathrm{Var}_{y,c}\mleft[ \hat{U}(\pi) \mid q\mright] \mright] + 1 \Big) . 
        \label{eq:var_new}
    \end{equation}
    Next, since we have assumed a rank-based examination model (Section~\ref{sec:background:cltr}), the examinations of documents are independent.
    This allows us to rewrite the variance conditioned on a single query:
    \begin{align}
        \mathrm{Var}_{y,c}\mleft[\hat{U}(\pi \mid q)\mright] 
        & =  \mathrm{Var}_{y,c}\mleft[ \sum_{d \in D_{q}}^{}   \frac{\rho(d)}{\rho_{0}(d)} c(d, q) \mright] \\ \nonumber
        & = \sum_{d \in D_{q}}^{} \mathrm{Var}_{y,c}\mleft[  \frac{\rho(d)}{\rho_{0}(d)} c(d, q)  \mright] \\ \nonumber
        & \leq  \sum_{d \in D_{q}} \mathbb{E}_{c, y} \mleft[ \mleft( \frac{\rho(d)}{\rho_{0}(d)} c(d, q) \mright)^{2} \mright]. \nonumber 
    \end{align}
   Since: $c(d, q)^2 = c(d, q)$, we can further rewrite to:
    \begin{align}
    \sum_{d \in D_{q}} \mathbb{E}_{c, y} \mleft[ \mleft( \frac{\rho(d)}{\rho_{0}(d)} c(d, q) \mright)^{2} \mright]
         &= \sum_{d \in D_q} \mathbb{E}_{c, y} \mleft[ \mleft( \frac{\rho(d)}{\rho_{0}(d)} \mright)^{2} c(d, q)  \mright]   \\
         &=  \sum_{d \in D_q}   \mleft( \frac{\rho(d)}{\rho_{0}(d)} \mright)^{2} P(C=1 \mid d, q, \pi_{0}).    \nonumber
    \end{align}
    Next, we use Eq.~\ref{click-model} and~\ref{policy-aware-exposure1} to substitute the click probability;
    subsequently, we replace the examination propensities with normalized counterparts using Eq.~\ref{norm-exposure} and~\ref{eq:ratio};
    and lastly, we upper bound the result using the fact that $P(R=1 | d, q)  \leq 1$:
    \begin{equation}
    \begin{split}
    \sum_{d \in D_{q}}  \mathbb{E}_{c, y} \mleft[ \mleft( \frac{\rho(d)}{\rho_{0}(d)} c(d, q) \mright)^{2} \mright]
     &=  \sum_{d \in D_q}    \rho_{0}(d) \mleft( \frac{\rho(d)}{\rho_{0}(d)} \mright)^{2}  P(R=1 \,|\, d, q) \\
    & = \sum_{d \in D_q}  \mathrm{Z}\, \rho'_{0}(d) \mleft( \frac{\rho'(d)}{\rho'_{0}(d)} \mright)^{2}  P(R=1 | d, q) \\  
    & \leq \mathrm{Z}   \sum_{d \in D_q}  \rho'_{0}(d) \mleft( \frac{\rho'(d)}{\rho'_{0}(d)} \mright)^{2}. 
   \end{split}
    \end{equation}
   Finally, we place this upper bound for a single query back into the expectation over all queries (Eq.~\ref{eq:vardecom1}):
    \begin{equation}
        \frac{1}{N} \,  \mathbb{E}_{q} \mleft[ \mathrm{Var}_{y,c} \mleft[ \hat{U}(\pi ) \mid q \mright] \mright]   
         \leq \frac{ \mathrm{Z} }{N} \mathbb{E}_{q} \bigg[ \sum_{d \in D_q}  \rho'_{0}(d) \mleft( \frac{\rho'(d)}{\rho'_{0}(d)} \mright)^{2}  \bigg].
         \label{eq:varprooflast}
    \end{equation}
    Therefore, by Eq.~\ref{eq:vardecom1}, \ref{eq:varprooflast}, \ref{eq:var_new}, and the definition of exposure-based divergence in Eq.~\ref{eq:exposuredivergence}, it is a proven upper bound of the variance.
\end{proof}

\subsection{Exposure-divergence bound on performance}
\label{sec:perfbound1}

Using the upper bound on the variance of an \ac{CLTR} \ac{IPS} estimator that was proven in Theorem~\ref{var-theorm}, we can now introduce a generalization bound for the \ac{CLTR} estimator.

\begin{theorem}
\label{CLTR-bound1}
    Given the true utility $U(\pi)$ (Eq.~\ref{true-utility1}) and its exposure-based \ac{IPS} estimate $\hat{U}(\pi)$ (Eq.~\ref{cltr-obj}), for the ranking policy $\pi$ and the logging policy $\pi_{0}$ with expected exposures $\rho(d)$ and $\rho_{0}(d)$, respectively,
    the following generalization bound holds with probability $1 - \delta$:\footnote{The following proof differs slightly from the original proof published in~\cite{gupta2023safe}, as we overlooked an additional constant term in the original paper.}
\begin{equation}
    U(\pi) \geq \hat{U}(\pi) -  \sqrt{ \frac{\mathrm{Z}}{N}  \Big(\frac{1-\delta}{\delta}\Big) d_2(\rho \,\Vert\, \rho_0)} - \sqrt{ \frac{\mathrm{1}}{N}  \Big(\frac{1-\delta}{\delta}\Big)}.
\label{variance}
\end{equation}
\end{theorem}
\begin{proof}
    As per Cantelli's inequality~\cite{ghosh2002probability}, given an estimator $\hat{X}$ with the expected value $\mathbb{E}[\hat{X}]$ and variance $\mathrm{Var}[\hat{X}]$, the following tail-bound holds:
    \begin{equation}
        P(\hat{X} - \mathbb{E}[\hat{X}] \geq \lambda ) \leq \frac{\mathrm{Var}[\hat{X}]}{\mathrm{Var}[\hat{X}] + \lambda^2}. 
    \label{cantelli}
    \end{equation} 
    Since $\lambda > 0$ is a free parameter, we can define $\delta$ such that:
    \begin{equation}
        \delta =  \frac{\mathrm{Var}[\hat{X}]}{\mathrm{Var}[\hat{X}] + \lambda^2},
         \qquad
         \lambda = \sqrt{\frac{1-\delta}{\delta} \mathrm{Var}[\hat{X}]}.
    \end{equation} 
    Consequently, the following inequality holds:
    \begin{equation}
        P(\mathbb{E}[\hat{X}] \geq \hat{X} - \lambda ) \geq 1-\delta.
         \label{cantelli-inequality} 
    \end{equation}
    Building on this inequality, the following inequality must hold with probability $1-\delta$:
    \begin{equation}
        U(\pi) \geq \hat{U}(\pi) - \sqrt{ \frac{1-\delta}{\delta} \mathrm{Var}_{q,y,c}\mleft[\hat{U}(\pi)\mright]}.
        \label{inequality} 
    \end{equation}
    Next, we replace the variance with the upper bound from Theorem~\ref{var-theorm}, which results in the following bound:
    \begin{equation}
        U(\pi) \geq \hat{U}(\pi) -  \sqrt{ \frac{\mathrm{Z}}{N}  \Big(\frac{1-\delta}{\delta}\Big) d_2(\rho \,\Vert\, \rho_0) + \Big( \frac{1-\delta}{\delta N} \Big)}.
    \end{equation}
    By applying the Cauchy–Schwarz inequality, we get:
    \begin{equation}
        U(\pi) \geq \hat{U}(\pi) -  \sqrt{ \frac{\mathrm{Z}}{N}  \Big(\frac{1-\delta}{\delta}\Big) d_2(\rho \,\Vert\, \rho_0)} - \sqrt{ \frac{\mathrm{1}}{N}  \Big(\frac{1-\delta}{\delta}\Big)}.
    \end{equation}
    This completes the proof.
\end{proof}

\begin{figure}[t]
\centering
\includegraphics[width=0.75\columnwidth,trim= 0 -0.25cm 0 0,clip]{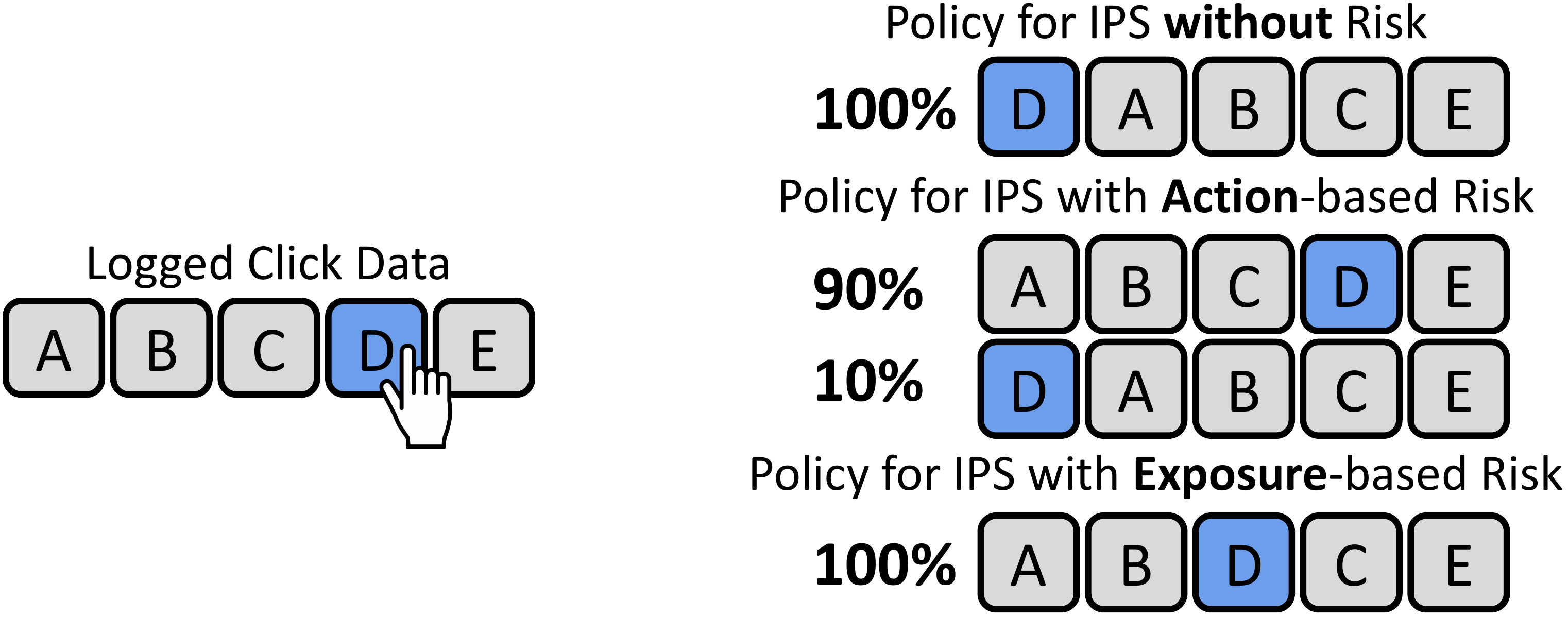}
\caption{
Example comparison of the optimal policy for a single logged click according to three different risk estimators.
}
\label{fig:riskcomparison}
\end{figure}

\header{Risk in \ac{CLTR}}
Based on the generalization bound proposed in Theorem~\ref{CLTR-bound1}, we see that it proposes the following measure of risk:
$\text{Risk}(\delta) = \sqrt{ \frac{\mathrm{Z}}{N}  \big(\frac{1-\delta}{\delta}\big) d_2(\rho \,\Vert\, \rho_0)}$ (cf.\ Eq.~\ref{eq:generalizationbound}).
Clearly, this risk is mostly determined by the exposure-based divergence between the new policy and the logging policy.
Thereby, it states that the greater the difference between how exposure is spread over documents by the logging policy and the new policy, the higher the risk involved.
Therefore, to optimize this lower bound, one has to balance the maximization of the estimated utility $\hat{U}(\pi)$ and the minimization of risk by not letting $\pi$ differ too much from $\pi_0$ in terms of exposure.

Furthermore, we see that our measure of risk diminishes as $N$ increases.
As a result, the risk term will overwhelm the \ac{IPS} term when $N$ is very low, as there is much risk involved when estimating based on a few interactions.
Conversely, when $N$ is very large, the risk term mostly disappears, as the \ac{IPS} estimate is more reliable when based on large numbers of interactions.
Thus, during optimization, the generalization bound is expected to mostly help with avoiding initial decreases in performance, while still converging at the same place as the standard \ac{IPS} estimator.

Lastly, the $\delta$ parameter determines the \emph{safety} that is provided by the risk, where a lower $\delta$ makes it more likely that the generalization bound holds.
Accordingly, as $\delta$ increases the risk term becomes smaller and will thus have less effect on optimization.

To the best of our knowledge, this is the first exposure-based generalization bound, which makes it the first method designed for safe optimization in the \ac{CLTR} setting.

\header{Illustrative comparison}
To emphasize the working and novelty of our exposure-based risk, a comparison of the optimal policies for action-based risk, exposure-based risk, and no risk are shown in Figure~\ref{fig:riskcomparison}.
We see that \ac{IPS} without a risk term places the once-clicked document at the first position, with 100\% probability.
This is very risky, as it greatly impacts the ranking while only being based on a single observation.
The action-based risk tries to mitigate this risk with a probabilistic policy that gives most probability to the logging policy ranking (90\%) and the remainder to the IPS ranking (10\%).
In contrast, with exposure-based risk, the optimal policy makes the risk and utility trade-off in a single ranking, that mostly follows the logging policy but places the clicked document slightly higher.

This example illustrates that because action-based risk does not have a similarity measure between rankings, it can only produce a probabilistic interpolation between the logging policy and IPS rankings.
Alternatively, because exposure-based risk does have such a measure, it produces a ranking that is neither the logging ranking nor the IPS ranking, but one with an exposure distribution that is similar to both.
Thereby, exposure-based risk has a more elegant and natural method of balancing utility maximization and risk minimization in the \ac{CLTR} setting.

\section{A Novel Counterfactual Risk Minimization Method for LTR}
\label{risk-cltr}

Now that we have the proven generalization bound described in Section~\ref{sec:perfbound1} (Theorem~\ref{CLTR-bound1}),
we can propose a novel risk-aware \ac{CLTR} method for optimizing it.
Accordingly, the aim of our method is to find the policy that maximizes this high-confidence lower bound on the true performance.
In formal terms, we have the following optimization problem:
\begin{equation}
    \max_{\pi}  \hat{U}(\pi) -  \sqrt{ \frac{\mathrm{Z}}{N}  \Big(\frac{1-\delta}{\delta}\Big) d_2(\rho \,\Vert\, \rho_0)}.
\label{objgenbound1}
\end{equation}
Note that we ignore the term $\sqrt{ \frac{\mathrm{1}}{N}  \Big(\frac{1-\delta}{\delta}\Big)}$, since it is constant with respect to the policy $\pi$, and therefore can be disregarded for optimization purposes.
We propose to train a stochastic policy $\pi$ via stochastic gradient descent, therefore, we need to derive the gradient and find a method of computing it.
For the computation of the gradient w.r.t.\ the utility $\hat{U}(\pi)$, the first part of Eq.~\ref{objgenbound1}, we refer to several prior work that discusses this topic extensively~\citep{oosterhuis2021computationally, oosterhuis2020policy, yadav2021policy}.
Thus, we can focus our attention on the second part of Eq.~\ref{objgenbound1}:
\begin{equation}
\nabla_{\!\pi} \sqrt{ \frac{\mathrm{Z}}{N}  \Big(\frac{1-\delta}{\delta}\Big) d_2(\rho \,\Vert\, \rho_0)}
         = \sqrt{ \frac{\mathrm{Z}(1-\delta)}{ 4N\delta d_2(\rho \,\Vert\, \rho_0)}} \nabla_{\!\pi} d_2(\rho \,\Vert\, \rho_0).
\end{equation}
To derive the gradient of the exposure-based divergence function, we use the relation between $\rho$ and $\rho'$ from Eq.~\ref{total-exposure1} and \ref{eq:ratio}:
\begin{equation}
\begin{split}
    \nabla_{\!\pi} d_2(\rho \,\Vert\, \rho_0)
    &=  \nabla_{\!\pi} \mathbb{E}_{q} \bigg[ \sum_{d \in D_q} \!  \rho'_{0}(d) \mleft( \frac{\rho'(d)}{\rho'_{0}(d)} \mright)^{2}   \bigg] 
    \\
    &= \frac{2}{\mathrm{Z}} \, \mathbb{E}_{q} \bigg[ \! \sum_{d \in D_q} \!\! \frac{\rho(d)}{\rho_{0}(d)} \nabla_{\!\pi} \rho(d)   \bigg] .
\end{split}
\end{equation}
\noindent Thus, we only need the gradient w.r.t.\ the exposure of a document ($\nabla_{\!\pi} \rho(d)$) to complete our derivation.
If $\pi$ is a \ac{PL} ranking model, one can make use of the specialized gradient computation algorithm from~\citep{oosterhuis2021computationally}.
However, for this chapter, we will not make further assumptions about $\pi$ and apply the more general log-derivate trick from the REINFORCE algorithm~\cite{williams1992simple}:
\begin{equation}
\nabla_{\!\pi} \rho(d) = \mathbb{E}_{y \sim \pi} \big[ P\big(E=1 \mid \text{rank}(d \mid y)\big)  \big]  \nabla_{\!\pi} \log \pi(y).
\end{equation}
Putting all of the previous elements back together, gives us the gradient w.r.t.\ the exposure-based risk function:
\begin{equation}
\mbox{}\hspace*{-2.5mm}
  \sqrt{\!\frac{
    1-\delta
   }{ N \delta \, \mathrm{Z}  \, d_2(\rho \,\Vert\, \rho_0) }} 
      \mathbb{E}_{q, y \sim \pi} \!\mleft[
      \!\mleft( \sum_{k=1}^K \! \frac{\rho(y_k)}{\rho_{0}(y_k)} P(E=1 |\, k)\!\mright) \nabla_{\!\pi}\!\log \pi(y) 
     \!\mright],
     \hspace*{-2.5mm}\mbox{}
\end{equation}
where $y_k$ is the document at rank $k$ in ranking $y$.
For a close approximation of this gradient, we substitute the gradient with the queries from the given dataset, and the rankings sampled from $\pi$ during optimization~\cite{williams1992simple,oosterhuis2021computationally}.

Similarly, since the exact computation of  is $d_2(\rho \,\Vert\, \rho_0)$ infeasible in practice, we introduce a sample-based empirical divergence estimator:
\begin{equation}
        \hat{d}_{2}(\rho \mid\mid \rho_{0}) = \frac{1}{N} \sum_{i=1}^{N} \sum_{d \in D_{q_i}}^{}   \rho_{0}'(d) \mleft( \frac{\rho'(d)}{\rho_{0}'(d)} \mright)^{2} .
    \label{empirical-div}
\end{equation}
This is an unbiased estimate of the true divergence given that the sampling process is truly Monte Carlo~\cite{james1980monte}.

\section{Experimental Setup}

For our experiments, we follow the semi-synthetic experimental setup that is common in the \ac{CLTR} literature~\citep{oosterhuis2021unifying,oosterhuis2021robust,joachims2017unbiased,vardasbi2020inverse}.
We make use of the three largest publicly available \ac{LTR} datasets: Yahoo!\ Webscope~\cite{chapelle2011yahoo}, MSLR-WEB30k~\citep{qin2013introducing}, and Istella~\citep{dato2016fast}.
The datasets consist of queries, a preselected list of documents per query, query-document feature vectors, and manually-graded relevance judgements for each query-document pair.
To generate clicks, we follow previous work~\citep{oosterhuis2021unifying,oosterhuis2021robust,vardasbi2020inverse} and train a logging policy on a $3\%$ fraction of the relevance judgements.
This simulates a real-world setting, where a production ranker trained on manual judgements is used to collect click logs, which can then be used for subsequent click-based optimization. 
Typically, in real-world ranking settings, given that the production ranker is used on live-traffic, it is deemed as a safe policy that can be trusted with real users.

We simulate a top-$K$ ranking setup~\cite{oosterhuis2020policy} where five documents are presented at once.
Clicks are generated with our assumed click model (Eq.~\ref{click-model}) and the following rank-based position-bias:
\begin{equation}
    P(E=1 \mid q, d, y) = 
\begin{cases}
    \left(\frac{1}{\textrm{rank}(d \vert y)}\right)^2& \text{if } \textrm{rank}(d \mid y) \leq 5,\\
    0              & \text{otherwise}.
\end{cases}    
\end{equation}
In real-world click data, the observed \ac{CTR} is typically very low~\citep{saito2020open,chen2019tiangong,li2010contextual};\ 
hence, to simulate such a sparse click settings, we apply the following transformation from relevance judgements to relevance probabilities:
\begin{equation}
    P(R = 1 \mid q, d) = 0.025 * rel(q,d) + 0.2,
    \label{click-model-simul}
\end{equation}
where $rel(q,d) \in \{0,1,2,3,4\}$ is the relevance judgement for the query-document pair and $0.2$ is added as click noise. 
During training, the only available data consists of clicks generated on the training and validation sets, no baseline method has access to the underlying relevance judgements (except the skyline).

Furthermore, we assume a setting where the exact logging policy is not available during training.
As a result, the $\hat{\rho}_0$ propensities have to estimated, we use a simple frequency estimate following~\citep{oosterhuis2021unifying}:
\begin{equation}
\mbox{}\hspace*{-2mm}
    \hat{\rho}_0(d ) = \sum^N_{i=1} \frac{\mathds{1}\big[q = q_i\big]}{\sum^N_{j=1} \mathds{1}\big[q = q_j\big]}  P\big( E= 1 \mid \text{rank}(d \mid y_i)\big). \label{prop-estimate}
\hspace*{-2mm}\mbox{}    
\end{equation}
For the action-based baselines, the action propensities $\hat{\pi}_0(y \mid q)$ are similarly estimated based on observed frequencies:
\begin{equation}
        \hat{\pi}_0(y \,|\, q) = \prod_{k=1}^{K-1} \hat{\pi}_0(y_k \,|\, q), \hspace*{1mm} %
        \hat{\pi}_0(y_k \,|\, q) = \sum^N_{j=1} \frac{ \mathds{1}\big[y_k = y_j] }{ \sum^N_{j=1} \mathds{1}\big[q = q_j\big] },
    \label{action-estimation}
\end{equation}
where $\hat{\pi}_0(y_k \,|\, q)$ is the estimated probability of $d$ appearing at rank $k$ for query $q$. 
As is common in \ac{CLTR}~\citep{oosterhuis2020learning, saito2021counterfactual, joachims2017unbiased}, we clip propensities by $10 / \sqrt{N}$ in the training set, to reduce variance, but not in the validation set.

We optimize neural \ac{PL} ranking models~\citep{oosterhuis2021computationally} with early stopping based on validation clicks to prevent overfitting. 
For the REINFORCE policy-gradient, we follow~\cite{yadav2021policy} and use the average reward per query as a control-variate for variance reduction.

As our evaluation metric, we compute NDCG@5 metric using the relevance judgements on the test split of each dataset~\citep{jarvelin2002cumulated}. 
All reported results are averages over ten independent runs, significant testing is performed with a two-sided student-t test.

Finally, the following methods are included in our comparisons:
\begin{enumerate}[label=(\roman*), leftmargin=*]
    \item \emph{Naive}. As the most basic baseline, we train on the generated clicks without any correction (equivalent to $\forall d, \, \rho_0(d)=1 $). %
     \item  \emph{Skyline.} To compare with the highest possible performance, this baseline is trained on the actual relevance judgements.
    \item  \emph{Action-based IPS.} Standard IPS estimation (Eq.~\ref{eq:actionips}) that is not designed for ranking and thus uses action-based propensities.
    \item  \emph{Action-based \ac{CRM}.} Standard \ac{CRM} (Eq.~\ref{crm-loss}) that is also not designed for ranking, for the risk function we use the action-based divergence function in Eq.~\ref{eq:actionbaseddiv1}.
    \item  \emph{Exposure-based IPS}. The IPS estimator designed for \ac{CLTR} with exposure-based propensities (Eq.~\ref{cltr-obj}).
    The most important baseline, as it is the prevalent approach in the field~\cite{oosterhuis2020policy,oosterhuis2021unifying}.
     \item  \emph{Exposure-based \ac{CRM}.} Our proposed \ac{CRM} method (Eq.~\ref{objgenbound1}) using a risk function based on exposure-based divergence.
\end{enumerate}

\section{Results and Discussion}

\subsection{Comparison with baseline methods}
The main results of our experimental comparison are presented in Figure~\ref{fig:mainresults_part1} and \ref{fig:mainresults_part2}, and Table~\ref{tab:dcgresults_part1} and \ref{tab:dcgresults_part2}.
Figure~\ref{fig:mainresults_part1} and \ref{fig:mainresults_part2} display the performance curves of the different methods as the number of logged interactions ($N$) increases.
Table~\ref{tab:dcgresults_part1} and \ref{tab:dcgresults_part2} present the performance at $N\in\{4 \cdot 10^2, 4 \cdot 10^7, 10^9\}$ and indicate whether the observed differences with our exposure-based \ac{CRM} method are statistically significant.

We start by considering the performance curves in Figure~\ref{fig:mainresults_part1} and \ref{fig:mainresults_part2}.
We see that both the action-based and exposure-based \ac{IPS} baselines have an initial period of very similar performance that is far below the logging policy.
Around $N\approx10^4$ their performance is comparable to the logging policy, and finally at $N=10^9$ the exposure-based IPS has reached optimal performance, while the performance of action-based IPS is still far from optimal.
We can attribute this initial poor performance to the high variance problem of IPS estimation;
when $N$ is small, variance is at its highest, resulting in risky and sub-optimal optimization by the IPS estimators.
However, even when $N=10^9$, the variance of the action-based IPS estimator is too high to reach optimal performance, due to its extremely small propensities.
This illustrates why the introduction of exposure-based propensities was so important to the \ac{CLTR} field, and that even exposure-based IPS produces unsafe optimization when little data is available or variance from interactions is high.
\begin{figure}[!t]
    \centering
    {\renewcommand{\arraystretch}{0.01}%
    \setlength{\tabcolsep}{0.04cm}%
    \begin{tabular}{c r r}
        & 
        \multicolumn{1}{c}{\small Yahoo! Webscope}
        & 
        \multicolumn{1}{c}{\small MSLR-WEB30k}
        \\[2mm]
        \rotatebox[origin=lt]{90}{\hspace{0.77cm}\small NDCG@5} &
        \includegraphics[scale=0.475]{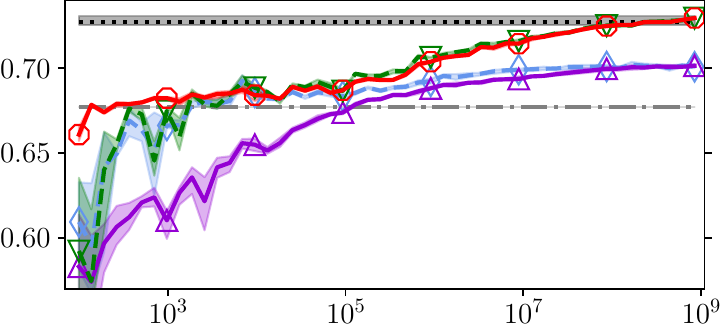} &
        \includegraphics[scale=0.475]{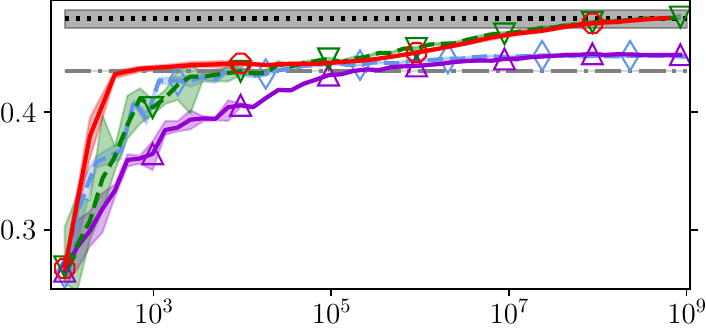}
        \\
        \rotatebox[origin=lt]{90}{\hspace{0.65cm} \small NDCG@5} &
        \includegraphics[scale=0.475]{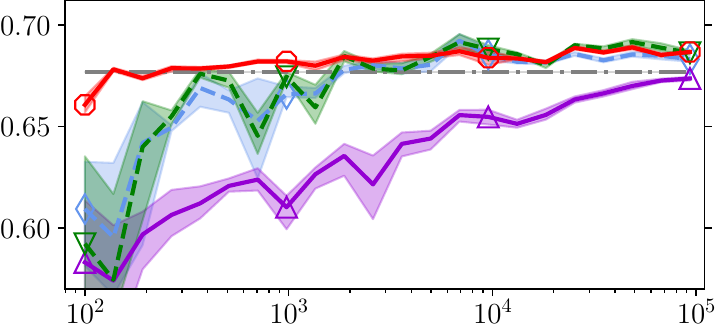} &
        \includegraphics[scale=0.475]{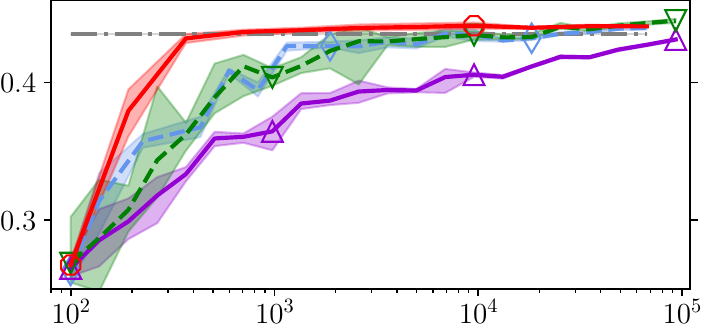}
        \\
        & \multicolumn{1}{c}{\small Number of interactions simulated ($N$)}
        & \multicolumn{1}{c}{\small Number of interactions simulated ($N$)}
        \\[2mm]
    \end{tabular}
    } %
    \vspace{0.3\baselineskip}
    \begin{tabular}{c}
        \multicolumn{1}{c}{\includegraphics[scale=0.51]{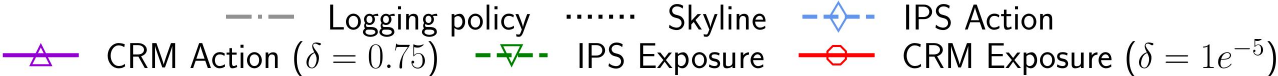}}
    \end{tabular}
    \caption{
        Performance in NDCG@5 of various \ac{IPS} and \ac{CRM} methods for \ac{CLTR} on Yahoo! Webscope and MSLR-WEB30k datasets.
        The top-row presents the results when the size of the training data is varied from $10^2$ to $10^9$.
        The bottom-row is a zoomed-in view, focusing on the low-data region from $10^2$ to $10^5$.
        Results are averages over 10 runs; shaded areas indicate 80\% confidence intervals.
    }
    \label{fig:mainresults_part1}
    \vspace{0.5em}
\end{figure}
\begin{figure}[h]
    \centering
    {\renewcommand{\arraystretch}{0.01}%
    \setlength{\tabcolsep}{0.04cm}%
    \begin{tabular}{c r}
        &
         \multicolumn{1}{c}{\small Istella}
        \\[2mm]
        \rotatebox[origin=lt]{90}{\hspace{0.77cm}\small NDCG@5} &
        \includegraphics[scale=0.475]{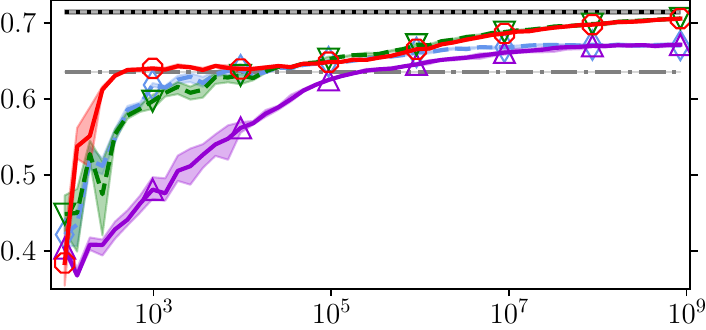}
        \\
        \rotatebox[origin=lt]{90}{\hspace{0.65cm} \small NDCG@5} &
        \includegraphics[scale=0.475]{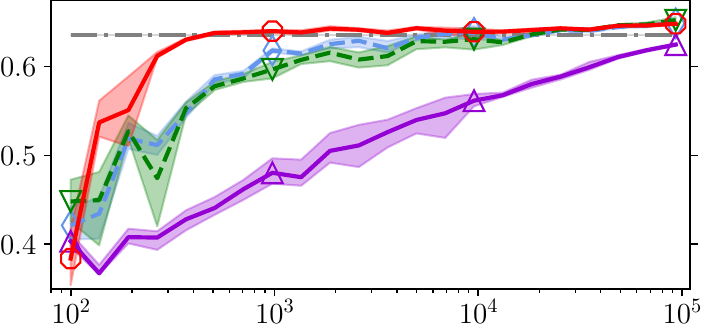}
        \\
        & \multicolumn{1}{c}{\small Number of interactions simulated ($N$)}
        \\[2mm]
    \end{tabular}
    } %
    
    \vspace{0.3\baselineskip}
    \begin{tabular}{c}
        \multicolumn{1}{c}{\includegraphics[scale=0.51]{04-safety_sigir/figures/legend1.pdf}}
    \end{tabular}
    \caption{
        Performance in NDCG@5 of various \ac{IPS} and \ac{CRM} methods for \ac{CLTR} on the Istella dataset.
        The top-row presents the results when the size of the training data is varied from $10^2$ to $10^9$.
        The bottom-row is a zoomed-in view, focusing on the low-data region from $10^2$ to $10^5$.
        Results are averages over 10 runs; shaded areas indicate 80\% confidence intervals.
    }
    \label{fig:mainresults_part2}
\end{figure}

Next, we consider whether action-based \ac{CRM} is able to mitigate the high variance problem of action-based \ac{IPS}.
Despite being a proven generalization bound, Figure~\ref{fig:mainresults_part1} and \ref{fig:mainresults_part2} clearly show us that action-based \ac{CRM} only leads to decreases in performance compared to its IPS counterpart.
It appears that this happens because the logging policy is not available in our setup, and the propensities have to be estimated from logged data.
Consequently, the action-based risk pushes the optimization to mimic the exact rankings that were observed during logging.
Thus, due to the variance introduced from the sampling of rankings from the logging policy, it appears that action-based \ac{CRM} has an even higher variance problem than action-based \ac{IPS}.
As expected, our results thus clearly indicate that action-based \ac{CRM} is also unsuited for the \ac{CLTR} setting, to our surprise; it is substantially worse than its IPS counterpart.

\begin{table}[!t]
    \centering
    \caption{\small
      NDCG@5 performance for Yahoo! Webscope and MSLR-WEB30k datasets under different settings for several values of $N$, the number of logged interactions in the simulated training set.
        Values are averages over 10 independent runs on the held-out test sets; bold figures mark the highest score.
        Differences from the exposure-based \ac{CRM} are assessed with a two-sided Student's \textit{t}-test: $^{\tiny \blacktriangledown}$ denotes significantly lower (p\,$<$\,0.01), while $^{\tiny \star}$ indicates no significant difference.}
    \label{tab:dcgresults_part1}
    \vspace{0.5em}    %
    \setlength\tabcolsep{3pt}
    \resizebox{\columnwidth}{!}{%
      \begin{tabular}{l ccc ccc}
        \toprule
        & \multicolumn{3}{c}{Yahoo! Webscope}
        & \multicolumn{3}{c}{MSLR-WEB30k} \\
        \cmidrule(r){2-4} \cmidrule(r){5-7}
        & $N=4\!\cdot\!10^{2}$ & $N=4\!\cdot\!10^{7}$ & $N=10^{9}$
        & $N=4\!\cdot\!10^{2}$ & $N=4\!\cdot\!10^{7}$ & $N=10^{9}$ \\
        \midrule
        Logging              & 0.677 & 0.677 & 0.677              & 0.435 & 0.435 & 0.435 \\
        Skyline              & 0.727 & 0.727 & 0.727              & 0.479 & 0.479 & 0.479 \\
        \midrule
        Naive              & 0.652 \small (0.021)$^{\tiny \blacktriangledown}$ & 0.694 \small (0.000)$^{\tiny \blacktriangledown}$ & 0.695 \small (0.000)$^{\tiny \blacktriangledown}$             & 0.353 \small (0.003)$^{\tiny \blacktriangledown}$ & 0.448 \small (0.000)$^{\tiny \blacktriangledown}$ & 0.448 \small (0.001)$^{\tiny \blacktriangledown}$ \\
        Action IPS              & 0.656 \small (0.008)$^{\tiny \blacktriangledown}$ & 0.701 \small (0.001)$^{\tiny \blacktriangledown}$ & 0.701 \small (0.001)$^{\tiny \blacktriangledown}$             & 0.359 \small (0.007)$^{\tiny \blacktriangledown}$ & 0.448 \small (0.001)$^{\tiny \blacktriangledown}$ & 0.448 \small (0.001)$^{\tiny \blacktriangledown}$ \\
        Action \ac{CRM}              & 0.617 \small (0.004)$^{\tiny \blacktriangledown}$ & 0.698 \small (0.001)$^{\tiny \blacktriangledown}$ & 0.700 \small (0.001)$^{\tiny \blacktriangledown}$             & 0.359 \small (0.005)$^{\tiny \blacktriangledown}$ & 0.448 \small (0.001)$^{\tiny \blacktriangledown}$ & 0.449 \small (0.001)$^{\tiny \blacktriangledown}$ \\
        Exp.\ IPS              & 0.659 \small (0.010)$^{\tiny \blacktriangledown}$ & \textbf{0.723 \small (0.001)}$^{\tiny \star}$ & \textbf{0.730 \small (0.001)}$^{\tiny \star}$             & 0.389 \small (0.014)$^{\tiny \blacktriangledown}$ & \textbf{0.474 \small (0.001)}$^{\tiny \star}$ & \textbf{0.481 \small (0.001)}$^{\tiny \star}$ \\
        Exp.\ \ac{CRM}             & \textbf{0.677 \small (0.001)\phantom{$^{\tiny \blacktriangledown}$}} & \textbf{0.723 \small (0.001)\phantom{$^{\tiny \star}$}} & \textbf{0.730 \small (0.000)\phantom{$^{\tiny \star}$}}             & \textbf{0.434 \small (0.001)\phantom{$^{\tiny \blacktriangledown}$}} & \textbf{0.473 \small (0.001)\phantom{$^{\tiny \star}$}} & \textbf{0.480 \small (0.001)\phantom{$^{\tiny \star}$}} \\
        \bottomrule
      \end{tabular}%
    } %
\end{table}

\begin{table}[!t]
    \centering
    \caption{
     \small NDCG@5 performance for Istella dataset under different settings for several values of $N$,
      the number of logged interactions in the simulated training set.
      Reported numbers are averages over 10 independent runs evaluated on the held-out test-sets; 
       bold numbers indicate the highest performance.
      Statistical significance for differences with the exposure-based \ac{CRM} 
       are measured via a two-sided student-t test:
      $^{\tiny \blacktriangledown}$ indicates methods with significantly lower NDCG ($p<0.01$), 
       and $^{\tiny \star}$ no significant difference.
    }
    \label{tab:dcgresults_part2}
    \vspace{0.5em}
    \small  %
    \begin{tabular}{l ccc}
        \toprule
        & \multicolumn{3}{c}{Istella}\\
        \cmidrule(r){2-4}
        & $N=4\cdot10^2$ &  $N=4 \cdot 10^7$ &  $N=10^9$ \\
        \midrule
        Logging              & 0.635 & 0.635 & 0.635 \\
        Skyline              & 0.714 & 0.714 & 0.714 \\
        \midrule
        Naive              & 0.583 \small (0.007)$^{\tiny \blacktriangledown}$ & 0.661 \small (0.001)$^{\tiny \blacktriangledown}$ & 0.661 \small (0.001)$^{\tiny \blacktriangledown}$ \\
        Action IPS              & 0.578 \small (0.004)$^{\tiny \blacktriangledown}$ & 0.671 \small (0.001)$^{\tiny \blacktriangledown}$ & 0.671 \small (0.002)$^{\tiny \blacktriangledown}$ \\
        Action \ac{CRM}              & 0.449 \small (0.013)$^{\tiny \blacktriangledown}$ & 0.668 \small (0.002)$^{\tiny \blacktriangledown}$ & 0.672 \small (0.001)$^{\tiny \blacktriangledown}$ \\
        Exp. IPS              & 0.576 \small (0.010)$^{\tiny \blacktriangledown}$ & \textbf{0.696 \small (0.001)}$^{\tiny \star}$ & \textbf{0.706 \small (0.001)}$^{\tiny \star}$ \\
        Exp. \ac{CRM}             & \textbf{0.635 \small (0.001)\phantom{$^{\tiny \blacktriangledown}$}} & \textbf{0.695 \small (0.001)\phantom{$^{\tiny \star}$}} & \textbf{0.706 \small (0.001)\phantom{$^{\tiny \star}$}} \\
        \bottomrule
    \end{tabular}
\end{table}

Finally, we examine the performance of our novel exposure-based \ac{CRM} method.
Similar to the other methods, there is an initial period of low performance, but in stark contrast, this period ends very quickly;
on Yahoo!\ logging policy performance is reached when $N \approx 125$, on MSLR-WEB30k when $N\approx350$ and on Istella when $N\approx400$.
For comparison, exposure-based IPS needs $N\approx1100$ on Yahoo!, $N\approx10^4$ on MLSR-WEB30k and $N\approx1.1\cdot10^4$ on Istella to do the same; meaning that our \ac{CRM} method needs roughly $89\%$, $97\%$ and $97\%$ fewer interactions, respectively.
In addition, Table~\ref{tab:dcgresults_part1} and \ref{tab:dcgresults_part2} indicate that the logging policy performance is matched on all datasets when $N=400$ by exposure-based \ac{CRM}, where it also outperforms all baseline methods.
We note that there is still an initial period of low performance, because the logging policy is unavailable at training, and thus, its behavior still has to be estimated from logged interactions.
It is possible that in settings where the logging policy is fully known during training, this initial period is eliminated entirely.
Nevertheless, our results show that exposure-based \ac{CRM} reduces the initial periods of poor performance due to variance by an enormous magnitude.

\begin{figure}[!t]
    \centering
    \renewcommand{\arraystretch}{0.01}
    \setlength{\tabcolsep}{0.015cm}
    \begin{tabular}{c r r}
    &
     \multicolumn{1}{c}{\small Yahoo! Webscope}
    &
     \multicolumn{1}{c}{\small MSLR-WEB30k}
    \\
    \rotatebox[origin=lt]{90}{\hspace{0.2cm} \small \it Action-based \ac{CRM}}
    \rotatebox[origin=lt]{90}{\hspace{0.65cm} \small NDCG@5} &
    \includegraphics[scale=0.475]{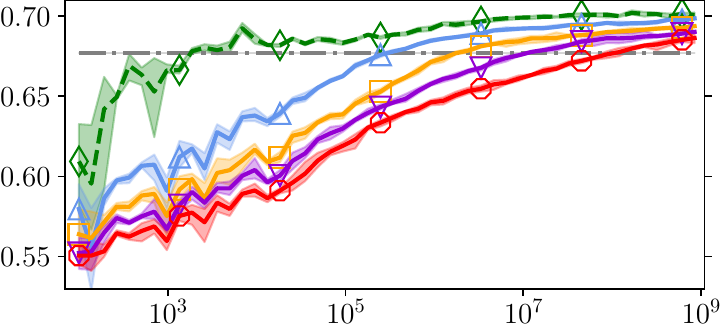} &
    \includegraphics[scale=0.475]{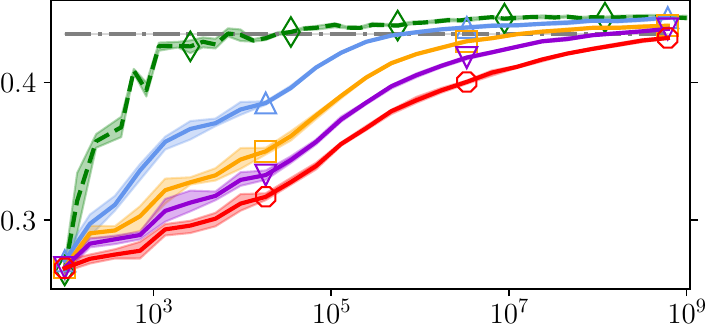}
    \\
    \rotatebox[origin=lt]{90}{\hspace{0.1cm}\small\it Exposure-based \ac{CRM}}
    \rotatebox[origin=lt]{90}{\hspace{0.65cm} \small NDCG@5} &
    \includegraphics[scale=0.475]{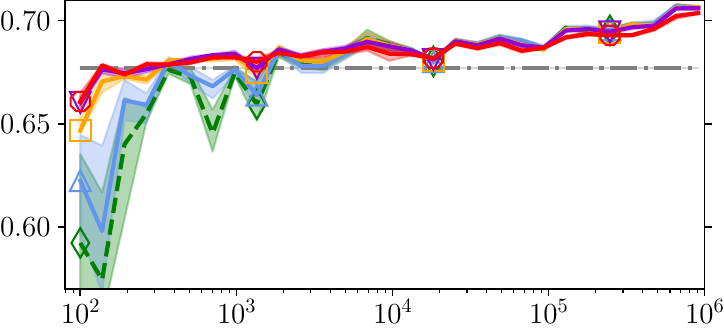} &
    \includegraphics[scale=0.475]{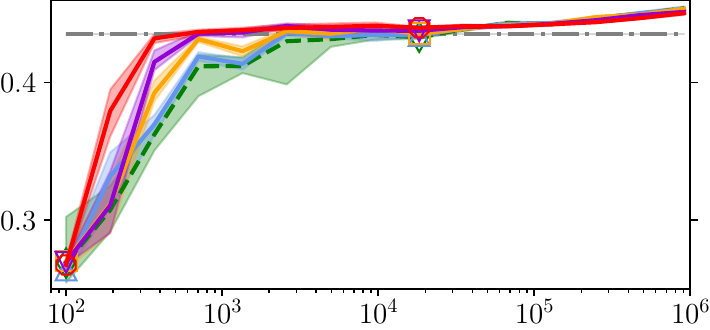}
    \\
    & \multicolumn{1}{c}{\small Number of interactions simulated ($N$)}
    & \multicolumn{1}{c}{\small Number of interactions simulated ($N$)}
    \\[2mm]
    \end{tabular}
    
    \vspace{0.3\baselineskip}
    \includegraphics[scale=0.45]{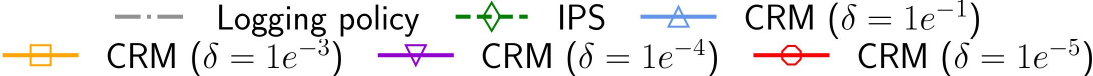} 
    \caption{
        Performance of \ac{CRM} methods with varying confidence parameters ($\delta$) on Yahoo! Webscope and MSLR-WEB30k datasets.
        Top-row: action-based \ac{CRM} baseline; bottom-row: our exposure-based \ac{CRM} method.
        Results are averages of 10 runs; shaded areas indicate 80\% confidence intervals.
    }
    \label{fig:ablationresults_part1}
\end{figure}

Furthermore, while the initial period is clearly improved, we should also consider whether there is a trade-off with the rate of convergence.
Surprisingly, Figure~\ref{fig:mainresults_part1} and \ref{fig:mainresults_part2} do not display any noticeable decrease in performance when compared with exposure-based IPS.
Moreover, Table~\ref{tab:dcgresults_part1} and \ref{tab:dcgresults_part2} show the differences between exposure-based IPS and \ac{CRM} are barely measurable and not statistically significant when $N \in \{4\cdot 10^7,10^9\}$.
We know from the risk formulation in Eq.~\ref{objgenbound1} that the weight of the risk term decreases as $N$ increases at a rate of $1/\sqrt{N}$.
In other words, the more data is available, the more optimization is able to diverge from the logging policy.
It appears that this balances utility maximization and risk minimization so well that we are unable to observe any downside of applying exposure-based \ac{CRM} instead of IPS.
Therefore, we conclude that, compared to all baseline methods and across all datasets, exposure-based \ac{CRM} drastically reduces the initial period of low performance, matches the best rate of convergence of all baseline, and has optimal performance at convergence.

\subsection{Ablation study on the confidence parameter}

To gain insights into how the confidence parameter $\delta$ affects the trade-off between safety and utility, an ablation study over various $\delta$ values was performed for both \ac{CRM} methods. 

The top-rows of Figure~\ref{fig:ablationresults_part1} and \ref{fig:ablationresults_part2} show us the performance of action-based \ac{CRM}, and contrary to expectation, a decrease in $\delta$ corresponds to a considerably worse performance.
For the sake of clarity, in theory, $\delta$ is inversely tied to safety, a lower $\delta$ should result in less divergence from the safe logging policy~\citep{oosterhuis2021robust}.
Conversely, we see that action-based \ac{CRM} displays the opposite trend.
We think this further confirms our hypothesis that a frequency estimate of action-based divergence has an even higher variance problem than action-based IPS.
Consequently, a higher weight to the risk function results in worse performance.
This further confirms our previous conclusion that action-based \ac{CRM} is unsuited for the \ac{CLTR} setting, regardless of how the $\delta$ parameter is tuned.

In contrast, the bottom-rows of Figure~\ref{fig:ablationresults_part1} and \ref{fig:ablationresults_part2} display the expected trend for exposure-based \ac{CRM};
as $\delta$ decreases the resulting performance gets closer to the logging policy.
With $\delta=0.1$, \ac{CRM} performs extremely close to its IPS counterpart, as optimization is less constrained to mimic the logging policy here.
Decreasing $\delta$ appears to have diminishing returns, as the difference between $\delta=10^{-4}$ and $\delta=10^{-5}$ is marginal.
Importantly, we do not observe any downsides to setting $\delta=10^{-5}$, thus we have not reached a point in our experiments where $\delta$ is set too conservatively.
This suggests that exposure-based \ac{CRM} is very robust to the setting of the $\delta$ parameter, and that a sufficiently low $\delta$ does not require fine-tuning.
Therefore, this shows that the improvements we observed when comparing with baseline methods, did not stem from a fine-tuning of $\delta$.
Thus, we can conclude that this robustness further increases the safety that is provided by exposure-based \ac{CRM}, as there is also little risk involved in the tuning of the $\delta$ parameter.

\begin{figure}[!t]
    \centering
    \renewcommand{\arraystretch}{0.01}
    \setlength{\tabcolsep}{0.015cm}
    \begin{tabular}{c r}
    &
     \multicolumn{1}{c}{\small Istella}
    \\
    \rotatebox[origin=lt]{90}{\hspace{0.2cm} \small \it Action-based \ac{CRM}}
    \rotatebox[origin=lt]{90}{\hspace{0.65cm} \small NDCG@5} &
    \includegraphics[scale=0.475]{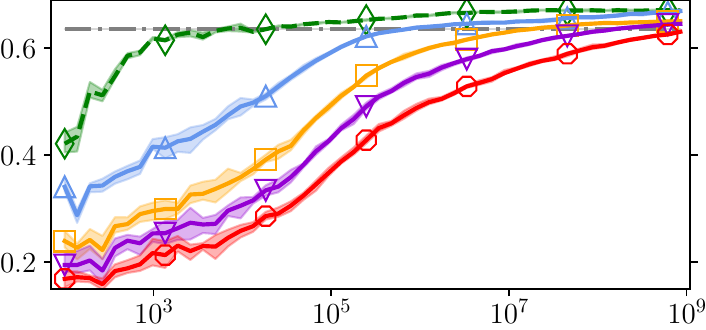}
    \\
    \rotatebox[origin=lt]{90}{\hspace{0.1cm}\small\it Exposure-based \ac{CRM}}
    \rotatebox[origin=lt]{90}{\hspace{0.65cm} \small NDCG@5} &
    \includegraphics[scale=0.475]{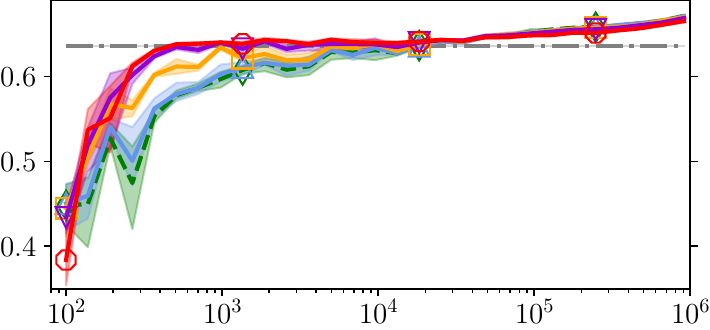}
    \\
    & \multicolumn{1}{c}{\small Number of interactions simulated ($N$)}
    \\[2mm]
    \end{tabular}
    
    \vspace{0.3\baselineskip}
    \includegraphics[scale=0.45]{04-safety_sigir/figures/legend_abl_action.pdf}
    
    \caption{
        Performance of \ac{CRM} methods with varying confidence parameters ($\delta$) on the Istella dataset.
        Top-row: action-based \ac{CRM} baseline; bottom-row: our exposure-based \ac{CRM} method.
        Results are averages of 10 runs; shaded areas indicate 80\% confidence intervals.
    }
    \label{fig:ablationresults_part2}
\end{figure}

\section{Conclusion}

In this chapter, we introduced the first \acf{CRM} method designed for \ac{CLTR}, that relies on a novel exposure-based divergence function.
In contrast with existing action-based \ac{CRM} methods, exposure-based divergence avoids the problem of the enormous combinatorial action space when ranking, by measuring the dissimilarity between policies based on how they distribute exposure to documents.
As a result, exposure-based \ac{CRM} optimization produces policies that rank similar to the logging policy when it is risky to follow \ac{IPS}, i.e., when little data is available or variance is very high.
Consequently, our experimental results show that it almost completely removes initial periods of detrimental performance;
to be precise, our method needed 89\% to 97\% fewer interactions than state-of-the-art \ac{IPS} to match production system performance.
Importantly, we observed no downsides in its application, as it maintained the same rate and point of convergence as \ac{IPS}, in all tested experimental settings.
Therefore, we conclude that our exposure-based \ac{CRM} method provides the safest \ac{CLTR} methods so far, as it almost completely alleviates the risk of decreasing the performance of a production system.

These improvements have big implications for practitioners who work on ranking systems in real-world settings, since the almost complete reduction of initial detrimental performance removes the main risks involved in applying \ac{CLTR}.
In other words, when applying our novel exposure-based \ac{CRM}, practitioners can have significantly less worry that the resulting policy will perform worse than their production system and hurt user experience.

In this chapter, we answer the broad research question (\ref{rq:safe1}) in affirmative. We derived a generalization bound for the counterfactual \ac{LTR} estimator, establishing a lower bound on the true ranking utility. 
We then demonstrate that optimizing this lower bound ensures safety, i.e., the resulting ranking policy after optimizing the lower bound is no worse than the current production policy. 
In practice, this property is useful when click data is scarce, mitigating the risk of deploying potentially harmful policies, thereby ensuring safe deployment.

The safety method presented in this chapter depends on the assumed user model (position-based click model), and relies on the assumption that user interaction data will follow the click model.
In settings where user interactions do not follow the assumed user model, the safety guarantees will not hold.
In the next chapter, we will discuss a method that guarantees safety even under adversarial user behavior settings.

\chapter{Practical and Robust Safety Guarantees for Advanced Counterfactual Learning-to-Rank}
\chaptermark{Safety Guarantees for Advanced Counterfactual Learning-to-Rank}
\label{chapter:01-online-evaluation2}

\footnote[]{This chapter was published as~\citep{gupta-2024-practical}.}
In Chapter~\ref{chapter:01-online-evaluation1}, we presented a safe counterfactual \ac{LTR} method that guarantees safe deployment by optimizing a lower confidence bound on the true ranking utility.
By optimizing the lower confidence bound on the true ranking utility via exposure-based risk minimization, is it guaranteed that the new ranking policy will be at least as good as the production/logging ranking policy.
However, these safety guarantees depend critically on assumptions regarding user behavior, i.e., the assumed click model. 
If the user behavior deviates from the assumed click model, the safety guarantees will be invalidated; which motivates the following research question:

\begin{enumerate}[label=\textbf{RQ2},ref={RQ\arabic*}]
\item \acl{rq:safe2}
\end{enumerate}

\noindent In this chapter, we introduce \ac{PRPO}, a novel safe deployment method ensuring safety for counterfactual \ac{LTR} without reliance on user behavior assumptions, guaranteeing robust safety even under adversarial conditions.
Further, we extend the safety guarantees from position-based click model and \ac{IPS} estimator introduced in Chapter~\ref{chapter:01-online-evaluation2} to trust-bias click model~\cite{agarwal2019addressing,vardasbi2020inverse} and doubly robust counterfactual estimator~\cite{oosterhuis2022doubly}.

\section{Introduction}

\Ac{CLTR}~\cite{joachims2017unbiased, wang2016learning, oosterhuis2020learning,gupta2024unbiased} concerns the optimization of ranking systems based on user interaction data using \ac{LTR} methods~\cite{liu2009learning}.
A main advantage of \ac{CLTR} is that it does not require manual relevance labels, which are costly to produce~\citep{qin2010letor,chapelle2011yahoo} and often do not align with actual user preferences~\cite{sanderson2010test}.
Nevertheless, \ac{CLTR} also brings significant challenges since user interactions only provide a heavily biased form of implicit feedback~\citep{gupta2024unbiased}.
User clicks are affected by many different factors, for example, the position at which an item is displayed  in a ranking~\citep{wang2018position, craswell2008experimental}.
Thus, click frequencies provide a biased indication of relevance, that is often more representative of how an item was displayed than actual user preferences~\citep{agarwal2019addressing, wang2016learning}.

To correct for this bias, early \ac{CLTR} applied \ac{IPS}, which weights clicks inversely to the estimated effect of position bias~\cite{joachims2017unbiased, wang2016learning}.
Later work expanded this approach to correct for other forms of bias, e.g., item-selection bias~\citep{oosterhuis2020policy, ovaisi2020correcting} and trust bias~\citep{agarwal2019addressing, vardasbi2020inverse}, and more advanced \ac{DR} estimation~\citep{oosterhuis2022doubly}.
Using these methods, standard \ac{CLTR} aims to create an unbiased estimate of relevance (or user preference) from click frequencies.
In other words, their goal is to output an estimate per document with an expected value that is equal to their relevance.

However, unbiased estimates of \ac{CLTR} have their limitations.
Firstly, they assume a model of user behavior and require an accurate estimate of this model.
If the assumed model is incorrect~\citep{vardasbi2020inverse, oosterhuis2020policy} or its estimated parameters are inaccurate~\citep{joachims2017unbiased, oosterhuis2022doubly}, then their unbiasedness is not guaranteed.
Secondly, even when unbiased, the estimates are subject to variance~\cite{oosterhuis2022reaching}.
As a result, the actual estimated values are often erroneous, especially when the available data is sparse~\citep{oosterhuis2022doubly, gupta2023safe}.
Accordingly, unbiased \ac{CLTR} does not guarantee that the ranking models it produces have optimal performance~\cite{oosterhuis2022reaching}.

\header{Safe \acl{CLTR}}
There are risks involved in applying \ac{CLTR} in practice. 
In particular, there is a substantial risk that a learned ranking model is deployed that degrades performance compared to the previous production system~\citep{gupta2023safe,oosterhuis2021robust,jagerman2020safe}.
This can have negative consequences to important business metrics, making \ac{CLTR} less attractive to practitioners.
To remedy this issue, a \emph{safe} \ac{CLTR} approach was proposed by \citet{gupta2023safe}; see Chapter~\ref{chapter:01-online-evaluation1}.
Their approach builds on IPS-based \ac{CLTR} and adds exposure-based risk regularization, which keeps the learned model from deviating too much from a given safe model.
Thereby, under the assumption of a position-biased user model, the safe \ac{CLTR} approach can guarantee an upper bound on the probability of the model being worse than the safe model.

\header{Limitations of the current safe \ac{CLTR} method}
Whilst safe \ac{CLTR} is an important contribution to the field, it has two severe limitations -- both are addressed by this chapter.
Firstly, the existing approach is only applicable to \ac{IPS} estimation, which is no longer the state-of-the-art in the field~\citep{gupta2024unbiased}, and it assumes a rank-based position bias model~\citep{craswell2008experimental, wang2016learning}, the most basic user behavior model in the field.
Secondly, because its guarantees rely on assumptions about user behavior, it can only provide a conditional notion of safety.
Moreover, since user behavior can be extremely heterogeneous, it is unclear whether a practitioner could even determine whether the safety guarantees would apply to their application.

\header{Main contributions}
Our first contribution in this chapter addresses the mismatch between the existing safe \ac{CLTR} approach and recent advances in \ac{CLTR}.
We propose a novel generalization of the exposure-based regularization term that provides safety guarantees for both \ac{IPS} and \ac{DR} estimation, also under more complex models of user behavior that cover both position and trust bias.
Our experimental results show that our novel method reaches higher levels of performance significantly faster, while avoiding any notable decreases of performance.
This is especially beneficial since \ac{DR} is known to have detrimental performance when very little data is available~\citep{oosterhuis2022doubly}.

Our second contribution in this chapter provides an unconditional notion of safety. We take inspiration from advances in \ac{RL}~\cite{wang2020truly,schulman2017proximal,liu2019neural,wang2019trust,queeney2021generalized} and propose the novel \acfi{PRPO} method.
\ac{PRPO} removes incentives for \ac{LTR} methods to rank documents too much higher than a given safe ranking model would.
Thereby, \ac{PRPO} imposes a limit on the performance difference between a learned model and a safe model, in terms of standard ranking metrics.
Importantly, \ac{PRPO} is easily applicable to \emph{any} gradient-descent-based \ac{LTR} method, and makes \emph{no assumptions} about user behavior.
In our experiments, \ac{PRPO} prevents any notable decrease in performance even under extremely adversarial circumstances, where other methods fail.
Therefore, we believe \ac{PRPO} is the first \emph{unconditionally} safe \ac{LTR} method. %

Together, our contributions in this chapter bring important advances to the theory of safe \ac{CLTR}, by proposing a significant generalization of the existing approach with theoretical guarantees, and the practical appeal of \ac{CLTR}, with the first robustly safe \ac{LTR} method: \ac{PRPO}.
All source code to reproduce our experimental results is available at: 
\url{https://github.com/shashankg7/cikm-safeultr}.

\vspace{-1.1mm}

\section{Related Work}
\label{sec:relatedwork}
\textbf{\Acl{CLTR}.}
\citet{joachims2017unbiased} introduced the first method for \ac{CLTR}, a \ac{LTR} specific adaptation of \ac{IPS} from the bandit literature~\cite{swaminathan2015batch,joachims2016counterfactual,saito2021counterfactual,gupta-2024-optimal,gupta2023ictir, gupta-2023-first-abstract} to correct for position bias.
They weight each user interaction according to the inverse of its examination probability, i.e., its inverse propensity, during learning to correct for the position bias in the logged data. 
This weighting will remove the effect of position bias from the final ranking policy.
\citet{oosterhuis2020policy} extended this method for the top-$K$ ranking setting with item-selection bias, where any item placed outside the top-$K$ positions gets zero exposure probability, i.e., an extreme form of position bias. 
They proposed a policy-aware propensity estimator, where the propensity weights used in \ac{IPS} are conditioned on the logging policy used to collect the data. 

\citet{agarwal2019addressing} introduced an extension of \ac{IPS}, known as Bayes-IPS, to correct for \textit{trust bias}, an extension of position-bias, with false-positive clicks at the higher ranks, because of the users' trust in the search engine. 
\citet{vardasbi2020inverse} proved that Bayes-IPS cannot correct for trust bias and introduced an affine-correction method and unbiased estimator. 
\citet{oosterhuis2021unifying} combined the affine-correction with a policy-aware propensity estimator to correct for trust bias and item-selection bias simultaneously.
Recently, \citet{oosterhuis2022doubly} introduced a \ac{DR}-estimator for \ac{CLTR}, which combines the existing \ac{IPS}-estimator with a regression model to overcome some of the challenges with the \ac{IPS}-estimator. 
The proposed \ac{DR}-estimator corrects for item-selection and trust biases, with lower variance and improved sample complexity. 

\header{Safe policy learning from user interactions}
In the context of offline evaluation for contextual bandits, \citet{thomas2015high} introduced a high-confidence off-policy evaluation framework.
A confidence interval is defined around the empirical off-policy estimates, and there is a high probability that the \textit{true} utility can be found in the interval. 
\citet{jagerman2020safe} extended this framework for safe deployment in the contextual bandit learning setup. 
The authors introduce a \ac{SEA} method that selects with high confidence between a safe behavior policy and the newly learned policy. 
In the context of \ac{LTR}, \citet{oosterhuis2021robust} introduced the \ac{GENSPEC} method, which safely selects between a feature-based and tabular \ac{LTR} model. 
For off-policy learning, \citet{swaminathan2015batch} introduced a \ac{CRM} framework for the contextual bandit setup. 
They modify the \ac{IPS} objective for bandits to include a regularization term, which explicitly controls for the variance of the \ac{IPS}-estimator during learning, thereby overcoming some of the problems with the high-variance of \ac{IPS}.
\citet{wu2018variance} extended the \ac{CRM} framework by using a \textit{risk} regularization, which penalizes mismatches in the action probabilities under the new policy and the behavior policy. 
In the previous chapter (Chapter~\ref{chapter:01-online-evaluation1}) we made this general safe deployment framework effective in the \ac{LTR} setting. 
We proposed an exposure-based risk regularization method where the difference in the document exposure distribution under the new and logging policies is penalized. 
When click data is limited, risk regularization ensures that the performance of the new policy is similar to the logging policy, ensuring safety. 

To the best of our knowledge, the methodology proposed in Chapter~\ref{chapter:01-online-evaluation1} is the only method for safe policy learning in the \ac{LTR} setting. 
While it guarantees safe ranking policy optimization, it has two main limitations:
\begin{enumerate}[label=(\roman*)]
        \item It is only applicable to the \ac{IPS} estimator; and  
        \item It is only applicable under the position-based click model assumption, the most basic click model in the \ac{CLTR} literature~\citep{joachims2017unbiased, oosterhuis2020learning, gupta2024unbiased}. 
\end{enumerate} 

\header{Proximal policy optimization} 
In the broader context of \ac{RL}, \acfi{PPO} was introduced as a policy gradient method for training \ac{RL} agents to maximize long-term rewards~\cite{wang2020truly,schulman2017proximal,liu2019neural,wang2019trust,queeney2021generalized}.
\Ac{PPO} clips the importance sampling ratio of action probability under the new policy and the current behavior policy, and thereby, it prevents the new policy to deviate from the behavior policy by more than a certain margin.
\ac{PPO} is not directly applicable to \ac{LTR}, for the same reasons that the \ac{CRM} framework is not: the combinatorial action space of \ac{LTR} leads to extremely small propensities that \ac{PPO} cannot effectively manage~\cite{gupta2023safe}.

\section{Background}
\subsection{Learning to rank}
The goal in \ac{LTR} is to find a ranking policy ($\pi$) that optimizes a given ranking metric~\cite{liu2009learning}. 
Formally, given a set of documents ($D$), a distribution of queries $Q$, and the true relevance function ($P(R=1 \mid d)$), \ac{LTR} aims to maximize the following utility function:
\begin{equation}
    U(\pi) =  \sum_{q \in Q} P(q \mid Q) \sum_{d \in D} \omega(d \mid \pi) \; P(R=1 \mid d), \label{true-utility}
\end{equation}
where $\omega(d \mid \pi)$ is the weight of the document for a given policy $\pi$. The weight can be set accordingly to optimize for a given ranking objective, for example, setting the weight to:
\begin{equation}
    \omega_{\text{DCG}}(d \mid q, \pi) = \mathbb{E}_{y \sim \pi( \cdot \mid q)} \mleft[ (\log_2(\textrm{rank}(d \mid y) + 1))^{-1} \mright], \label{rho1}
\end{equation}
optimizes \ac{DCG}~\citep{jarvelin2002cumulated}.
For this chapter, we aim to optimize the expected number of clicks, so we set the weight accordingly~\cite{oosterhuis2022doubly,gupta2023safe,yadav2021policy}. 

\subsection{Assumptions about user click behavior}
The optimization of the true utility function (Eq.~\ref{true-utility}) requires access to the document relevance ($P(R=1 \mid d)$). 
In the \ac{CLTR} setting, the relevances of documents are not available, and instead, click interaction data is used to estimate them~\cite{joachims2017unbiased, wang2016learning, oosterhuis2020learning}.
However, naively using clicks to optimize a ranking system can lead to sub-optimal ranking policies, as clicks are a biased indicator of relevance~\cite{chuklin-click-2015, joachims2002optimizing, joachims2017unbiased, craswell2008experimental}.
\ac{CLTR} work with theoretical guarantees starts by assuming a model of user behavior.
The earliest \ac{CLTR} works~\citep{joachims2017unbiased, wang2016learning} assume a basic model originally proposed by \citet{craswell2008experimental}: 
\begin{assumption}[\emph{The rank-based position bias model}]
    \label{assumption:positionbias}
     The probability of a click on document $d$ at position $k$ is the product of the rank-based examination probability and document relevance:
    \begin{equation}    
    P(C = 1 \mid d, k) = P(E=1 \mid k) P(R=1 \mid d) = \alpha_k P(R=1 \mid d).
    \end{equation}
\end{assumption}

\noindent%
Later work has proposed more complex user models to build on~\citep{gupta2024unbiased}.
Relevant to our work is the model proposed by \citet{agarwal2019addressing}, and its re-formulation by \citet{vardasbi2020inverse}; it is a generalization of the above model to include a form of trust bias:

\begin{assumption}[\emph{The trust bias model}]
\label{assumption:trustbias}
The probability of a click on document $d$ at position $k$ is an affine transformation of the relevance probability of $d$ in the form:
\begin{equation}
        P(C = 1 \mid d, k) = \alpha_k P(R=1 \mid d) + \beta_k,  \label{affine-click-model}
\end{equation}
where $\forall k, \alpha_k \in [0,1] \land \beta_k\in [0,1] \land (\alpha_k + \beta_k) \in [0,1]$.
\end{assumption}

\noindent%
Whilst it is named after trust bias, this model actually captures three forms of bias that were traditionally categorized separately: rank-based position bias, item-selection bias, and trust bias.
Position bias was originally approached as the probability that a user would examine an item, which would decrease at lower positions in the ranking~\citep{wang2018position, craswell2008experimental, joachims2017unbiased, wang2016learning}.
In the trust bias model, this effect can be captured by decreasing $\alpha_k + \beta_k$ as $k$ increases.
Additionally, with $\forall k, \beta_k = 0$, the trust bias model is equivalent to the rank-based position bias model.
Item-selection bias refers to users being unable to see documents outside a top-$K$, where they receive zero probability of being examined or interacted with~\citep{oosterhuis2020policy}.
This can be captured by the trust bias model by setting $\alpha_k + \beta_k = 0$ when $k > K$.
Lastly, the key characteristic of trust bias is that users are more likely to click on non-relevant items when they are near the top of the ranking~\citep{agarwal2019addressing}.
This can be captured by the model by making $\beta_k$ larger as $k$ decreases~\citep{vardasbi2020inverse}.
Thereby, the trust bias model is in fact a generalization of most of the user models assumed by earlier work~\citep{oosterhuis2021unifying, gupta2024unbiased}. The following works all assume models that fit Assumption~\ref{assumption:trustbias}: \citep{vardasbi2020inverse, agarwal2019addressing, oosterhuis2021unifying, oosterhuis2020policy, wang2016learning, wang2018position, oosterhuis2022doubly, oosterhuis2021robust, gupta2023safe, agarwal2019general, ovaisi2020correcting}.

\subsection{Counterfactual learning to rank}
This section details the \emph{policy-aware} \acfi{IPS} estimator proposed by \citet{oosterhuis2020policy} and the \acfi{DR} estimator by \citet{oosterhuis2022doubly}.

First, let $\mathcal{D}$ be a set of logged interaction data:
$
    \mathcal{D} = \big\{q_i, y_i, c_i \big\}^N_{i=1}$,
where each of the $N$ interactions consists of a query $q_i$, a displayed ranking $y_i$, and click feedback $c_i(d) \in \{0,1\}$ that indicates whether the user clicked on the document $d$ or not.
Both policies use propensities that are the expected $\alpha$ values for each document:
\begin{equation}
        \rho_{0}(d \mid q_i, \pi_0) =  \mathbb{E}_{y \sim \pi_{0}(q_i)} \big[ \alpha_{k(d)}  \big] = \rho_{i,0}(d).
    \label{policy-aware-exposure}
\end{equation}
Similarly, to keep our notation short, we also use $\omega(d \mid q_i, \pi) = \omega_i(d)$.
Next, the policy-aware \ac{IPS} estimator is defined as:
\begin{equation}
    \hat{U}_{\text{IPS}}(\pi) = \frac{1}{N} \sum_{i=1}^{N} \sum_{d \in D}  \frac{\omega_i(d)}{\rho_{i,0}(d)} c_i(d).
    \label{cltr-obj-ips-positionbias}
\end{equation}
\citet{oosterhuis2020policy} prove that under the rank-based position bias model (Assumption~\ref{assumption:positionbias}) and when $\forall (i,d),  \rho_{i,0}(d) > 0$, this estimator is unbiased: $\mathbb{E}[\hat{U}_{\text{IPS}}(\pi)] = U(\pi)$.

The \ac{DR} estimator improves over the policy-aware \ac{IPS} estimator in terms of assuming the more general trust bias model (Assumption~\ref{assumption:trustbias}) and having lower variance.
\citet{oosterhuis2022doubly} proposes the usage of the following $\omega$ values for the policy $\pi$:
\begin{equation}
\omega(d \mid q_i, \pi) = \mathbb{E}_{y \sim \pi(q_i)} \big[ \alpha_{k(d)}  + \beta_{k(d)} \big] = \omega_i(d),
    \label{eq:omega}
\end{equation}
since with these values $U$ (Eq.~\ref{true-utility}) becomes the number of expected clicks on relevant items under the trust bias model; $U = (\alpha_{k}  + \beta_{k})P(R=1 \mid d,q) = P(C = 1, R = 1 \mid k,d,q)$.
We follow this approach and define the $\omega$ values for the logging policy $\pi_{0}$ as:
\begin{equation}
    \omega_{0}(d \mid q_i, \pi_{0}) = \mathbb{E}_{y \sim \pi_{0}(q_i)} \big[ \alpha_{k(d)}  + \beta_{k(d)} \big] = \omega_{i,0}(d).
    \label{eq:omega_logging}
\end{equation}
The \ac{DR} estimator uses predicted relevances in its estimation, i.e., using predictions from a regression model.
Let $\hat{R}_{i}(d) \approx P(R = 1 \mid d, q_i)$ indicate a predicted relevance; then the utility according to these predictions is:
\begin{equation}
    \hat{U}_{\text{DM}}(\pi) = \frac{1}{N} \sum_{i=1}^{N} \sum_{d \in D} \omega_i(d) \hat{R}_{i}(d).
\end{equation}
The \ac{DR} estimator starts with this predicted utility and adds an \ac{IPS}-based correction to remove its bias:
\begin{align}
    \hat{U}_{\text{DR}}(\pi) = \hat{U}_{\text{DM}}(\pi) +  \frac{1}{N} \sum_{i=1}^{N} \sum_{d \in D}  \frac{\omega_{i}(d)}{\rho_{i,0}(d)} \left(c_i(d) - \alpha_{k_i(d)}\hat{R}_{i}(d) -  \beta_{k_i(d)} \right). 
    \label{cltr-obj-dr}
\end{align}
Thereby, the corrections of the \ac{IPS} part of the \ac{DR} estimator will be smaller if the predicted relevances are more accurate.
\citet{oosterhuis2022doubly} proves that under the assumption of the trust bias model (Assumption~\ref{assumption:trustbias}), the \ac{DR} estimator is unbiased when $\forall (i,d), \rho_{i,0}(d) > 0 \lor \hat{R}_{i}(d) = P(R = 1 \mid d, q_i)$ and has less variance if $0 \leq \hat{R}_{i}(d) \leq 2P(R = 1 \mid d, q_i)$.
The author also shows that the \ac{DR} estimator needs less data to reach the same level of ranking performance as \ac{IPS}, with especially large improvements when applied to top-$K$ rankings~\citep{oosterhuis2022doubly}.

\subsection{Safety in counterfactual learning to rank}
\Ac{IPS}-based \ac{CLTR} methods, despite their unbiasedness and consistency, suffer from the problem of high-variance~\cite{gupta2024unbiased,oosterhuis2022doubly,joachims2017unbiased}.
 Specifically, if the logged click data is limited, training an \ac{IPS}-based method can lead to an unreliable and unsafe ranking policy~\cite{gupta2023safe}.
 The problem of \textit{safe} policy learning is well-studied in the bandit literature~\citep{thomas2015high, jagerman2020safe, swaminathan2015batch, wu2018variance}. \citet{swaminathan2015batch} proposed the first risk-aware off-policy learning method for bandits, with their risk term quantified as the variance of the \ac{IPS}-estimator. 
\citet{wu2018variance} proposed an alternative method for risk-aware off-policy learning, where the risk is quantified using a Renyi divergence between the action distribution of the new policy and the logging policy~\citep{renyi1961measures}.
Thus, both consider it a risk for the new policy to be too dissimilar to the logging policy, which is presumed safe.
Whilst effective at standard bandit problems, these risk-aware methods are not effective for ranking tasks due to their enormous combinatorial action spaces and correspondingly small propensities.

As a solution for \ac{CLTR}, in Chapter~\ref{chapter:01-online-evaluation1} we introduced a risk-aware \ac{CLTR} approach that uses divergence based on the exposure distributions of policies.
They first introduce normalized propensities: $\rho'\!(d) = \rho / Z$, with a normalization factor $Z$ based on $K$:
\begin{equation}
        Z = \!\! \sum_{d \in D}^{}\! \rho(d)  = \!\! \sum_{d \in D} \! \mathbb{E}_{y \sim \pi } \big[ \alpha_{k(d)}  \big]  \! 
         = \mathbb{E}_{y \sim \pi} \left[  \sum_{k=1}^K  \alpha_{k(d)}  \right] \!=  \!\! \sum_{k=1}^K  \alpha_k.
    \label{total-exposure}
\end{equation}
Since $\rho'\!(d) \in [0,1]$ and $\sum_d \rho'\!(d) = 1$, they can be treated as a probability distribution that indicates how exposure is spread over documents.
In Chapter~\ref{chapter:01-online-evaluation1}, we use Renyi divergence to quantify how dissimilar the new policy is from the logging policy:
\begin{equation}
    d_2(\rho \,\Vert\, \rho_0) = \mathbb{E}_{q} \left[ \sum_{d} \mleft(\frac{\rho'(d)}{\rho_{0}'(d)}\mright)^2 \rho_{0}'(d) \right],
    \label{eq:actionbaseddiv}
\end{equation} 
with the corresponding empirical estimate based on the log data ($\mathcal{D}$) defined as:
 \begin{equation}
    \hat{d}_2(\rho \,\Vert\, \rho_0) = \frac{1}{N} \sum_{i=1}^{N} \sum_{d} \mleft(\frac{\rho'_{i}(d)}{\rho_{i,0}'(d)}\mright)^2 \rho_{i,0}'(d).
    \label{eq:actionbaseddiv_emp}
\end{equation} 
Based on this divergence term, they propose the following risk-aware \ac{CLTR} objective, with parameter $\delta$:
\begin{equation}
    \max_{\pi}  \hat{U}_{\text{IPS}}(\pi) -  \sqrt{ \frac{Z}{N}  \Big(\frac{1-\delta}{\delta}\Big) \hat{d}_2(\rho \,\Vert\, \rho_0)}.
\label{objgenbound}
\end{equation}
Thereby, the existing safe \ac{CLTR} approach penalizes the optimization procedure from learning ranking behavior that is too dissimilar from the logging policy in terms of the distribution of exposure.
The weight of this penalty decreases as the number of datapoints $N$ increases, thus it maintains the same point of convergence as standard \ac{IPS}.
Yet, initially when little data is available and the effect of variance is the greatest, it forces the learned policy to be very similar to the safe logging policy.
\citet{gupta2023safe} prove that their objective bounds the real utility with a probability of $1-\delta$:
\begin{equation}
P\mleft( U(\pi) \geq \hat{U}_{\text{IPS}}(\pi)\! -  \sqrt{ \frac{Z}{N}  \Big(\frac{1-\delta}{\delta}\Big) d_2(\rho \,\Vert\, \rho_{0})\!} - \sqrt{ \frac{\mathrm{1}}{N}  \Big(\frac{1-\delta}{\delta}\Big)} \mright) \geq 1 - \delta.
\end{equation}
However, their proof of safety relies on the rank-based position bias model (Assumption~\ref{assumption:positionbias}) and their approach is limited to the basic \ac{IPS} estimator for \ac{CLTR}.

\subsection{Proximal policy optimization}
In the more general \acf{RL} field, \acfi{PPO} was introduced as a method to restrict a new policy $\pi$ from deviating too much from a previously rolled-out policy $\pi_0$~\cite{Schulmanetal_ICLR2016, schulman2017proximal}.
In contrast with the earlier discussed methods, \ac{PPO} does not make use of a divergence term but uses a simple clipping operation in its optimization objective.
Let $s$ indicate a state, $a$ an action and $R$ a reward function, the \ac{PPO} loss is:
\begin{equation}
    U^{PPO}\!(s, a, \pi, \pi_0) = \mathbb{E} \mleft[ \min \mleft( \frac{\pi(a \,|\, s)}{\pi_0(a \,|\, s)} R(a \,|\, s) , g\big(\epsilon, R(a \,|\, s) \big)  \mright) \mright],
    \end{equation}
where $g$ creates a clipping threshold based on the sign of $R(a \,|\, s)$:
\begin{equation}
    g\mleft(\epsilon, R(a \,|\, s)  \mright) = 
\begin{cases}
    \left(1+\epsilon\right) R(a \,|\, s) & \text{if }R(a \,|\, s)  \geq 0,\\
    \left(1-\epsilon\right) R(a \,|\, s) & \text{otherwise}.
\end{cases}    
\end{equation}
The clipping operation removes incentives for the optimization to let $\pi$ deviate too much from $\pi_0$, since there are no further increases in $U^{PPO}$ when $\pi(a \,|\, s) > (1+ \epsilon)\pi_0(a \,|\, s)$ or  $\pi(a \,|\, s) < (1- \epsilon)\pi_0(a \,|\, s)$, depending on the sign of $R(a \,|\, s)$.
Similar to the previously discussed general methods, \ac{PPO} is not effective when directly applied to the \ac{CLTR} setting due to the combinatorial action space and corresponding extremely small propensities (for most $a$ and $s$: $\pi_0(a \,|\, s) \simeq 0$).

\section{Extending Safety to Advanced CLTR}
\label{sec:method:existingCLTR}
In this section, we introduce our first contribution of this chapter: our extension of the safe \ac{CLTR} method to address trust bias and \ac{DR} estimation. 

\subsection{Method: Safe doubly-robust CLTR}
For the safe \ac{DR} \ac{CLTR} method, we extend the generalization bound from the existing \ac{IPS} estimator and position bias~\cite[Eq.~26]{gupta2023safe} to the \ac{DR} estimator and trust bias. 
    \label{sec:perfbound}
    \begin{theorem}
    \label{CLTR-bound}
        Given the true utility $U(\pi)$ (Eq.~\ref{true-utility}) and its exposure-based \ac{DR} estimate $\hat{U}_{\text{DR}}(\pi)$ (Eq.~\ref{cltr-obj-dr}) of the ranking policy $\pi$ with the logging policy $\pi_{0}$ and the metric weights $\omega$ and $\omega_{0}$ (Eq.~\ref{eq:omega} and \ref{eq:omega_logging}), assuming the trust bias click model (Assumption~\ref{assumption:trustbias}),
        the following generalization bound holds with probability $1 - \delta$: \footnote{The following proof differs slightly from the original proof published in~\cite{gupta-2024-practical}, as we overlooked an additional constant term in the original paper.}
        \begin{equation}
            \begin{split}
              P\Bigg(
                  U(\pi)\;\ge\;
                  \hat{U}_{\text{DR}}(\pi)
                  \;-\;
                  \Bigl(1+\max_{k}\tfrac{\beta_k}{\alpha_k}\Bigr)
                  &\Bigl(
                    \sqrt{\frac{2Z}{N}\,
                      \Big(\frac{1-\delta}{\delta}\Big)\,
                             d_2(\omega\Vert\omega_0)}
                      \\
                      &+\sqrt{\frac{1}{N}\,
                      \Big(\frac{1-\delta}{\delta}\Big)}
                  \Bigr)
              \Bigg)
              \;\ge\;
              1-\delta .
            \end{split}
          \end{equation}
    \end{theorem}
    \begin{proof}
        For a proof, we refer to the appendix (Theorem~\ref{sec:perfbound-proof}).
    \end{proof}
    
\noindent%
Given the novel generalization bound from Theorem~\ref{CLTR-bound}, we define the safe \ac{DR} \ac{CLTR} objective as follows:
\begin{equation}
    \max_{\pi}  \hat{U}_{\text{DR}}(\pi) - \mleft( 1 + \max_{k} \frac{ \beta_k}{\alpha_k} \mright) \sqrt{ \frac{2 Z}{N}  \Big(\frac{1-\delta}{\delta}\Big) \hat{d}_2(\omega \,\Vert\, \omega_{0})},
\label{dr-objgenbound}
\end{equation}
where $\hat{d}_2(\omega \,\Vert\, \omega_0)$ is defined analogously to Eq.~\ref{eq:actionbaseddiv}.
Note that we ignore the term $\sqrt{ \frac{\mathrm{1}}{N}  \Big(\frac{1-\delta}{\delta}\Big)}$, since it is constant with respect to the policy $\pi$, and therefore can be disregarded for optimization purposes.
The objective optimizes the lower-bound on the true utility function, through a linear combination of the empirical \ac{DR} estimator ($\hat{U}_{\text{DR}}(\pi)$) and the empirical risk regularization term ($\hat{d}_2(\omega \,\Vert\, \omega_0)$). 
In a setting where click data is limited, our safe \ac{DR} objective will weight the risk regularization term higher, and as a result, the objective ensures that the new policy stays close to the safe logging policy. 
When a sufficiently high volume of click data is collected, and thus we have higher confidence in the \ac{DR} estimate, the objective falls back to its \ac{DR} objective counterpart.

For the choice of the ranking policy ($\pi$), we propose to optimize a stochastic ranking policy $\pi$ with a gradient descent-based method. 
For the gradient calculation, we refer to previous work~\cite{yadav2021policy,gupta2023safe,oosterhuis2022doubly}.

\header{Conditions for safe \ac{DR} \ac{CLTR}}
Finally, we note that besides the explicit assumption that user behavior follows the trust bias model (Assumption~\ref{assumption:trustbias}), there is also an important implicit assumption in this approach.
Namely, the approach assumes that the bias parameters (i.e., $\alpha$ and $\beta$) are known, a common assumption in the \ac{CLTR} literature~\citep{oosterhuis2022doubly, oosterhuis2020learning}.
However, in practice, either of these assumptions could not hold, i.e., user behavior could not follow the trust bias model, or a model's bias parameters could be wrongly estimated.
Additionally, in adversarial settings where clicks are intentionally misleading or incorrectly logged~\cite{radlinski2007addressing,liu2023black,castillo2011adversarial}, the user behavior assumptions do not hold, and, the generalization bound of our \ac{DR} \ac{CLTR} is not guaranteed to hold.
Thus, whilst it is an important advancement over the existing safe \ac{CLTR} method~\cite{gupta2023safe}, our approach is limited to only providing a \emph{conditional} form of safety.

\section{Method: Proximal Ranking Policy Optimization (PRPO)}

Inspired by the limitations of the method introduced in Section~\ref{sec:method:existingCLTR} and the \ac{PPO} method from the \ac{RL} field (Section~\ref{sec:relatedwork}), we propose the first \emph{unconditionally} safe \ac{CLTR} method: \acfi{PRPO}.
Our novel \ac{PRPO} method is designed for practical safety by making \emph{no assumptions} about user behavior.
Thereby, it provides the most robust safety guarantees for \ac{CLTR} yet.

For safety, instead of relying on a high-confidence bound (e.g., Eq.~\ref{objgenbound} and~\ref{dr-objgenbound}), \ac{PRPO} guarantees safety by removing the incentive for the new policy to rank documents too much higher than the safe logging policy. 
This is achieved by directly clipping the ratio of the metric weights for a given query $q_i$ under the new policy $\omega_i(d)$, and the logging policy ($\omega_{i,0}(d)$), i.e., $\frac{\omega_{i}(d)}{\omega_{i,0}(d)}$ to be bounded in a fixed predefined range: $\mleft[\epsilon_{-}, \epsilon_{+}\mright]$. 
As a result, the \ac{PRPO} objective provides no incentive for the new policy to produce weights $\omega_i(d)$ outside of the range: $\epsilon_{-} \cdot \omega_{i,0}(d) \leq \omega_i(d) \leq \epsilon_{+} \cdot \omega_{i,0}(d)$.

Before defining the \ac{PRPO} objective, we first introduce a term $r(d|q)$ that represents an unbiased \ac{DR} relevance estimate, weighted by $\omega_{0}$, for a single document-query pair (cf.\ Eq.~\ref{cltr-obj-dr}):
\begin{equation}
\begin{split}
        r(d | q)  = \omega_{0}(d | q)\hat{R}(d | q) +   \frac{\omega_{0}(d | q)}{\rho_{0}(d | q)} \! \sum_{i \in \mathcal{D} : q_i = q \hspace{-0.7cm}}  \mleft(c_i(d)  - \alpha_{k_i(d)}\hat{R}(d | q)   -  \beta_{k_i(d)}\mright).
\end{split}        
\end{equation}
For the sake of brevity, we drop $\pi$ and $\pi_{0}$ from the notation when their corresponding value is clear from the context.
This enables us to reformulate the \ac{DR} estimator around the ratios between the metric weights $\omega$ and $\omega_0$ (cf.\ Eq.~\ref{cltr-obj-dr}):
\begin{equation}
    \hat{U}_{\text{DR}}(\pi) =  \sum_{q,d \in \mathcal{D}}  \frac{\omega(d\mid q)}{\omega_{0}(d \mid q)} r(d \mid q).
    \label{eq:dr_reformulate}
\end{equation}
Before defining the proposed \ac{PRPO} objective, we first define the following clipping function:
\begin{equation}
    f(x,\epsilon_{-}, \epsilon_{+}, r) = 
\begin{cases}
    \min(x, \epsilon_{+}) \cdot r  & r \geq 0,\\
    \max(x, \epsilon_{-}) \cdot r & \text{otherwise}.
\end{cases}  
\label{eq:clipping_function}  
\end{equation}
Given the reformulated \ac{DR} estimator (Eq.~\ref{eq:dr_reformulate}), and the clipping function (Eq.~\ref{eq:clipping_function}), the \ac{PRPO} objective can be defined as follows:
\begin{equation}
    \hat{U}_{\text{PRPO}}(\pi) =  \sum_{q,d \in \mathcal{D}}  f\mleft(\frac{\omega(d \mid q)}{\omega_{0}(d \mid q)}, \epsilon_{-}, \epsilon_{+}, r(d \mid q)\mright).
    \label{eq:prpo_obj}
\end{equation}
Figure~\ref{fig:prpo} visualizes the effect the clipping of \ac{PRPO} has on the optimization incentives.
We see how the clipped and unclipped weight ratios progress as documents are placed on different ranks.
The unclipped weights keep increasing as documents are moved to the top of the ranking, when $r>1$, or to the bottom, when $r<1$.
Consequently, optimization with unclipped weight ratios aims to place these documents at the absolute top or bottom positions.
Conversely, the clipped weights do not increase beyond their clipping threshold, which for most document is reached before being placed at the very top or bottom position.
As a result, optimization with clipped weight ratios will not push these documents beyond these points in the ranking.
For example, when $r>0$, we see that there is no incentive to place a document at higher than rank 6, if it was placed at rank 8 by the logging policy.
Similarly, placement higher than rank 4 leads to no gain if the original rank was 6, and higher than rank 3 leads to no improvement gain from an original rank of 4.
Vice versa, when $r<0$, each document has a rank, where placing it lower than that rank brings no increase in clipped weight ratio.
Importantly, this behavior only depends on the metric and the logging policy; \ac{PRPO} makes \emph{no further assumptions}.

\begin{figure}[!t]
    \begin{center}
    \includegraphics[scale=0.48]{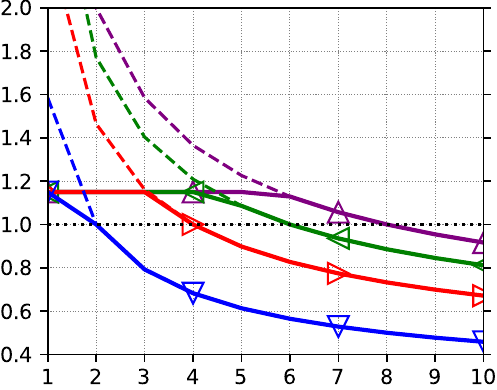}
    \includegraphics[scale=0.48]{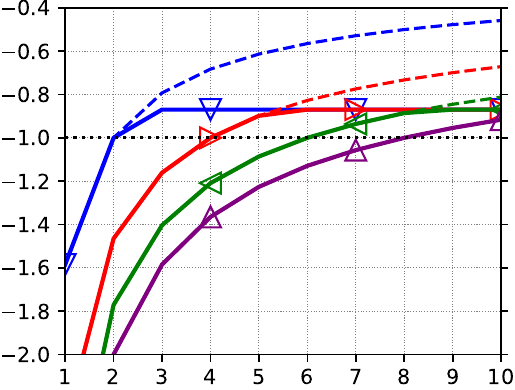}\\
    \includegraphics[width=0.99\columnwidth]{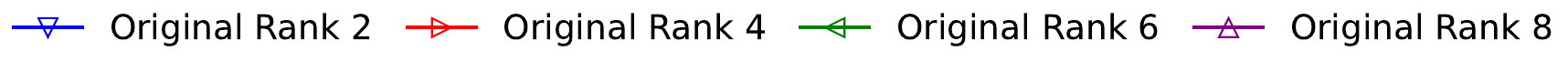}
    \end{center}
        \caption{
        Weight ratios in the clipped \ac{PRPO} objective (solid lines) and the unclipped counterparts (dashed lines), as documents are moved from four different original ranks.
        Left: positive relevance, $r=1$; right: negative relevance, $r=-1$;
        x-axis: new rank for document;
        y-axis: unclipped weight ratios (dashed lines), $r\cdot\omega_i(d)/\omega_{i,0}(d)$;
        and 
        clipped \ac{PRPO} weight ratios (solid lines),
        $f\mleft(\omega_i(d)/\omega_{i,0}(d), \epsilon_{-} = 1.15^{-1}, \epsilon_{+}= 1.15, r=\pm1\mright)$.
        DCG metric weights used: $\omega_i(d) = \log_2(\textmd{rank}(d \mid q_i, \pi) + 1)^{-1}$.
        }
        \label{fig:prpo}
\end{figure}

Whilst the clipping of \ac{PRPO} is intuitive, we can prove that it provides the following formal form of unconditional safety:
\begin{theorem}
\label{PRPO-proof}
Let $q$ be a query, $\omega$ be metric weights, $y_0$ be a logging policy ranking, and $y^*(\epsilon_{-},\epsilon_{+})$ be the ranking that optimizes the \ac{PRPO} objective in Eq.~\ref{eq:prpo_obj}.
Assume that $\forall d, \in \mathcal{D}, r(d \mid q) \not= 0$.
Then, 
for any $\Delta \in \mathbb{R}_{\geq0}$, there exist values for $\epsilon_{-}$ and $\epsilon_{+}$ that guarantee that the difference between the utility of $y_0$ and $y^*(\epsilon_{-},\epsilon_{+})$ is bounded by $\Delta$:
\begin{equation}
\forall \Delta \! \in \mathbb{R}_{\geq0}, \exists \epsilon_{-} \!\!\in \mathbb{R}_{\geq0}, \epsilon_{+}\!\! \in \mathbb{R}_{\geq0} \;\; | U(y_0) -  U(y^*(\epsilon_{-}, \epsilon_{+}))  | \leq \Delta.
\label{eq:prpotheorem}
\end{equation}
\end{theorem}
\begin{proof}
A proof is given in Appendix~\ref{sec:prpo-proof}.
\end{proof}

\header{Adaptive clipping}
Theorem~\ref{PRPO-proof} describes a very robust sense of safety, as it shows \ac{PRPO} can be used to prevent any given decrease in performance without assumptions.
However, it also reveals that this safety comes at a cost; \ac{PRPO} prevents both decreases and increases of performance.
This is very common in safety approaches, as there is a generally a tradeoff between risks and rewards~\citep{gupta2023safe}.
Existing safety methods, such as the safe \ac{CLTR} approach of Section~\ref{sec:method:existingCLTR}, generally, loosen their safety measures as more data becomes available, and the risk is expected to have decreased~\citep{thomas2015high}.

We propose a similar strategy for \ac{PRPO} through adaptive clipping, where the effect of clipping decreases as the number of datapoints $N$ increases.
Specifically, we suggest using a monotonically decreasing $\delta(N)$ function such that $\lim_{N \rightarrow \infty} \delta(N) = 0$.
The $\epsilon$ parameters can then be obtained through the following transformation: $\epsilon_{-} = \delta(N)$ and $\epsilon_{+} = \frac{1}{\delta(N)}$.
This leads to a clipping range of $[\delta(N), \frac{1}{\delta(N)}]$, and in the limit: $\lim_{N \rightarrow \infty}$, it becomes: $[0,\infty]$.
In other words, as more data is gathered, the effect of \ac{PRPO} clipping eventually disappears, and the original objective is recovered.
The exact choice of $\delta(N)$ determines how quickly this happens.

\header{Gradient ascent with PRPO and possible extensions}
Finally, we consider how the \ac{PRPO} objective should be optimized. This turns out to be very straightforward when we look at its gradient.
The clipping function $f$ (Eq.~\ref{eq:clipping_function}) has a simpler gradient involving an indicator function on whether $x$ is inside the bounded range:
\begin{equation}
\nabla_{x} f(x, \epsilon_{-}, \epsilon_{+}, r)
=  \mathds{1}\big[
(r > 0 \land x \leq \epsilon_{+})
\lor
(r < 0 \land x \geq \epsilon_{-})
\big] r.
\end{equation}
Applying the chain rule to the \ac{PRPO} objective (Eq.~\ref{eq:prpo_obj}) reveals:
\begin{equation}
\nabla_{\!\pi} \hat{U}_{\text{PRPO}}(\pi) = \sum_{q,d \in \mathcal{D}}
	\underbrace{\!\!\!\!
	\Big[\nabla_{\!\pi} \frac{\omega(d | q)}{\omega_{0}(d | q)} \Big]
	}_\text{\hspace{-1cm}grad. for single doc.\hspace{-1cm}}
	\underbrace{\!\!
	\nabla_{\!\pi} f\bigg( \! \frac{\omega(d | q)}{\omega_{0}(d | q)}, \epsilon_{-}, \epsilon_{+}, r(d | q) \! \bigg)
	}_\text{\vphantom{g}indicator reward function}.
\end{equation}
Thus, the gradient of \ac{PRPO} simply takes the importance weighted metric gradient per document, and multiplies it with the indicator function and reward.
As a result, \ac{PRPO} is simple to combine with existing \ac{LTR} algorithms, especially ones that use policy-gradients~\cite{williams1992simple}, such as PL-Rank~\citep{oosterhuis2021computationally, oosterhuis2022learning} or StochasticRank~\citep{ustimenko2020stochasticrank}.
For methods in the family of LambdaRank~\citep{wang2018lambdaloss, burges2010ranknet, burges2006learning}, it is a matter of replacing the $|\Delta DCG|$ term with an equivalent for the PRPO bounded metric.

Lastly, we note that whilst we introduced \ac{PRPO} for \ac{DR} estimation, it can be extended to virtually any relevance estimation by choosing a different $r$;
e.g., one can easily adapt it for \ac{IPS}~\citep{oosterhuis2021unifying,joachims2017unbiased}, or relevance estimates from a click model~\citep{chuklin-click-2015}, etc.
In this sense, we argue \ac{PRPO} can be seen as a framework for robust safety in \ac{LTR}.

\section{Experimental Setup}

For our experiments, we follow the semi-synthetic experimental setup that is prevalent in the \ac{CLTR} literature~\citep{oosterhuis2022doubly,oosterhuis2021robust,vardasbi2020inverse,gupta2023safe}.
We make use of the three largest publicly available \ac{LTR} datasets: Yahoo!\ Webscope~\cite{chapelle2011yahoo}, MSLR-WEB30k~\citep{qin2013introducing}, and Istella~\citep{dato2016fast}.
The datasets consist of queries, a preselected list of documents per query, query-document feature vectors, and manually-graded relevance judgments for each query-document pair.

{\renewcommand{\arraystretch}{0.01}
\setlength{\tabcolsep}{0.04cm}
\begin{figure}[!t]
\centering
\begin{tabular}{c r r}
&
 \multicolumn{1}{c}{\small Yahoo! Webscope}
&
 \multicolumn{1}{c}{\small MSLR-WEB30k}
\\
\rotatebox[origin=lt]{90}{\hspace{1.4cm}\small NDCG@5} &
\includegraphics[scale=0.455]{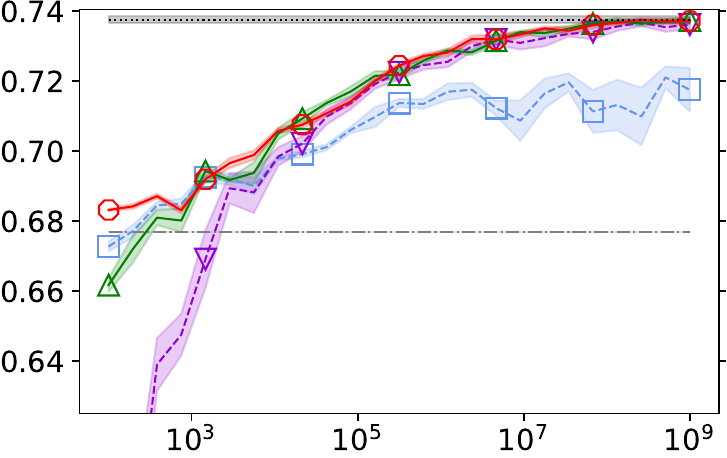} &
\includegraphics[scale=0.455]{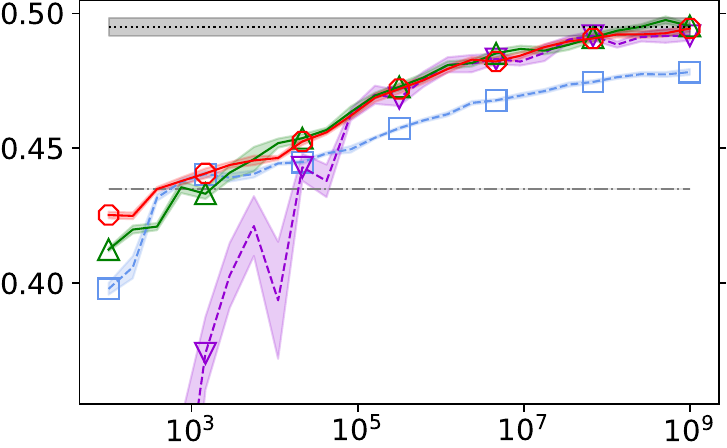}
\\
& \multicolumn{1}{c}{\small Number of interactions simulated ($N$)}
& \multicolumn{1}{c}{\small Number of interactions simulated ($N$)}
\\[2mm]
\end{tabular}

\vspace{0.3\baselineskip}
\includegraphics[scale=0.40]{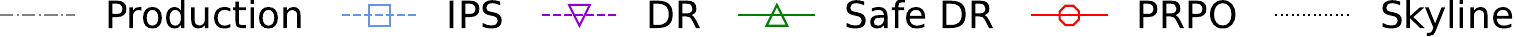}

\caption{
Performance in terms of NDCG@5 of the \ac{IPS}, \ac{DR} and proposed safe \ac{DR} ($\delta=0.95$) and \ac{PRPO} ($\delta(N)=\frac{100}{N}$) methods for \ac{CLTR} on Yahoo! Webscope and MSLR-WEB30k datasets.
The results are presented varying size of training data ($N$), with number of simulated queries varying from $10^2$ to $10^9$.
Results are averaged over 10 runs; the shaded areas indicate 80\% prediction intervals. 
}
\label{fig:mainresults_safe_part1}
\end{figure}
}
{\renewcommand{\arraystretch}{0.01}
\setlength{\tabcolsep}{0.04cm}
\begin{figure}
\centering
\begin{tabular}{c r}
&
 \multicolumn{1}{c}{\small Istella}
\\[2mm]
\rotatebox[origin=lt]{90}{\hspace{1.4cm}\small NDCG@5} &
\includegraphics[scale=0.455]{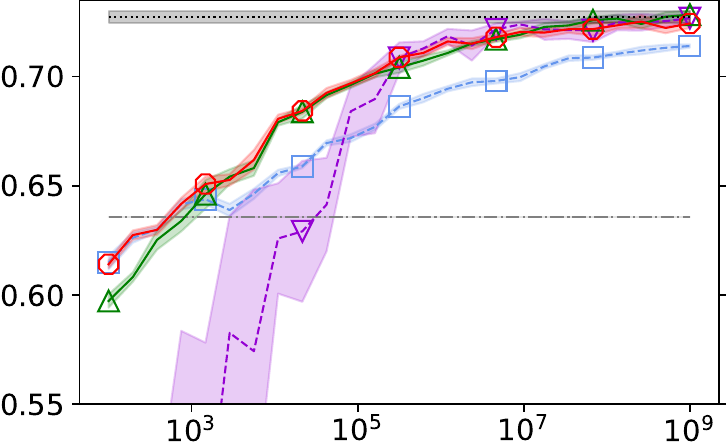}
\\
& \multicolumn{1}{c}{\small Number of interactions simulated ($N$)}
\\[2mm]
\end{tabular}

\vspace{0.3\baselineskip}
\includegraphics[scale=0.46]{05-safety_cikm/figures/legend1.pdf}

\caption{
Performance in terms of NDCG@5 of the \ac{IPS}, \ac{DR} and proposed safe \ac{DR} ($\delta=0.95$) and \ac{PRPO} ($\delta(N)=\frac{100}{N}$) methods for \ac{CLTR} on the Istella dataset.
The results are presented varying size of training data ($N$), with number of simulated queries varying from $10^2$ to $10^9$.
Results are averaged over 10 runs; the shaded areas indicate 80\% prediction intervals. 
}
\label{fig:mainresults_safe_part2}
\end{figure}
}

Following~\cite{oosterhuis2022doubly,vardasbi2020inverse,gupta2023safe}, we train a production ranker on a $3\%$ fraction of the training queries and their corresponding relevance judgments.
The goal is to simulate a real-world setting where a ranker trained on manual judgments is deployed in production and is used to collect click logs.
The collected click logs can then be used for \ac{LTR}. 
We assume the production ranker is safe, given that it would serve live traffic in a real-world setup. 

We simulate a top-$K$ ranking setup~\cite{oosterhuis2020policy} where only $K=5$ documents are displayed to the user for a given query, and any document beyond that gets zero exposure.
To get the relevance probability, we apply the following transformation: $P(R=1 \mid q, d) = 0.25 \cdot rel(q,d)$, where $rel(q,d) \in \{0,1,2,3,4\}$ is the relevance judgment for the given query-document pair.
We generate clicks based on the trust bias click model (Assumption~\ref{assumption:trustbias}):
\begin{equation}
    P(C=1 \mid q, d, k) = \alpha_k P(R=1 \mid q, d) + \beta_k.   
\label{click_simulation}
\end{equation}
The trust bias parameters are set based on the empirical observation in \citep{agarwal2019addressing}: $\alpha = [0.35, 0.53, 0.55, 0.54, 0.52]$, and $\beta = [0.65, 0.26, 0.15$, $0.11$, $0.08]$.
For \ac{CLTR} training, we only use the training and validation clicks generated via the click simulation process (Eq.~\ref{click_simulation}).
To test the robustness of the safe \ac{CLTR} methods in a setting where the click model assumptions do not hold,
we simulate an \emph{adversarial click model}, where the user clicks on the irrelevant document with a high probability and on a relevant document with a low click probability.
We define the adversarial click model as:
\begin{equation}
    P(C=1 \mid q, d, k) = 1 - \mleft( \alpha_k P(R=1 \mid q, d) + \beta_k \mright).   
\label{click_simulation_adv}
\end{equation}
Thereby, we simulate a maximally \emph{adversarial} user who clicks on documents with a click probability that is inversely correlated with the assumed trust bias model (Assumption~\ref{assumption:trustbias}).

Further, we assume that the logging propensities have to be estimated.
For the logging propensities $\rho_0$, and the logging metric weights ($\omega_0$), we use a simple Monte Carlo estimate~\citep{gupta2023safe}:
\begin{equation}
    \hat{\rho}_0(d ) = \frac{1}{N} \sum^N_{i=1: y_i \sim \pi_{0}\hspace{-1cm}}   \alpha_{k_{i}(d)},
    \quad
    \hat{\omega}_0(d ) = \frac{1}{N} \sum^N_{i=1: y_i \sim \pi_{0}\hspace{-0.8cm}} \mleft( \alpha_{k_{i}(d)} + \beta_{k_{i}(d)} \mright).
    \label{prop-estimate1}  
\end{equation}
For the learned policies ($\pi$), we optimize \ac{PL} ranking models~\citep{oosterhuis2021computationally} using the REINFORCE policy-gradient method~\cite{yadav2021policy,gupta2023safe}.
We perform clipping on the logging propensities (Eq.~\ref{policy-aware-exposure}) only for the training clicks and not for the validation set. 
Following previous work, we set the clipping parameter to $10 / \sqrt{N}$~\cite{gupta2023safe,oosterhuis2021unifying}. 
We do not apply the clipping operation for the logging metric weights (Eq.~\ref{eq:omega_logging}).
To prevent overfitting, we apply early stopping based on the validation clicks.
For variance reduction, we follow~\cite{yadav2021policy,gupta2023safe} and use the average reward per query as a control-variate.

As our evaluation metric, we compute the NDCG@5 metric using the relevance judgments on the test split of each dataset~\citep{jarvelin2002cumulated}. 
Finally, the following methods are included in our comparisons:
\begin{enumerate}[leftmargin=*,label=(\roman*)]
    \item  \emph{IPS}. The \ac{IPS} estimator with affine correction~\cite{vardasbi2020inverse,oosterhuis2021unifying} for \ac{CLTR} with trust bias (Eq.~\ref{cltr-obj-ips-positionbias}).
    \item  \emph{Doubly Robust}. The \ac{DR} estimator for \ac{CLTR} with trust bias (Eq.~\ref{cltr-obj-dr}). 
    This is the most important baseline for this chapter, given that the \ac{DR} estimator is the state-of-the-art \ac{CLTR} method~\cite{oosterhuis2022doubly}.
    \item  \emph{Safe \ac{DR}}. Our proposed safe \ac{DR} \ac{CLTR} method (Eq.~\ref{dr-objgenbound}), which relies on the trust bias assumption (Assumption~\ref{assumption:trustbias}).
     \item  \emph{\ac{PRPO}}. Our proposed \acfi{PRPO} method for safe \ac{DR} \ac{CLTR} (Eq.~\ref{eq:prpo_obj}).
     \item  \emph{Skyline.} \ac{LTR} method trained on the true relevance labels. Given that it is trained on the real relevance signal, the skyline performance is the upper bound on any \ac{CLTR} methods performance. 
\end{enumerate}

\section{Results and Discussion}
\begin{figure}[!th]
    \centering
    {\renewcommand{\arraystretch}{0.01}%
    \setlength{\tabcolsep}{0.04cm}%
    \begin{tabular}{c r r}
        &
         \multicolumn{1}{c}{\small Yahoo! Webscope}
        &
         \multicolumn{1}{c}{\small MSLR-WEB30k}
        \\[2mm]
        \rotatebox[origin=lt]{90}{\hspace{0.77cm}\small NDCG@5} &
        \includegraphics[scale=0.475]{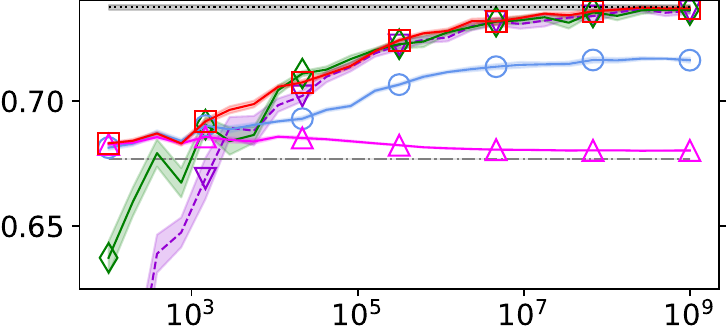} &
        \includegraphics[scale=0.475]{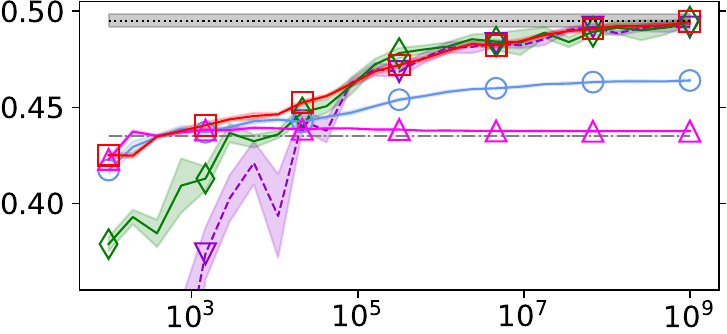}
        \\
        & \multicolumn{1}{c}{\small Number of interactions simulated ($N$)}
        & \multicolumn{1}{c}{\small Number of interactions simulated ($N$)}
        \\[2mm]
        & \multicolumn{2}{c}{\includegraphics[scale=0.5]{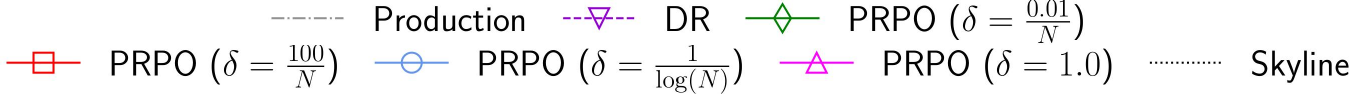}}
        \\[4mm]
        \rotatebox[origin=lt]{90}{\hspace{0.65cm} \small NDCG@5} &
        \includegraphics[scale=0.475]{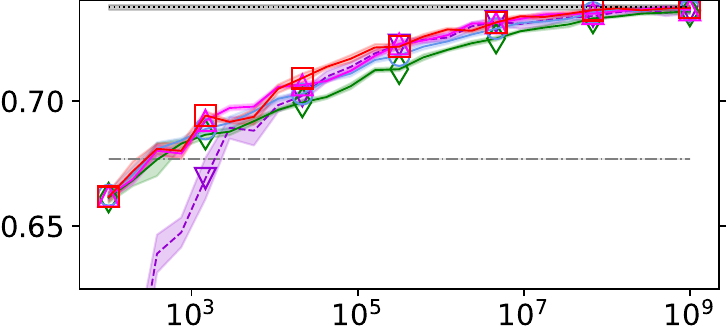} &
        \includegraphics[scale=0.475]{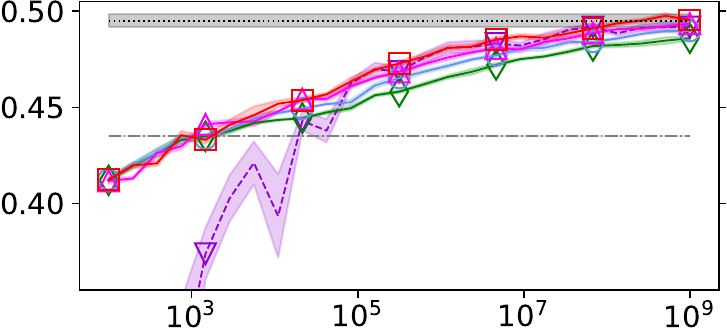}
        \\
        & \multicolumn{1}{c}{\small Number of interactions simulated ($N$)}
        & \multicolumn{1}{c}{\small Number of interactions simulated ($N$)}
        \\[2mm]
    \end{tabular}
    } %
    \vspace{0.3\baselineskip}
    \includegraphics[scale=0.5]{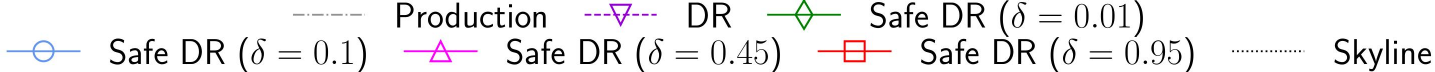}
    \caption{
        Performance of the safe \ac{DR} and \ac{PRPO} with varying safety parameter ($\delta$) on Yahoo! Webscope and MSLR-WEB30k datasets. 
        Top row: sensitivity analysis of \ac{PRPO} with varying clipping parameter ($\delta$) over varying dataset sizes $N$. 
        Bottom row: sensitivity analysis for the safe \ac{DR} method with varying safety confidence parameter ($\delta$). Results are averaged over 10 runs; shaded areas indicate $80\%$ prediction intervals.
    }
    \label{fig:ablationresults1}
\end{figure}

\noindent \textbf{Comparision with baseline methods.}
Figure~\ref{fig:mainresults_safe_part1} and \ref{fig:mainresults_safe_part2} present the main results with different \ac{CLTR} estimators with varying amounts of simulated click data.
Amongst the baselines, we see that the \ac{DR} estimator converges to the skyline much faster than the \ac{IPS} estimator. 
The \ac{IPS} estimator fails to reach the optimal performance even after training on $10^9$ clicks, suggesting that it suffers from a high-variance problem. 
This aligns with the findings in \citep{oosterhuis2022doubly}. 
As to safety, when the click data is limited ($N < 10^5$), the \ac{DR} estimator performs much worse than the logging policy, i.e., it exhibits unsafe behavior, which can lead to a negative user experience if deployed online. 
A likely explanation is that when click data is limited, the regression estimates ($\hat{R}(d)$, Eq.~\ref{cltr-obj-dr}) have high errors, resulting in a large performance degradation, compared to \ac{IPS}.

Our proposed safety methods, safe \ac{DR} and \ac{PRPO}, reach the performance of the logging policy within $\sim$500 queries on all datasets. 
For the safe \ac{DR} method, we set the confidence parameter $\delta=0.95$. For the \ac{PRPO} method, we set $\delta(N)=\frac{100}{N}$.  
On the MSLR and the ISTELLA dataset, we see that \ac{PRPO} reaches logging policy performance with almost $10^3$ fewer queries than the \ac{DR} method.  
Thus, our proposed methods, safe \ac{DR} and \ac{PRPO}, can be safely deployed, and avoid the initial period of bad performance of \ac{DR}, whilst providing the same state-of-the-art performance at convergence.

\begin{figure}[!t]
    \centering
    {\renewcommand{\arraystretch}{0.01}%
     \setlength{\tabcolsep}{0.04cm}%
     \begin{tabular}{c r}
       & \multicolumn{1}{c}{\small Istella} \\[2mm]
       \rotatebox[origin=lt]{90}{\hspace{0.77cm}\small NDCG@5} &
         \includegraphics[scale=0.475]{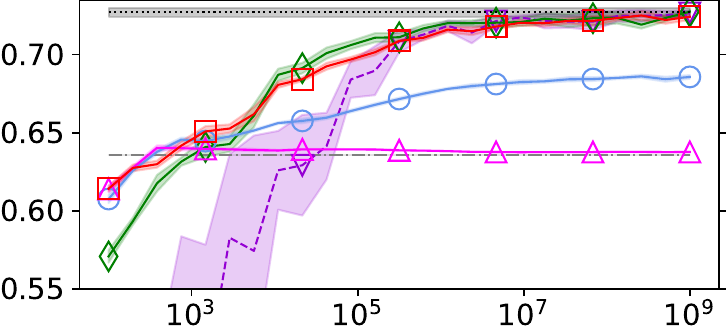} \\
       & \multicolumn{1}{c}{\small Number of interactions simulated ($N$)} \\[2mm]
     \end{tabular}
    }
    \vspace{0.3\baselineskip}
    \includegraphics[scale=0.47]{05-safety_cikm/figures/legend_prpo_main.pdf}
  \end{figure}
    \vspace{-0.1cm}
  \begin{figure}[!t]
    \centering
    {\renewcommand{\arraystretch}{0.01}%
     \setlength{\tabcolsep}{0.04cm}%
     \begin{tabular}{c r}
       \rotatebox[origin=lt]{90}{\hspace{0.65cm}\small NDCG@5} &
         \includegraphics[scale=0.475]{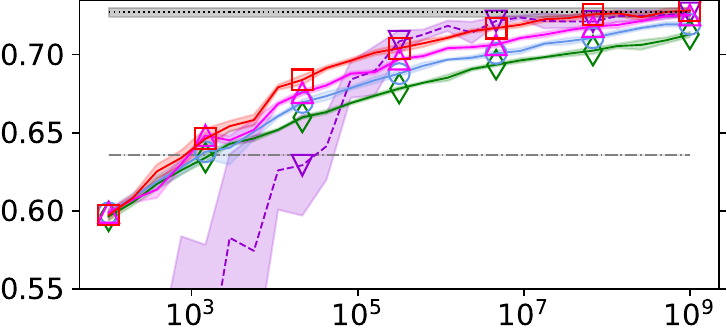} \\
       & \multicolumn{1}{c}{\small Number of interactions simulated ($N$)} \\[2mm]
     \end{tabular}
    }
    \vspace{0.3\baselineskip}
    \includegraphics[scale=0.47]{05-safety_cikm/figures/legend_risk_main.pdf}
    \caption{%
      Performance of the safe \ac{DR} and \ac{PRPO} with varying safety parameter ($\delta$) on ISTELLA dataset.
      Top row: sensitivity analysis of \ac{PRPO} with varying clipping parameter ($\delta$) over varying dataset sizes $N$.
      Bottom row: sensitivity analysis for the safe \ac{DR} method with varying safety confidence parameter ($\delta$).
      Results are averaged over 10 runs; shaded areas indicate $80\%$ prediction intervals.
    }
    \label{fig:ablationresults2}
\end{figure}

\begin{figure}[!t]
    \centering
    {\renewcommand{\arraystretch}{0.01}%
    \setlength{\tabcolsep}{0.04cm}%
    \begin{tabular}{c r r}
        &
        \multicolumn{1}{c}{\small Yahoo! Webscope}
        &
        \multicolumn{1}{c}{\small MSLR-WEB30k}
        \\
        \rotatebox[origin=lt]{90}{\hspace{0.77cm}\small NDCG@5} &
        \includegraphics[scale=0.475]{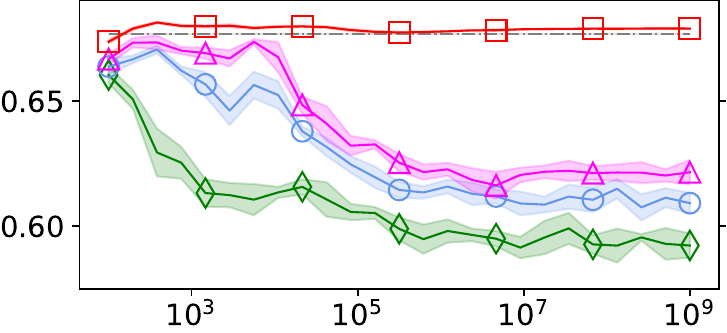} &
        \includegraphics[scale=0.475]{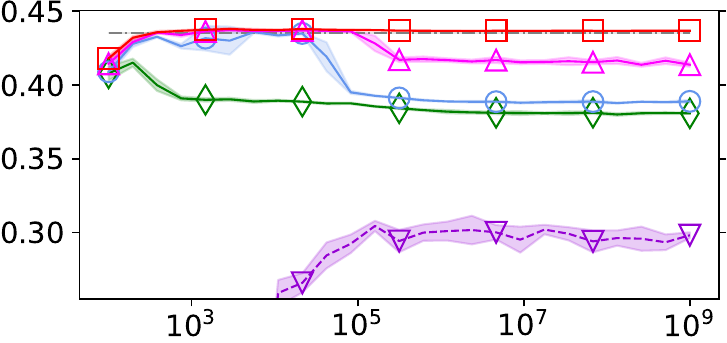}
        \\
        & \multicolumn{1}{c}{\small Number of interactions simulated ($N$)}
        & \multicolumn{1}{c}{\small Number of interactions simulated ($N$)}
        \\[4mm]
        & \multicolumn{2}{c}{\includegraphics[scale=0.5]{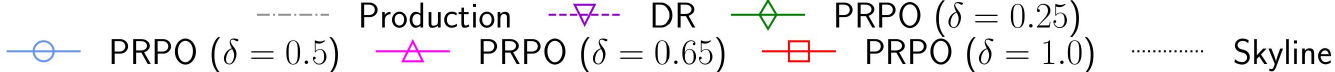}} \\[2mm]
        \rotatebox[origin=lt]{90}{\hspace{0.65cm} \small NDCG@5} &
        \includegraphics[scale=0.475]{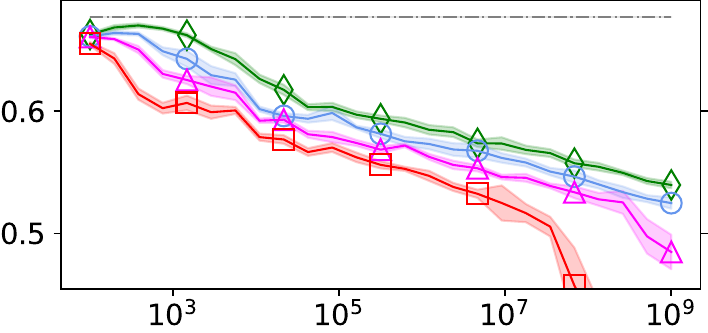} &
        \includegraphics[scale=0.475]{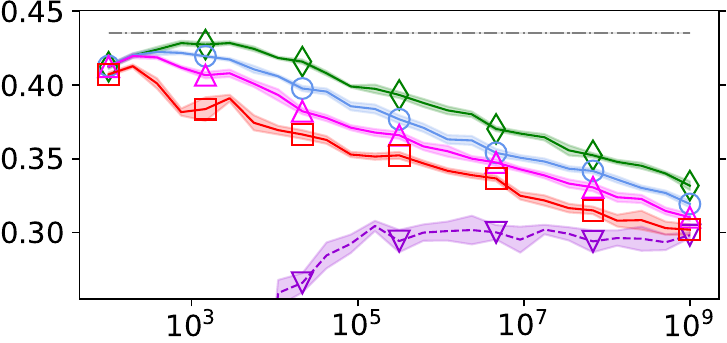}
        \\
        & \multicolumn{1}{c}{\small Number of interactions simulated ($N$)}
        & \multicolumn{1}{c}{\small Number of interactions simulated ($N$)}
        \\[2mm]
    \end{tabular}
    } %
    
    \vspace{0.3\baselineskip}
    \includegraphics[scale=0.5]{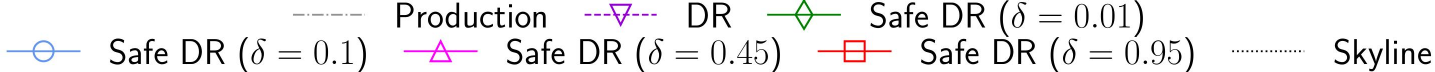}
    
    \caption{
        Performance of the proposed safe \ac{DR} and \ac{PRPO} with the adversarial click model on the Yahoo! and MSLR datasets.
        Top: sensitivity analysis results for the \ac{PRPO} method with varying clipping parameter ($\delta$).
        Bottom: sensitivity analysis for the safe \ac{DR} method with varying safety confidence parameter ($\delta$).
        Results are averaged over 10 independent runs; the shaded areas indicate $80\%$ prediction intervals.
    }
    \label{fig:ablationresults_adv1}
\end{figure}

\begin{figure}[!t]
    \centering
    {\renewcommand{\arraystretch}{0.01}%
    \setlength{\tabcolsep}{0.04cm}%
    \begin{tabular}{c r}
        & \multicolumn{1}{c}{\small Istella} \\[2mm]
        \rotatebox[origin=lt]{90}{\hspace{0.77cm}\small NDCG@5} &
        \includegraphics[scale=0.475]{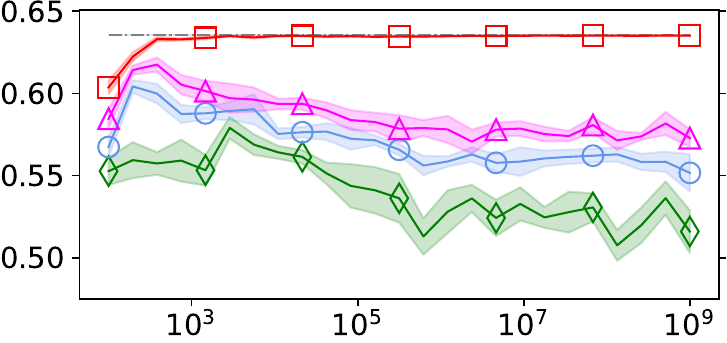} \\
        & \multicolumn{1}{c}{\small Number of interactions simulated ($N$)} \\[2mm]
    \end{tabular}
    }
    \vspace{0.3\baselineskip}
    \includegraphics[scale=0.5]{05-safety_cikm/figures/legend_abl_prpo.pdf}
    \par\bigskip
    {\renewcommand{\arraystretch}{0.01}%
    \setlength{\tabcolsep}{0.04cm}%
    \begin{tabular}{c r}
        \rotatebox[origin=lt]{90}{\hspace{0.65cm} \small NDCG@5} &
        \includegraphics[scale=0.475]{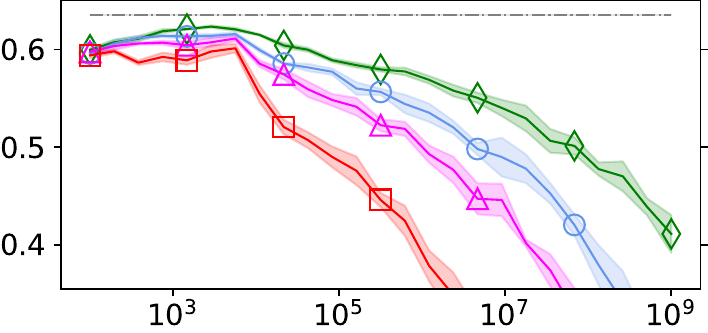} \\
        & \multicolumn{1}{c}{\small Number of interactions simulated ($N$)} \\[2mm]
    \end{tabular}
    }
    \vspace{0.3\baselineskip}
    \includegraphics[scale=0.5]{05-safety_cikm/figures/legend_abl_risk.pdf}
    \caption{
        Performance of the proposed safe \ac{DR} and \ac{PRPO} with the adversarial click model on the ISTELLA dataset.
        Top: sensitivity analysis results for the \ac{PRPO} method with varying clipping parameter ($\delta$).
        Bottom: sensitivity analysis for the safe \ac{DR} method with varying safety confidence parameter ($\delta$).
        Results are averaged over 10 independent runs; the shaded areas indicate $80\%$ prediction intervals.
    }
    \label{fig:ablationresults_adv2}
\end{figure}

\header{Sensitivity analysis of the safety parameter}
To understand the tradeoff between safety and utility, we performed a sensitivity analysis by varying the safety parameter ($\delta$) for the safe \ac{DR} method and \ac{PRPO}. 
The top rows of Figure~\ref{fig:ablationresults1} and \ref{fig:ablationresults2} show us the performance of the \ac{PRPO} method with different choices of the clipping parameter $\delta$ as a function of dataset size ($N$). 
We report results with the setting of the $\delta$ parameter, which results in different clipping widths. 
For the setting $\delta=\frac{0.01}{N}$ and $\delta = \frac{100}{N}$, the clipping range width grows linearly with the dataset size $N$. Hence, the resulting policy is safer at the start 
but converges to the \ac{DR} estimator when $N$ increases. 
With $\delta=\frac{0.01}{N}$, the clipping range is wider at the start. As a result, it is more unsafe than when $\delta=\frac{100}{N}$, which is the safest amongst all. 
For the case where the range grows logarithmically ($\delta=\frac{1}{\log(N)}$), the method is more conservative throughout, i.e., it is closer to the logging policy since the clipping window grows only logarithmically with $N$.
For the extreme case where the clipping range is a constant ($\delta=1$), \ac{PRPO} avoids any change w.r.t.\ the logging policy, and as a result, it sticks closely to the logging policy. 

The bottom rows of Figure~\ref{fig:ablationresults1} and \ref{fig:ablationresults2} show the performance of the safe \ac{DR} method with varying confidence parameter values ($\delta$). 
Due to the nature of the generalization bound (Eq.~\ref{dr-objgenbound}), the confidence parameter is restricted to: $0 \leq \delta \leq 1$.
We vary the confidence parameters in the range $\delta \in \{0.01,0.1,0.45,0.95\}$.
We note that a lower $\delta$ value results in higher safety, and vice-versa. 
Until $N < 10^5$, there is no noticeable difference in performance. 
For the Yahoo!\ Webscope dataset, almost all settings result in a similar performance. 
For the MSLR and ISTELLA datasets, when $N < 10^5$, a lower $\delta$ value results in a more conservative policy, i.e., a policy closer to the logging policy. 
However, the performance difference with different setups is less drastic than with the \ac{PRPO} method. 
Thus, we note that the safe \ac{DR} method is \emph{less flexible} in comparison to \ac{PRPO}.

Therefore, compared to our safe \ac{DR} method, we conclude that our \ac{PRPO} method provides practitioners with greater flexibility and control when deciding between safety and utility. 

\header{Robustness analysis using an adversarial click model}
To verify our initial claim that our proposed \ac{PRPO} method provides safety guarantees \emph{unconditionally}, we report results with clicks simulated via the adversarial click model (Eq.~\ref{click_simulation_adv}). 
With the adversarial click setup, the initial user behavior assumptions (Assumption~\ref{assumption:trustbias}) \emph{do not hold}. 
The top rows of Figure~\ref{fig:ablationresults_adv1} and \ref{fig:ablationresults_adv2} show the performance of the \ac{PRPO} method with different safety parameters when applied to the data collected via the adversarial click model. 
We vary the $\delta$ parameter for \ac{PRPO} in the range $\{0.25,0.5,0.65,1.0\}$, e.g.,  $\delta=0.5$ results in $\epsilon_{-}=0.5$ and $\epsilon_{+} = 2$.
With the constant clipping range ($\delta=1$), we notice that after $\sim$400 queries, the \ac{PRPO} methods performance never drops below the safe logging policy performance. 
For greater values of $\delta$, there are drops in performance but they are all bounded. 
For the Yahoo! Webscope dataset, the maximum drop in the performance is $\sim$12$\%$; for the MSLR30K dataset, the maximum performance drop is $\sim$10$\%$; and finally, for the Istella dataset, the maximum drop is $\sim$20$\%$.
Clearly, these observations show that \ac{PRPO} provides robust safety guarantees, that are reliable even when user behavior assumptions are wrong.

In contrast, the generalization bound of our safe \ac{DR} method (Theorem~\ref{CLTR-bound}) holds only when the user behavior assumptions are true. 
This is not the case in the bottom rows of Figure~\ref{fig:ablationresults_adv1} and \ref{fig:ablationresults_adv2}, which show the performance of the safe \ac{DR} method under the adversarial click model. 
Even with the setting where the safety parameters have a high weight ($\delta=0.01$), as the click data size increases, the performance drops drastically. 
Regardless of the exact choice of $\delta$, the effect of the regularization of safe \ac{DR} disappears as $N$ grows, thus in this adversarial setting, it is only a matter of time before the performance of safe \ac{DR} degrades dramatically.

\section{Conclusion}
In this chapter, we have introduced the first safe \ac{CLTR} method that uses state-of-the-art \ac{DR} estimation and corrects trust bias.
This is a significant extension of the existing safety method for \ac{CLTR} that was restricted to position bias and \ac{IPS} estimation.
However, in spite of the importance of this extended safe \ac{CLTR} approach, it heavily relies on user behavior assumptions.
We argue that this means it only provides a \emph{conditional} concept of safety, that may not apply to real-world settings.
To address this limitation, we have made a second contribution: the \acfi{PRPO} method.
\ac{PRPO} is the first \ac{LTR} method that provides \emph{unconditional} safety, that is applicable regardless of user behavior.
It does so by removing incentives to stray too far away from a safe ranking policy.
Our experimental results show that even in the extreme case of adversarial user behavior  \ac{PRPO} results in safe ranking behavior, unlike existing safe \ac{CLTR} approaches.

\ac{PRPO} easily works with existing \ac{LTR} algorithms and relevance estimation techniques.
We believe it provides a flexible and generic framework that enables practitioners to apply the state-of-the-art \ac{CLTR} method with strong and robust safety guarantees. 
Future work may apply the proposed safety methods to exposure-based ranking fairness~\cite{oosterhuis2021computationally,yadav2021policy} and to safe online \ac{LTR}~\cite{oosterhuis2021unifying}.

In this chapter, we answer the broad research question (\ref{rq:safe2}) in affirmative. 
We introduce \ac{PRPO}, a novel safe counterfactual \ac{LTR} method that does not rely on any user behavior assumptions with robust safety guarantees, even under adversarial conditions.

So far, in the first part of the thesis, we have discussed safe deployment strategies for contextual bandits with combinatorial action space -- for example, ranking for web search.
In the second part of the thesis, we will switch the discussion to contextual bandits for traditional single-action recommender systems.

\begin{subappendices}

\section{Appendix: Extended Safety Proof}\label{appendix-cikm}

\begin{lemma}
\label{cov-lemma}
Under the trust bias click model (Assumption~\ref{assumption:trustbias}), and
    given the trust bias parameter $\alpha_k, \beta_k$, the regression model estimates $\hat{R}_d$ and click indicator $c(d)$, the following holds:
    \begin{equation}
    \mathrm{Cov}_{y,c}\mleft[ c(d)  -  \beta_{k(d)}, \alpha_{k(d)}\hat{R}_d \mright] \geq 0.
    \end{equation}
\end{lemma}
\begin{proof}
\vspace{-0.25\baselineskip}
    The covariance term can be rewritten as:
    \begin{align}
        &\mathrm{Cov}_{y,c}\mleft[ c(d) -  \beta_{k(d)}, \alpha_{k(d)}\hat{R}_d \mright] \nonumber \\
         &= \mathbb{E}_{y,c} \mleft[ (c(d)  -  \beta_{k(d)}) \alpha_{k(d)}\hat{R}_d  \mright] - \mathbb{E}_{y,c} \mleft[ c(d)  -  \beta_{k(d)} \mright] \mathbb{E}_{y} \mleft[ \alpha_{k(d)}\hat{R}_d  \mright] \nonumber \\    
         &= \hat{R}_d \Big(\mathbb{E}_{y,c} \mleft[ c(d) \alpha_{k(d)} \mright]   - \mathbb{E}_{y} \mleft[ \beta_{k(d)} \alpha_{k(d)}  \mright] - R_d \; \rho_{0}(d)^2  \Big),
         \label{cov-expand}
    \end{align}
    where use $\rho_{0}(d)=\mathbb{E}_{y,c} \mleft[ \alpha_{k(d)}  \mright]$ and $\mathbb{E}_{y,c}\mleft[ (c_i(d) - \beta_{k_i(d)} )/\rho_{0}(d) \mright]  = R_d$~\cite{oosterhuis2022doubly}.
    Expanding the first expectation term in the expression:
    \begin{align}
        &\underset{y,c}{\mathbb{E}} \mleft[ c(d) \alpha_{k(d)} \mright] = \sum_{y \in \pi_0} \pi_0(y) \alpha_{k(d)} P(C=1 \mid d,y) \nonumber \\
        &= \sum_{y \in \pi_0} \pi_0(y) \alpha_{k(d)} \cdot \mleft( \alpha_{k(d)} R_d + \beta_{k(d)} \mright) = R_d \mathbb{E}_{y} \mleft[ \alpha_{k(d)}^2 \mright] + \mathbb{E}_{y} \mleft[ \alpha_{k(d)} \beta_{k(d)} \mright] \nonumber \\
        &= \sum_{y \in \pi_0} \pi_0(y) \mleft[ \alpha_{k(d)}^2 R_d + \alpha_{k(d)} \beta_{k(d)} \mright],  
    \end{align}
    where we substitute click model equation $P(C=1 \mid d,y)$ (Eq.~\ref{cltr-obj-dr}). Substituting it back in Eq.~\ref{cov-expand}, we get:
    \begin{align}
        \mathrm{Cov}_{y,c}\mleft[ c(d) -  \beta_{k(d)}, \alpha_{k(d)}\hat{R}_d \mright] &= R_d \mathbb{E}_{y} \mleft[ \alpha_{k(d)}^2 \mright] - R_d \; \mathbb{E}_{y} \mleft[ \alpha_{k(d)}  \mright]^2  \nonumber \\[-1ex]
        R_d \Big(\mathbb{E}_{y} \mleft[ \alpha_{k(d)}^2 \mright] -  \mathbb{E}_{y} \mleft[ \alpha_{k(d)}  \mright]^2 \Big) &= R_d \mathrm{Var}_{y} \mleft[ \alpha_k(d) \mright] \geq 0.    \qedhere
    \end{align}
\end{proof}

\subsection{Proof of Theorem~\ref{CLTR-bound}}\label{sec:perfbound-proof}
\renewcommand{\thetheorem}{3.4.1}
\begin{theorem}
        Given the true utility $U(\pi)$ (Eq.~\ref{true-utility}) and its exposure-based \ac{DR} estimate $\hat{U}_{\text{DR}}(\pi)$ (Eq.~\ref{cltr-obj-dr}) of the ranking policy $\pi$ with the logging policy $\pi_{0}$ and the metric weights $\omega$ and $\omega_{0}$ (Eq.~\ref{eq:omega} and \ref{eq:omega_logging}), assuming the trust bias click model (Assumption~\ref{assumption:trustbias}),
        the following generalization bound holds with probability $1 - \delta$:\footnote{The following proof differs slightly from the original proof published in~\cite{gupta-2024-practical}, as we overlooked an additional constant term in the original paper.}
        \begin{equation}
            \begin{split}
              P\Bigg(
                  U(\pi)\;\ge\;
                  \hat{U}_{\text{DR}}(\pi)
                  \;-\;
                  \Bigl(1+\max_{k}\tfrac{\beta_k}{\alpha_k}\Bigr)
                  &\Bigl(
                    \sqrt{\frac{2Z}{N}\,
                      \Big(\frac{1-\delta}{\delta}\Big)\,
                             d_2(\omega\Vert\omega_0)}
                      \\
                      &+\sqrt{\frac{1}{N}\,
                      \Big(\frac{1-\delta}{\delta}\Big)}
                  \Bigr)
              \Bigg)
              \;\ge\;
              1-\delta .
            \end{split}
          \end{equation}
    \end{theorem}
    \begin{proof} %
        As per Cantelli's inequality~\cite{ghosh2002probability}, the following inequality must hold with probability $1-\delta$:
        \begin{equation}
            U(\pi) \geq \hat{U}_{\text{DR}}(\pi) - \sqrt{ \frac{1-\delta}{\delta} \mathrm{Var}_{q,y,c}\mleft[\hat{U}_{\text{DR}}(\pi)\mright]}.
            \label{inequality_1} 
        \end{equation}
        Following a similar approach as previous works~\cite{gupta2023safe,wu2018variance}, we look for an upper-bound on the variance of the \ac{DR} estimator.
        From the definition of $\hat{U}_{\text{DR}}(\pi)$ (Eq.~\ref{cltr-obj}) and the assumption that queries $q$ are \ac{i.i.d}, the variance of the counterfactual estimator can be expanded by applying the law of total variance as follows:
        \begin{equation}
            \begin{aligned}
            \!\!\!\! \mathrm{Var}_{q,y,c}\mleft[\hat{U}_{\text{DR}}(\pi )\mright] \!\!=\!\! \frac{1}{N} \Big( \mathbb{E}_{q} \mleft[ \mathrm{Var}_{y,c}\mleft[ \hat{U}_{\text{DR}}(\pi) \mid q\mright] \mright] \!+\! \mathrm{Var}_{q} \mleft[ \mathbb{E}_{y,c}\mleft[ \hat{U}_{\text{DR}}(\pi) \mid q\mright] \mright] \Big) .
            \end{aligned}
            \label{eq:vardecom2}
        \end{equation}    
        The second term (variance over queries) can be expanded as follows:
        \begin{align}
            \mathrm{Var}_{q} \mleft[ \mathbb{E}_{y,c}\mleft[ \hat{U}_{\text{DR}}(\pi) \mid q\mright] \mright] \nonumber
            & = \mathbb{E}_{q} \mleft[ \mathbb{E}_{y,c}\mleft[ \hat{U}_{\text{DR}}(\pi) \mid q\mright] ^2 \mright] -  \mathbb{E}_{q} \mleft[ \mathbb{E}_{y,c}\mleft[ \hat{U}_{\text{DR}}(\pi) \mid q\mright] \mright] ^2  \nonumber  \\
            & \leq \mathbb{E}_{q} \mleft[ \mathbb{E}_{y,c}\mleft[ \hat{U}_{\text{DR}}(\pi) \mid q\mright] ^2 \mright] \\ 
            & = \mathbb{E}_{q} \mleft[ \mleft[ U(\pi) \mid q\mright] ^2 \mright] \\ \nonumber
            & \leq  1, \nonumber 
        \end{align}
        where in the second step, we use the unbiasedness property of the doubly robust estimator~\cite[Eq.~37]{oosterhuis2022doubly}, and used the fact that the true utility is non-zero, i.e. $U(\pi) \geq 0$.
        In the last step, we made use of the fact that the true utility is bounded, and is upper bounded by $1$ (safe to assume if the utility is normalized, for ex: normalized discounted cumulative gain, or click-through rate.), which results in the following bound for the doubly-robust variance:
        \begin{equation}
            \mathrm{Var}_{q,y,c}\mleft[\hat{U}_{\text{DR}}(\pi)\mright] \leq  \frac{1}{N} \Big( \mathbb{E}_{q} \mleft[ \mathrm{Var}_{y,c}\mleft[ \hat{U}_{\text{DR}}(\pi) \mid q\mright] \mright] + 1 \Big) . 
            \label{eq:var_new_dr}
        \end{equation}
        Now, focusing on the first part of the doubly-robust variance, from the definition of $\hat{U}_{\text{DR}}(\pi)$ (Eq.~\ref{cltr-obj-dr}), the variance of the \ac{DR} estimator (for a single query) can be expressed as the variance of the second term (Eq.~\ref{cltr-obj-dr}):
        \begin{equation}
            {\mathrm{Var}_{y,c}}\mleft[\hat{U}_{\text{DR}}(\pi )\mright] \!=\!  \frac{1}{N} \underset{y,c}{\mathrm{Var}}\mleft[ \sum_{d \in D}  \frac{\omega(d)}{\rho_{0}(d)} \big(c(d) \!-\! \alpha_{k(d)}\hat{R}(d) \!-\!  \beta_{k(d)} \big) \mright]. \!\mbox{}
            \label{eq:vardecom}
        \end{equation}
        Using Assumption~\ref{assumption:trustbias} and assuming that document examinations are independent from each other~\cite{gupta2023safe}, we rewrite further:
        \begin{equation}
        \begin{split}
            & N \cdot \underset{y,c}{\mathrm{Var}}\mleft[\hat{U}_{\text{DR}}(\pi )\mright] 
            =  \sum_{d \in D_{q}}^{} \underset{y,c}{\mathrm{Var}}\mleft[  \frac{\omega(d)}{\rho_{0}(d)} \big(c(d) - \alpha_{k(d)}\hat{R}(d) -  \beta_{k(d)} \big)  \mright] \\[-1.2ex]
            & \qquad\qquad =  \sum_{d \in D_{q}}^{} \mleft( \frac{\omega(d)}{\rho_{0}(d)} \mright)^2 \underset{y,c}{\mathrm{Var}}\mleft[ c(d)  -  \beta_{k(d)} - \alpha_{k(d)}\hat{R}(d)  \mright].
            \end{split}
            \label{var-dr}
        \end{equation}
       The total variance can be split into the following:
        \begin{align}
        &\mathrm{Var}_{y,c}\mleft[ c(d)  -  \beta_{k(d)} - \alpha_{k(d)}\hat{R}_{i}(d)  \mright]  = \mathrm{Var}_{y,c}\mleft[ \alpha_{k(d)}\hat{R}(d)  \mright]   \\
        &\;\;\;\;\;\;\;\; + \mathrm{Var}_{y,c}\mleft[ c(d)  -  \beta_{k(d)} \mright] - 2 \mathrm{Cov}_{y,c}\mleft[ c(d)  -  \beta_{k(d)}, \alpha_{k(d)}\hat{R}(d) \mright]. \nonumber
        \end{align}
        Using Lemma~\ref{cov-lemma}, we upper-bound the total variance term to:
        \begin{align}
            &\mathrm{Var}_{y,c}\mleft[ c(d)  -  \beta_{k(d)} - \alpha_{k(d)}\hat{R}(d)  \mright] \nonumber \\
            & \leq \mathrm{Var}_{y,c}\mleft[ \alpha_{k(d)}\hat{R}(d)  \mright] + \mathrm{Var}_{y,c}\mleft[ c(d)  -  \beta_{k(d)} \mright]. 
        \label{var-split}
        \end{align}
        Next, we consider the two variance terms separately; with the variance of the first term following:
    \begin{equation*}
        \underset{y,c}{\mathrm{Var}}\mleft[ \alpha_{k(d)}\hat{R}(d) \mright] = \underset{y,c}{\mathrm{Var}}\mleft[ \alpha_{k(d)} \mright] \hat{R}(d)^2 
         \leq \mathbb{E}_{y,c} \mleft[ \alpha_{k(d)}^2 \mright] \leq \mathbb{E}_{y} \mleft[ \alpha_{k(d)} \mright],
    \end{equation*}
       where we make use of the fact that $\hat{R}_d^2 \leq 1$, and $\alpha \in [0,1] \rightarrow \alpha_k^2 \leq \alpha_k$.
       Next, we consider the second term:
        \begin{align}
        &\mathrm{Var}_{y,c}\mleft[ c(d)  -  \beta_{k(d)} \mright] \leq  \mathbb{E}_{y,c} \mleft[ \big(c(d)  -  \beta_{k(d)} \big)^2 \mright] \\
        & \quad = \mathbb{E}_{y,c} \mleft[ c(d)^2  +  \beta_{k(d)}^2 - 2 c(d) \beta_{k(d)} \mright] 
        \leq \mathbb{E}_{y,c} \mleft[ c(d) \mright]   + \mathbb{E}_{y} \mleft[ \beta_{k(d)} \mright],  \nonumber
        \end{align}
       since $c(d)^2=c(d)$, $\beta_k^2 \leq \beta_k$, and $ \mathbb{E}_{y,c} \mleft[ c(d) \beta_{k(d)} \mright] \geq 0$.
        Substituting the click probabilities with Eq.~\ref{affine-click-model}, we get:
        \begin{align}
            & \mathbb{E}_{y,c} \mleft[ c(d) \mright]   + \mathbb{E}_{y,c} [ \beta_{k(d)} ] = \mathbb{E}_{y,c}[ \alpha_{k(d)}] P(R=1 |\, d) + 2 \, \mathbb{E}_{y,c} [ \beta_{k(d)}] \nonumber \\
            &\leq  \mathbb{E}_{y} \mleft[ \alpha_{k(d)} \mright] + 2 \, \mathbb{E}_{y}[ \beta_{k(d)}],
        \end{align}
        where we use the fact that $P(R=1 \mid d) \leq 1$.
        Putting together the bounds on both parts of Eq.~\ref{var-split}, we have:
        \begin{equation}
        \mathrm{Var}_{y,c}\mleft[ c(d)  -  \beta_{k(d)} - \alpha_{k(d)}\hat{R}(d)  \mright] \leq 2 \omega_0(d),
        \end{equation}
        where $\omega_0(d) = \mathbb{E}_{y} \mleft[ \alpha_{k(d)} \mright] + \mathbb{E}_{y} \mleft[ \beta_{k(d)} \mright]$. 
        Substituting the final variance upper bound in Eq.~\ref{var-dr}, we get:
        \begin{align}
             \underset{y,c}{\mathrm{Var}} & \mleft[ \sum_{d \in D}  \frac{\omega(d)}{\rho_{0}(d)} \big(c(d) - \alpha_{k(d)}\hat{R}(d) -  \beta_{k(d)} \big) \mright] \nonumber \\
            & \leq  2 \sum_{d \in D_{q}}^{} \! \mleft( \frac{\omega(d)}{\rho_{0}(d)} \mright)^2 \!\! \omega_0(d) \nonumber \\
            &=  2 \sum_{d \in D_{q}}^{} \mleft( \frac{\omega(d)}{\rho_{0}(d)} \mright)^2 \omega_0(d)  \mleft(\frac{\omega_0(d)}{\omega_0(d)}\mright)^2  \nonumber \\
            & =  2 \sum_{d \in D_{q}}^{} \mleft( \frac{\omega(d)}{\omega_{0}(d)} \mright)^2 \omega_0(d)  \mleft(\frac{\omega_0(d)}{\rho_0(d)}\mright)^2,
            \label{divergence}
        \end{align}
        where we multiply and divide by $\omega_0(d)^2$ in the third step.  
        Finally, we make use of the fact: $\frac{\omega_{0}(d)}{\rho_{0}(d)} \leq \max_{\pi_{0}} \frac{\omega_{0}(d)}{\rho_{0}(d)} \leq 1 + \max_{k} \frac{\beta_k}{\alpha_k}$,
        and
       put everything back together:
       \begin{align}
        N \cdot \mathrm{Var}_{y,c}\mleft[\hat{U}_{\text{DR}}(\pi )\mright] & \leq  2 Z \mleft( 1 + \max_{k} \frac{ \beta_k}{\alpha_k} \mright)^2 \sum_{d \in D_{q}}^{} \mleft( \frac{\omega'(d)}{\omega_{0}'(d)} \mright)^2 \omega_0'(d) \nonumber \\[-1.2ex]
        &= 2 Z \mleft( 1 + \max_{k} \frac{ \beta_k}{\alpha_k} \mright)^2 d_2(\omega \,\Vert\, \omega_0).
    \end{align}
    where $d_2(\omega \,\Vert\, \omega_0)$ is the Renyi divergence between the normalized expected exposure $\omega'(d)$ and $\omega_{0}'(d)$ (cf.\ Eq.~\ref{eq:actionbaseddiv}).
    Next, we replace the variance with the Renyi divergence-based term, and substituting back into the upper-bound on variance in Eq.~\ref{inequality_1} results in the following:
    \begin{equation}
        \!\!\! U(\pi)\ge \hat{U}_{\text{DR}}(\pi)
                -(1+\max_{k}\tfrac{\beta_k}{\alpha_k})
                \bigl(
                   \sqrt{\tfrac{2Z}{N}\left(\tfrac{1-\delta}{\delta}\right)d_2(\omega\Vert\omega_0)
                   +\tfrac{1}{N}\left(\tfrac{1-\delta}{\delta}\right)}
                \bigr)
                \!\!\;\ge\;\!\!
        1-\delta .
    \end{equation}
    By applying the Cauchy–Schwarz inequality, we get:
    \begin{equation}
     \!\!\! U(\pi)\ge \hat{U}_{\text{DR}}(\pi)
                -(1+\max_{k}\tfrac{\beta_k}{\alpha_k})
                \bigl(
                   \sqrt{\tfrac{2Z}{N}\left(\tfrac{1-\delta}{\delta}\right)d_2(\omega\Vert\omega_0)}
                   +\sqrt{\tfrac{1}{N}\left(\tfrac{1-\delta}{\delta}\right)}
                \bigr)
                \!\! \;\ge\;\!\!
        1-\delta .
        \end{equation}
    This completes the proof.
    \end{proof}

\subsection{Proof of Theorem~\ref{PRPO-proof}}\label{sec:prpo-proof}
    \begin{proof}
    Given a logging policy ranking $y_0$, a user defined metric weight $\omega$, and non-zero $r(d \mid q)$, for the choice of the clipping parameters $\epsilon_{-} = \epsilon_{+} = 1$, 
    the ranking $y^*(\epsilon_{-},\epsilon_{+})$ that maximizes the \ac{PRPO} objective (Eq.~\ref{eq:prpo_obj}) will be the same as the logging ranking $y_0$, i.e. $y^*(\epsilon_{-},\epsilon_{+})=y_0$.
    This is trivial to prove since any change in ranking can only lead in a decrease in the clipped ratio weights, and thus, a decrease in the \ac{PRPO} objective.
    Therefore, $y^*(\epsilon_{-}=1,\epsilon_{+}=1)=y_0$ when $\epsilon_{-} = \epsilon_{+} = 1$.
    Accordingly: $| U(y_0) -  U(y^*(\epsilon_{-}=1, \epsilon_{+}=1)) | = 0$ directly implies Eq.~\ref{eq:prpotheorem}.
    This completes our proof.
\end{proof}

\noindent Whilst the above proof is performed in the extreme case where $\epsilon_{-} = \epsilon_{+} = 1$ and the optimal ranking has the same utility as the logging policy ranking,
other choices of $\epsilon_{-}$ and $\epsilon_{+}$ bound the difference in utility to a lesser degree and allow for more deviation.
As our experimental results show, the power of PRPO is that it gives practitioners direct control over this maximum deviation.

\end{subappendices}

\part{Robust and Efficient Reinforcement Learning for Recommendation and Diffusion Models}

\chapter{Optimal Baseline Corrections for Off-policy Contextual Bandits}
\label{chapter:01-online-evaluation3}

\footnote[]{This chapter was published as~\citep{gupta-2024-optimal}.} 
So far, the first part of this thesis has examined contextual bandits with \textit{combinatorial} action spaces -- for example, web-search ranking and slate recommendation. 
The second part shifts attention to contextual bandits that select \textit{a single action}, such as top-1 recommendations or reinforcement-learning-based fine-tuning of foundation models. 
This chapter zooms in on the top-1 recommendation case.

Our aim is to increase sample efficiency in off-policy evaluation and learning from logged user interactions. 
Although inverse propensity scoring (IPS) is unbiased in expectation, it suffers from high variance~\cite{saito2021counterfactual,gupta2023safe}. 
Variance-reduction techniques -- most notably the doubly robust (DR) estimator~\cite{oosterhuis2022doubly} and self-normalized IPS (SNIPS)~\cite{Swaminathan2015} -- mitigate this problem through additive and multiplicative baseline corrections, respectively~\cite{Swaminathan2015,joachims2018deep}. 
Yet, the literature still lacks a unifying lens on these approaches, which motivates the following research question:

\begin{enumerate}[label=\textbf{RQ3},ref={RQ\arabic*},resume]
\item \acl{rq:recsys1}%
\end{enumerate}

\noindent To address \textbf{RQ3}, we introduce the $\beta$-IPS estimator, which places IPS, DR, and SNIPS inside a single baseline-correction framework. This, in turn, raises a second question:

\begin{enumerate}[label=\textbf{RQ4},ref={RQ\arabic*},resume]
\item \acl{rq:recsys}%
\end{enumerate}

\noindent Within the unified $\beta$-IPS framework, we ask whether a variance-optimal baseline $\beta^{*}$ can be derived analytically. 
In the latter part of this chapter we show that it can, presenting a closed-form expression for $\beta^{*}$ that minimizes variance for both off-policy evaluation and learning.

\section{Introduction \& Motivation}
Recommender systems have undergone a paradigm shift in the last few decades, moving their focus from \emph{rating} prediction in the days of the Netflix Prize~\cite{Bennet2007}, to \emph{item} prediction from implicit feedback~\cite{Rendle2022} and \emph{ranking} applications gaining practical importance~\cite{Steck2013,Jeunen2023_nDCG}.
Recently, work that applies ideas from the algorithmic \emph{decision-making} literature to recommendation problems has become more prominent~\cite{Vasile2020,Saito2021,gupta2024unbiased, jeunen2022consequences}.
While this line of research is not inherently new~\cite{Shani2002,li2010contextual}, methods based on contextual bandits (or reinforcement learning by extension) have now become widespread in the recommendation field~\cite{McInerney2018,Mehrotra2020,Bendada2020, Jeunen2021_TopK, Yi2023, Su2024, Briand2024}.
The \emph{off-policy} setting is particularly attractive for practitioners~\cite{vandenAkker2024}, as it allows models to be trained and evaluated in an offline manner~\cite{chen2019top, Dong2020, ma2020off, Jeunen2020, Jeunen2021_Pessimism, Chen2021, Liu2022, chen2022actorcritic, Jeunen2023,gupta2023safe,gupta2023ictir,gupta-2024-practical,gupta-2023-first-abstract}.
Indeed, methods exist to obtain unbiased \emph{offline} estimators of \emph{online} reward metrics, which can then be optimized directly~\cite{Jeunen2021_Thesis}.

Research at the forefront of this area typically aims to find Pareto-optimal solutions to the bias-variance trade-off that arises when choosing an estimator: reducing variance by accepting a small bias~\cite{Ionides2008,su2020doubly}, by introducing control variates~\cite{Dudik2014,Swaminathan2015}, or both~\cite{Su2019}.
Control variates are especially attractive as they (asymptotically) preserve the unbiasedness of the widespread inverse propensity scoring (IPS) estimator.
Additive control variates give rise to baseline corrections~\cite{greensmith2004variance}, regression adjustments~\cite{Freedman2008}, and doubly robust estimators~\cite{Dudik2014}.
Multiplicative control variates lead to self-normalised estimators~\cite{Kong1992,Swaminathan2015}.
Previous work has proven that for off-policy \emph{learning} tasks, the multiplicative control variates can be re-framed using an equivalent additive variate~\cite{joachims2018deep,Budylin2018}, enabling mini-batch optimization methods to be used.
We note that the self-normalised estimator is only \emph{asymptotically} unbiased: a clear disadvantage for evaluation with finite samples.
The common problem which most existing methods tackle is that of \emph{variance reduction} in offline value estimation, either for learning or for evaluation.
The common solution is the application of a control variate, either multiplicative or additive~\cite{Owen2013}.
However, to the best of our knowledge, there is no  work that attempts to unify these methods.
Our work in this chapter addresses this gap by presenting these methods in a unifying framework of baseline corrections which, in turn, allows us to find the optimal baseline correction for variance reduction.

In the context of off-policy learning, adding to the well-known equivalence between reward-translation and self-normalisation described by \citet{joachims2018deep}, we demonstrate that the equivalence extends to baseline corrections, regression adjustments, and doubly robust estimators with a constant reward model.
Further, we derive a novel baseline correction method for off-policy learning that minimizes the variance of the gradient of the (unbiased) estimator. We further show that the baseline correction can be estimated in a closed-form fashion, allowing for easy practical implementation.

In line with recent work on off-policy evaluation/learning for recommendation~\cite{Jeunen2021_TopK,Jeunen2023_AuctionGym,Jeunen2021_Pessimism,rohde2018recogym,Saito2021_OPE}, we adopt an off-policy simulation environment to emulate real-world recommendation scenarios, such as stochastic rewards, large action spaces, and controlled randomisation.
This choice also encourages future reproducibility~\cite{saito2020open}.
Our experimental results indicate that our proposed baseline correction for gradient variance reduction enables substantially faster convergence and lower gradient variance during learning.

In addition, we derive a closed-form solution to the optimal baseline correction for off-policy evaluation, i.e., the one that minimizes the variance of the estimator itself. 
Importantly, since our framework only considers unbiased estimators, the variance-optimality implies overall optimality.
Our experimental results show that this leads to lower errors in policy value estimation than widely used doubly-robust and SNIPS estimators~\cite{Dudik2014, Swaminathan2015}.

All source code to reproduce our experimental results is available at: \url{https://github.com/shashankg7/recsys2024_optimal_baseline}.

\section{Background and Related Work}\label{sec:sec1}
The goal of this section is to introduce common contextual bandit setups for recommendation, both on-policy and off-policy.

\subsection{On-policy contextual bandits}
\label{sec:general}

We address a general contextual bandit setup~\cite{saito2022counterfactual,joachims2016counterfactual} with contexts $X$, actions $A$, and rewards $R$.
The context typically describes \emph{user} features, actions are the \emph{items} to recommend, and rewards can be any type of \emph{interaction} logged by the platform.
A policy $\pi$ defines a conditional probability distribution over actions $x$: $\mathsf{P}(A=a\mid X=x,\Pi=\pi) \equiv \pi(a \mid x)$.
Its \emph{value} is the expected reward it yields:
\begin{equation}\label{eq:onpolicy_reward}
    V(\pi) = \mathop{\mathbb{E}}_{x\sim\mathsf{P}(X)}\Big[\mathop{\mathbb{E}}_{a \sim \pi(\cdot \mid x)}\mleft[ R \mright]\Big].
\end{equation}
When the policy $\pi$ is deployed, we can estimate this quantity by averaging the rewards we observe.
We denote the expected reward for action $a$ and context $r$ as $r(a,x) \coloneqq \mathbb{E}[R \mid X=x;A=a]$.

In the field of contextual bandits (and reinforcement learning (RL) by extension), one often wants to learn $\pi$ to maximise $V(\pi)$~\cite{sutton2018reinforcement,lattimore2020bandit}.
This is typically achieved through gradient ascent.
Assuming $\pi_{\theta}$ is parameterised by $\theta$, we iteratively update with learning rate $\eta$:
\begin{equation}
    \theta_{t+1} = \theta_{t} + \eta \nabla_{\theta}(V(\pi_{\theta})).
\end{equation}
Using the well-known REINFORCE ``log-trick''~\cite{Williams1992}, the above gradient can be formulated as an expectation over sampled actions, whereby tractable Monte Carlo estimation is made possible:
\begin{align}
    \nabla_{\theta}(V(\pi_{\theta})) &=  \nabla_{\theta}\mleft( \mathop{\mathbb{E}}_{x\sim\mathsf{P}(X)}\mleft[\mathop{\mathbb{E}}_{a \sim \pi_{\theta}(\cdot\mid x)}\mleft[ R \mright]\mright] \mright) \nonumber \\
    &= \nabla_{\theta}\mleft( \int\sum_{a \in \mathcal{A}} \pi_{\theta}(a \mid x) r(a,x) \mathsf{P}(X=x) \rm{d}x   \mright) \nonumber \\
    &=  \int\sum_{a \in \mathcal{A}} \nabla_{\theta}\mleft( \pi_{\theta}(a \mid x) r(a,x) \mright)\mathsf{P}(X=x)\rm{d}x \label{eq:mc_sample}   \\
    &=  \int\sum_{a \in \mathcal{A}} \pi_{\theta}(a \mid x) \nabla_{\theta}\mleft( \log(\pi_{\theta}(a \mid x)) r(a,x) \mright) \mathsf{P}(X=x)\rm{d}x \nonumber \\
    &=  \mathop{\mathbb{E}}_{x\sim\mathsf{P}(X)}\mleft[\mathop{\mathbb{E}}_{a \sim \pi_{\theta}(\cdot\mid x)}\mleft[ \nabla_{\theta}\mleft( \log(\pi_{\theta}(a\mid x)) R \mright) \mright]\mright]. \nonumber
\end{align}
This provides an unbiased estimate of the gradient of $V(\pi_{\theta})$. However, it may be subject to high variance due to the inherent variance of $R$.
Several techniques have been proposed in the literature that aim to alleviate this, mostly using additive \emph{control variates}.

Control variates are random variables with a known expectation~\cite[\S 8.9]{Owen2013}.
If the control variate is correlated with the original estimand -- in our case $V(\pi_{\theta})$ -- they can be used to reduce the estimator's variance.
A natural way to apply control variates to a sample average estimate for Eq.~\ref{eq:onpolicy_reward} is to estimate a model of the reward $\widehat{r}(a,x)\approx \mathbb{E}[R|X=x;A=a]$ and subtract it from the observed rewards~\cite{Freedman2008}.
This is at the heart of key RL techniques (i.a., generalised advantage estimation~\cite{Schulmanetal_ICLR2016}), and it underpins widely used methods to increase sensitivity in online controlled experiments~\cite{Deng2013,Poyarkov2016,Budylin2018,Baweja2024}.
As such, it applies to both \emph{evaluation} and \emph{learning} tasks.
We note that if the model $\widehat{r}(a,x)$ is biased, this bias propagates to the resulting estimator for $V(\pi_{\theta})$.

Alternatively, instead of focusing on reducing the variance of $ V(\pi_{\theta})$ directly, other often-used approaches tackle the variance of its gradient estimates $\nabla_{\theta}(V(\pi_{\theta}))$ instead.

Observe that $\mathop{\mathbb{E}}_{a \sim \pi_{\theta}(\cdot|x)}\mleft[ \nabla_{\theta}\mleft( \log(\pi_{\theta}(a\mid x))\mright) \mright]=0$~\cite[Eq. 12]{mohamed2020monte}.
This implies that a translation on the rewards in Eq.~\ref{eq:mc_sample} does not affect the unbiasedness of the gradient estimate.
Nevertheless, as such a translation can be framed as an additive control variate, it will affect its variance.
Indeed, ``\emph{baseline corrections}'' are a well-known variance reduction method for on-policy RL methods~\cite{greensmith2004variance}.
For a dataset consisting of logged contexts, actions and rewards $\mathcal{D} = \{(x_i,a_i,r_i)_{i=1}^{N}\}$, we apply a \emph{baseline} control variate $\beta$ to the estimate of the final gradient to obtain:
\begin{equation}
\begin{split}
    \nabla_{\theta}(V(\pi_{\theta}))
    &\approx
    \widehat{\nabla_{\theta}(V_{\beta}(\pi_{\theta}))} 
    \\
    &=
    \frac{1}{\mleft|\mathcal{D}\mright|}
    \sum_{(x,a,r) \in \mathcal{D}}   (r-\beta) \nabla_{\theta} \log \pi_{\theta}(a \mid x).
\end{split}    
    \label{eq:onpolicy_grad}
\end{equation}
\citet{Williams1988} originally proposed to use the average observed reward for $\beta$. Subsequent work has derived optimal baselines for general on-policy RL scenarios~\cite{Dayan1991, greensmith2004variance}.
However, to the best of our knowledge, \emph{optimal baselines for on-policy contextual bandits have not been considered in previous work}.

\noindent
\textbf{Optimal baseline for on-policy bandits.}
The optimal baseline $\beta$ for the on-policy gradient estimate in Eq.~\ref{eq:onpolicy_grad} is the one that minimizes the variance of the gradient estimate.
In accordance with earlier work~\cite{greensmith2004variance},
we define the variance of a vector random variable as the sum of the variance of its individual components.
Therefore, the optimal baseline is given by:
\begin{align}
    & \argmin_\beta \mathrm{Var} \mleft( \widehat{\nabla_{\theta}(V_{\beta}(\pi_{\theta}))} \mright) \nonumber \\
    &= \argmin_\beta  \frac{1}{\mleft|\mathcal{D}\mright|} \mathop{\mathrm{Var}}\mleft[ \nabla_{\theta}\mleft( \log(\pi_{\theta}(a \mid x)) \mleft( r - \beta \mright) \mright) \mright]  \\
    &= \argmin_\beta \frac{1}{\mleft|\mathcal{D}\mright|} \mathop{\mathbb{E}}\mleft[  \nabla_{\theta} \log(\pi_{\theta}(a \mid x))^{\top} \nabla_{\theta} \log(\pi_{\theta}(a \mid x))  \mleft( r - \beta \mright)^2 \mright] \hspace{-1em} \label{eq:step2} \\
    & \quad - \frac{1}{\mleft|\mathcal{D}\mright|}  \mathop{\mathbb{E}}\mleft[  \nabla_{\theta} \log(\pi_{\theta}(a \mid x)) \mleft( r - \beta \mright) \ \mright]^{\top} \mathop{\mathbb{E}}\mleft[  \nabla_{\theta} \log(\pi_{\theta}(a \mid x)) \mleft( r - \beta \mright) \mright] \nonumber  \\
    &= \argmin_\beta  \frac{1}{\mleft|\mathcal{D}\mright|} \mathop{\mathbb{E}}\mleft[ \| \nabla_{\theta} \log(\pi_{\theta}(a \mid x))\|^2_2 \mleft( r -\beta \mright)^2  \mright],
    \label{eq:onpolicy_optimal_baseline}
\end{align}
where we ignore the second term in Eq.~\ref{eq:step2}, since it is independent of $\beta$~\cite[Eq. 12]{mohamed2020monte}.
The result from this derivation (Eq.~\ref{eq:onpolicy_optimal_baseline}) reveals that the optimal baseline can be obtained by solving the following equation:
\begin{align}
    \mbox{}\hspace*{-2mm}
    \frac{\partial \mathrm{Var}\big(\widehat{\nabla_{\theta}(V_{\beta}(\pi_{\theta}))} \big)}{\partial \beta} &= 
    \frac{2}{\mleft|\mathcal{D}\mright|} \mathop{\mathbb{E}}\mleft[  |\nabla_{\theta} \log(\pi_{\theta}(a \mid x))|^2_2 \mleft( \beta -  r  \mright)  \mright] \!= 0,
    \hspace*{-1mm}\mbox{}
\end{align}
which results in the following optimal baseline correction:
\begin{equation}
    \beta^{*} =  \frac{\mathop{\mathbb{E}}\mleft[  |\nabla_{\theta} \log(\pi_{\theta}(a \mid x))|^2_2 r(a,x) \mright]}{\mathop{\mathbb{E}}\mleft[  |\nabla_{\theta} \log(\pi_{\theta}(a \mid x))|^2_2 \mright]},
\end{equation}
and the empirical estimate of the optimal baseline correction:
\begin{equation}
    \widehat{\beta^{*}} =  \frac{\sum_{(x,a,r) \in \mathcal{D}} \mleft[  |\nabla_{\theta} \log(\pi_{\theta}(a \mid x))|^2_2 r(a,x) \mright]}{\sum_{(x,a,r) \in \mathcal{D}}\mleft[  |\nabla_{\theta} \log(\pi_{\theta}(a \mid x))|^2_2 \mright]}.
\end{equation}
This derivation follows the more general derivation from \citet{greensmith2004variance} for partially observable Markov decision processes (POMDPs). We have not encountered its use in the existing bandit literature applied to recommendation problems. In Section~\ref{sec:grad_var}, we show that a similar line of reasoning can be applied to derive a variance-optimal gradient for the off-policy contextual bandit setup.

\subsection{Off-policy estimation for general bandits}
\label{sec:general_OPE}

Deploying $\pi$ is a costly prerequisite for estimating $V(\pi)$, that comes with the risk of deploying a possible poorly valued $\pi$.
Therefore, commonly in real-world model validation pipelines, practitioners wish to estimate $V(\pi)$ \emph{before} deployment.
Accordingly, we will address this \emph{counterfactual} evaluation scenario that falls inside the field of off-policy estimation (OPE)~\cite{Saito2021_OPE,Vasile2020}. 

The expectation $V(\pi)$ can be unbiasedly estimated using samples from a \emph{different} policy $\pi_{0}$ through \emph{importance sampling}, also known as inverse propensity score weighting (IPS)~\cite[\S 9]{Owen2013}:
\begin{equation}\label{eq:imp_sampl}
    \mathop{\mathbb{E}}_{x\sim\mathsf{P}(X)}\mleft[\mathop{\mathbb{E}}_{a \sim \pi(\cdot|x)}\mleft[ R \mright]\mright]
    = \mathop{\mathbb{E}}_{x\sim\mathsf{P}(X)}\mleft[\mathop{\mathbb{E}}_{a \sim \pi_{0}(\cdot|x)}\mleft[ \frac{\pi(a\mid x)}{\pi_{0}(a\mid x)} R \mright]\mright].
\end{equation}
To ensure that the so-called \emph{importance weights} $\frac{\pi(a\mid x)}{\pi_{0}(a\mid x)}$ are well-defined, we assume ``\emph{common support}'' by the logging policy: $\forall a \in \mathcal{A}, x \in \mathcal{X}: \pi(a\mid x) > 0 \implies \pi_{0}(a\mid x) > 0$.

From Eq.~\ref{eq:imp_sampl}, we can derive an unbiased estimator for $V(\pi)$ using contexts, actions and rewards logged under $\pi_{0}$, denoted by $\mathcal{D}$:
\begin{equation}\label{eq:IPS}
    \widehat{V}_{\rm IPS}(\pi,\mathcal{D}) = \frac{1}{\mleft|\mathcal{D}\mright|} \sum_{(x,a,r) \in \mathcal{D}} \frac{\pi(a\mid x)}{\pi_{0}(a\mid x)} r.
\end{equation}
To keep our notation brief, we suppress subscripts when they are clear from the context.
In the context of gradient-based optimization methods, we often refer to a minibatch $\mathcal{B} \subset \mathcal{D}$ instead of the whole dataset, as is typical for, e.g., stochastic gradient descent (SGD).

If we wish to learn a policy that maximises this estimator, we need to estimate its gradient for a batch $\mathcal{B}$.
Whilst some previous work has applied a REINFORCE estimator~\cite{chen2019top,ma2020off,chen2022actorcritic}, we use a straightforward Monte Carlo estimate for the gradient:
\begin{equation}\label{eq:MC_gradient_IPS}
    \nabla \widehat{V}_{\rm IPS}(\pi, \mathcal{B}) = \frac{1}{\mleft|\mathcal{B}\mright|} \sum_{(x,a,r) \in \mathcal{B}} \frac{\nabla\pi(a \mid x)}{\pi_{0}(a \mid x)} r.
\end{equation}
Importance sampling -- the bread and butter of unbiased off-policy estimation -- often leads to increased variance compared to on-policy estimators.
Several variance reduction techniques have been proposed specifically to combat the excessive variance of $\widehat{V}_{\rm IPS}$~\cite{Ionides2008,Dudik2014,Swaminathan2015}.
Within the scope of this chapter, we only consider techniques that reduce variance \emph{without} introducing bias.

\noindent
\textbf{Self-normalised importance sampling.}
The key idea behind \emph{self-normalisation}~\cite[\S 9.2]{Owen2013} is to use a \emph{multiplicative} control variate to rescale $\widehat{V}_{\rm IPS}(\pi,\mathcal{D})$.
An important observation for this approach is that for any policy $\pi$ and a dataset $\mathcal{D}$ logged under $\pi_{0}$, the expected average of importance weights should equal 1~\cite[\S 5]{Swaminathan2015}:
\begin{equation}\label{eq:control_variate}
   \mathop{\mathbb{E}}_{\mathcal{D} \sim \mathsf{P}(\mathcal{D})}\mleft[
   \frac{1}{\mleft|\mathcal{D}\mright|} \sum_{(x,a,r) \in \mathcal{D}}\frac{\pi(a\mid x)}{\pi_{0}(a\mid x)}
   \mright] = 1.
\end{equation}
Furthermore, as this random variable (Eq.~\ref{eq:control_variate}) is likely to be correlated with the IPS estimates, we can expect that its use as a control variate will lead to reduced variance (see~\cite[e.g.,][]{Kong1992}).
This gives rise to the asymptotically unbiased and parameter-free self-normalised IPS (SNIPS) estimator, with $ S \coloneqq  \frac{1}{D} \sum_{(x,a,r) \in \mathcal{D}} \frac{\pi(a \mid x)}{\pi_{0}(a \mid x)}$ as its normalization term:
\begin{equation}\label{eq:SNIPS}
    \widehat{V}_{\rm SNIPS}(\pi,\mathcal{D}) =  \frac{\sum_{(x,a,r) \in \mathcal{D}} \frac{\pi(a \mid x)}{\pi_{0}(a \mid x)} r}{\sum_{(x,a,r) \in \mathcal{D}} \frac{\pi(a \mid x)}{\pi_{0}(a \mid x)}}
    =
    \frac{\widehat{V}_{\rm IPS}(\pi,\mathcal{D})}{S}.
\end{equation}
Given the properties of being asymptotically unbiased and para\-meter-free, this estimator is often a go-to method for off-policy \emph{evaluation} use-cases~\cite{Saito2021_OPE}.
An additional advantage is that the SNIPS estimator is invariant to translations in the reward, which cannot be said for $\widehat{V}_{\rm IPS}$.
Whilst the formulation in Eq.~\ref{eq:SNIPS} is not obvious in this regard, it becomes clear when we consider its gradient:
\begin{align}\label{eq:SNIPS_gradient}
    &\nabla \widehat{V}_{\rm SNIPS}(\pi,\mathcal{D}) =  \nabla\mleft(\frac{\sum_{(x,a,r)} \frac{\pi(a \mid x)}{\pi_{0}(a \mid x)} r}{\sum_{(x,a)} \frac{\pi(a \mid x)}{\pi_{0}(a \mid x)}}\mright)  \nonumber \\ 
 &\qquad\qquad\qquad =\frac{\mleft( \sum_{(x,a,r)} \frac{\nabla \pi(a \mid x)}{\pi_{0}(a \mid x)} r  \mright) \mleft( \sum_{(x,a)} \frac{ \pi(a \mid x)}{\pi_{0}(a \mid x)}  \mright) }{\mleft({\sum_{(x,a)} \frac{ \pi(a \mid x)}{\pi_{0}(a \mid x)}}\mright)^2}  \nonumber \\
 & \quad  - \frac{\mleft( \sum_{(x,a,r)} \frac{ \pi(a \mid x)}{\pi_{0}(a \mid x)} r  \mright) \mleft( \sum_{(x,a)} \frac{\nabla \pi(a \mid x)}{\pi_{0}(a \mid x)}  \mright) }{\mleft({\sum_{(x,a)} \frac{  \pi(a \mid x)}{\pi_{0}(a \mid x)}}\mright)^2} \\
 &= \frac{\sum_{(x_{i},a_{i}, r_{i})} \sum_{(x_{j},a_{j}, r_{j})}\frac{\pi(a_{i}\mid x_{i})\nabla\pi(a_{j}\mid x_{j})}{\pi_{0}(a_{i}\mid x_{i})\pi_{0}(a_{j}\mid x_{j})}(r_{j}-r_{i}) }{\mleft({\sum_{(x,a)} \frac{ \pi(a \mid x)}{\pi_{0}(a \mid x)}}\mright)^2}\nonumber  \\ 
 &= \frac{\sum\limits_{(x_{i},a_{i}, r_{i})} \sum\limits_{(x_{j},a_{j}, r_{j})}\frac{\pi(a_{i} \mid x_{i})\pi(a_{j} \mid x_{j})}{\pi_{0}(a_{i} \mid x_{i})\pi_{0}(a_{j} \mid x_{j})} \nabla\log\pi(a_{j} \mid x_{j})(r_{j}-r_{i}) }{\mleft({\sum_{(x,a)} \frac{ \pi(a \mid x)}{\pi_{0}(a \mid x)}}\mright)^2}.\nonumber
\end{align}
Indeed, as the SNIPS gradient relies on the \emph{relative difference} in observed reward between two samples, a constant correction would not affect it (i.e., if $\overline{r} = r - \beta$, then $r_j-r_i \equiv \overline{r}_{j}-\overline{r}_{i}$).

\citet{Swaminathan2015} effectively apply the SNIPS estimator (with a variance regularisation term~\cite{swaminathan2015batch}) to off-policy \emph{learning} scenarios.
Note that while $\widehat{V}_{\rm IPS}$ neatly decomposes into a single sum over samples, $\widehat{V}_{\rm SNIPS}$ no longer does.
Whilst this may be clear from the gradient formulation in Eq.~\ref{eq:SNIPS_gradient}, a formal proof can be found in~\cite[App. C]{joachims2018deep}.
This implies that mini-batch optimization methods (which are often necessary to support learning from large datasets) are no longer directly applicable to $\widehat{V}_{\rm SNIPS}$.

\citet{joachims2018deep} solve this by re-framing the task of maximising $\widehat{V}_{\rm SNIPS}$ as an optimization problem on $\widehat{V}_{\rm IPS}$ with a constraint on the self-normalisation term.
That is, if we define:
\begin{equation}
\mbox{}\hspace*{-2mm}
    \pi^{\star} \!=\! \argmax_{\pi \in \Pi} \widehat{V}_{\rm SNIPS}(\pi, \mathcal{D}), \text{ with } S^{\star} \!=\! \frac{1}{\mleft|\mathcal{D}\mright|}
    \!\sum_{(x,a,r) \in \mathcal{D}} \!
    \frac{\pi^{\star}(a\mid x)}{\pi_{0}(a \mid x)},
\end{equation}
then, we can equivalently state this as:
\begin{equation}\label{eq:constr_opt}
    \pi^{\star} = \argmax_{\pi \in \Pi} \widehat{V}_{\rm IPS}(\pi, \mathcal{D}), \text{ s.th. } \frac{1}{\mleft|\mathcal{D}\mright|}\sum_{(x,a,r) \in \mathcal{D}} \frac{\pi(a \mid x)}{\pi_{0}(a \mid x)} = S^{\star}.
\end{equation}
\citet{joachims2018deep} show via the Lagrange multiplier method that this optimization problem can be solved by optimising for $\widehat{V}_{\rm IPS}$ with a translation on the reward:
\begin{equation}\label{eq:banditnet}
\begin{split}
    \pi^{\star} =  \argmax_{\pi \in \Pi} \widehat{V}_{\lambda^{\star}\text{-}{\rm IPS}}(\pi, \mathcal{D}), \text{where }\\ 
    \widehat{V}_{\lambda\text{-}{\rm IPS}}(\pi,\mathcal{D}) = \frac{1}{\mleft|\mathcal{D}\mright|} \sum_{(x,a,r) \in \mathcal{D}} \frac{\pi(a \mid x)}{\pi_{0}(a \mid x)} (r-\lambda).
\end{split}
\end{equation}
This approach is called BanditNet~\cite{joachims2018deep}.
Naturally, we do not know $\lambda^{\star}$ beforehand (because we do not know $S^{\star}$), but we do know that $S^{\star}$ should concentrate around 1 for large datasets (see Eq.~\ref{eq:control_variate}).
\citet{joachims2018deep} essentially propose to treat $\lambda$ as a hyper-parameter to be tuned in order to find $S^{\star}$.

\noindent
\textbf{Doubly robust estimation. }
Another way to reduce the variance of $\widehat{V}_{\rm IPS}$ is to use a model of the reward $\widehat{r}(a,x)\approx \mathbb{E}[R|X=x;A=a]$.
Including it as an additive control variate in Eq.~\ref{eq:IPS} gives rise to the doubly robust (DR) estimator, deriving its name from its unbiasedness if \emph{either} the logging propensities $\pi_{0}$ or the reward model $\widehat{r}$ is unbiased~\cite{Dudik2014}:
\begin{align}\label{eq:DR}
    &\widehat{V}_{\rm DR}(\pi,\mathcal{D}) = \\
    &\;\;\; \frac{1}{\mleft|\mathcal{D}\mright|} \sum_{(x,a,r) \in \mathcal{D}} \mleft(\frac{\pi(a\mid x)}{\pi_{0}(a\mid x)} (r-\widehat{r}(a,x)) + \sum_{a^{\prime} \in \mathcal{A}} \pi(a^{\prime}\mid x)\widehat{r}(a^{\prime},x)\mright). \nonumber
\end{align}
Several further extensions have been proposed in the literature: one can optimize the reward model $\widehat{r}(a,x)$ to minimize the resulting variance of $\widehat{V}_{\rm DR}$~\cite{Farajtabar2018}, further parameterise the trade-off relying on $\widehat{V}_{\rm IPS}$ or  $\widehat{r}(a,x)$~\cite{Su2019}, or shrink the IPS weights to minimize a bound on the MSE of the resulting estimator~\cite{su2020doubly}.
One disadvantage of this method, is that practitioners are required to fit the secondary reward model $\widehat{r}(a,x)$, which might be costly and sample inefficient.
Furthermore, variance reduction is generally not guaranteed, and stand-alone $\widehat{V}_{\rm IPS}$ can be empirically superior in some scenarios~\cite{Jeunen2020REVEAL}.

\section{Unifying Off-Policy Estimators}
Section~\ref{sec:sec1} provides an overview of (asymptotically) unbiased estimators for the value of a policy.
We have introduced the contextual bandit setting, detailing often used variance reduction techniques for both on-policy (i.e., regression adjustments and baseline corrections) and off-policy estimation (i.e., self-normalisation and doubly robust estimation).
In this section, we demonstrate that they perform equivalent optimization as baseline-corrected estimation.
Subsequently, we characterize the baseline corrections that either minimize the variance of the estimator, or that of its gradient.

\subsection{A unified off-policy estimator}

\textbf{Baseline corrections for $\nabla \widehat{V}_{{\rm IPS}}(\pi,\mathcal{D})$.}
Baseline corrections are common in on-policy estimation, but occur less often in the off-policy literature.
The estimator is obtained by removing a baseline control variate $\beta \in \mathbb{R}$ from the reward of each action, while also adding it to the estimator:
\begin{equation}    
    \widehat{V}_{\beta\text{-}{\rm IPS}} = \beta + \frac{1}{\mleft|\mathcal{D}\mright|}\sum_{(x,a,r) \in \mathcal{D}} \frac{\pi(a \mid x)}{\pi_{0}(a \mid x)} (r-\beta).
    \label{eq:beta_ips_ope}
\end{equation} 
Its unbiasedness is easily verified:
\begin{equation}
\begin{split}
    \mathop{\mathbb{E}} \mleft[\widehat{V}_{\beta\text{-}{\rm IPS}} \mright] &= \mathop{\mathbb{E}} \mleft[\beta \mright] + \mathop{\mathbb{E}} \mleft[\frac{\pi(a \mid x)}{\pi_{0}(a \mid x)} (r-\beta)  \mright] \\ 
    &= \beta + \mathop{\mathbb{E}} \mleft[\frac{\pi(a \mid x)}{\pi_{0}(a \mid x)} r  \mright] - \beta \qquad = V(\pi).
\end{split}
\label{eq:betaIPSunbiased}
\end{equation}
From an optimization perspective, we are mainly interested in the gradient of the $\widehat{V}_{{\rm \beta\text{-}IPS}}$ objective:
\begin{equation}\label{eq:MC_gradient_bIPS}
    \nabla \widehat{V}_{\beta\text{-}{\rm IPS}}(\pi,\mathcal{B}) = \frac{1}{\mleft|\mathcal{B}\mright|} \sum_{(x,a,r) \in \mathcal{B}}  \frac{ \nabla \pi(a \mid x)}{\pi_{0}(a \mid x)} \mleft(r - \beta \mright) .
\end{equation}
Our key insight is that SNIPS and certain doubly-robust estimators have an equivalent gradient to the proposed $\beta$-IPS estimator.
As a result, optimizing them is equivalent to optimizing $\widehat{V}_{\beta\text{-}{\rm IPS}}$ for a specific $\beta$ value.

\noindent\textbf{Self-normalisation through BanditNet and $\widehat{V}_{\lambda\text{-}{\rm IPS}}(\pi,\mathcal{D})$.}
If we consider the optimization problem for SNIPS that is solved by BanditNet in Eq.~\ref{eq:banditnet}~\cite{joachims2018deep}, we see that its gradient is given by:
\begin{equation}
    \label{eq:banditnet_grad}
    \nabla \widehat{V}_{\lambda\text{-}{\rm IPS}}(\pi,\mathcal{B}) = \frac{1}{\mleft|\mathcal{B}\mright|} \sum_{(x,a,r) \in \mathcal{B}}  \frac{ \nabla \pi(a \mid x)}{\pi_{0}(a \mid x)} \mleft(r - \lambda \mright).
\end{equation}

\noindent \textbf{Doubly robust estimation via $\widehat{V}_{\rm DR}(\pi,\mathcal{D})$}.
As mentioned, a nuisance of doubly robust estimators is the requirement of fitting a regression model $\widehat{r}(a,x)$.
Suppose that we instead treat $\widehat{r}$ as a single scalar hyper-parameter, akin to the BanditNet approach.
Then, the gradient of such an estimator would be given by:
\begin{equation}
    \nabla \widehat{V}_{\widehat{r}\text{-}{\rm DR}}(\pi,\mathcal{B}) = \frac{1}{\mleft|\mathcal{B}\mright|} \sum_{(x,a,r) \in \mathcal{B}}  \frac{ \nabla \pi(a \mid x)}{\pi_{0}(a \mid x)} \mleft(r - \widehat{r} \mright).
    \label{eq:dr_grad_const_rew}
\end{equation}

\noindent%
Importantly, these three approaches are motivated through entirely different lenses: minimizing gradient variance, applying a multiplicative control variate to reduce estimation variance, and applying an additive control variate to improve robustness.
But they result in equivalent gradients, and thus, in equivalent optima.
Specifically, for optimization, the estimators are equivalent when $\beta\equiv\lambda\equiv\widehat{r}$.

This equivalence implies that the choice between these three approaches is not important.
Since the simple baseline correction estimator $\widehat{V}_{\beta\text{-}{\rm IPS}}$ (Eq.~\ref{eq:beta_ips_ope}) has an equivalence with all SNIPS estimators and all doubly-robust estimators with a constant reward, we propose that $\widehat{V}_{\beta\text{-}{\rm IPS}}$ should be seen as an estimator that unifies all three approaches.
Accordingly, we argue that the real task is to find the optimal $\beta$ value for $\widehat{V}_{\beta\text{-}{\rm IPS}}$, since this results in an estimator that is at least as optimal as any estimator in the underlying families of estimators, and possibly superior to them.

The remainder of this section describes the optimal $\beta$ values for minimizing gradient variance and estimation value variance.

\subsection{Minimizing gradient variance}
\label{sec:grad_var}
Similar to the on-policy variant derived in Eq.~\ref{eq:onpolicy_optimal_baseline}, we can derive the optimal baseline in the off-policy case as the one which results in the minimum variance for the gradient estimate given by Eq.~\ref{eq:MC_gradient_IPS}:
\begin{align}
    \argmin_\beta \mathrm{Var} & \mleft( \nabla_{\theta} (\widehat{V}_{\beta\text{-}{\rm IPS}}(\pi_{\theta},\mathcal{B}))  \mright) \\
    &= \argmin_\beta  \frac{1}{\mleft|\mathcal{B}\mright|}  \mathop{\mathrm{Var}}\mleft[ \frac{ \nabla\pi(a \mid x)}{\pi_{0}(a \mid x)}  \mleft(r - \beta\mright)  \mright] \\
    &= \argmin_\beta  \frac{1}{\mleft|\mathcal{B}\mright|}  \mathop{\mathbb{E}}\mleft[  \| \nabla \pi(a \mid x))\|^2_2  \Big( \frac{r - \beta}{\pi_{0}(a \mid x)} \Big)^2 \ \mright] \label{eq:grad_var_step2} \\
    &\qquad\qquad\qquad - \frac{1}{\mleft|\mathcal{B}\mright|} \|  
 \mathop{\mathbb{E}}\mleft[  \frac{ \nabla\pi(a \mid x)}{\pi_{0}(a \mid x)}  \mleft(r - \beta\mright)  \mright]  \|^2_2 \nonumber \\
    &=\argmin_\beta  \frac{1}{\mleft|\mathcal{B}\mright|}  \mathop{\mathbb{E}}\mleft[  \frac{ \| \nabla \pi(a \mid x))\|^2_2}{\pi_{0}(a \mid x)^2}  \mleft( r -\beta \mright)^2  \ \mright] \label{eq:quadratic}, 
\end{align}
where we can ignore the second term of the variance in Eq.~\ref{eq:grad_var_step2}, since it is independent of $\beta$~\cite[Eq. 12]{mohamed2020monte}. The optimal baseline can be obtained by solving for:
\begin{equation}
    \frac{\partial \mathrm{Var}\mleft(\nabla (\widehat{V}_{\beta\text{-}{\rm IPS}}(\pi,\mathcal{B}))\mright)}{\partial \beta} = 
    \frac{2}{\mleft|\mathcal{B}\mright|}  \mathop{\mathbb{E}}\mleft[ \frac{ \| \nabla \pi(a \mid x))\|^2_2}{\pi_{0}(a \mid x)^2} \mleft( \beta -  r  \mright) \mright] = 0,
\label{eq:first_order_condition}
\end{equation}
which results in the following optimal baseline:
\begin{equation}\label{eq:optimal_offpolicy}
    \beta^{*} =  \frac{\mathop{\mathbb{E}}\limits_{x, a \sim \pi_{0}, r}\mleft[ \frac{ |\nabla \pi(a\mid x))\|^2_2}{\pi_{0}(a\mid x)^2} r(a,x) \mright]}{\mathop{\mathbb{E}}\limits_{x, a \sim \pi_{0}, r}\mleft[  \frac{ |\nabla \pi(a \mid x))|^2_2}{\pi_{0}(a\mid x)^2} \mright]},
\end{equation}
with its empirical estimate given by:
\begin{equation}
    \widehat{\beta^{*}} =  \frac{\sum_{(x,a,r) \in \mathcal{B}} \mleft[  \frac{ |\nabla \pi(a \mid x))|^2_2}{\pi_{0}(a\mid x)^2} r \mright]}{\sum_{(x,a,r) \in \mathcal{B}}\mleft[  \frac{ |\nabla \pi(a\mid x))\|^2_2}{\pi_{0}(a\mid x)^2} \mright]}.
\end{equation}

\noindent%
Note that this expectation is over actions sampled by the \emph{logging} policy.
As a result, we can obtain Monte Carlo estimates of the corresponding expectations.
The derivation has high similarity with the on-policy case (cf.\ Section~\ref{sec:general})~\cite{greensmith2004variance}.
Nevertheless, we are unaware of any work on off-policy learning that uses it.
\citet{joachims2018deep} refer to the on-policy variant with: ``\emph{we cannot sample new roll-outs from the current policy under consideration, which means we cannot use the standard variance-optimal estimator used in REINFORCE}.''
Since the expectation is over actions sampled by the \emph{logging} policy and not the \emph{target} policy, we have shown that we do not need new roll-outs.
Thereby, our estimation strategy is a novel off-policy approach that estimates the variance-optimal baseline.

\begin{theorem}
\label{thrm:min_grad_var}
Within the family of gradient estimators with a global additive control variate, i.e., $\beta$-IPS (Eq.~\ref{eq:MC_gradient_bIPS}), IPS (Eq.~\ref{eq:MC_gradient_IPS}), BanditNet (Eq.~\ref{eq:banditnet_grad}), and DR with a constant correction (Eq.~\ref{eq:dr_grad_const_rew}), $\beta$-IPS with our proposed choice of $\beta$ in Eq.~\ref{eq:optimal_offpolicy} has minimal gradient variance.
\end{theorem}
\begin{proof}
Eq.~\ref{eq:first_order_condition} shows that the $\beta$ value in Eq.~\ref{eq:optimal_offpolicy} attains a minimum.
Because the variance of the gradient estimate (Eq.~\ref{eq:grad_var_step2}) is a quadratic function of $\beta$, and hence a convex function (Eq.~\ref{eq:quadratic}), it must be the global minimum for the gradient variance. 
\end{proof}

\subsection{Minimizing estimation variance}
\label{sec:estimator_var}

Besides minimizing gradient variance, one can also aim to minimize the variance of estimation, i.e., the variance of the estimated value.
We note that the $\beta$ value for minimizing estimation does not need to be the same value that minimizes gradient variance.
Furthermore, since $\widehat{V}_{\beta\text{-}{\rm IPS}}$ is unbiased, any estimation error will entirely be driven by variance.
As a result, the value for $\beta$ that results in minimal variance will also result in minimal estimation error:
\begin{align}
    & \argmin_\beta \mathrm{Var} \mleft( \widehat{V}_{\beta\text{-}{\rm IPS}}(\pi, \mathcal{D}) \mright) \\
    &\qquad = \argmin_\beta  \frac{1}{\mleft|\mathcal{D}\mright|} \mathop{\mathrm{Var}}\mleft[ \frac{ \pi(a \mid x)}{\pi_{0}(a \mid x)}  \mleft(r - \beta\mright)  \mright] \\
    &\qquad = \argmin_\beta \frac{1}{\mleft|\mathcal{D}\mright|}  \mathop{\mathbb{E}}\mleft[ \Big( \frac{ \pi(a \mid x)}{\pi_{0}(a \mid x)}  \mleft(r - \beta\mright) \Big)^2 \ \mright]\\
    & \qquad\qquad\qquad\qquad\qquad - \frac{1}{\mleft|\mathcal{D}\mright|} \Big(  
 \mathop{\mathbb{E}}\mleft[  \frac{ \pi(a \mid x)}{\pi_{0}(a \mid x)}  \mleft(r - \beta\mright)  \mright]  \Big)^2 \nonumber \\
    &\qquad = \argmin_\beta \frac{1}{\mleft|\mathcal{D}\mright|}  \mathop{\mathbb{E}}\mleft[  \mleft(\frac{ \pi(a \mid x)}{\pi_{0}(a \mid x)}\mright) ^2 \mleft( r -\beta \mright)^2  \ \mright]\\
    & \qquad\qquad\qquad\qquad\qquad  - \frac{1}{\mleft|\mathcal{D}\mright|} \Big(  
 \mathop{\mathbb{E}}\mleft[  \frac{ \pi(a \mid x)}{\pi_{0}(a \mid x)}  r  \mright] - \beta  \Big)^2.  \nonumber %
\end{align}
The minimum is obtained by solving for the following equation:
\begin{align}
    &\frac{\partial \mleft(\mathrm{Var} \mleft( \widehat{V}_{\beta\text{-}{\rm IPS}}(\pi, \mathcal{D}) \mright)\mright)}{\partial \beta} \label{eq:gradientzeroMSE}  \\
    &\quad\; = \frac{2}{\mleft|\mathcal{D}\mright|}  \mathop{\mathbb{E}}\mleft[  \mleft(\frac{\pi(a \mid x)}{\pi_{0}(a \mid x)}\mright) ^2\mleft( \beta -  r  \mright)  \ \mright] - \frac{2}{\mleft|\mathcal{D}\mright|} \mleft( \beta - \mathop{\mathbb{E}}\mleft[  \frac{ \pi(a \mid x)}{\pi_{0}(a \mid x)}  r  \mright] \mright) = 0, \nonumber
\end{align}
which results in the following optimal baseline:
\begin{align}\label{eq:optimal_offpolicy_beta}
    \beta^{*} &=  \frac{\mathop{\mathbb{E}}\mleft[  \mleft(\mleft(\frac{ \pi(a \mid x)}{\pi_{0}(a \mid x)}\mright) ^2  - \frac{ \pi(a \mid x)}{\pi_{0}(a \mid x)}\mright) r(a,x)  \mright]}{\mathop{\mathbb{E}}\mleft[  \mleft(\frac{ \pi(a \mid x)}{\pi_{0}(a \mid x)}\mright) ^2 - \mleft(\frac{ \pi(a \mid x)}{\pi_{0}(a \mid x)}\mright) \mright]}.
\end{align}
We can estimate $\beta^{*}$ using logged data, resulting in a Monte Carlo estimate of the optimal baseline.
Such a sample estimate will not be unbiased (because it is a ratio of expectations), but the bias will vanish asymptotically (similar to the bias of the $\widehat{V}_{\rm SNIPS}$ estimator). 

Next, we formally prove that optimal estimator variance leads to overall optimality (in terms of the MSE of the estimator).

\begin{theorem}
Within the family of offline estimators with a global additive control variate, i.e., $\beta$-IPS (Eq.~\ref{eq:beta_ips_ope}), IPS (Eq.~\ref{eq:IPS}), and DR with a constant correction (Eq.~\ref{eq:dr_grad_const_rew}), $\beta$-IPS with our proposed $\beta$ in Eq.~\ref{eq:optimal_offpolicy_beta} has the minimum mean squared error (MSE): %
 \begin{equation}
     {\rm MSE}(\hat{V}(\pi)) = \mathbb{E}_{\mathcal{D}} \left[ (\hat{V}(\pi, \mathcal{D}) - V(\pi))^2 \right].
     \label{eq:estimator_mse}
\end{equation}
\end{theorem}
\begin{proof}
The MSE of any off-policy estimator $\hat{V}(\pi, \mathcal{D})$ can be decomposed in terms of the bias and variance of the estimator~\cite{su2020doubly}:
 \begin{equation}
     \text{MSE}(\hat{V}(\pi)) = \text{Bias}(\hat{V}(\pi), \mathcal{D})^2 + \text{Variance}(\hat{V}(\pi), \mathcal{D}),
\end{equation}
where the bias of the estimator is defined as:
\begin{equation}
\text{Bias}(\hat{V}(\pi), \mathcal{D}) = \left| \mathbb{E}_{\mathcal{D}} \mleft[ \hat{V}(\pi, \mathcal{D}) - V(\pi)\mright]  \right|,
\end{equation}
and the variance of the estimator is defined previously (see Section~\ref{sec:estimator_var}).
Eq.~\ref{eq:betaIPSunbiased} proves that $\beta$-IPS is unbiased: $\text{Bias}(\hat{V}(\pi), \mathcal{D})=0$.
Thus, the minimum variance (Eq.~\ref{eq:gradientzeroMSE}) implies minimum MSE.
\end{proof}

\noindent%
We note that SNIPS is not covered by this theorem, as it is only asymptotically unbiased.
As a result, the variance reduction brought on by SNIPS might be higher than that by $\beta$-IPS, but as it introduces bias, its estimation error (MSE) is not guaranteed to be better.
Our experimental results below indicate that our method is always at least as good as SNIPS, and outperforms it in most cases, in both learning and evaluation tasks.

\begin{figure*}[t]
    \centering
    \includegraphics[width=\linewidth]{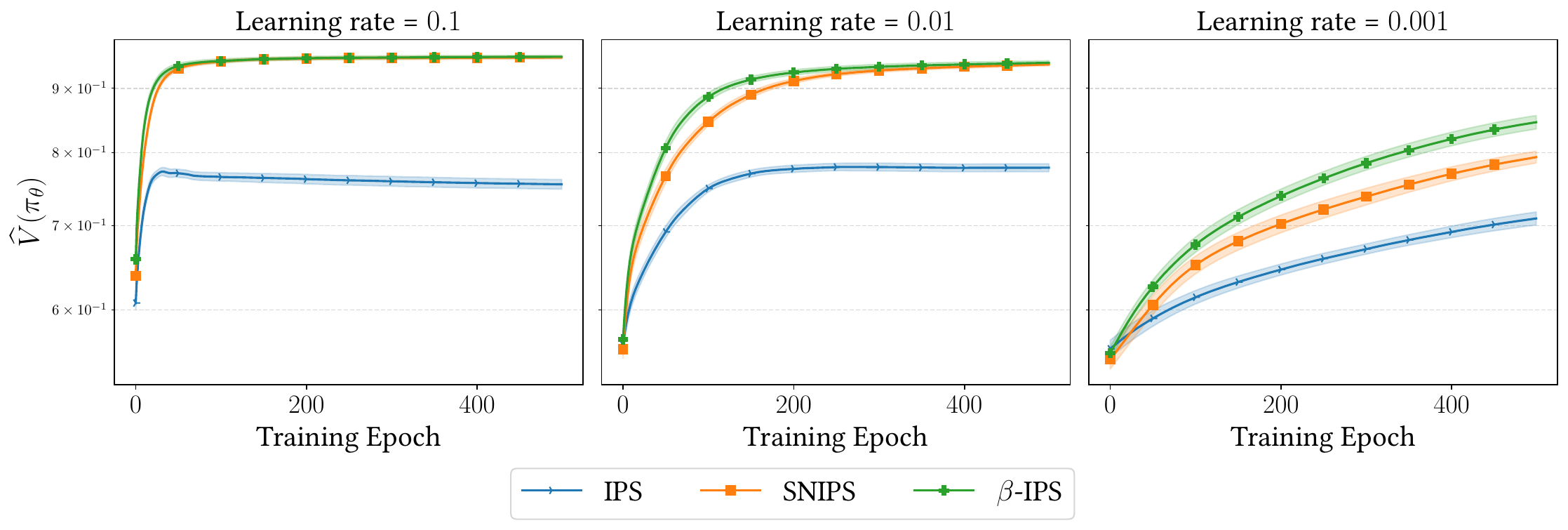}
    \caption{Performance of different off-policy learning methods trained in a full-batch gradient descent fashion in terms of the policy value on the test set. x-axis corresponds to the training epoch during the optimization (we use a maximum of 500 epochs for all methods), and y-axis corresponds to the policy value.
    A decaying learning rate is used.
    Reported results are averages over 32 independent runs with 95\% confidence interval. }
\label{fig:full_batch_val}
\end{figure*}
\begin{figure*}[th]
    \centering
    \includegraphics[width=\linewidth]{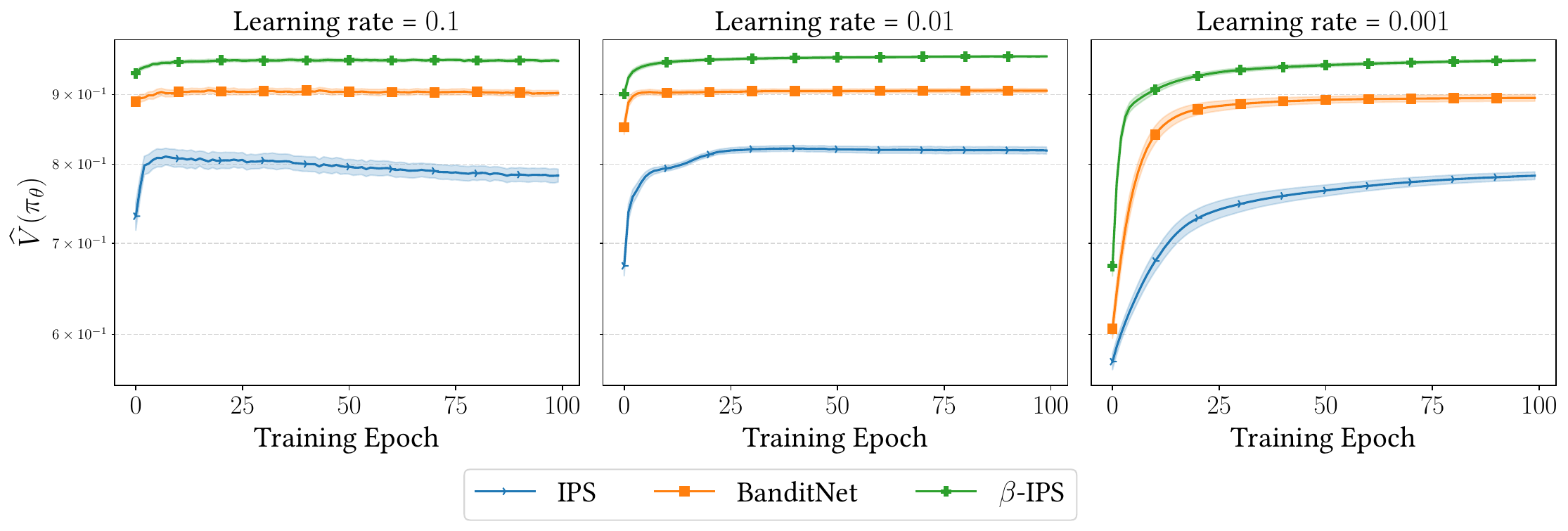}
    \caption{Performance of different off-policy learning methods trained in a mini-batch gradient descent fashion in terms of policy value on the test set. The axis labels are similar to Figure~\ref{fig:full_batch_val}.}
\label{fig:mini_batch_val}
\end{figure*}

\section{Experimental Setup}
In order to evaluate off-policy learning and evaluation methods, we need access to logged data sampled from a stochastic policy involving logging propensities (exact or estimated) along with the corresponding context and action pairs.
Recent work that focuses on off-policy learning or evaluation for contextual bandits in recommender systems follows a supervised-to-bandit conversion process to simulate a real-world bandit feedback dataset~\cite{Jeunen2021_TopK,Jeunen2021_Pessimism, Jeunen2023_AuctionGym, Saito2021_OPE,su2020doubly,Su2019}, or conducts a live experiment on actual user traffic to evaluate the policy in an \emph{on-policy} or \emph{online} fashion~\cite{chen2019top,chen2022actorcritic}.
In this work, we adopt the Open Bandit Pipeline (OBP) to simulate, in a reproducible manner, real-world recommendation setups with stochastic rewards, large action spaces, and controlled randomization~\cite{rohde2018recogym}.
Although the Open Bandit Pipeline simulates a generic offline contextual bandit setup, there is a strong correspondence to real-world recommendation setups where the environment context vector corresponds to the user context and the actions correspond to the items recommended to the user.
Finally, the reward corresponds to the user feedback received on the item (click, purchase, etc.).
As an added advantage, the simulator allows us to conduct experiments in a \emph{realistic} setting where the logging policy is sub-optimal to a controlled extent, the logged data size is limited, and the action space is large.
In addition, we conduct experiments with real-world recommendation logs from the OBP for off-policy evaluation.\footnote{\url{https://research.zozo.com/data.html}} 

The research questions we answer with our experimental results in this chapter are:
\begin{enumerate}[label=\textbf{RQC\arabic*},leftmargin=*]
    \item Does the proposed \emph{estimator-variance-minimizing} baseline correction (Eq.~\ref{eq:optimal_offpolicy_beta}) improve off-policy learning (OPL) in a full-batch setting?
    \item Does the proposed \emph{gradient-variance-minimizing} baseline correction (Eq.~\ref{eq:optimal_offpolicy}) improve OPL in a mini-batch setting? 
    \item How does the proposed \emph{gradient-variance-minimizing} baseline correction (Eq.\\ \ref{eq:optimal_offpolicy}) affect gradient variance during OPL? 
    \item Does the proposed \emph{estimator-variance-minimizing} baseline correction (Eq.~\ref{eq:optimal_offpolicy_beta}) improve off-policy evaluation (OPE) performance?
\end{enumerate}
\begin{figure*}
    \centering
    \includegraphics[width=\linewidth]{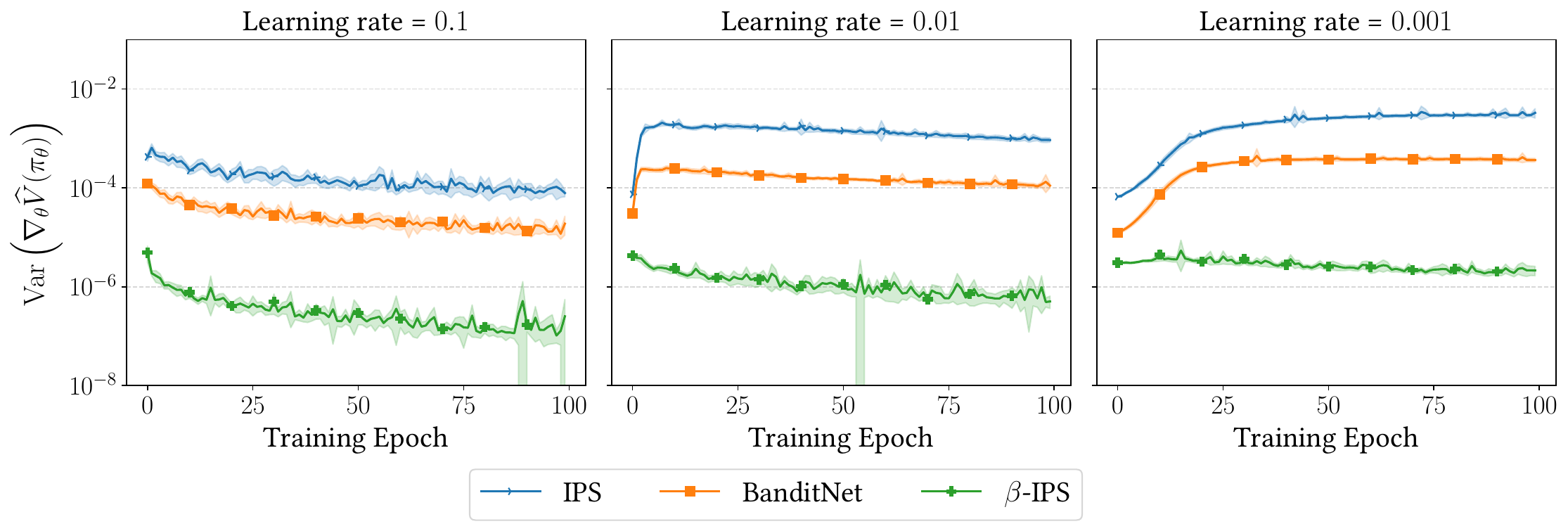}
    \caption{Empirical variance of the gradient of different off-policy learning estimators in a mini-batch optimization setup with varying learning rates (in title). We compute gradient variance for each mini-batch during training and then report the average value across all mini-batches in a training epoch. Results are averaged across 32 independent runs with 95\% confidence interval.}
\label{fig:grad_var}
\end{figure*}
\newpage

\section{Results and Discussion}
\label{recsys-results-disc}
\subsection{Off-policy learning performance (RQC1--3) }
To evaluate the performance of the proposed $\beta$-IPS method on an OPL task, we consider two learning setups:
\begin{enumerate}[leftmargin=*]
    \item  \emph{Full-batch}. In this setup, we  directly optimize the $\beta$-IPS policy value estimator (Eq.~\ref{eq:beta_ips_ope}) with the optimal baseline correction, which minimizes the variance of the value (Eq.~\ref{eq:optimal_offpolicy_beta}). Given that the optimal baseline correction involves a ratio of two expectations, optimizing the value function directly via a mini-batch stochastic optimization is not possible for the same reason as the SNIPS estimator, i.e., it is not possible to get an unbiased gradient estimate with a ratio function~\cite{joachims2018deep}. Therefore, for this particular setting,  we use a \emph{full-batch} gradient descent method for the optimization, where the gradient is computed over the entire training dataset. 
    \item  \emph{Mini-batch}. In this setup, we focus on optimizing the $\beta$-IPS policy value estimator with the baseline correction, which minimizes the gradient estimate (Eq.~\ref{eq:optimal_offpolicy}). This setup translates to a traditional machine learning training setup where the model is optimized in a stochastic mini-batch fashion. 
\end{enumerate}

\noindent \textbf{Full-batch.} The results for the full-batch training in terms of the policy value on the test set are reported in Figure~\ref{fig:full_batch_val}, over the number of training epochs. %
To minimize the impact of external factors, we use a linear model without bias, followed by a softmax to generate a distribution over all actions, given a context vector $x$ (this is a common setup, see~\cite[e.g.,][]{Jeunen2020,Jeunen2020REVEAL,Sakhi2020}).
We note that the goal of this chapter is not to get the maximum possible policy value on the test set but rather to evaluate the effect of baseline corrections on gradient and estimation variance. The simple model setup allows us to easily track the empirical gradient variance, given that we have only one parameter vector. 

An advantage of the full-batch setup is that we can compute the gradient of the SNIPS estimator directly~\cite{Swaminathan2015}.
SNIPS is a natural baseline method to consider, along with the traditional IPS estimator.
Because of practical concerns, we only consider 500 epochs of optimization.
Additionally, we use the state-of-the-art and widely used Adam optimizer~\cite{kingma2014adam}. 

The IPS method converges to a lower test policy value in comparison to the SNIPS and the proposed $\beta$-IPS methods, even after 500 epochs.
A likely reason is the high-variance of the IPS estimator~\cite{Dudik2014}, which can cause it to get stuck in  bad local minima. 

The methods with a control variate, i.e., SNIPS (with multiplicative control variate) and $\beta$-IPS (with additive control variate) converge to substantially better test policy values.
In terms of the convergence speed, $\beta$-IPS converges to the optimal value faster than the SNIPS estimator, most likely because it has lower estimator variance than SNIPS. %
With this, we can answer RQC1 as follows: in the full-batch setting, our proposed optimal baseline correction enables $\beta$-IPS to converge faster than SNIPS at similar performance.
 
\noindent \textbf{Mini-batch.} The results for mini-batch training in terms of the test policy value are reported in Figure~\ref{fig:mini_batch_val}.
Different from the full-batch setup, where the focus is on reducing the variance of the \emph{estimator} value (Section~\ref{sec:estimator_var}), in the mini-batch mode, the focus is on reducing the variance of the gradient estimate (Section~\ref{sec:grad_var}).
The model and training setup are similar to the full-batch mode, except that we fixed the batch size to 1024 for the mini-batch experiments. 
Preliminary results indicated that the batch size hyper-parameter has a limited effect.

Analogous to the full-batch setup, the IPS estimator results in a lower test policy value, most likely because of the high gradient variance which prevents convergence to high performance.
In contrast, due to their baseline corrections, BanditNet (Eq.~\ref{eq:banditnet_grad}) and $\beta$-IPS have a lower gradient variance.
Accordingly, they also converge to better performance~\cite{bottou2018optimization}, i.e., resulting in superior test policy values. 

Amongst these baseline-corrected gradient-based methods (BanditNet and $\beta$-IPS), our proposed $\beta$-IPS estimator outperforms BanditNet as it provides a policy with substantially higher value.
The differences are observed over different choices of learning rates.
Thus we answer RQC2 accordingly: in the mini-batch setting, our proposed gradient-minimizing baseline method results in considerably higher policy value compared to both IPS and BanditNet.

\begin{figure*}[!t]
    \centering
    \includegraphics[width=\linewidth]{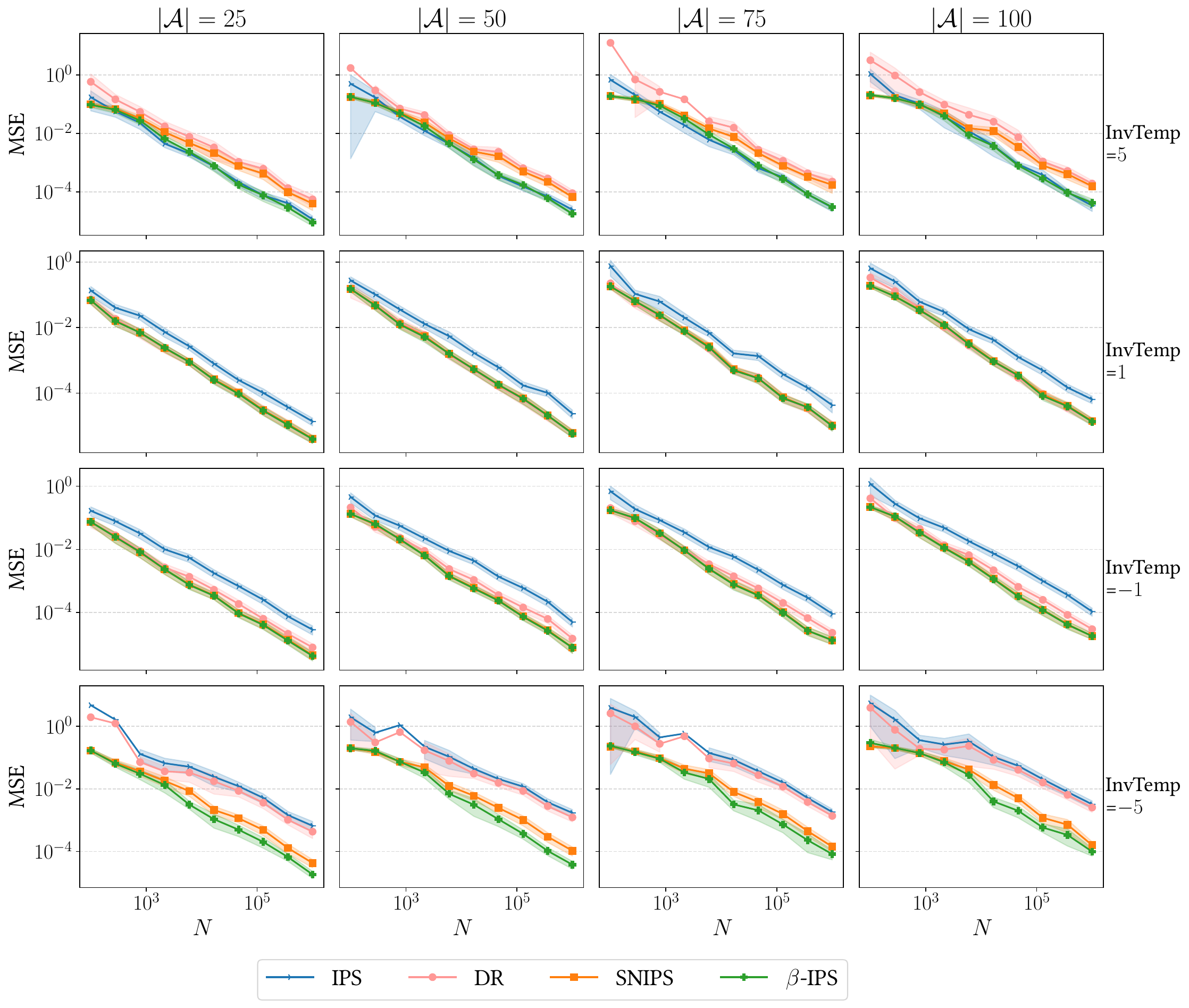}
    \caption{Mean Squared Error (MSE) of different off-policy estimators with varying action space (from left to right), and varying inverse temperature parameter of the softmax logging policy (from top to bottom). X-axis corresponds to the size of the logged data simulated (ranging from $10^2$ to $10^6$), and the y-axis corresponds to the MSE (evaluated over 100 independent samples of the synthetic data) along with 95\% confidence interval. Each row corresponds to a different setting of inverse temperature of the softmax logging policy. 
    We only consider unbiased (asymptotically or otherwise) estimators.}
    \label{fig:ope_mse}
\end{figure*}

Next, we directly consider the empirical gradient variance of different estimators;
Figure~\ref{fig:grad_var} reports the average mini-batch gradient variance per epoch.
As expected, the IPS estimator has the highest gradient variance by a large margin. %
For BanditNet, we observe a lower gradient variance, which is the desired result of the additive baseline it employs.
Finally, we observe that our proposed method $\beta$-IPS has the lowest gradient variance.
This result corroborates the theoretical claim (Theorem~\ref{thrm:min_grad_var}) that the $\beta$-IPS estimator has the lowest gradient variance amongst all global additive control variates (including IPS and BanditNet).
Our answer to RQC3 is thus clear: our proposed $\beta$-IPS results in considerably lower gradient variance compared to BanditNet and IPS.

\subsection{Off-policy evaluation performance (RQC4)}
To evaluate the performance of the proposed $\beta$-IPS method, which minimizes the estimated policy value (Eq.~\ref{eq:optimal_offpolicy_beta}), in an OPE task, results are presented in Figure~\ref{fig:ope_mse}. The target policy (to be evaluated) is a logistic regression model trained via the IPS objective (Eq.~\ref{eq:MC_gradient_IPS}) on logged data and evaluated on a separate full-information test set. 
We evaluate the MSE of the estimated policy value against the \textit{true} policy value (Eq.~\ref{eq:estimator_mse}). 
To evaluate the MSE of different estimators realistically, we report results with varying degrees of the optimality of the behavior policy (decided by the inverse temperature parameter of the softmax) and with a varying cardinality of the action space. A positive (and higher) inverse softmax temperature results in a increasingly optimal behavior policy (selects action with highest reward probability), and a negative (and lower) inverse softmax temperature parameter results in an increasingly sub-optimal behavior policy (selects actions with lowest reward probability). 
Our proposed $\beta$-IPS method has the lowest MSE in all simulated settings.
Interestingly, the proposed $\beta$-IPS has a lower MSE than the DR method, which has a regression model-based control variate, arguably more powerful than the constant control variate from the proposed $\beta$-IPS method.
Similar observations have been made in previous work, e.g., \citet{Jeunen2020REVEAL} reported that the DR estimator's performance heavily depends on the logging policy.

\begin{table}[t]
\caption{Comparison of different OPE methods on real-world recommender system logs of ZOZOTOWN from a campaign targeted towards men with a uniformly random production policy. We report the mean relative absolute error (with std).}
\label{tab:method_comparison}
\centering
\begin{tabular}{lc}
\toprule
\textbf{OPE estimator} & \textbf{Abs. relative error} $\downarrow$ \\
\midrule
IPS & 0.1277 (0.0142) \\
SNIPS & 0.1113 (0.0372) \\
DR & 0.1144 (0.0366) \\
$\beta$-IPS & \textbf{0.1078 (0.0383)} \\
\bottomrule
\end{tabular}
\end{table}

Depending on the setting, we see that $\beta$-IPS either has performance comparable to the SNIPS estimator, i.e., when inverse temperature $\in \{-1, 1\}$;
or noticeably higher performance than SNIPS, i.e., when inverse temperature $\in \{-5, 5\}$.

\noindent \textbf{Real-world evaluation.} To evaluate different estimators in a real-world recommender systems setup, we report the results of OPE from the production logs of a real-world recommender system in Table~\ref{tab:method_comparison}.
Similar to the simulation setup, the proposed $\beta$-IPS has the lowest absoluate relative error amongst all estimators in the comparison.
In conclusion, we answer RQC4:
our proposed policy-value variance minimizing baseline method results in substantially improved MSE, compared to IPS, SNIPS and DR, in offline evaluation tasks that are typical recommender system use-cases.

\section{Conclusion and Future Work}

In this chapter, we have proposed to unify different off-policy estimators as equivalent additive baseline corrections.
We look at off-policy evaluation and learning settings and propose baseline corrections that minimize the variance in the estimated policy value and the empirical gradient of the off-policy learning objective.
Extensive experimental comparisons on a synthetic benchmark with realistic settings show that our proposed methods improve performance in the off-policy estimation (OPE) and off-policy learning (OPL) tasks.

We believe our work in this chapter represents a significant step forward in the understanding and use of off-policy estimation methods (for both evaluation and learning use-cases), since we show that the prevalent SNIPS estimator can be improved upon with essentially no cost, as our proposed method is parameter-free and -- in contrast with SNIPS -- it retains the unbiasedness that comes with IPS.
Future work may apply a similar approach to offline reinforcement learning setups~\cite{levine2020offline}, or consider extensions of our approach for ranking applications~\cite{London2023}.

In this chapter, we answer the \ref{rq:recsys}, and \ref{rq:recsys1} in the affirmative.
First, we present a unifying framework for off-policy evaluation and learning tasks :$\beta$-IPS.
The new estimator $\beta$-IPS combines most commonly estimators such as: inverse propensity scoring (IPS), doubly robust (DR), and self-normalized IPS (SNIPS). 
Further, we presented a closed-form solution baseline correction term -- $\beta^{}$ that minimizes variance for both off-policy learning and evaluation tasks. 

Broadly, in this chapter, we explored off-policy evaluation and learning estimators derived from logged user interactions in recommender systems.
In the next chapter, we turn to contextual-bandit learning within a diffusion-model framework, aiming to optimize arbitrary user-defined objectives.%

\newpage

\begin{subappendices}
\section{Appendix: Off-policy Estimator Variance}

\begin{figure*}[!ht]
    \centering
    \includegraphics[width=\textwidth]{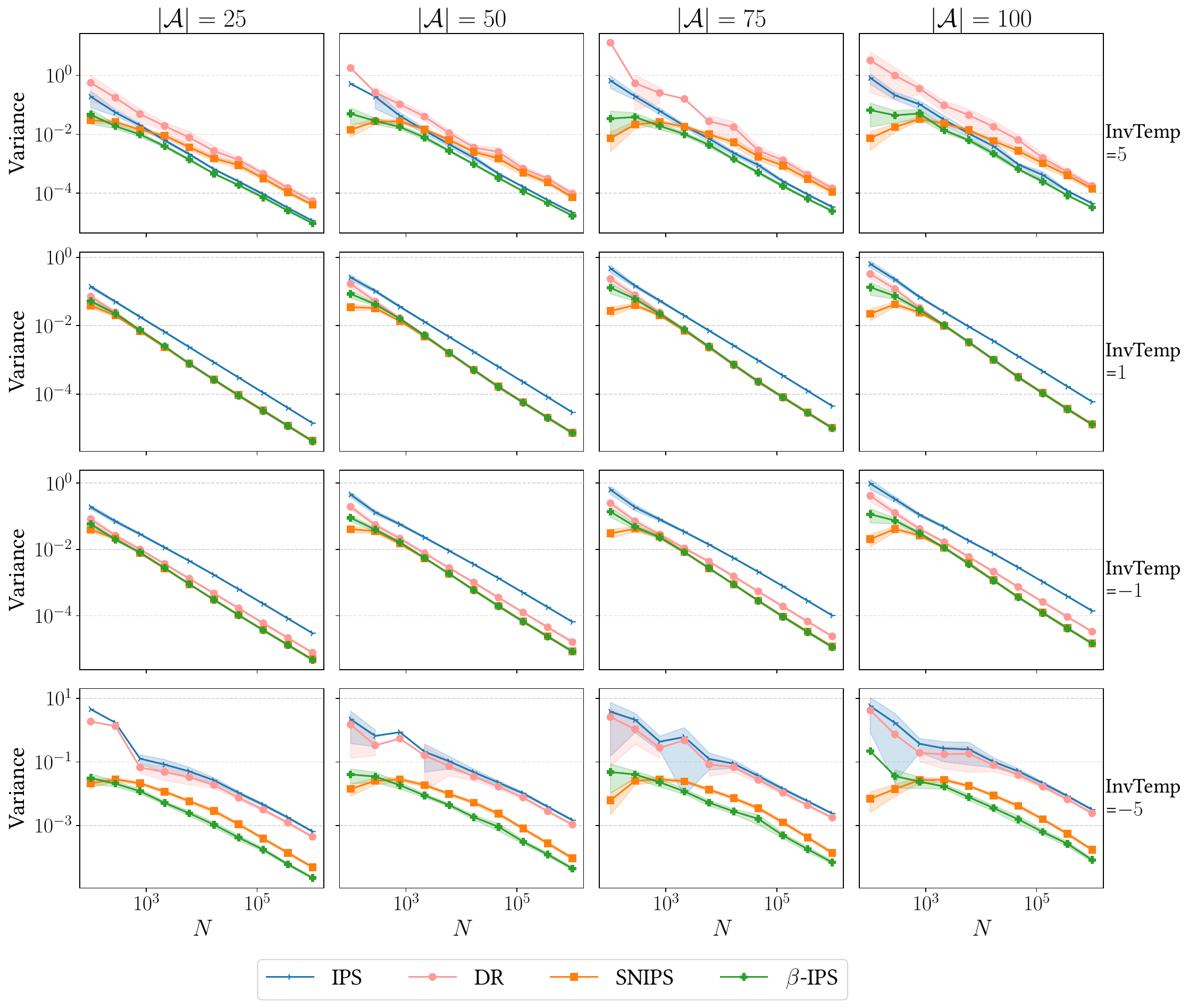}
    \caption{Empirical variance of different off-policy estimators with varying action space (from left to right), and varying sub-optimality of a temperature-based softmax behavior policy (from top to bottom). The x-axis corresponds to the size of the logged data simulated (ranging from $10^2$ to $10^6$), and the y-axis corresponds to the variance of different estimators (evaluated over 100 independent samples of the synthetic data) along with 95\% confidence interval. Each row corresponds to a different optimality level of the logging policy, decided by the inverse temperature parameter. We only consider unbiased (asymptotically or otherwise) estimators.}
    \label{fig:ope_sample_var}
\end{figure*}
\noindent In this chapter appendix, we report additional results from the experimental section (Section~\ref{recsys-results-disc}) from the main chapter), answering RQC4. Specifically, we look the the empirical variance of various offline estimators for the task of off-policy evaluation. The mean squared error (MSE) of different offline estimators are reported in Figure~\ref{fig:ope_mse}. In this appendix, we report the empirical variance of various offline estimators in Figure~\ref{fig:ope_sample_var}. 

From the figure, it is clear that our proposed $\beta$-IPS estimator with estimator variance minimizing $\beta$ value (Eq.~\ref{eq:optimal_offpolicy_beta}) results in the lowest empirical variance in most of the cases. It is interesting to note that when the logged data is limited ($N < 10^3$), sometimes the SNIPS estimator has lower estimator variance. We suspect that the reason could be a bias in the estimate of the variance-optimal $\beta$ estimate (Eq.~\ref{eq:optimal_offpolicy_beta}), when the dataset size is small, given that it is a ratio estimate of expectations. For practical settings, i.e., when $N > 10^3$, the proposed estimator $\beta$-IPS results in a minimum sample variance, thereby empirically validating the effectiveness of our proposed $\beta$-IPS estimator for the task of OPE. 
\end{subappendices}

\chapter{A Simple and Effective Reinforcement Learning Method for Text-to-Image Diffusion Models}
\chaptermark{A Reinforcement Learning Method for Text-to-Image Diffusion Models}
\label{chapter:01-online-evaluation4}

\footnote[]{This chapter was published as~\citep{gupta2025simple}.} 
So far in this thesis, we focused on learning from user interactions via contextual bandits within ranking or recommendation systems. 
However, the contextual bandit framework has recently also been effectively employed in fine-tuning foundation models, such as textual \ac{LLMs} and diffusion models.
The typical choice of method for model optimization is proximal policy optimization (PPO)~\cite{schulman2017proximal}. 
While effective, recent research has highlighted computational advantages of REINFORCE (policy gradient methods) over PPO for text-based \ac{LLMs}~\cite{ahmadian2024back}. 
Given PPO's ongoing challenges with variance and sample inefficiency, we consider improvements through our final research question:

\begin{enumerate}[label=\textbf{RQ5},ref={RQ\arabic*},resume]
\item \acl{rq:loop}
\end{enumerate}

\noindent In this chapter, we systematically compare PPO and REINFORCE for diffusion model fine-tuning. 
In the first part of this chapter, we demonstrate that REINFORCE exhibits inferior sample efficiency compared to PPO. 
Subsequently, we propose \ac{LOOP}, an enhancement to PPO achieving superior performance with the same number of input prompts by generating multiple actions per prompt.

\section{Introduction}
\label{sec:intro_icml}

Diffusion models have emerged as a powerful tool for generative modeling~\cite{sohl2015deep,ho2020denoising}, with a strong capacity to model complex data distributions from various modalities, like images~\cite{rombach2022high}, text~\cite{austin2021structured}, natural molecules~\cite{xu2023geometric}, and videos~\cite{blattmann2023align}.

Diffusion models are typically pre-trained on a large-scale dataset, such that they can subsequently generate samples from the same data distribution. The training objective typically involves maximizing the data distribution likelihood. This pre-training stage helps generate high-quality samples from the model. However, some applications might require optimizing a custom reward function, for example, optimizing for generating aesthetically pleasing images~\cite{xu2024imagereward}, semantic alignment of image-text pairs based on human feedback~\cite{schuhmann2022laion}, or generating molecules with specific properties~\cite{wang2024fine}. 

To optimize for such black-box objectives, \ac{RL}-based fine-tuning has been successfully applied to diffusion models~\cite{fan2024reinforcement,black2023training,wallace2024diffusion,li2024aligning,gu2024diffusion}. For \ac{RL}-based fine-tuning, the reverse diffusion process is treated as a \ac{MDP}, wherein prompts are treated as part of the input state, the generated image at each time-step is mapped to an action, which receives a reward from a fixed reward model (environment in standard \ac{MDP}), and finally the diffusion model is treated as a policy, which we optimize to maximize rewards. For optimization, typically \ac{PPO} is applied~\cite{fan2024reinforcement,black2023training}. In applications where getting a reward model is infeasible or undesirable, ``RL-free'' fine-tuning (typically offline) can also be applied~\cite{wallace2024diffusion}. For this chapter, we only focus on diffusion model fine-tuning using ``online'' \ac{RL} methods, specifically \ac{PPO}~\cite{schulman2017proximal}.

\begin{figure*}[!ht]
    \centering
    \rotatebox{90}{%
      \begin{minipage}{0.95\textheight}    %
        \centering
        \setlength{\tabcolsep}{1.2mm}
        \begin{tabular}{l@{~}ccccc}
          \raisebox{1.5cm}{SD v2} &
            \includegraphics[width=2.9cm]{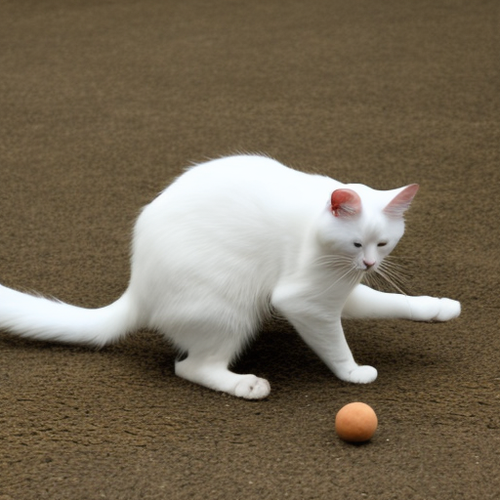} &
            \includegraphics[width=2.9cm]{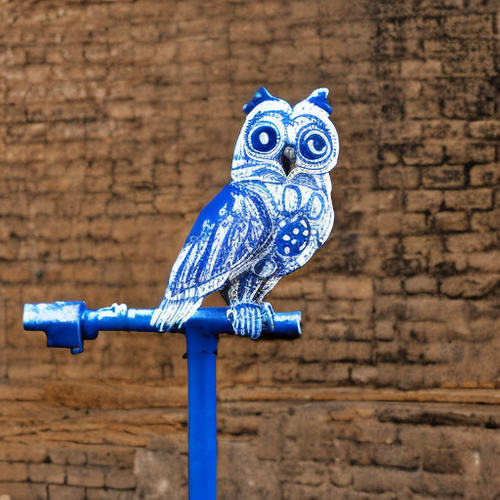} &
            \includegraphics[width=2.9cm]{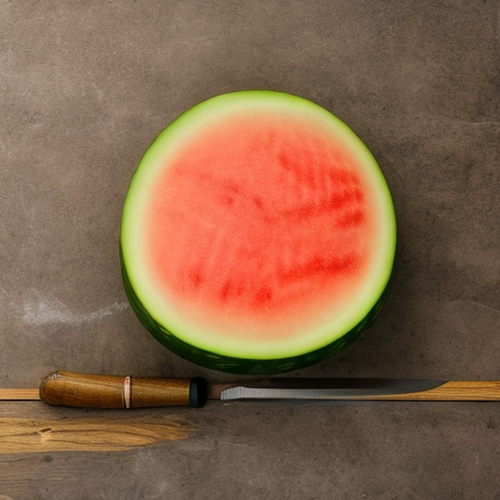} &
            \includegraphics[width=2.9cm]{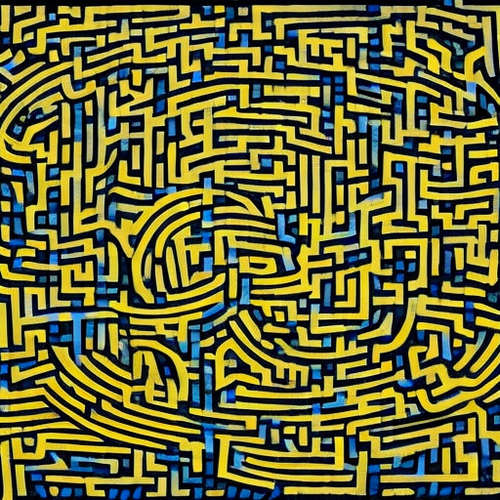} &
            \includegraphics[width=2.9cm]{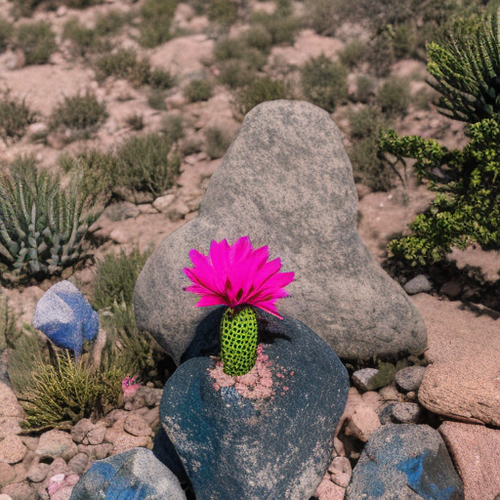}
          \\
          \raisebox{1.5cm}{PPO} &
            \includegraphics[width=2.9cm]{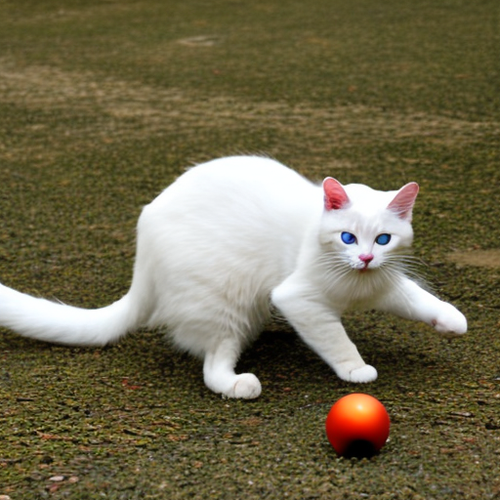} &
            \includegraphics[width=2.9cm]{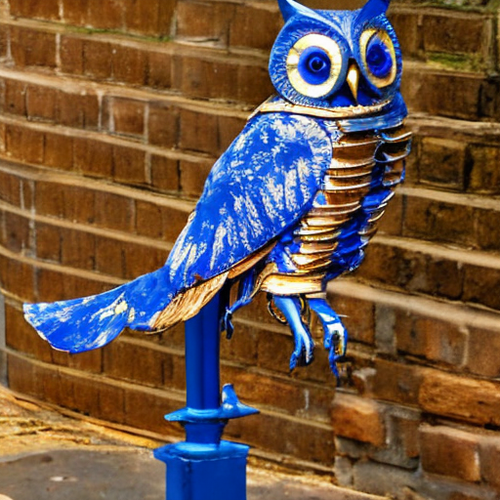} &
            \includegraphics[width=2.9cm]{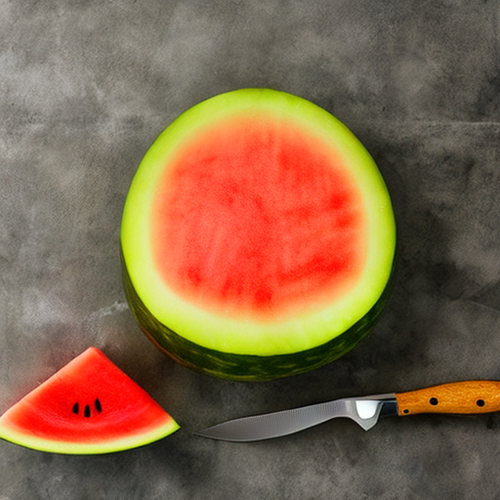} &
            \includegraphics[width=2.9cm]{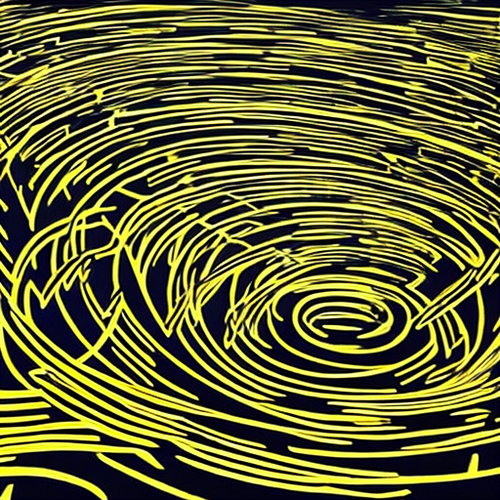} &
            \includegraphics[width=2.9cm]{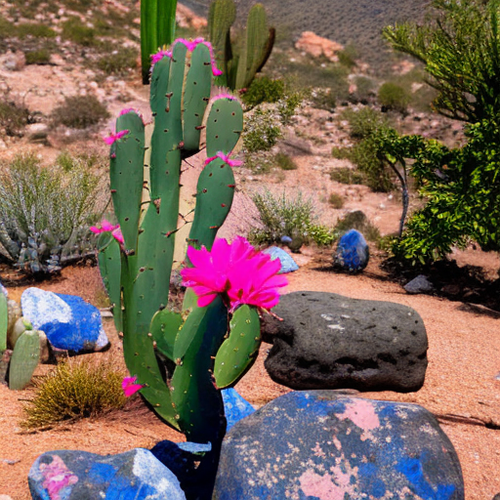}
          \\
          \raisebox{1.5cm}{LOOP} &
            \includegraphics[width=2.9cm]{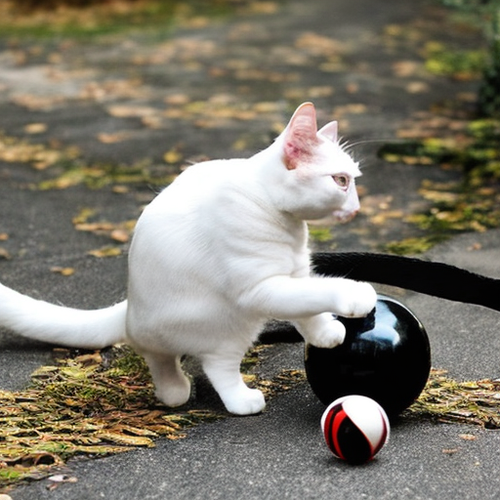} &
            \includegraphics[width=2.9cm]{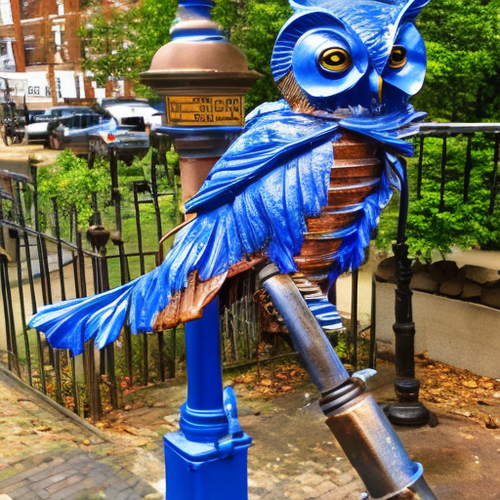} &
            \includegraphics[width=2.9cm]{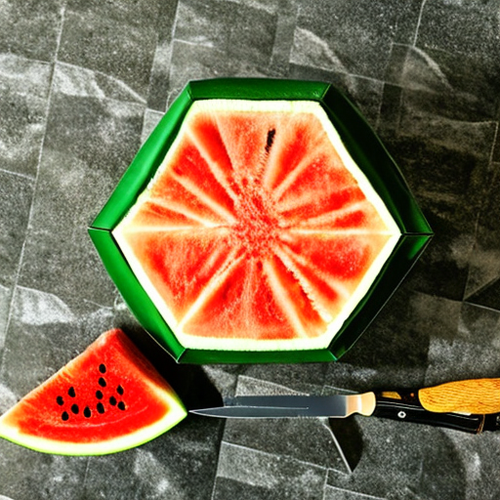} &
            \includegraphics[width=2.9cm]{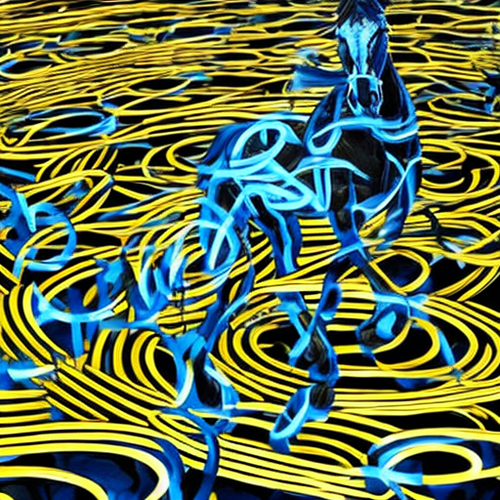} &
            \includegraphics[width=2.9cm]{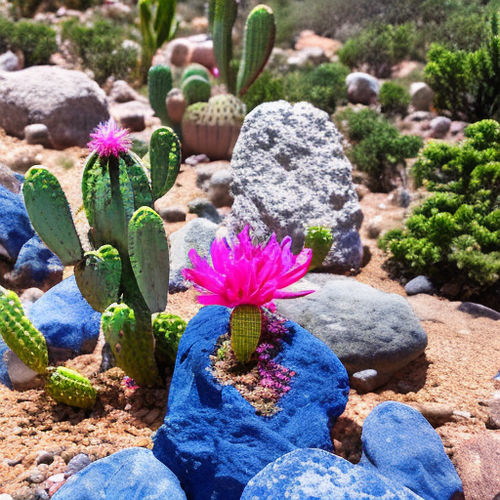}
          \\
          & \parbox{2.9cm}{\small ``White cat playing with black ball''\\ \\}
          & \parbox{2.9cm}{\small ``Mechanical owl with cobalt blue feathers on a rusted bronze lamppost''}
          & \parbox{2.9cm}{\small ``A hexagonal watermelon placed beside a knife''\\}
          & \parbox{3cm}{\small ``Black horse with glowing cyan patterns in maze of floating golden rings''}
          & \parbox{2.9cm}{\small ``Neon pink cactus growing on a cobalt blue rock''\\}
          \\
        \end{tabular}
        \vspace*{-0.5\baselineskip}
        \caption{\textbf{LOOP improves attribute binding}. Qualitative examples generated with SD 2.0 (first row), \ac{PPO} (second row) and \ac{LOOP} $k{=}4$ (third row).  
        In each prompt SD and \ac{PPO} miss at least one listed attribute, whereas \ac{LOOP} succeeds.}
        \label{fig:improve-attribute-binding}
      \end{minipage}%
    }%
  \end{figure*}

An advantage of \ac{PPO} is that it removes the incentive for the new policy to deviate too much from the previous reference policy, via importance sampling and clipping operation~\cite{schulman2017proximal}. While effective, \ac{PPO} can have a significant computational overhead. In practice, \ac{RL} fine-tuning for diffusion models via \ac{PPO} requires concurrently loading three models in memory: 
\begin{enumerate}[label=(\roman*)]
    \item The \textbf{reference policy}: The base policy, which is usually initialized with the pre-trained diffusion model.
    \item The \textbf{current policy}: The policy that is \ac{RL} fine-tuned, and also initialized with the pre-trained diffusion model.
    \item The \textbf{reward model}: Typically, a large vision-language model, trained via supervised fine-tuning objective~\cite{lee2023aligning}, which assigns a scalar reward to the final generated image during the online optimization stage. 
    \end{enumerate}
This can result in a considerable computational burden, given that each policy can potentially have millions of parameters.
In addition to its computational overhead, \ac{PPO} is also known to be sensitive to hyper-parameters~\cite{engstrom2019implementation,zheng2023delve,huang2024n+}.

Simpler approaches, like REINFORCE~\cite{williams1992simple} avoid such complexities, and could theoretically be more efficient. However, in practice, they often suffer from high variance and instability. Recently, a variant of REINFORCE: \ac{RLOO}~\citep{kool2019buy} was proposed which samples multiple sequences per input prompt, and a baseline correction term to reduce the variance, however, it still suffers from sample inefficiency.

This raises a fundamental question about the \textbf{efficiency-effectiveness} trade-off in \ac{RL}-based diffusion fine-tuning. In this chapter, first we systematically explore this trade-off between \textit{efficiency} -- a lower computational cost, and reduced implementation complexity (i.e., fewer hyper-parameters) -- and \textit{effectiveness} -- stable training, and final performance. 
We compare a simple REINFORCE approach with the standard \ac{PPO} framework, demonstrating that while REINFORCE greatly reduces computational complexity, it comes at the cost of reduced performance.

Motivated by this finding, we propose a novel \ac{RL} for diffusion fine-tuning method, \ac{LOOP}, which combines the best of the both worlds. To reduce the variance during policy optimization, \ac{LOOP} uses multiple actions (diffusion trajectories) and a (REINFORCE) baseline correction term per input prompt.
To maintain the stability and robustness of \ac{PPO}, \ac{LOOP} uses clipping and importance sampling.

We choose the text-to-image compositionality benchmark (T2I-CompBench; \citeauthor{huang2023t2i}, \citeyear{huang2023t2i}) as our primary evaluation benchmark. Text-to-image models often fail to satisfy an essential reasoning ability of attribute binding, i.e., the generated image often fails to \textit{bind} certain \textit{attributes} specified in the instruction prompt~\cite{huang2023t2i,ramesh2022hierarchical,fu2024enhancing}. 
As illustrated in Figure~\ref{fig:improve-attribute-binding}, LOOP outperforms previous diffusion methods on attribute binding.
As attribute binding is a key skill necessary for real-world applications, we choose the T2I-CompBench benchmark alongside two other common tasks: aesthetic image generation and image-text semantic alignment.   

To summarize, our main contributions are as follows:

\noindent \textbf{PPO vs.\ REINFORCE efficiency-effectiveness trade-off.} We systematically study how design elements like clipping, reference policy, value function in \ac{PPO} compare to a simple REINFORCE method, highlighting the efficiency-effectiveness trade-off in diffusion fine-tuning. To the best of our knowledge, we are the first ones to present such a systematic study, highlighting the trade-offs in diffusion fine-tuning.

\noindent \textbf{Introducing LOOP.} We propose \ac{LOOP}, a novel \ac{RL} for diffusion fine-tuning method combining the best of REINFORCE and \ac{PPO}. \ac{LOOP} uses multiple diffusion trajectories and a REINFORCE baseline correction term for variance reduction, as well as clipping and importance sampling from \ac{PPO} for robustness and sample efficiency.
    
\noindent \textbf{Empirical validation.} To validate our claims empirically, we conduct experiments on the T2I-CompBench benchmark image compositionality benchmark. The benchmark evaluates the attribute binding capabilities of the text-to-image generative models and shows that LOOP succeeds where previous text-to-image generative models often fail. We also evaluate LOOP on two common objectives from the \ac{RL} for diffusion literature: image aesthetics, and text-image semantic alignment~\cite{black2023training}.

\section{Background and Related Work}

\subsection{Diffusion models}
We focus on denoising diffusion probabilistic models (DDPM) as the base model for text-to-image generative modeling~\cite{ho2020denoising,sohl2015deep}. Briefly, given a conditioning context variable $\mathbf{c}$ (text prompt in our case), and the data sample $\mathbf{x}_0$, DDPM models $p(\mathbf{x}_0 \mid \mathbf{c})$ via a Markov chain of length $T$, with the following dynamics:
\begin{equation}
    p_{\theta}(\mathbf{x}_{0:T} \mid \mathbf{c}) = p(\mathbf{x}_{T} \mid \mathbf{c}) \prod_{t=1}^{T} p_{\theta}(\mathbf{x}_{t-1} \mid \mathbf{x}_{t}, \mathbf{c}).
    \label{eq:diffusion_mc}
\end{equation}
Image generation in diffusion model is achieved via the following ancestral sampling scheme, which is a reverse diffusion process:
\begin{equation}
\begin{aligned} 
\mathbf{x}_T {} &\sim \mathcal{N}(\mathbf{0}, \mathbf{I}), \\
\mathbf{x}_{t} &{}\sim N\left(\mathbf{x}_t \mid \mu_\theta(\mathbf{x}_t, \mathbf{c}, t), \sigma^2_\theta I\right),
 \forall t \in [0, T-1],
\end{aligned}
\end{equation}
where the distribution at time-step $t$ is assumed to be a multivariate normal distribution with the predicted mean $\mu_\theta(\mathbf{x}_t, \mathbf{c}, t)$, and a constant variance.

\subsection{Proximal policy optimization for RL}
\Ac{PPO} was introduced for optimizing a policy with the objective of maximizing the overall reward in the \ac{RL} paradigm~\cite{schulman2017proximal}. \ac{PPO} removes the incentive for the current policy $\pi_t$ to diverge from the previous policy $\pi_{t-1}$ outside the range $[ 1-\epsilon, 1+\epsilon]$, where $\epsilon$ is a hyper-parameter. As long as the subsequent policies are closer to each other in the action space, the monotonic policy improvement bound guarantees a monotonic improvement in the policy's performance as the optimization progresses. This property justifies the clipping term in the mathematical formulation of the \ac{PPO} objective function~\cite{schulman2015trust,achiam2017constrained,queeney2021generalized}. Formally, \ac{PPO} the objective function is:
\begin{equation}
\mbox{}\hspace*{-2mm}
J(\theta) \! = \!
\mathbb{E} \!\!\left[ \min \!\left( \!r_t(\theta) \hat{A}_t, \, \!\text{clip}\left(r_t(\theta), 1-\epsilon, 1+\epsilon\right)\!\hat{A}_t\! \right)\!\right]\!,
\!\mbox{}
\label{eq:ppo_obj1}
\end{equation}
where $r_t(\theta)=\frac{\pi_t(a \mid c)}{\pi_{t-1}(a \mid c)}$ is the importance sampling ratio between the current policy $\pi_{t}(a \mid c)$ and the previous reference policy $\pi_{t-1}(a \mid c)$, $\hat{A}_t$ is the advantage function~\cite{sutton2018reinforcement}, and the $\text{clip}$ operator restricts the importance sampling ratio in the range $[ 1-\epsilon, 1+\epsilon]$.

\subsection{RL for text-to-image diffusion models}
The diffusion process can be viewed as an \ac{MDP} $\left( \mathcal{S}, \mathcal{A}\right.$, $\left.\mathcal{P}, \mathcal{R}, \rho_0 \right)
$, where $\mathcal{S}$ is the state space, $\mathcal{A}$ is the action space, $\mathcal{P}$ is the state transition kernel, $\mathcal{R}$ is the reward function, and $\rho_0$ is the distribution of initial state $\mathbf{s_0}$. In the context of text-to-image diffusion models, the \ac{MDP} is defined as:
\begin{equation}
\begin{split}
  \mathbf{s_t} = (\mathbf{c}, t, \mathbf{x_t}), & \;\; \pi_{\theta}(\mathbf{a_t} \mid \mathbf{s_t}) = p_\theta(\mathbf{x_{t-1}} \mid \mathbf{x_t}, \mathbf{c}),  \\ 
\mathcal{P}(\mathbf{s_{t+1}} \mid \mathbf{s_t}, \mathbf{a_t}) &=  \delta \big( \mathbf{c}, \mathbf{a_{t}} \big), \quad \mathbf{a_t} = \mathbf{x_{t-1}},  \\
\rho_0(\mathbf{s_0}) &= \big(p(\mathbf{c}), \delta_T, \mathcal{N}(0, \mathbf{I})\big),  \\
\mathcal{R}(\mathbf{s_t}, \mathbf{a_t}) &= 
    \begin{cases}
        r(\mathbf{x_0}, \mathbf{c}) & \text{if } t = 0, \\
        0 & \text{otherwise.}
    \end{cases}
\label{eq:mdp_formal}
\end{split}
\end{equation}
The input state $\mathbf{s_t}$ is defined in terms of the context (prompt features), sampled image at the given time-step $t$. The policy $\pi_{\theta}$ is the diffusion model itself. The state transition kernel is a dirac delta function $\delta$ with the current sampled action $\mathbf{x}_{t}$ as the input. The reward is assigned only at the last step in the reverse diffusion process, when the final image is generated. The initial state $\rho_0$ corresponds to the last state in the forward diffusion process: $\mathbf{x}_T$. 

\subsection{PPO for diffusion fine-tuning}
The objective function of \ac{RL} fine-tuning for a diffusion policy $\pi_{\theta}$ can be defined as follows:
\begin{equation}
\begin{split}
J_{\theta}(\pi) &= \mathbb{E}_{\tau\sim p(\tau\mid\pi_{\theta})}\left[\sum_{t=0}^{T} \mathcal{R}(\mathbf{s}_t,\mathbf{a}_t)\right] \\
&= \mathbb{E}_{\tau\sim p(\tau\mid\pi_{\theta})}\left[r(\mathbf{x}_0, \mathbf{c})\right],
\end{split}
\label{eq:objective}
\end{equation}
where the trajectory $\mathbf{\tau} =\{\mathbf{x}_T, \mathbf{x}_{T-1},\ldots,\mathbf{x}_0\}$ refers to the reverse diffusion process (Eq.~\ref{eq:diffusion_mc}), and the total reward of the trajectory is the reward of the final generated image $\mathbf{x}_0$ (Eq.~\ref{eq:mdp_formal}). We ignore the KL-regularized version of the equation, which is commonly applied in the RLHF for LLM literature~\cite{zhong2024dpo,zeng2024token,rafailov2024direct}, and proposed by \citet{fan2024reinforcement} in the context of RL for diffusion models. As shown by~\citet{black2023training}, adding the KL-regularization term makes no empirical difference in terms of the final performance. 
The \ac{PPO} objective is given as:
\begin{align}
 J_{\theta}^{\mathrm{PPO}}(\pi) = \mathbb{E} \bigg[\sum_{t=0}^{T} \! 
\text{clip}\!\left(\!\frac{\pi_{\theta}(\mathbf{x}_{t-1} \mid \mathbf{x}_t, \mathbf{c})}{\pi_\text{old}(\mathbf{x}_{t-1} \mid \mathbf{x}_t, \mathbf{c})}, 1\!-\!\epsilon, 1\!+\!\epsilon\!\right) 
    r(\mathbf{x}_0, \mathbf{c})\bigg],
\label{eq:ppo_objective}
\end{align}
where the clipping operation removes the incentive for the new policy $\pi_{\theta}$ to differ from the previous round policy $\pi_{\text{old}}$~\cite{schulman2017proximal,black2023training}.

\section{REINFORCE vs.\ PPO: An~Efficiency-Effectiveness Trade-Off}
\label{sect:revist_rl}
In this section, we explore the efficiency-effectiveness trade-off between two prominent reinforcement learning methods for diffusion fine-tuning: REINFORCE and PPO. Understanding this trade-off is crucial for selecting the appropriate algorithm given constraints on computational resources and desired performance outcomes.

In the context of text-to-image diffusion models, we aim to optimize the policy $\pi$ to maximize the expected reward $ \mathcal{R}(x_{0:T}, c) = r(x_{0}, c)$. Our objective function is defined as:
\begin{equation}
J_{\theta}(\pi) = \mathbb{E}_{c \sim p(C), x_{0:T} \sim p_{\theta}(x_{0:T}\mid c)}\left[ r(x_{0}, c)\right].
\label{eq:cb_objective}
\end{equation}
\noindent \textbf{REINFORCE for gradient calculation.}
For optimizing this objective, the REINFORCE policy gradient (also known as \ac{SF})~\cite{williams1992simple} provides the following gradient estimate:
\begin{equation}
\begin{split}
\nabla_{\theta} J_{\theta}^{\mathrm{SF}}(\pi) &{}=\mathbb{E}_{\mathbf{x}_{0:T}}\left[ \nabla_{\theta} \log \mleft( \prod_{t=1}^{T} p_{\theta}\left(\mathbf{x}_{t-1} \mid \mathbf{x}_{t}, \mathbf{c}\right)\mright) r\left(\mathbf{x}_{0}, \mathbf{c}\right)\right] \\
&{}= \mathbb{E}_{\mathbf{x}_{0:T}}\left[\sum_{t=0}^{T} \nabla_{\theta} \log p_{\theta}\left(\mathbf{x}_{t-1} \mid \mathbf{x}_{t}, \mathbf{c}\right) r\left(\mathbf{x}_{0}, \mathbf{c}\right)\right],
\end{split}
\label{eq:cb_pg}
\end{equation}
where the second step follows from the reverse diffusion policy decomposition (Eq.~\ref{eq:diffusion_mc}).

In practice, a batch of trajectories is sampled from the reverse diffusion distribution, i.e., $\mathbf{x}_{0:T} \sim p_{\theta}(\mathbf{x}_{0:T})$, and a Monte Carlo estimate of the REINFORCE policy gradient (Eq.~\ref{eq:cb_pg}) is calculated for the model update. 

\noindent  \textbf{REINFORCE with baseline correction.}
To reduce variance of the REINFORCE estimator, a common trick is to subtract a constant baseline correction term from the reward function~\cite{greensmith2004variance,mohamed2020monte}:
\begin{equation}
\begin{split}
\nabla_{\theta} J_{\theta}^{\mathrm{SFB}}(\pi) = \mathbb{E}\mleft[\sum_{t=0}^{T} \nabla_{\theta} \log p_{\theta}\mleft(\mathbf{x}_{t-1} \mid \mathbf{x}_{t}, \mathbf{c}\mright) \mleft(r\mleft(\mathbf{x}_{0}, \mathbf{c}\mright)-b_t\mright)\mright].
\end{split}
\label{eq:cb_pg_b}
\end{equation}

\noindent  \textbf{REINFORCE Leave-one-out (RLOO).}
To further reduce the variance of the REINFORCE estimator, \ac{RLOO} samples $K$ diffusion trajectories per prompt ($\{\mathbf{x}^{k}_{0:T}\} \sim \pi(. \mid \mathbf{c}))$, for a better Monte Carlo estimate of the expectation~\cite{kool2019buy,ahmadian2024back}. The \ac{RLOO} estimator is:
\begin{align}
 \nabla_{\theta} J_{\theta}^{\mathrm{RLOO}}(\pi) = \mathbb{E}\mleft[\!K^{-1}\sum_{k=0}^{K} \sum_{t=0}^{T} \nabla_{\theta} \log p_{\theta}\mleft(\mathbf{x}^{k}_{t-1} \mid \mathbf{x}^{k}_{t}, \mathbf{c}\mright) \mleft(r\mleft(\mathbf{x}^{k}_{0}, \mathbf{c}\mright)-b_t\mright)\!\mright]\!.
\label{eq:rloo}
\end{align}

\noindent However, REINFORCE-based estimators have a significant disadvantage: they do not allow sample reuse (i.e., reusing trajectories collected from previous policies) due to a distribution shift between policy gradient updates during training. Sampled trajectories can only be used once, prohibiting mini-batch updates. This makes it \textit{sample inefficient}. 

To allow for sample reuse, the \ac{IS} trick can be applied~\cite{schulman2015trust,Owen2013}:
\begin{equation}
    J_{\theta}^{\textrm{IS}}(\pi) = \mathbb{E}_{c_t \sim p(C), a_t \sim \pi_{\text{old}}(a_t\mid c_t)}\left[\frac{\pi_{\theta}(a_t \mid c_t)}{\pi_{\text{old}}(a_t \mid c_t)} \mathcal{R}_t \right],
    \label{eq:is_obj}
\end{equation}
where $\pi_{\theta}$ is the \textit{current} policy to be optimized, and $\pi_{\text{old}}$ is the policy from the previous update round. With the \ac{IS} trick, we can sample trajectories from the current policy in a batch, store it in a temporary buffer, and re-use them to apply mini-batch optimization~\cite{schulman2017proximal}.

\noindent \textbf{Motivation for \ac{PPO}.} With the \ac{IS} trick, the samples from the old policy can be used to estimate the policy gradient under the current policy $\pi_{\theta}$ (Eq.~\ref{eq:cb_pg}) in a statistically unbiased fashion~\cite{Owen2013}, i.e., in expectation the \ac{IS} and REINFORCE gradients are equivalent (Eq.~\ref{eq:is_obj}, Eq.~\ref{eq:cb_pg}).
Thus, potentially, we can improve the sample efficiency of REINFORCE gradient estimation with \ac{IS}. 

While unbiased, the \ac{IS} estimator can exhibit high variance~\cite{Owen2013}. This high variance may lead to unstable training dynamics. Additionally, significant divergence between the current policy $ \pi_{\theta} $ and the previous policy $ \pi_{\textrm{old}} $ can result in the updated diffusion policy performing worse than the previous one~\cite{schulman2015trust,achiam2017constrained}. Next, we will prove this formally. We note that this result has previously been established by \cite{achiam2017constrained} for the more general \ac{RL} setting. In this chapter, we extend this finding to the context of diffusion model fine-tuning.  

A key component of the proof relies on the distribution of states under the current policy, i.e., $d^{\pi}(s)$. 
In the case of diffusion models, the state transition kernel $P(s_{t+1} \mid s_t, a_t)$ is deterministic, because the next state consists of the action sampled from the previous state (Eq.~\ref{eq:mdp_formal}), i.e., $ P(s_{t+1} \mid s_t, a_t)=1$.
While the state transition kernel is deterministic, the distribution of states is stochastic, given that it depends on the action at time $t$, which is sampled from the policy (Eq.~\ref{eq:mdp_formal}).
We define the state distribution as: 
\begin{definition}
Given the distribution over contexts $\mathbf{c} \sim p(\mathbf{C})$, the (deterministic) distribution over time $t=\delta(t)$, and the diffusion policy $\pi$, the state distribution at time $t$ is:
\begin{equation*}
    p(\mathbf{s}_t \mid \pi) = 
    p(\mathbf{c}) \delta(t) \!\!\int_{\mathbf{x}_{t+1}} \hspace{-0.55cm} \pi(\mathbf{x}_t \mid \mathbf{x}_{t+1}, \mathbf{c}, t) \pi(\mathbf{x}_{t+1} \mid \mathbf{c}, t) \,\mathbf{dx}_{t+1}.     
\end{equation*}
\end{definition}
\noindent Subsequently, the normalized discounted state visitation distribution can be defined as: 
\begin{equation}
   d^{\pi}(\mathbf{s})= (1 - \gamma) \sum_{t=0}^\infty \gamma^t p(\mathbf{s}_t = \mathbf{s} \mid  \pi).
\end{equation}
The advantage function is defined as: $A^{\pi_{k}}(\mathbf{s},\mathbf{a}) = Q^{\pi_{k}}(\mathbf{s}, \mathbf{a}) - V^{\pi_{k}}(\mathbf{s})$~\cite{sutton2018reinforcement}. 
Given this, 
the monotonic policy improvement bound can be derived:
\begin{theorem}
\label{theorem:1}
Consider a current policy $\pi_{k}$. For any future policy $\pi$, we have:
\begin{align*}
J(\pi) - J(\pi_{k}) &\geq \frac{1}{1-\gamma} \mathbb{E}_{(s,a)\sim d^{\pi_{k}}} \left[ \frac{\pi(a\mid s)}{\pi_{k}(a\mid s)} A^{\pi_{k}}(s,a) \right] \\
&- \frac{2\gamma C^{\pi,\pi_{k}}}{(1-\gamma)^{2}} \mathbb{E}_{s\sim d^{\pi_{k}}} \left[ \mathrm{TV}(\pi(\cdot\mid s),\pi_{k}(\cdot\mid s)) \right],
\end{align*}
where $C^{\pi,\pi_{k}} = \max_{s\in S}  |\mathbb{E}_{a\sim\pi(\cdot\mid s)} \left[ A^{\pi_{k}}(s,a) \right]|$ and $\mathrm{TV}(\pi(\cdot\mid s),\pi_{k}(\cdot\mid s))$ represents the total variation distance between the policies $\pi(\cdot\mid s)$ and $\pi_{k}(\cdot\mid s)$~\cite{achiam2017constrained}.
\end{theorem}
\noindent A direct consequence of this theorem is that when optimizing a policy with the \ac{IS} objective (Eq.~\ref{eq:is_obj}), to guarantee that the new policy will improve upon the previous policy, the policies should not diverge too much. Therefore, we need to apply a constraint on the current policy. This can be achieved by applying the clipping operator in the \ac{PPO} objective (Eq.~\ref{eq:ppo_objective})~\cite{queeney2021generalized,achiam2017constrained,schulman2017proximal,gupta-2024-practical}. 

This gives rise to an \textit{efficiency-effectiveness trade-off} between REINFORCE and \ac{PPO}. REINFORCE offers greater computational and implementation efficiency due to its simplicity, but it comes at the cost of lower sample efficiency and potential suboptimal performance. In contrast, \ac{PPO} is more computationally demanding and involves more complex hyper-parameter tuning, yet it achieves higher performance and reliable policy improvements during training. 

We note that a similar trade-off analysis was performed in the context of \ac{RL} fine-tuning for large language models (LLM)~\cite{ahmadian2024back}.
However, their analysis was limited to an empirical study, whereas we present a theoretical analysis in addition to the empirical analysis. To the best of our knowledge, we are the first to conduct such a study for diffusion methods.

\section{Method: Leave-One-Out PPO (LOOP) for~Diffusion~Fine-tuning}
\label{sec:PLOO_method}
We demonstrated the importance of \ac{PPO} in enhancing sample efficiency and achieving stable improvements during training for diffusion fine-tuning. Additionally, we showcased the RLOO method's effectiveness in reducing the variance of the REINFORCE method.
In this section, we introduce our proposed method, \textbf{\ac{LOOP}}, a novel \ac{RL} for diffusion fine-tuning method. We start with highlighting the potential high-variance in the \ac{PPO} objective.

The expectation in the \ac{PPO} loss (Eq.~\ref{eq:ppo_objective}) is typically estimated by sampling a single trajectory for a given prompt~$c$: 
\begin{equation} 
\mbox{}\hspace*{-1mm}
 \sum_{t=0}^{T} \text{clip}\mleft(\frac{\pi_{\theta}(\mathbf{x}_{t-1} \mid \mathbf{x}_t, \mathbf{c})}{\pi_\text{old}(\mathbf{x}_{t-1} \mid \mathbf{x}_t, \mathbf{c})}, 1-\epsilon, 1+\epsilon\!\mright) r(\mathbf{x}_0, \mathbf{c}),\!
\label{eq:ppo_estimator}
\end{equation}
where $\mathbf{x}_{0:T} \sim \pi_{old}$.
Even though the single sample estimate is an unbiased Monte-Carlo approximation of the expectation, it has high-variance~\cite{Owen2013}. Additionally, the \ac{IS} term ($\frac{\pi_{\theta}(\mathbf{x}_{t-1} \mid \mathbf{x}_t, \mathbf{c})}{\pi_\text{old}(\mathbf{x}_{t-1} \mid \mathbf{x}_t, \mathbf{c})}$) can also contribute to high-variance of the \ac{PPO} objective~\cite{Swaminathan2015,xie2023dropout}. Both factors combined, can lead to high-variance, and unstable training of the \ac{PPO}.

Taking inspiration from \ac{RLOO} (Eq.~\ref{eq:rloo}), we sample $K$ independent trajectories from the previous policy for a given prompt $c$, and apply a baseline correction term from each trajectories reward, to reduce the variance of the estimator:
\begin{equation}
    \begin{aligned}
    \!\!\hat{J}_{\theta}^{\mathrm{LOOP}}(\pi) = \frac{1}{K} \!\sum_{i=1}^{K} \!\sum_{t=0}^{T} \text{clip}\!\left(\frac{\pi_{\theta}(\mathbf{x}^{i}_{t-1} \mid \mathbf{x}^{i}_t, c)}{\pi_\text{old}(\mathbf{x}^{i}_{t-1} \mid \mathbf{x}^{i}_t, c)}, 1-\!\epsilon, 1+\!\epsilon\right)\!\cdot\!\left(r(\mathbf{x}^{i}_0, \mathbf{c}) - b^{i}\right),
\end{aligned}
    \label{eq:loop_estimator}
\end{equation}
where $\mathbf{x}^{i}_{0:T} \sim \pi_{old}, \forall i \in [1, K]$.
The baseline correction term $b^{i}$ reduces the variance of the gradient estimate, while being unbiased in expectation~\cite{gupta-2024-optimal,mohamed2020monte}. A simple choice of baseline correction can be the average reward across the $K$ trajectories, i.e.,
\begin{equation}
    b^{i} = \frac{1}{k} \sum_{i=0}^{K} r(\mathbf{x}^{i}_{0}).
\end{equation}
However, we choose the leave-one-out average baseline, with average taken across all samples in the trajectory, except the current sample $i$, i.e.,
\begin{equation}
    b^{i} = \frac{1}{k-1} \sum_{j \neq i} r(\mathbf{x}^{j}_{0}).
\end{equation}
\noindent Originally \ac{RLOO} sampling and baseline corrections were proposed in the context of REINFORCE, with a focus on on-policy optimization~\cite{ahmadian2024back,kool2019buy}, whereas we are applying these in the off-policy step of \ac{PPO}. We call this method \acfi{LOOP}.
Provenly, \ac{LOOP} has lower variance than \ac{PPO}:

\begin{proposition}
\label{prop:1}
The \ac{LOOP} estimator $\hat{J}_{\theta}^{\mathrm{LOOP}}(\pi)$ (Eq.~\ref{eq:loop_estimator}) has lower variance than the \ac{PPO} estimator $\hat{J}_{\theta}^{\mathrm{PPO}}(\pi)$ (Eq.~\ref{eq:ppo_estimator}):
\begin{equation}
    \mathrm{Var}\left[ \hat{J}_{\theta}^{\mathrm{LOOP}}(\pi) \right] < \mathrm{Var}\left[ \hat{J}_{\theta}^{\mathrm{PPO}}(\pi) \right].
\end{equation}
\end{proposition}
\begin{proof}
Since the sampled trajectories are independent:
\begin{align}
\mathrm{Var}\!\mleft[ \hat{J}_{\theta}^{\mathrm{LOOP}}\!(\pi) \mright]
\!=\! 
 \frac{1}{K^2} \mathrm{Var}\!\mleft[ \hat{J}_{\theta}^{\mathrm{PPO}}\!(\pi) \mright]
 \!<\!
 \mathrm{Var}\!\mleft[ \hat{J}_{\theta}^{\mathrm{PPO}}\!(\pi) \mright].
 \quad\;
\qedhere
\end{align}
\end{proof}

\section{Experimental Setup}

\textbf{Benchmark.} Text-to-image diffusion and language models often fail to satisfy an essential reasoning skill of attribute binding. Attribute binding reasoning capability refers to the ability of a model to generate images with attributes such as color, shape, texture, spatial alignment, (and others) specified in the input prompt.
In other words, generated images often fail to \textit{bind} certain \textit{attributes} specified in the instruction prompt~\cite{huang2023t2i,ramesh2022hierarchical,fu2024enhancing}. 
Since attribute binding seems to be a basic requirement for useful real-world applications, we choose the T2I-CompBench benchmark~\cite{huang2023t2i}, which contains multiple attribute binding/image compositionality tasks, and its corresponding reward metric to benchmark text-to-image generative models. We also select two common tasks from \ac{RL} for diffusion works: improving aesthetic quality of generation, and image-text semantic alignment~\cite{black2023training,fan2024reinforcement}. To summarize, we choose the following tasks for the \ac{RL} optimization:
\begin{inparaenum}[(\roman{enumi})]
    \item Color, 
    \item Shape,
    \item Texture,
    \item 2D Spatial,
    \item Numeracy,
    \item Aesthetic, and
    \item Image-text Alignment.
\end{inparaenum}
For all tasks, the prompts are split into training/validation prompts. We report the average reward on both training and validation split.

\textbf{Model.} As the base diffusion model, we use Stable diffusion V2~\cite{rombach2022high}, which is a latent diffusion model. For optimization, we fully update the UNet model, with a learning rate of $1e^{-5}$. We also tried LORA fine-tuning~\cite{hu2021lora}, but the results were not satisfactory, so we update the entire model instead.
The hyper-parameters are reported in Appendix~\ref{sec:hyper_params}.

\section{Results and Discussion}

\subsection{REINFORCE vs.\ PPO efficiency-effectiveness trade-off} 
\label{sec:revist_res}

We discuss our empirical results for the REINFORCE vs.\ \ac{PPO} efficiency-effectiveness trade-off.
Our empirical validation of the trade-off compares the following methods:\\

\noindent  \textbf{REINFORCE.} The REINFORCE policy gradient for diffusion fine-tuning (Eq.~\ref{eq:cb_pg}). \\
\noindent \textbf{REINFORCE with baseline correction.} We compare the REINFORCE policy gradient with a baseline correction (BC) term (Eq.~\ref{eq:cb_pg_b}). For the baseline term, we choose the average reward for the given prompt~\cite{black2023training}.
\\
\noindent \textbf{PPO.} The \ac{PPO} objective for diffusion fine-tuning with importance sampling and clipping (Eq.~\ref{eq:ppo_objective}).

\begin{figure*}[!t]
    \centering
    \renewcommand{\arraystretch}{0.9}
    \setlength{\tabcolsep}{0.06cm}
    \begin{tabular}{c c c c c}
        & \small Color 
        & \small Shape 
        & \small Texture 
        & \small Aesthetic \\
        
        \raisebox{0.9cm}{\rotatebox[origin=c]{90}{\small Reward}}
        & \includegraphics[width=0.24\textwidth]{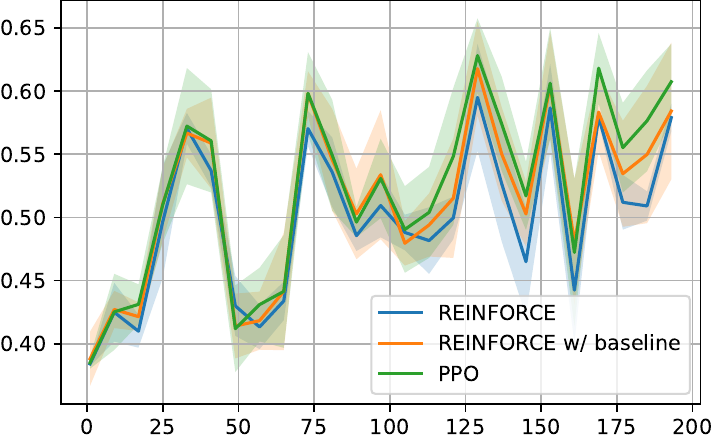}
        & \includegraphics[width=0.24\textwidth]{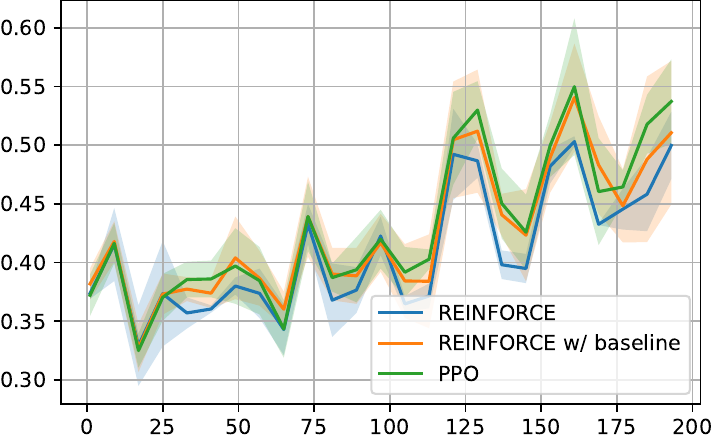}
        & \includegraphics[width=0.24\textwidth]{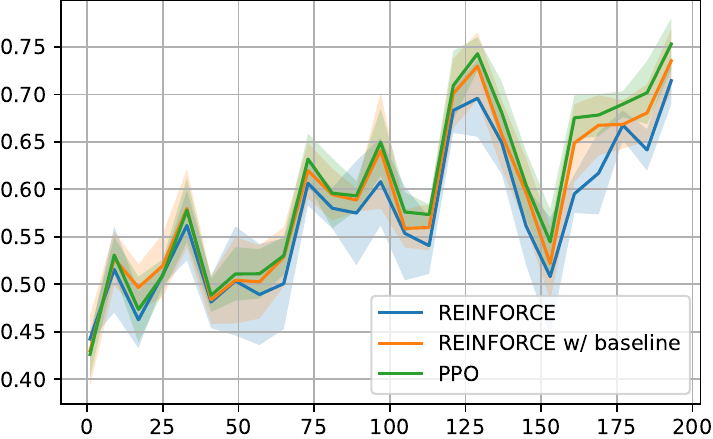}
        & \includegraphics[width=0.24\textwidth]{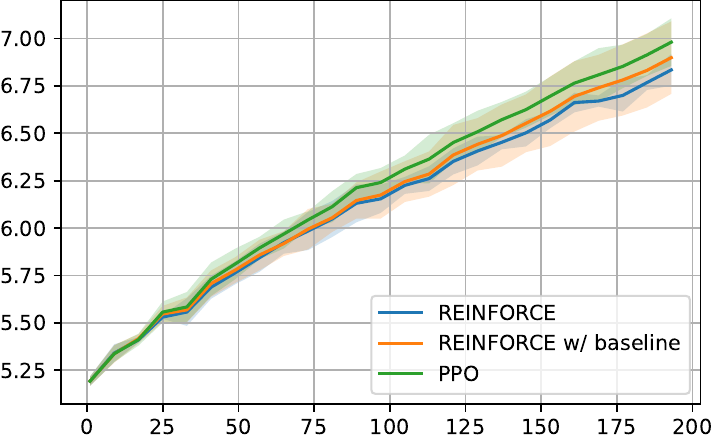} \\
        
        & \small Epochs 
        & \small Epochs 
        & \small Epochs 
        & \small Epochs \\
    \end{tabular}
    \caption{
    Evaluating REINFORCE vs.\ PPO trade-off by comparing: REINFORCE (Eq.~\ref{eq:cb_pg}), REINFORCE with baseline correction term (Eq.~\ref{eq:cb_pg_b}), and PPO (Eq.~\ref{eq:ppo_objective}). We evaluate on the T2I-CompBench benchmark over three image attributes: Color, Shape, and Texture. We also compare on the aesthetic task. Y-axis corresponds to the training reward, x-axis corresponds to the training epoch. Results are averaged over 3 runs; shaded areas indicate 80\% prediction intervals.}
    \label{fig:revisit_rl_results}
\end{figure*}

Figure~\ref{fig:revisit_rl_results} shows the training reward over epochs for the attributes: Color, Shape, and Texture from the T2I-CompBench benchmark, and training reward from optimizing the aesthetic model. It is clear that REINFORCE policy gradient is not effective in terms of performance, as compared to other variants. Adding a baseline correction term indeed improves the training performance, validating the effectiveness of baseline in terms of training performance, possibly because of reduced variance. \ac{PPO} achieves the highest training reward, validating the effectiveness of importance sampling and clipping for diffusion fine-tuning. 
 
\begin{table}[!t]
    \centering
    \caption{Comparing REINFORCE with PPO on the T2I-CompBench benchmark over three image attributes: Color, Shape, and Texture. The metrics in this table are average
reward on an unseen test set (higher is better). For each prompt, average rewards over 10 independent generated images are calculated.}
    \label{tab:rl_basics}
    \setlength{\tabcolsep}{1.1mm}
    \begin{tabular}{l ccc}
        \toprule
        \textbf{Method} & \textbf{Color} $\uparrow$
 & \textbf{Shape} $\uparrow$ & \textbf{Texture} $\uparrow$ \\
        \midrule
        REINFORCE & 0.6438 & 0.5330 & 0.6359  \\
        REINFORCE w/ BC & 0.6351 & 0.5347 & 0.6656  \\
        PPO & \textbf{0.6821} & \textbf{0.5655} & \textbf{0.6909}  \\
        \bottomrule
    \end{tabular}
\end{table}

\begin{table*}[!t]
\centering
\caption{Comparing the performance of the proposed LOOP method with state-of-the-art baselines on the T2I-CompBench benchmark over image attributes such as Color, Shape, Texture, Spatial relation, and Numeracy. The metrics in this table are average reward on an unseen test set (higher is better). For each prompt we generate and average rewards across 10 different generated images.}
\label{tab:main_results}
\setlength{\tabcolsep}{2.0mm}
\begin{tabular}{lccccc}
\toprule
\textbf{Model} & \textbf{Color} $\uparrow$ & \textbf{Shape} $\uparrow$ & \textbf{Texture} $\uparrow$ & \textbf{Spatial} $\uparrow$ & \textbf{Numeracy} $\uparrow$ \\
\midrule
Stable v1.4 \cite{rombach2022high} & 0.3765 & 0.3576 & 0.4156 & 0.1246 & 0.4461 \\
Stable v2 \cite{rombach2022high} & 0.5065 & 0.4221 & 0.4922 & 0.1342 & 0.4579 \\
Composable v2 \cite{liu2022compositional} & 0.4063 & 0.3299 & 0.3645 & 0.0800 & 0.4261 \\
Structured v2 \cite{feng2022training} & 0.4990 & 0.4218 & 0.4900 & 0.1386 & 0.4550 \\
Attn-Exct v2 \cite{chefer2023attend} & 0.6400 & 0.4517 & 0.5963 & 0.1455 & 0.4767 \\
\midrule
GORS unbiased \cite{huang2023t2i} & 0.6414 & 0.4546 & 0.6025 & 0.1725 & -- \\
GORS \cite{huang2023t2i} & 0.6603 & 0.4785 & 0.6287 & 0.1815 & 0.4841 \\
\midrule
PPO \cite{black2023training} & 0.6821 & 0.5655 & 0.6909 & 0.1961 & 0.5102 \\
LOOP ($k=2$) & 0.6785 & 0.5746 & 0.6937 & 0.1800 & 0.5072 \\
LOOP ($k=3$) & 0.7515 & 0.6220 & 0.7353 & 0.1966 & 0.5242 \\
LOOP ($k=4$) & \textbf{0.7859} & \textbf{0.6676} & \textbf{0.7518} & \textbf{0.2136} & \textbf{0.5422} \\
\bottomrule
\end{tabular}
\end{table*}

\begin{table}[!t]
    \centering
    \caption{Comparing the performance of LOOP with \ac{PPO} on the aesthetic and image-text alignment tasks. Higher values are better.}
    \label{tab:method_attributes}
    \setlength{\tabcolsep}{1.5mm}
    \begin{tabular}{l cc}
        \toprule
        \textbf{Method} & \textbf{Aesthetic} $\uparrow$ & \textbf{Image Alignment} $\uparrow$ \\
        \midrule
        PPO \cite{black2023training} & 6.8135 & 20.466 \\
        LOOP ($k=2$) & 6.8617 & 20.788 \\
        LOOP ($k=3$) & 7.0772 & 20.619 \\
        LOOP ($k=4$) & \textbf{7.8606} & \textbf{20.909} \\
        \bottomrule
    \end{tabular}
\end{table}

\begin{figure*}[th!]
    \centering
    \rotatebox{90}{%
      \begin{minipage}{0.95\textheight}   %
        \centering
        \setlength{\tabcolsep}{0.1em}  %
        {\renewcommand{\arraystretch}{0.40}
          \begin{tabular}{c r r r r}
            &
            \multicolumn{1}{c}{\hspace{-0.5cm} Color} &
            \multicolumn{1}{c}{Shape} &
            \multicolumn{1}{c}{Texture} &
            \multicolumn{1}{c}{Numeracy} \\
            \rotatebox[origin=lt]{90}{\hspace{0.65cm}\small Reward} &
            \includegraphics[scale=0.35]{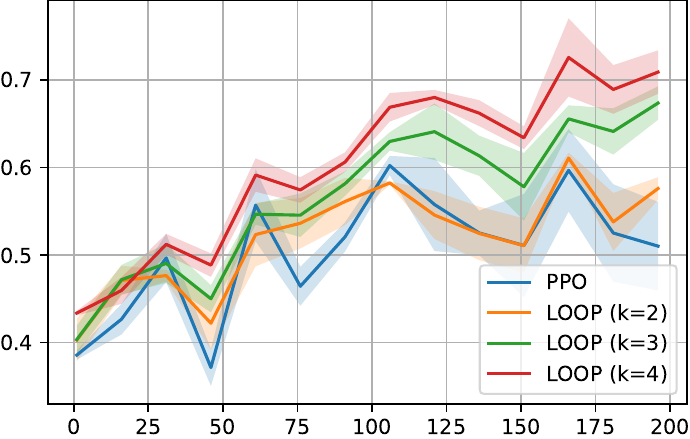} &
            \includegraphics[scale=0.35]{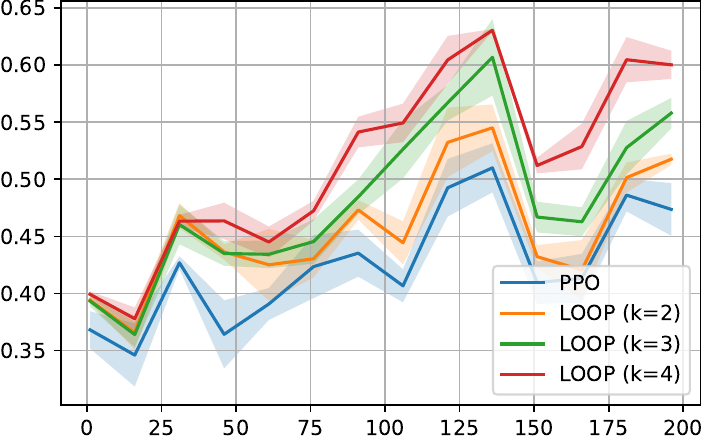} &
            \includegraphics[scale=0.35]{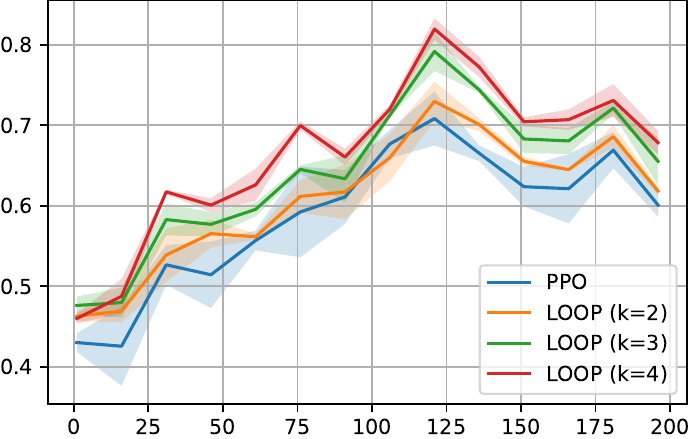} &
            \includegraphics[scale=0.35]{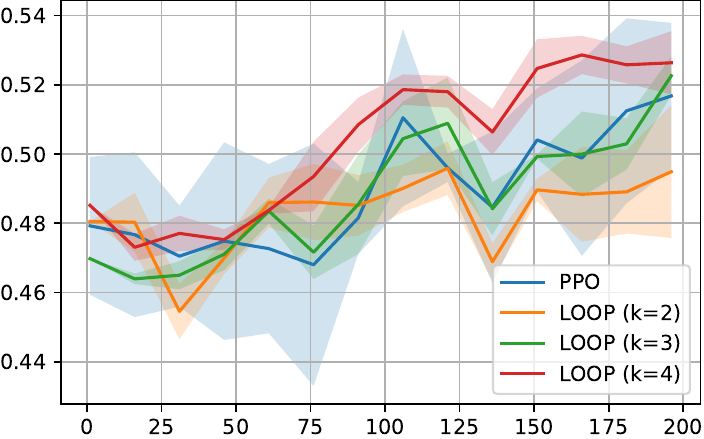}
          \end{tabular}
          \hspace{1em}%
          \begin{tabular}{c r r r}
            &
            \multicolumn{1}{c}{\hspace{0.5cm} Spatial} &
            \multicolumn{1}{c}{\hspace{0.5cm} Image-Text Alignment} &
            \multicolumn{1}{c}{\hspace{0.5cm} Aesthetic} \\
            \rotatebox[origin=lt]{90}{\hspace{0.65cm}\small Reward} &
            \includegraphics[scale=0.35]{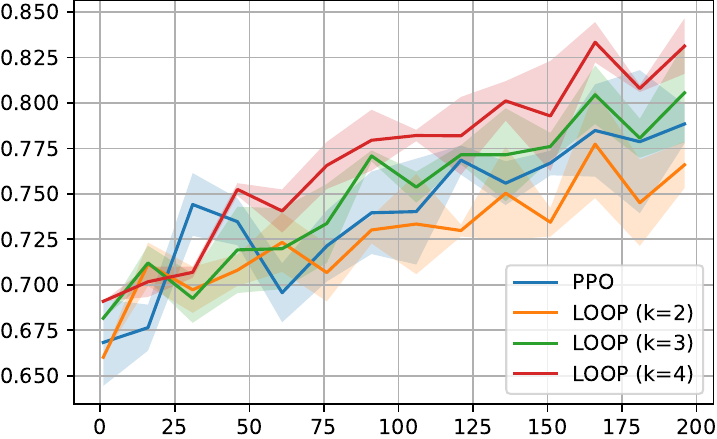} &
            \includegraphics[scale=0.35]{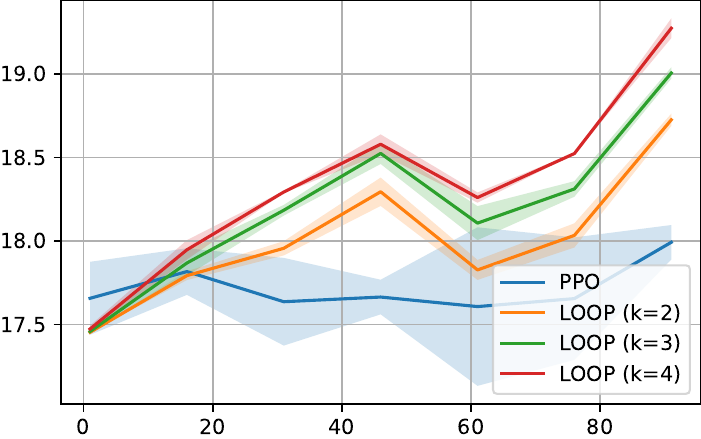} &
            \includegraphics[scale=0.35]{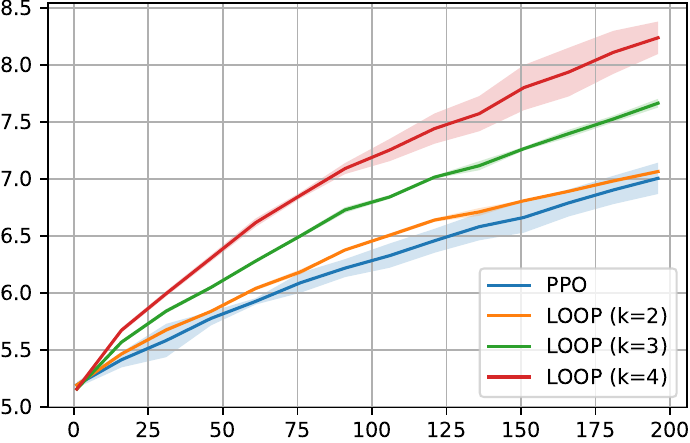}
          \end{tabular}
        }
        \vspace*{-0.5\baselineskip}
        \caption{Comparing PPO with proposed \ac{LOOP} on the T2I-CompBench benchmark with respect to image attributes: Color, Shape, Texture, Numeracy, and Spatial relationship. We also compare against aesthetic and image–text alignment tasks~\cite{black2023training}. The y-axis is the training reward; the x-axis is the training epoch. Results are averaged over three independent runs; shaded areas denote 80 \% prediction intervals.}
        \label{fig:main_results}
      \end{minipage}%
    }%
  \end{figure*}

{\renewcommand{\arraystretch}{0.9}

\setlength{\tabcolsep}{0.06cm}

\begin{figure}[ht!]

\centering

\begin{tabular}{c c c}

    & \small \hspace{0.5cm} Color
    & \small \hspace{0.5cm} Aesthetic \\

    \rotatebox[origin=lt]{90}{\hspace{0.7cm}\small Variance}
    & \includegraphics[width=0.48\columnwidth]{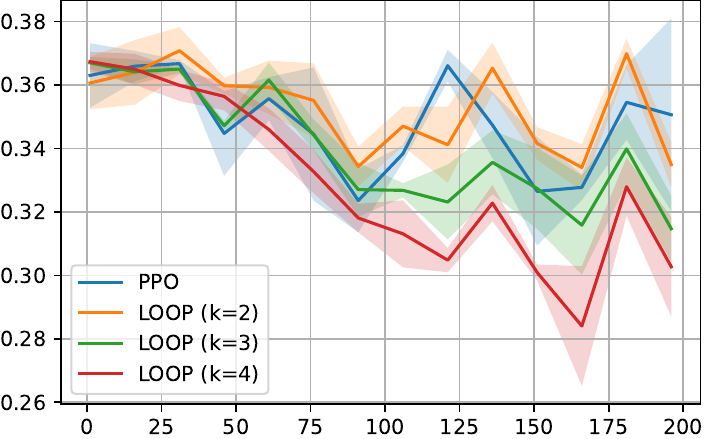}
    & \includegraphics[width=0.48\columnwidth]{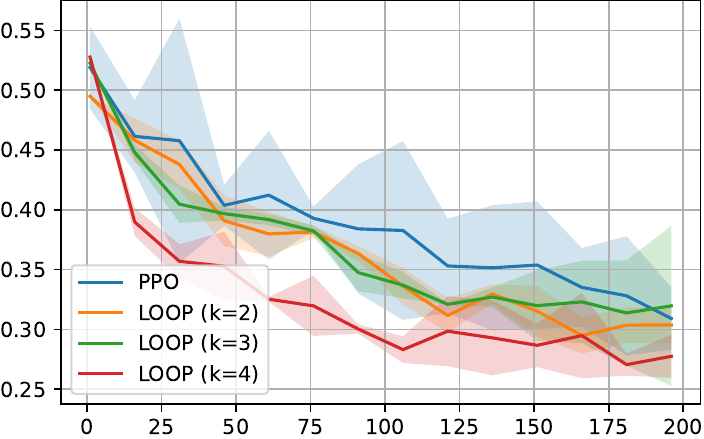} \\

    & \small \hspace{1.75em} Epochs
    & \small \hspace{1.75em} Epochs \\

\end{tabular}

\vspace*{-\baselineskip}

\caption{
Training reward variance for the color, and aesthetic task.
The y-axis corresponds to the training reward variance; the x-axis indicates the number of training epochs. Results are averaged over 3 runs; shaded areas indicate 80\% prediction intervals. We observe that higher values of samples reused (i.e., $k$) produce lower reward variance during training.
\vspace{-5mm}}

\label{fig:var_results}

\end{figure}
}

We also evaluate the performance on a separate validation set. For each validation prompt, we generate 10 independent images from the diffusion policy, and average the reward, finally averaging over all evaluation prompts. The validation results are reported in Table~\ref{tab:rl_basics}. The results are consistent with the pattern observed with the training rewards, i.e., REINFORCE with baseline provides a better performance than plain REINFORCE, suggesting that baseline correction indeed helps with the final performance.
Nevertheless, \ac{PPO} still performs better than REINFORCE. 

We now have empirical evidence supporting the \textit{efficiency-effectiveness trade-off} discussed in Section~\ref{sect:revist_rl}.
From these results, we can conclude that fine-tuning text-to-image diffusion models is more effective with \ac{IS} and clipping from \ac{PPO}, or baseline corrections from REINFORCE.
This bolsters our motivation for proposing \ac{LOOP} as an approach to effectively combine these methods.

\subsection{Evaluating LOOP}

Next we discuss the results from our proposed \ac{RL} for diffusion fine-tuning method, \ac{LOOP}.

\noindent \textbf{Performance during training.} Figure~\ref{fig:main_results} shows the training reward curves for different tasks, against number of epochs. \ac{LOOP} outperforms \ac{PPO} across all seven tasks consistently throughout training. This establishes the effectiveness of sampling multiple diffusion trajectories per input prompt, and the leave-one-out baseline correction term (Eq.~\ref{eq:rloo}) during training. The training reward curve is smoother for the aesthetic task, as compared to tasks from the T2I-CompBench benchmark. We hypothesise that improving the attribute binding property of the diffusion model is a harder task than improving the aesthetic quality of generated images.
 
Table~\ref{tab:main_results} reports the validation rewards across different tasks from the T2I-Comp\-Bench benchmark. \ac{LOOP} outperforms \ac{PPO} and other strong supervised learning based baseline significantly across all tasks. It shows that \ac{PPO} improves the attribute-binding reasoning ability of the diffusion model compared to other supervised learning based methods. 

For the aesthetic and image-text alignment objectives, the validation rewards are reported in Table~\ref{tab:method_attributes}. \ac{LOOP} results in a $\textbf{15.37}\%$ relative improvement over PPO for the aesthetic task, and a $\textbf{2.16}\%$ improvement over \ac{PPO} for the image-text alignment task.  

\noindent \textbf{Impact of the number of independent trajectories ($k$).} The LOOP variant with the number of independent trajectories where $K$ set to 4 performs the best across all tasks, followed by the variant $K=3$. This is intuitive given that Monte-Carlo estimates get better with more number of samples~\cite{Owen2013}. Surprisingly, the performance of the variant with $K=2$ is comparable to \ac{PPO}. 

\noindent \textbf{Impact on training variance.} We evaluate whether \ac{LOOP} results in a lower empirical variance than \ac{PPO}, as proved theoretically in Lemma~\ref{prop:1}. Figure~\ref{fig:var_results} reports the empirical reward variance during training for the color attribute and aesthetic objective. \ac{LOOP} results in a lower empirical variance than \ac{PPO}, thereby empirically validating our claim that \ac{LOOP} has lower variance than \ac{PPO}.

\noindent \textbf{Qualitative results.} For a qualitative evaluation of the attribute-binding reasoning ability, we present some  example image generations from SD, \ac{PPO}, and \ac{LOOP} in Figure~\ref{fig:improve-attribute-binding}. In the first example, the input prompt specifies a black colored ball with a white cat. Stable diffusion and \ac{PPO} fail to bind the color black with the generated ball, whereas \ac{LOOP} successfully binds that attribute. Similarly, in the third example, SD and \ac{PPO} fail to bind the hexagon shape attribute to the watermelon, whereas \ac{LOOP} manages to do that. In the fourth example, SD and \ac{PPO} fail to add the  horse object itself, whereas \ac{LOOP} adds the horse with the specified black color, and flowing cyan patterns.

\section{Conclusion}
We have studied the \textbf{efficiency-effectiveness} trade-off between two fundamental \ac{RL} for diffusion methods: REINFORCE, and \ac{PPO}. REINFORCE, while  computationally efficient and easier to implement, is subpar to \ac{PPO} in terms of sample efficiency and performance. Building on these insights, we have introduced a simple and effective \ac{RL} for diffusion method, \ac{LOOP}, which builds on the variance reduction techniques from REINFORCE and the effectiveness and robustness of \ac{PPO}. We have found that LOOP improves over diffusion models on multiple black-box objectives.
A limitation of LOOP is that sampling multiple diffusion trajectories per prompt can lead to more computational overhead and an increase in training time. A potential future direction would be to keep the effectiveness of \ac{LOOP} while maintaining the computational complexity of \ac{PPO}.

In this chapter, we answer the broad research question (\ref{rq:loop}) in the affirmative.
We systematically compare PPO and REINFORCE for diffusion model fine-tuning, where we demonstrate that REINFORCE exhibits inferior sample efficiency compared to PPO. 
Building on top of PPO, we propose \ac{LOOP}, which achieves superior performance with the same number of input prompts by generating multiple actions per prompt.

\begin{subappendices}
\section{Hyperparameter and Implementation Details}
\label{sec:hyper_params}
For REINFORCE (including REINFORCE with  baseline correction term), \ac{PPO}, and \ac{LOOP} the number of denoising steps ($T$) is set to 50. The diffusion guidance weight is set to 5.0. For optimization, we use AdamW~\cite{Loshchilov2017DecoupledWD} with a learning rate of $1e^{-5}$, and the weight decay of $1e{-4}$, with other parameters kept at the default value. We clip the gradient norm to $1.0$. We train all models using 8 A100 GPUs with a batch size of 4 per GPU. The clipping parameter $\epsilon$ for \ac{PPO}, and \ac{LOOP} is set to $1e^{-4}$.

\section{Additional Qualitative Examples}
We present some additional qualitative examples in this section.

\begin{figure*}[!ht]
  \centering
  \rotatebox{90}{%
    \begin{minipage}{\textheight}   %
      \centering
      \setlength{\tabcolsep}{1mm}
      \begin{tabular}{lccccc}
        \raisebox{1.5cm}{SD v2} &
          \includegraphics[width=2.9cm]{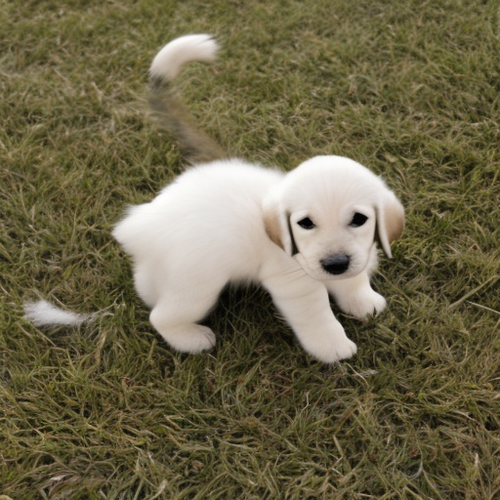} &
          \includegraphics[width=2.9cm]{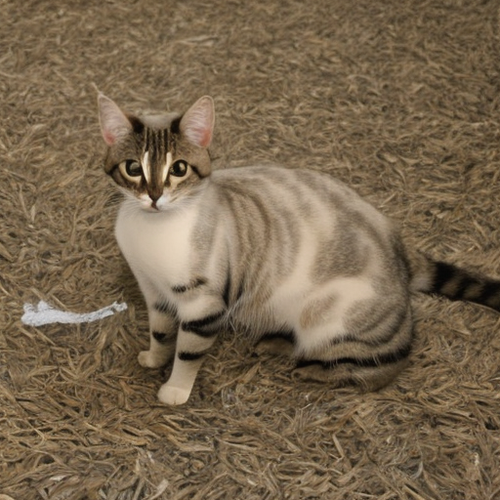} &
          \includegraphics[width=2.9cm]{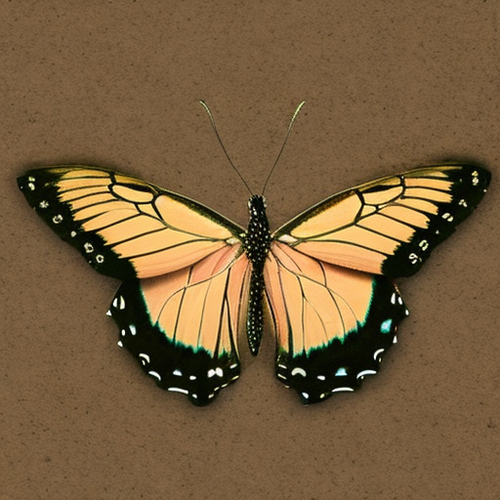} &
          \includegraphics[width=2.9cm]{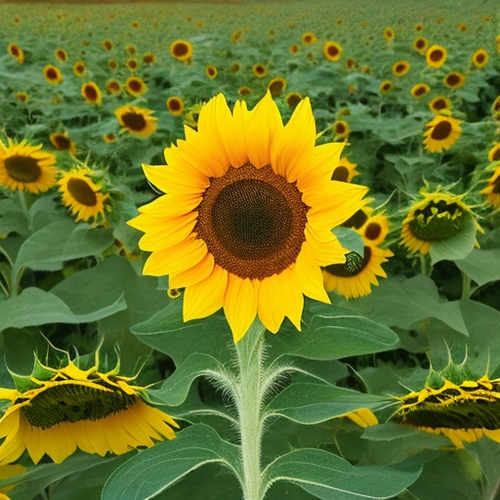} &
          \includegraphics[width=2.9cm]{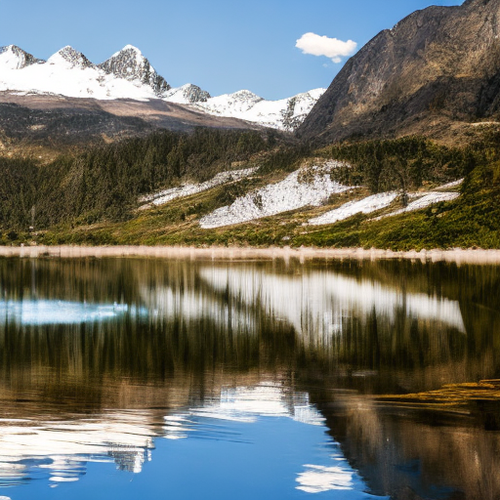}
        \\
        \raisebox{1.5cm}{PPO} &
          \includegraphics[width=2.9cm]{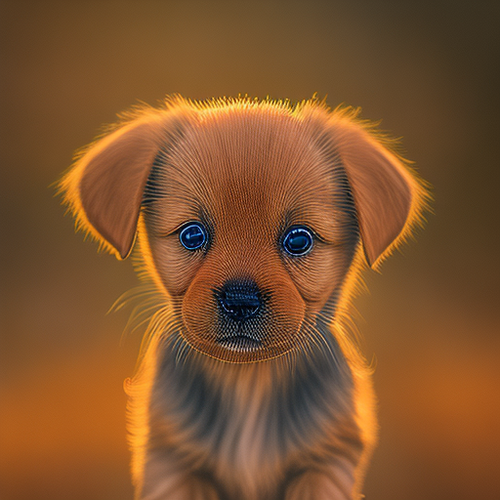} &
          \includegraphics[width=2.9cm]{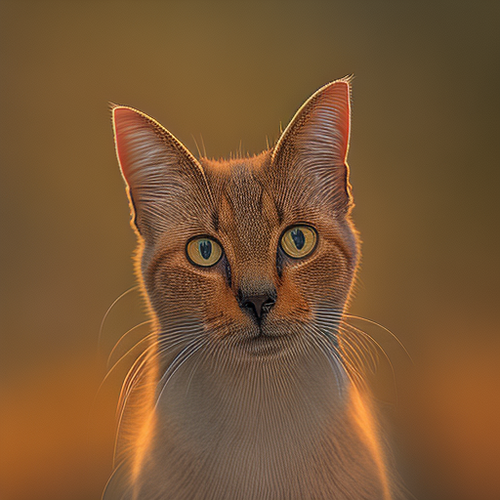} &
          \includegraphics[width=2.9cm]{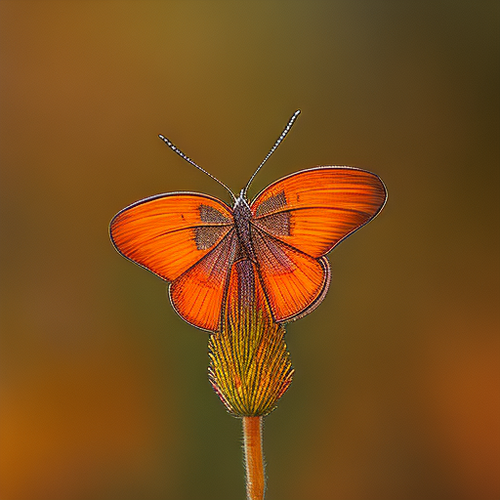} &
          \includegraphics[width=2.9cm]{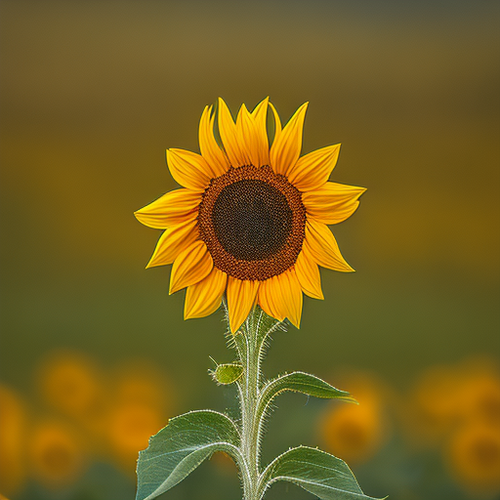} &
          \includegraphics[width=2.9cm]{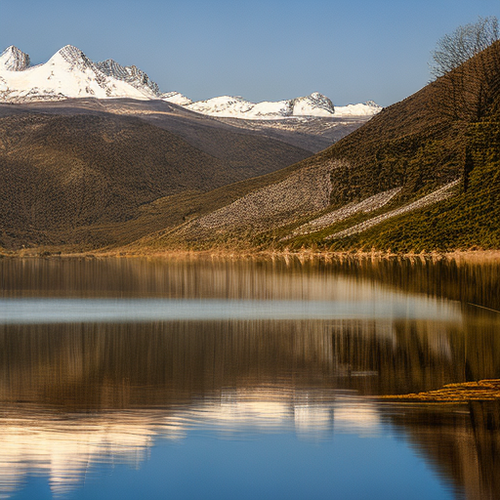}
        \\
        \raisebox{1.5cm}{LOOP} &
          \includegraphics[width=2.9cm]{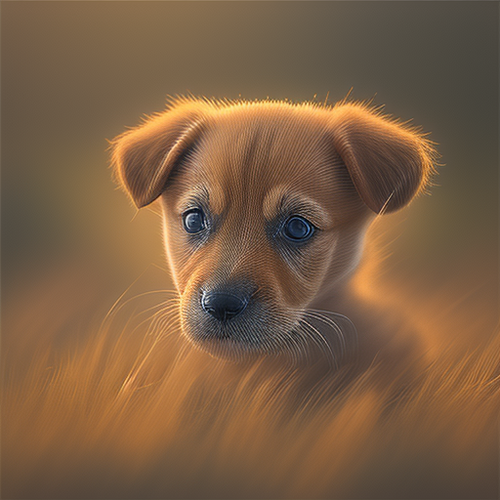} &
          \includegraphics[width=2.9cm]{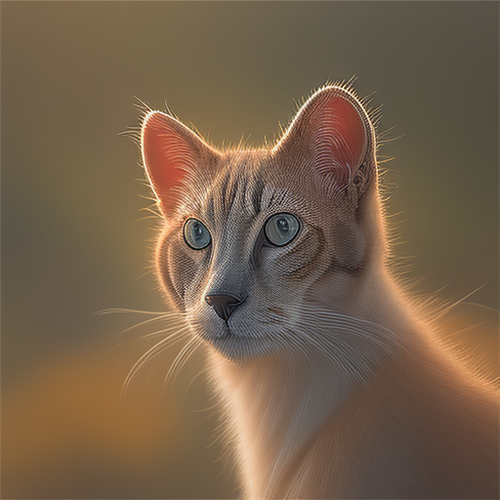} &
          \includegraphics[width=2.9cm]{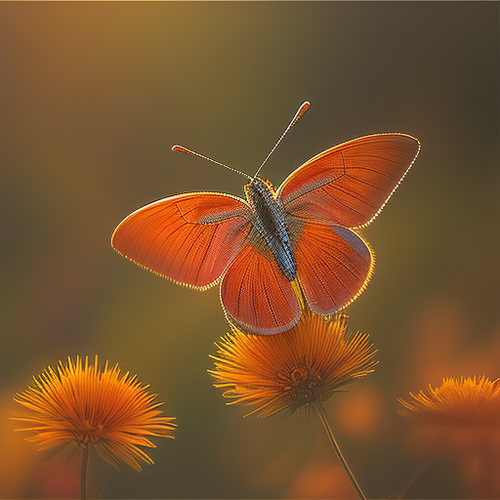} &
          \includegraphics[width=2.9cm]{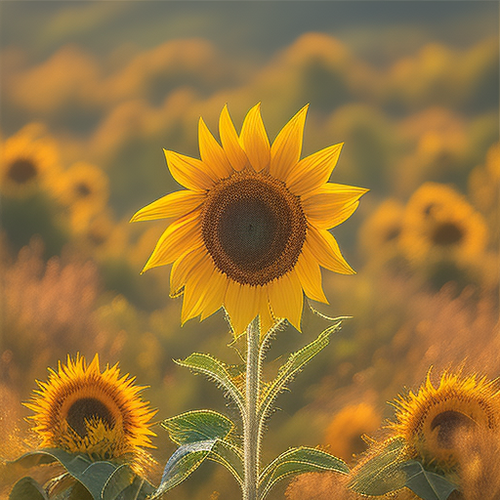} &
          \includegraphics[width=2.9cm]{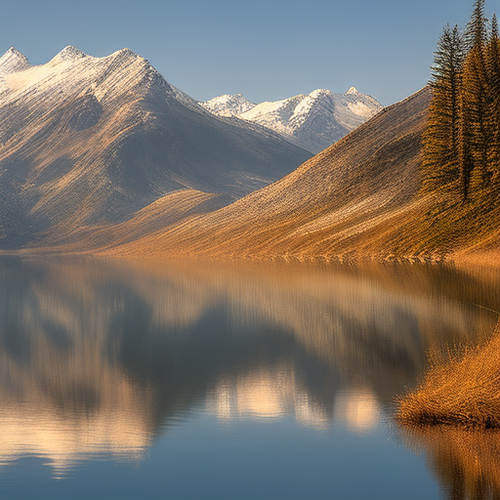}
        \\
        & \parbox{2.9cm}{\small ``A puppy dog''\\ \\}
        & \parbox{2.9cm}{\small ``A cat''\\ \\}
        & \parbox{2.9cm}{\small ``Butterfly''\\ \\}
        & \parbox{3cm}{\small ``Bright yellow sunflower in a green field''}
        & \parbox{2.9cm}{\small ``Crystal clear mountain lake reflecting snow-capped peaks''}
        \\
      \end{tabular}
      \vspace*{-0.5\baselineskip}
      \caption{\textbf{LOOP improves aesthetic quality}. Qualitative examples are presented from images generated via Stable Diffusion 2.0 (first row), \ac{PPO} (second row), and \ac{LOOP} $k=4$ (third row). \ac{LOOP} consistently generates more aesthetic images than \ac{PPO} and SD.}
      \label{tab:images_appendix1}
    \end{minipage}%
  }%
\end{figure*}

\newpage

\begin{figure*}[!h]
  \centering
  \rotatebox{90}{%
    \begin{minipage}{\textheight}   %
      \centering
      \setlength{\tabcolsep}{1mm}
      \begin{tabular}{lccccc}
        \raisebox{1.5cm}{SD v2} &
          \includegraphics[width=2.9cm]{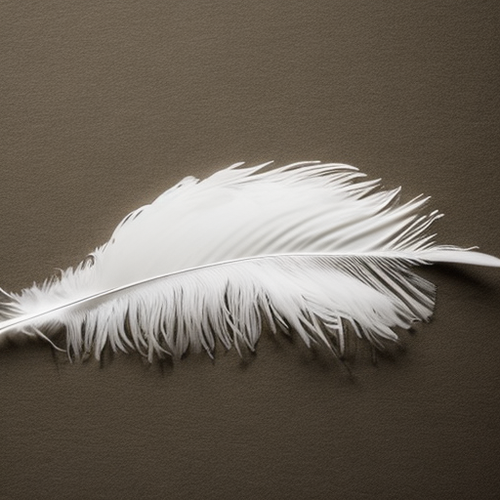} &
          \includegraphics[width=2.9cm]{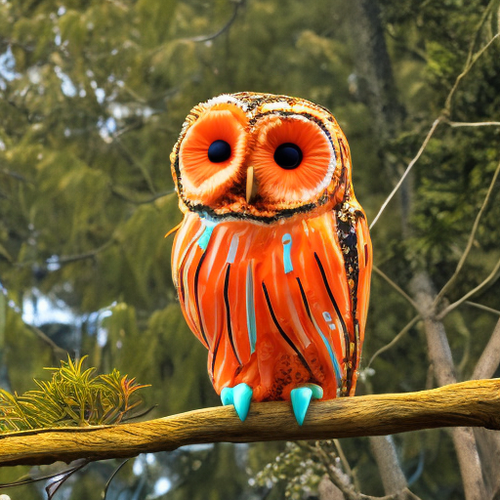} &
          \includegraphics[width=2.9cm]{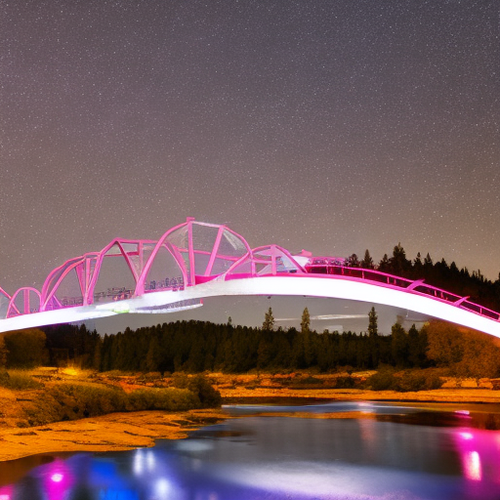} &
          \includegraphics[width=2.9cm]{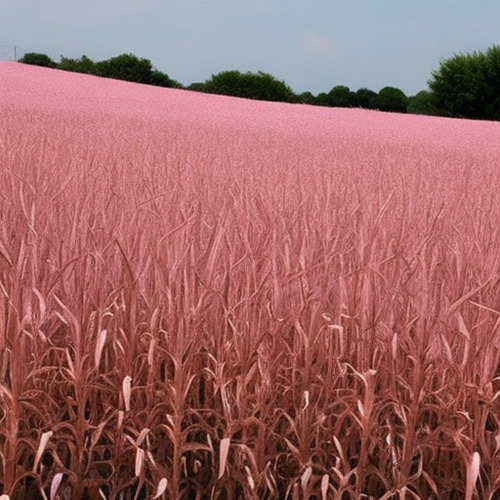} &
          \includegraphics[width=2.9cm]{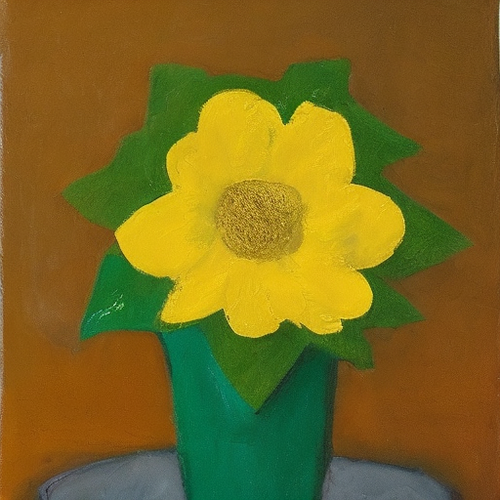}
        \\
        \raisebox{1.5cm}{PPO} &
          \includegraphics[width=2.9cm]{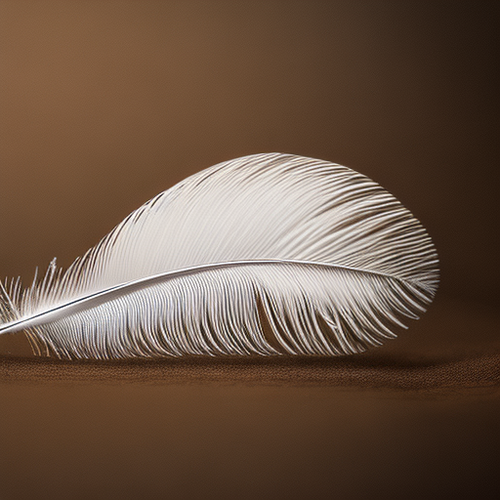} &
          \includegraphics[width=2.9cm]{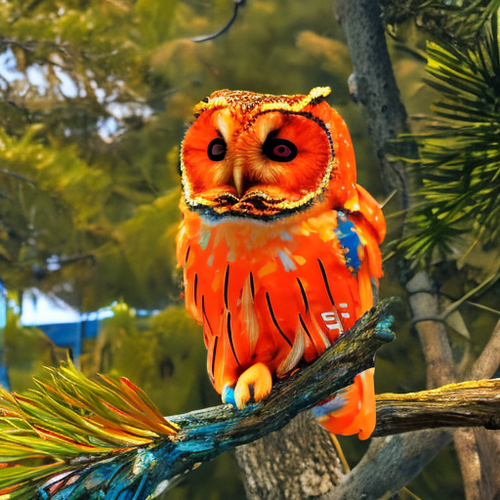} &
          \includegraphics[width=2.9cm]{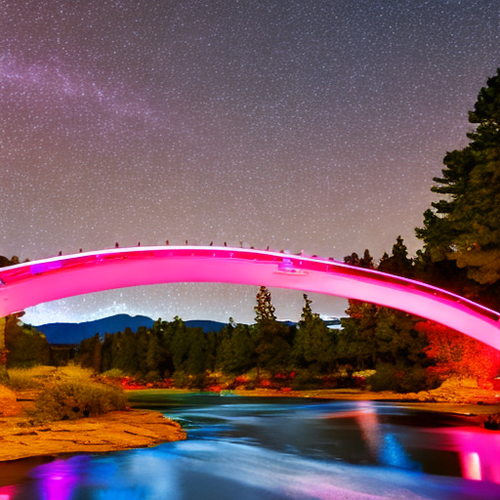} &
          \includegraphics[width=2.9cm]{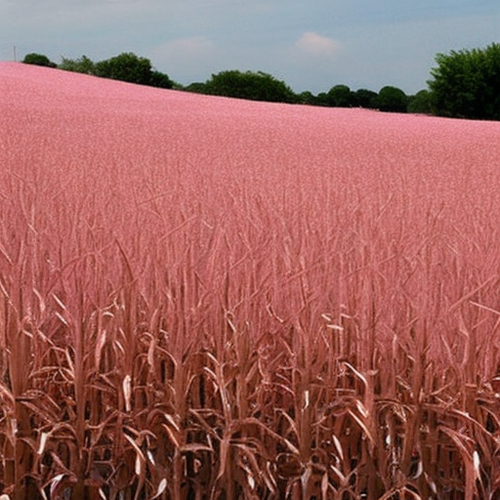} &
          \includegraphics[width=2.9cm]{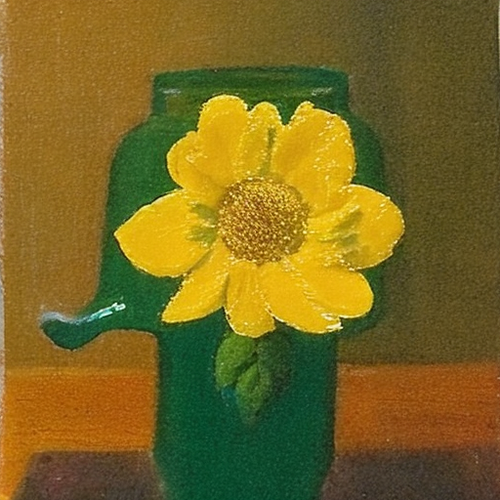}
        \\
        \raisebox{1.5cm}{LOOP} &
          \includegraphics[width=2.9cm]{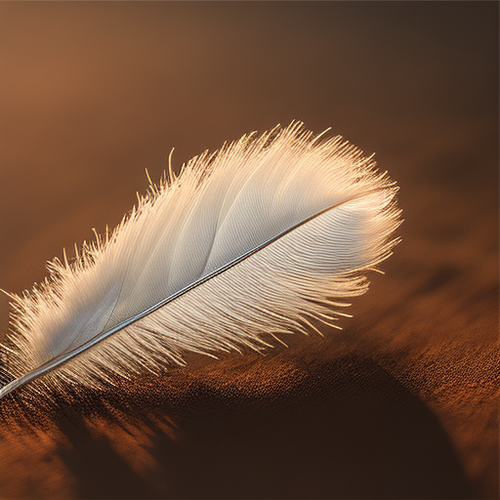} &
          \includegraphics[width=2.9cm]{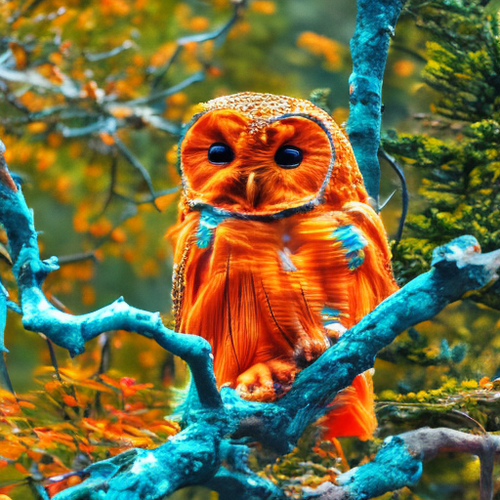} &
          \includegraphics[width=2.9cm]{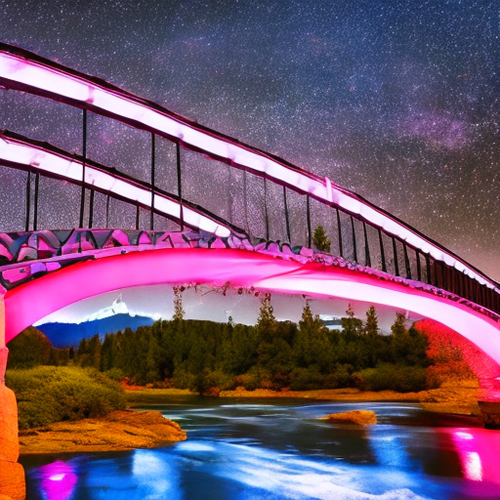} &
          \includegraphics[width=2.9cm]{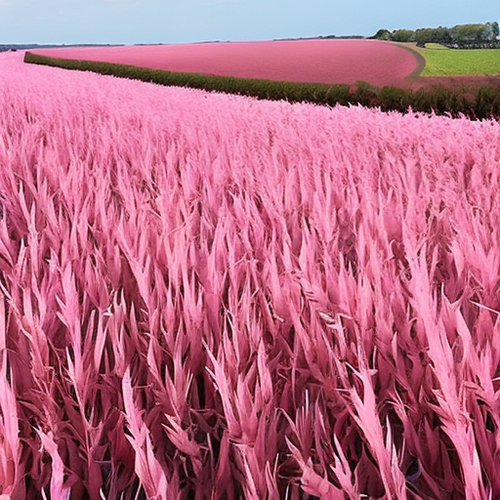} &
          \includegraphics[width=2.9cm]{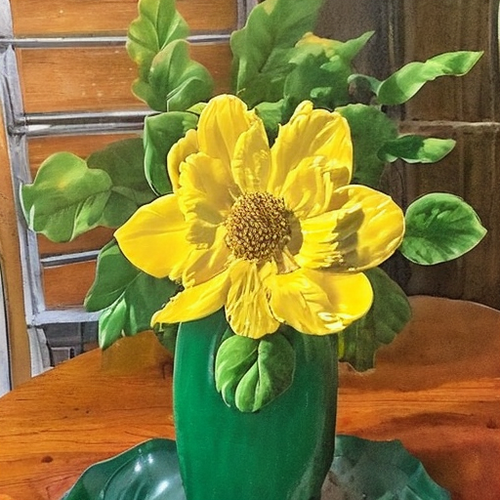}
        \\
        & \parbox{2.9cm}{\small ``A white feather on a black velvet surface''}
        & \parbox{2.9cm}{\small ``A neon orange owl sitting on a teal branch''}
        & \parbox{2.9cm}{\small ``Pink bridge over a glowing blue river''}
        & \parbox{3cm}{\small ``A pink cornfield''}
        & \parbox{2.9cm}{\small ``A yellow flower in a green vase''}
        \\
      \end{tabular}
      \vspace*{-0.5\baselineskip}
      \caption{Additional qualitative examples presented from images generated via Stable Diffusion 2.0 (first row), \ac{PPO} (second row), and \ac{LOOP} $k=4$ (third row). \ac{LOOP} consistently generates more aesthetic images than \ac{PPO} and SD (first, third, and fifth prompt). \ac{LOOP} also binds the color attribute (teal branch in the second example and pink cornfield in the fourth) where SD and \ac{PPO} fail.}
      \label{tab:images_appendix}
    \end{minipage}%
  }%
\end{figure*}

\end{subappendices}

\bookmarksetup{startatroot} %
\addtocontents{toc}{\bigskip} %

\chapter{Conclusions}
\label{chapter:conclusions}

\acresetall

In this thesis, we have investigated approaches to develop safe, robust and efficient reinforcement learning methods for real-world applications. 
Specifically, in the four chapters preceding this conclusion, we have demonstrated:
\begin{enumerate}[label=(\roman*)]
    \item A safe counterfactual LTR method which guarantees that the new ranking policy will be at least as good as the production/logging policy, presented in Chapter~\ref{chapter:01-online-evaluation1}.
    \item A robust safe counterfactual \ac{LTR} method where the safety guarantees are agnostic to the user behavior model, and hold even under adversarial click behavior settings, presented in Chapter~\ref{chapter:01-online-evaluation2}. 
    \item A closed-form baseline correction method for off-policy evaluation and learning for contextual bandits with guaranteed minimum variance, presented in Chapter~\ref{chapter:01-online-evaluation3}.
    \item An efficient reinforcement learning method for text-to-image diffusion fine-tuning, based on a simple and practical extension of the popular Proximal Policy Optimization (PPO) algorithm, with significantly improved performance, presented in Chapter~\ref{chapter:01-online-evaluation4}.
\end{enumerate} 

\section{Main Findings}
In this section we revisit the research questions presented in Chapter~\ref{chapter:introduction} followed by a summary of the most important findings.

\begin{enumerate}[label=\textbf{RQ\arabic*},ref={RQ\arabic*}]
    \item \acl{rq:safe1}
    \end{enumerate}
    
    \noindent The answer to this question is in the affirmative. In Chapter~\ref{chapter:01-online-evaluation1}, we derive a generalization bound for the counterfactual \ac{LTR} estimator, establishing a lower bound on the true ranking utility, the ideal target metric for optimization. 
    We demonstrated that optimizing this lower bound ensures a ranking policy no worse than the current production policy. 
    This property proves especially valuable when click data is scarce, mitigating the risk of deploying potentially harmful policies, thereby ensuring safe deployment.
    
    The broader implication of this work is in the practical \emph{deployment} stage for all modern search and recommender ranking systems.
    Using the presented safety techniques, search and recommendation teams can reliably deploy a ranking policy without risking deploying a policy with sub-optimal user experience.
    An example use-case is deploying a ranking policy in a new geographic region, with limited user interactions. 
    The safety method presented will help ensure that the new policy is never worse than the safe production policy. 

    A limitation of our work is that we assume the ranking policy to be stochastic in nature, i.e., for a given query context, the ranking policy generates different ranked lists at each time. 
    In certain real-world applications, stochastic rankings might be unfeasible, or not preferred because of external reasons.

    While \ref{rq:safe1} provides a probabilistic safety guarantee by optimizing the lower bound, these guarantees depend critically on assumptions regarding user behavior (click model). 
    Deviations from these assumptions invalidate the guarantees, which motivated the second research question:
    
    \begin{enumerate}[label=\textbf{RQ\arabic*},ref={RQ\arabic*},resume]
    \item \acl{rq:safe2}
    \end{enumerate}
    
    \noindent The answer to this question is in the affirmative; in Chapter~\ref{chapter:01-online-evaluation2}, we introduced \ac{PRPO}, a method ensuring safety for counterfactual \ac{LTR} without reliance on user behavior assumptions, guaranteeing robust safety even under adversarial conditions.
    
    The broader implication of this work is in providing a robust safe deployment framework. 
    In practice, the proposed method in this chapter provides reliability in the wild.
    Search engines or recommender systems that serve heterogeneous markets (or fast-shifting verticals like news) have shifting user preferences, and relying on a single user behavior assumption can have detrimental effects. 
    The robust safety method presented in this chapter can ensure safe deployment even under such dynamic heterogeneous markets. 

    Similar to the previous chapter, a limitation of this work is that we assume the ranking policy to be stochastic in nature, i.e., for a given query context, the ranking policy generates different ranked lists at each time. 
    A robust safety method for deterministic ranking policy might be preferred.

    In the context of off-policy evaluation and learning with single action contextual bandits, standard methods like IPS are unbiased but suffer from high variance. 
    Alternative methods, including doubly robust (DR) estimators and self-normalized IPS (SNIPS), reduce variance using additive and multiplicative baseline corrections respectively, yet lack a unifying framework. 
    This motivated our third research question:
    
    \begin{enumerate}[label=\textbf{RQ\arabic*},ref={RQ\arabic*},resume]
    \item \acl{rq:recsys1}
    \end{enumerate}
    
    \noindent The answer to this question is in the affirmative; in Chapter 4, we proposed the $\beta$-IPS estimator, integrating IPS, doubly robust methods, and Self-Normalized IPS under a unified baseline correction framework.
    
    \begin{enumerate}[label=\textbf{RQ\arabic*},ref={RQ\arabic*},resume]
    \item \acl{rq:recsys}
    \end{enumerate}
    
    \noindent The answer to this question is in the affirmative; in Chapter~\ref{chapter:01-online-evaluation3}, we presented a closed-form solution for $\beta^{}$ that minimizes variance for both learning and evaluation tasks. 
    Empirical evidence under different scenarios validates the effectiveness of our approach.
    
    A broader implication of this work is in providing a \emph{unified vocabulary/framework} for off-policy evaluation and learning tasks. 
    A common framework for off-policy evaluation and learning tasks reduces the burden of choosing an estimator in practice for practitioners. 
    Further, a closed-form solution for the baseline correction term makes the practical implementation of the estimator easier. 

    Contextual bandit theory previously discussed emphasizes user interactions within ranking or recommendation systems. 
    However, the framework has also been effectively employed in fine-tuning foundation models, such as large language models (LLMs) and diffusion models, typically using proximal policy optimization (PPO). Recent research highlights computational advantages of REINFORCE (policy gradient methods) over PPO for LLMs~\cite{ahmadian2024back}. Given PPO's ongoing challenges with variance and sample inefficiency, we consider improvements through our final research question:
    
    \begin{enumerate}[label=\textbf{RQ\arabic*},ref={RQ\arabic*},resume]
    \item \acl{rq:loop}
    \end{enumerate}
    
    \noindent The answer to this question is in the affirmative; in Chapter~\ref{chapter:01-online-evaluation4}, we systematically compare PPO and REINFORCE for diffusion model fine-tuning. Initially, we demonstrate that REINFORCE exhibits inferior sample efficiency compared to PPO. Subsequently, we propose \ac{LOOP}, an enhancement to PPO achieving superior performance with the same number of input prompts by generating multiple actions per prompt.

\section{Future Work}

Finally, this section addresses some limitations with the existing work and potential future directions of the research presented in this thesis. 

\subsection{Safety with real-world constraints}
First, the safe counterfactual \ac{LTR} methods presented in Chapter~\ref{chapter:01-online-evaluation1} and Chapter~\ref{chapter:01-online-evaluation2} are designed with a stochastic ranking policy in mind. 
In many real-world applications, deploying a stochastic policy might not be feasible, necessitating safety methods for deterministic ranking policies~\cite{guo2023inference}. 
A future direction along this line would be to add safety regularization to the top-K LambdaLoss \ac{LTR} method with deterministic ranking policy~\cite{oosterhuis2020policy}.

Regarding experiments, all of our evaluations are based on semi-synthetic simulations with click signal derived from the manual relevance judgments. 
Real-world experiments are typically conducted via A/B tests on actual users. 
As part of future work, applying the proposed safety methods to real-world user interaction data, followed by comprehensive A/B testing, would be particularly valuable.

Modern recommendation systems increasingly include a LLM that both selects content and generates natural-language explanations. 
Extending safe counterfactual LTR to such LLM-based ranking policies will require handling high-dimensional textual actions, possibly by constraining the language model output with exposure-based bounds handing the high-dimensional nature of action space.

\subsection{Extending optimal baseline corrections to reinforcement learning}
Next, the optimal-baseline correction method presented in Chapter~\ref{chapter:01-online-evaluation3} reduces the variance of existing off-policy evaluation and learning estimators for contextual bandits.
Extending the proposed optimal variance baseline to offline RL scenarios would be an interesting future research direction~\cite{levine2020offline}.

In moving from contextual bandits to full reinforcement learning settings, where decisions involve trajectories (sequences of actions), it would be interesting to derive an optimal scalar (or state-dependent) baseline that minimises variance while preserving unbiasedness.

Further, the optimal baseline correction presented for off-policy learning involves calculating and storing gradients of all parameters for each example separately. 
With large language models involving billions of parameters, storing and calculating gradients separately for each example could be practically challenging, presenting another potential avenue for future exploration. 
Compression techniques such as low-rank adapters, or selective checkpointing could bring the memory cost, opening up an interesting practical future direction.

\subsection{RL-based diffusion fine-tuning}
The diffusion fine-tuning setup presented in Chapter~\ref{chapter:01-online-evaluation4} is treated as a contextual bandit framework, with the entire reverse diffusion process treated as a single action.
A scalar reward is generated for the final image in the generation. 
Extending the model to a more traditional reinforcement learning setup, where we have a scalar reward at each step of the reverse diffusion process, could be an interesting future direction. 

Further, the reward signal for diffusion fine-tuning comes from an external reward model, which roughly represents the average population score for the corresponding task, and is not personalized. 
An interesting future direction would be to fine-tune foundational models directly with user interaction data to enable personalized generative models.
An example would be a personalized email writing assistant, which learns to generate tailored email text based on user interactions, such as edits, binary feedback, etc. 
In the context of image generation, an example use-case could be personalized music playlist cover generation, based on the user's interactions and preferences.

\subsection{Personalised generative models}
Finally, the reward model used for diffusion fine-tuning in Chapter~\ref{chapter:01-online-evaluation4} reflects an average user preference. 
Personalising generation to individual users raises new challenges: privacy, fairness and extreme data sparsity.
Promising future research directions in this area are:
\begin{itemize}
\item \textit{Privacy-preserving on-device fine-tuning}. Employ federated RL or secure aggregation to learn user-specific adapters without sharing raw images or prompts.
\item \textit{Meta-learning reward models}. Train a global model that can be rapidly adapted to a new user with a handful of interaction signals (edits, binary feedback).
\item \textit{Fairness and calibration}. Ensure that personalised models do not amplify sensitive-attribute biases by incorporating fairness constraints into the safe-bandit objective.
\end{itemize}

\backmatter

\renewcommand{\bibsection}{\chapter{Bibliography}}
\renewcommand{\bibname}{Bibliography}
\markboth{Bibliography}{Bibliography}
\renewcommand{\bibfont}{\footnotesize}
\setlength{\bibsep}{0pt}

\bibliographystyle{abbrvnat}
\bibliography{thesis}

\chapter{Summary}

This dissertation investigates how reinforcement learning methods can be made simultaneously safe, sample-efficient, and robust when trained only from logged user interactions. 
Under the unifying lens of contextual-bandit reinforcement learning, the work spans two application families: web-search ranking/recommendation and text-to-image diffusion models. 
In this thesis, we pair new theory with practical algorithms that (i) guarantee safe deployment against the production system, (ii) extract more signal from limited logs in an off-policy evaluation and learning setup, and (iii) scale to modern large-scale generative models.

In the first part of the thesis, we derive an exposure-based generalisation bound that upper-bounds the true ranking utility. 
Optimising the bound yields a counterfactual risk-minimisation (CRM) objective whose solution is provably no worse than the logging policy even with few clicks, resulting in safe deployment.
Further, we proposed a robust safe deployment method that extends safety to doubly-robust estimators, and retains guarantees under adversarial or mis-specified behaviour models.
The proposed method offers practitioners direct control over the maximum allowed utility drop.

In the second part of the thesis, shifting to single-action bandits, our contribution unifies IPS, self-normalised IPS and doubly robust estimators inside an unifying baseline-correction framework.
We propose a closed-form optimal baseline term that is proved to minimise both evaluation and policy-gradient variance.

In the final chapter we revisit the efficiency-effectiveness trade-off in a generative reinforcement learning setup. 
A systematic PPO-vs-REINFORCE study reveals an “efficiency–effectiveness” trade-off, inspiring leave-one-out PPO (LOOP). 
LOOP generates several diffusion trajectories per prompt and inserts a REINFORCE-style baseline inside PPO’s clipped objective, matching PPO quality while binding textual attributes more faithfully on text-to-image diffusion benchmark.

Finally, the thesis gives the following answers: (i) safety can be guaranteed for ranking --  with or without click-model assumptions; (ii) a single baseline parameter can unify and optimise bandit variance reduction; and (iii) lightly modified reinforcement learning algorithms can fine-tune large diffusion models efficiently. 
Together these advances demonstrate a path toward reliable, safe, and data-efficient reinforcement learning pipelines for real-world information access and generative AI and open avenues for extending safe-bandit theory to multitask and multi-objective foundation models.

\chapter{Samenvatting}

Dit proefschrift onderzoekt hoe \emph{reinforcement learning}-methoden tegelijkertijd veilig, sample-efficiënt en robuust kunnen worden gemaakt wanneer ze uitsluitend worden getraind op basis van geregistreerde gebruikersinteracties.

Onder de overkoepelende lens van \emph{contextual-bandit} reinforcement learning omvat het werk twee families van toepassingen: webzoekrangschikking/-aanbeveling en tekst-naar-afbeelding diffusiemodellen.

In dit proefschrift combineren we nieuwe theorie met praktische algoritmen die (i) veilige implementatie tegen het productiesysteem garanderen, (ii) meer signaal extraheren uit beperkte logs in een off-policy evaluatie- en leeropstelling, en (iii) opschalen naar moderne grootschalige generatieve modellen.

In het eerste deel van het proefschrift leiden we een \emph{exposure}gebaseerde generalisatiegrens af die de werkelijke bruikbaarheid van de rangschikking begrenst.
Het optimaliseren van de grens levert een contrafactische risicominimalisatiedoelstelling op waarvan de oplossing aantoonbaar niet slechter is dan het logbeleid, zelfs met weinig klikken, wat resulteert in een veilige implementatie. Verder hebben we een robuuste, veilige implementatiemethode voorgesteld die de veiligheid uitbreidt naar dubbelrobuuste schatters en garanties behoudt onder vijandige of verkeerd gespecificeerde gedragsmodellen.
De voorgestelde methode biedt professionals directe controle over de maximaal toegestane utiliteitsdaling.

In het tweede deel van het proefschrift, overgaand op \emph{single-action} bandits, verenigt onze bijdrage IPS, zelfgenormaliseerde IPS en dubbelrobuuste schatters binnen een uniform basislijncorrectiekader.
We stellen een gesloten, optimale basislijnterm voor waarvan bewezen is dat deze zowel de variantie in evaluatie als in beleidsgradiënt minimaliseert.

In het laatste hoofdstuk bekijken we de afweging tussen efficiëntie en effectiviteit in een generatieve reinforcement learning-opzet opnieuw. Een systematische PPO-versus-REINFORCE-studie onthult een afweging tussen efficiëntie en effectiviteit, wat inspireert tot leave-one-out PPO (LOOP). LOOP genereert meerdere diffusietrajecten per prompt en voegt een REINFORCE-achtige basislijn in binnen de afgeknipte doelstelling van PPO, die overeenkomt met de PPO-kwaliteit en tegelijkertijd tekstuele kenmerken getrouwer koppelt aan de tekst-naar-afbeelding diffusiebenchmark.

Ten slotte geeft het proefschrift de volgende antwoorden: (i) veiligheid kan worden gegarandeerd voor rangschikking – met of zonder aannames voor het klikmodel; (ii) één basislijnparameter kan de reductie van banditvariantie verenigen en optimaliseren; (iii) licht aangepaste \emph{reinforcement learning} algoritmen kunnen grote diffusiemodellen efficiënt verfijnen.

Samen tonen deze ontwikkelingen een pad naar betrouwbare, veilige en data-efficiënte \emph{reinforcement learning} pipelines voor real-world informatietoegang en generatieve AI, en openen ze mogelijkheden om de safe-bandittheorie uit te breiden naar multitask- en multi-objectieve basismodellen.

\end{document}